\documentclass[11pt]{article}
\usepackage{amsmath,amssymb}
\usepackage{color}
\usepackage[dvipsnames]{xcolor}
\usepackage{times}
\usepackage{bm}
\usepackage{natbib}
\usepackage{algorithm}
\usepackage{amssymb}
\usepackage{mathrsfs}
\usepackage{graphicx}
\usepackage{bbm}
\usepackage{mathtools}
\usepackage[flushleft]{threeparttable}
\usepackage{comment}
\usepackage{enumitem} %%%enumitem
\usepackage{arydshln}
\usepackage{pstricks}
\usepackage{tikz}
\usepackage{star}

\usepackage{subfigure}
\usepackage[colorlinks,
linkcolor=blue,
anchorcolor=blue,
citecolor=blue
]{hyperref}

\newcommand*{\supp}{\mathrm{supp}}

\newcommand{\rF}{\mathrm{\scriptstyle F}}

\newcommand{\new}{\textrm{new}}

\def\T{{ \mathrm{\scriptscriptstyle T} }}
\def\cov{{ \mathrm{cov} }}

\def\##1\#{\begin{align}#1\end{align}}
\def\$#1\${\begin{align*}#1\end{align*}}

\renewcommand{\tr}{\textrm{tr}}

\def\T{{ \mathrm{\scriptscriptstyle T} }} %%%transpose operator

\newcommand{\Rom}[1]{\text{\uppercase\expandafter{\romannumeral #1\relax}}}

\newcommand{\condition}{\;\middle|\;}
\newcommand{\ndata}{n}
\newcommand{\ndim}{d}
\newcommand{\as}{{\rm a.s.}}
\newcommand{\qas}{{\quad \rm a.s.}}
\newcommand{\npinfty}{\ndata,\ndim \rightarrow +\infty}
\newcommand{\zinfty}{z \rightarrow 0^{-}}

\DeclareRobustCommand\full  {\tikz[baseline=-0.6ex]\draw[thick] (0,0)--(0.5,0);}
\DeclareRobustCommand\dotted{\tikz[baseline=-0.6ex]\draw[thick,dotted] (0,0)--(0.54,0);}
\DeclareRobustCommand\dashed{\tikz[baseline=-0.6ex]\draw[thick,dashed] (0,0)--(0.54,0);}
\DeclareRobustCommand\chain {\tikz[baseline=-0.6ex]\draw[thick,dash dot] (0,0)--(0.5,0);}

\newcommand{\variablecondition}{X,\, \cW,\, \truesignal}
\newcommand{\riskcondition}{R_{\variablecondition}}
\newcommand{\biascondition}{B_{\variablecondition}}
\newcommand{\varcondition}{V_{\variablecondition}}

\newcommand{\poisson}{\text{Poisson}}
\newcommand{\stieltjes}{m}

\newcommand{\compstieltjes}{v}
\newcommand{\tcompstieltjes}{\Tilde{v}}

\newcommand{\stieltjescov}{\Theta}

\newcommand{\trainloss}{L}

\newcommand {\covmat}{\Sigma}
\newcommand {\hcovmat}{\widehat\covmat}
\newcommand {\tcovmat}{\Tilde{\covmat}}
\newcommand {\eigval}{\lambda}

\newcommand{\projmat}{\Pi}
\newcommand{\Bernoulliweight}{D}
\newcommand{\nresample}{B}
\newcommand{\pinv}{^{+}}
\newcommand{\inv}{^{-1}}
\newcommand{\transp}{^{\top}}
\newcommand{\sketchmat}{S}
\newcommand{\weight}{w}
\newcommand{\aspratio}{\gamma}

\newcommand{\limitdistr}{H}
\newcommand{\tlimitdistr}{\Tilde{H}}
\newcommand{\limitdistrweight}{\mu_{\weight}}
\newcommand{\tildelimitdistrweight}{\tilde{\mu}_{\weight}}
\newcommand{\ensambleweight}{k}

\newcommand{\sampleratio}{\theta}

\newcommand{\estimator}{\hat{\beta}^{\nresample}}

\newcommand{\estimatork}{\hat{\beta}_{k}}
\newcommand{\estimatorsketch}{\hat{\beta}^{1}}
\newcommand{\truesignal}{\beta}
\newcommand{\noiselev}{\sigma^2}
\newcommand{\signallev}{r^2}

\newcommand{\A}{A}
\newcommand{\B}{B}
\newcommand{\C}{C}

\newcommand{\ca}{a}
\newcommand{\cb}{b}

\newcommand{\iid}{\rm i.i.d.}
\newcommand{\ridgeless}{\hat\beta^{\rm mn}}

\newcommand{\Bern}{\rm{\scriptstyle Bern}}
\newcommand{\multi}{\rm{\scriptstyle multi}}

\usepackage{txfonts}

\usepackage{geometry}

\usepackage[subfigure]{tocloft}% http://ctan.org/pkg/tocloft
\makeatletter
\renewcommand{\numberline}[1]{%
  \@cftbsnum #1\@cftasnum~\@cftasnumb%
}
\makeatother

\newcommand{\scolor}[1]{{\color{magenta}#1}}
\newcommand{\scomment}[1]{\scolor{$\dagger$}\marginpar{\tiny\scolor{S:\ #1}}\hspace{-3pt}}

\newcommand{\wcolor}[1]{{\color{red}#1}}
\newcommand{\wcomment}[1]{\wcolor{$\dagger$}\marginpar{\tiny\wcolor{W:\ #1}}\hspace{-3pt}}

\newcommand{\nn}{\nonumber}

\begin{document}

\title{ \LARGE Ensemble  linear interpolators: The role of ensembling}     % Option 

\author{Mingqi Wu\thanks{Department of Mathematics and Statistics, Mcgill University, 805 Sherbrooke Street West, Montreal, Quebec H3A 0B9, Canada; Email: \texttt{mingqi.wu@mail.mcgill.ca}.} \and Qiang Sun\thanks{Corresponding author. Department of Statistical Sciences, University of Toronto, 100 St. George Street, Toronto, ON M5S 3G3, Canada; Email: \texttt{qiang.sun@utornoto.ca}.}}

%mingqi.wu@mcgill.ca, qsun.ustc@gmail.com

\date{August 31st, 2023 }

\maketitle

%\vspace{-0.25in}

% typeset the title of the contribution
\begin{abstract}

Interpolators are  unstable. For example, the mininum $\ell_2$ norm least square interpolator exhibits unbounded test errors when dealing with noisy data.  In this paper, we study how ensemble stabilizes and thus improves the generalization performance, measured by the out-of-sample prediction risk, of an individual interpolator. We focus on bagged linear interpolators, as bagging is a popular randomization-based ensemble method that can be implemented in parallel.  We  introduce the multiplier-bootstrap-based bagged least square estimator, which can then be formulated as an average of the sketched least square estimators. The proposed multiplier bootstrap encompasses the classical bootstrap with replacement as a special case, along with a more intriguing variant which we call the Bernoulli bootstrap. 

Focusing on the proportional regime where the sample size scales proportionally with the feature dimensionality, we investigate the out-of-sample prediction risks of the sketched and bagged least square estimators in both underparametrized and overparameterized regimes. 
Our results reveal the statistical roles of sketching and bagging. In particular, sketching modifies the aspect ratio and shifts the interpolation threshold of the minimum $\ell_2$ norm estimator.  However, the risk of the sketched estimator continues to be  unbounded around the interpolation threshold due to excessive variance. In stark contrast, bagging effectively mitigates this variance, leading to a bounded  limiting out-of-sample prediction risk. To further understand this stability improvement property, we establish that bagging acts as a form of implicit regularization, substantiated by the equivalence of the bagged estimator with  its explicitly regularized counterpart. We also discuss several extensions.

%As extensions of our core findings, we establish a link between the training error and the out-of-sample prediction risk, and delve into the calculation of adversarial risks.

%We also discuss several extensions. 

\end{abstract}

\noindent
{\bf Keywords}: Bagging, ensemble, interpolators, overparameterization, random matrix theory, sketching.

\tableofcontents

%%%%%%%%%%%%%%%%%%%%%%%%%%%%%%%%%%%%%%%%%%%
%%%%%%%%%%%%%%Introduction%%%%%%%%%%%%%%%%%
%%%%%%%%%%%%%%%%%%%%%%%%%%%%%%%%%%%%%%%%%%%
\section{Introduction}\label{sce:Intro}

Deep neural networks  have achieved remarkable performance across a broad spectrum  of tasks, such as computer vision, speech recognition, and natural language processing \citep{lecun2015deep, schmidhuber2015deep, wei2023overview}. These networks are often  overparameterized, allowing them to effortlessly interpolate the training data and achieve zero training errors. However, the  success of overparameterized neural networks comes at the cost of heightened instability. For example, neural networks can fail dramatically  in the presence of subtle and almost imperceptible image  perturbations in image classification tasks  \citep{goodfellow2018making} or  they may struggle  to  capture signals in low signal-to-noise ratio scenarios,  such as minor structural variations in image reconstruction tasks \citep{antun2020instabilities}.   The instability inherent to overparameterized models, or interpolators,  extends beyond deep neural networks; even simple models such as trees and linear models, exhibit instability, as evidenced by their unbounded test errors \citep{wyner2017explaining, belkin2018reconciling, hastie2019surprises}. 

%Instabilities are not only confined to deep neural networks.  Overparameterized models or interpolators exhibit instability in general. For example, simple models such as trees and linear models face instability challenges, as evidenced by their unbounded test errors \citep{wyner2017explaining, belkin2018reconciling, hastie2019surprises}. 

%In a matrix form, the model \eqref{eq:lm} can be written as
%\paragraph{Down to linear interpolators}

In this paper, we focus on least square interpolators.
Consider  independently and identically distributed (\iid) data $\cD := \{(x_i, y_i):  1 \leq i \leq \ndata\}$, which are  generated according to the model
\#\label{eq:lm}
&(x_i, \varepsilon_i) \sim (x,\varepsilon) \sim P_x\times P_\varepsilon, \nn\\
&y_i = x_i^\T \truesignal + \varepsilon_i, ~~ 1\leq i \leq n, 
\#
where  $P_x$ is a distribution on $\RR^\ndim$ and  $P_\varepsilon$ is a distribution on $\RR$ with mean $0$ and variance $\noiselev$.  In a matrix form, we can write 
$
Y = X \truesignal + E,
$
where $Y=(y_1,\ldots, y_n)^\T$, $X=(x_1,\ldots, x_n)^\T$, and $E = (\varepsilon_1,\ldots, \varepsilon_n)^\T$. 
Let the minimum $\ell_2$ norm, or min-norm for short, least square  estimator 
\citep{hastie2019surprises} be
\$
\ridgeless 
&:=  \argmin \left\{\|b\|_2: ~b~\text{minimizes}~ \|Y - Xb\|_2^2\right\} \\
&= (X^\T X)^{+}X^\T Y  = \lim_{\lambda \rightarrow 0^+}(X^\T  X + n\lambda I)^{-1} {X}^\T  {Y}, 
\$
where $(\cdot)^+$ denotes the Moore-Penrose pseudoinverse,  and $I\in \mathbb{R}^{\ndim \times \ndim}$ is the identity matrix. This estimator is also called the ridgeless least square estimator.  When $X$ possesses  full row  rank, typically the case for $p> n$, the min-norm  estimator becomes an interpolator, implying   $y_i = x_i^\T \ridgeless$ for $1\leq i \leq n$.

%and  is the converging limit of gradient descent when initialized  from zero \citep{hastie2019surprises}. 

%When $X$  has  full column rank, which is $n>p$ for example, the min-norm estimator reduces to the usual ordinary least square (OLS) estimator.  %$\ridgeless = (X^\T X)^{-1}X^\T y$,  %reduces to the usual ordinary least square (OLS) estimator, 
 
%In other words, the ridgeless least square estimator becomes an interpolator after the interpolation threshold $\ndim/n=1$. This estimator is the converging limit of gradient descent when initializing from zero    and agrees with the minimum-$\ell_2$-norm estimator for the linear regression problem \citep{zhang2021understanding}.  

\begin{figure}[t!]
       \centering    \includegraphics[width=0.60\textwidth]{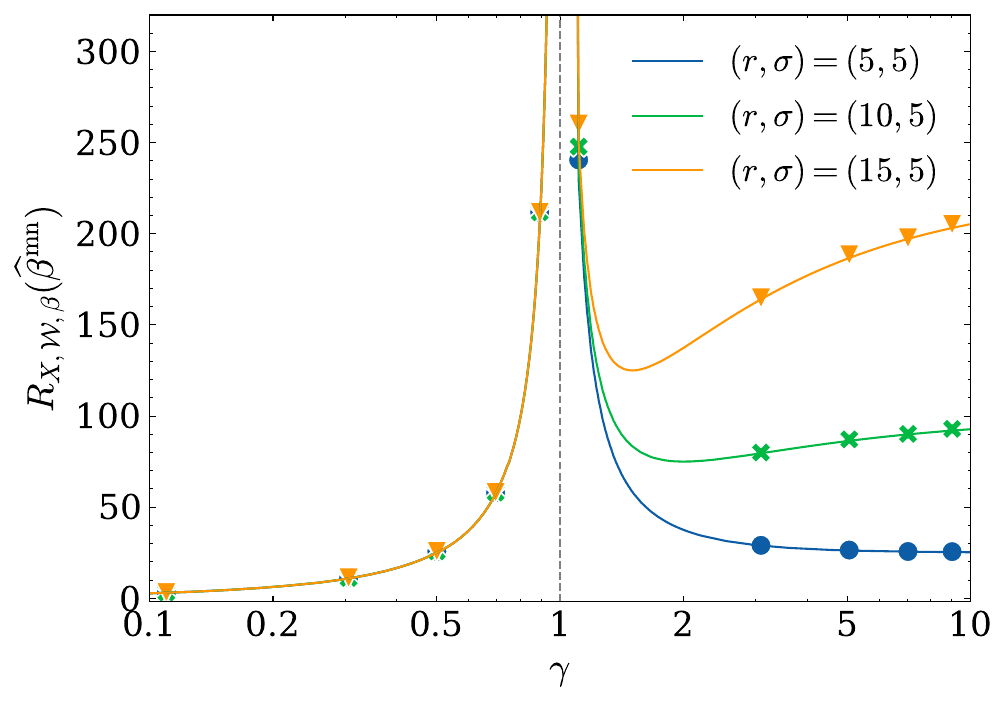}
    %\vspace{-1em}
    \caption{\small 
    Limiting out-of-sample prediction risk curves for the min-norm least square estimator, as functions of the aspect ratio $\gamma$. Rows of the feature matrix $X\in \RR^{n\times d}$ are \iid~drawn from $\mathcal{N} (0, I_d)$, errors $\varepsilon_i$ are \iid~drawn from $\cN(0, \noiselev)$, and $\beta$ is drawn from $\cN(0, r^2 I_d/d)$.   The blue, green, and yellow lines are theoretical risk curves with  $(r, \sigma)$ taking $(5, 5), (10, 5)$ and $(15, 5)$ respectively. The blue dots, green crosses, and orange triangles mark the finite-sample risks with  $n=400$, $\gamma$ varying in $[0.1,10]$ and $d = [n \gamma]$. Symbols are averages over 100 realizations. 
    }
    \label{fig:fig_1}
\end{figure}

%Existing literature centers on characterizing the out-of-sample (OS) prediction risks of interpolators and employing these characterizations  to elucidate the benign generalization properties of interpolators, in the hope of shedding light on the success of overparameterized neural networks \citep{belkin2018reconciling, ba2019generalization, liang2020just,  zhang2021understanding, mei2022generalization}. Taking a different perspective, we argue that interpolators are unstable. 

%\paragraph{Introducing the instability in terms of prediction risk} %$\ridgeless$ is unstable, as characterized  by its unbounded generalization performance.  

The least square interpolator $\ridgeless$ has unbounded test errors. 
To see this,  Figure \ref{fig:fig_1} plots its out-of-sample prediction risks  versus the dimension-to-sample-size ratio $\aspratio= \ndim/n$\footnote{To be rigorous, $\aspratio = \ndim/\ndata$ should be understood as $\aspratio \simeq \ndim/n$, that is, $\aspratio$ is asymptotically equivalent to  $\ndim/n$ or $\lim_{n\rightarrow} a_n/b_n = 1$.}, which is laso referred to as the aspect ratio.  Notably, the  out-of-sample prediction risk spikes around the interpolation threshold $\aspratio = 1$, leading to highly unstable predictions if the model size and thus the aspect ratio is not judiciously chosen. %as the sample size is often fixed  in practice. 
Recent empirical studies showed that ensembles of interpolators, including random forests and  deep ensembles, improve the predictive performance of individual interpolators \citep{lee2015m, wyner2017explaining}. For instance, \cite{wyner2017explaining} interpreted random forests as ``self-averaging interpolators" and hypothesized that such a behavior led to their success.  Similarly, \cite{lee2015m} demonstrated that deep ensembles, employing various ensembling strategies, effectively enhance the predictive performance of individual networks. However, the mechanisms through which ensembling fosters stability and improves the generalization performance of individual interpolators remain elusive. Therefore we ask the following questions:
\begin{quote}\label{q:role}
{\it What is the statistical role of ensembling?  How does ensembling  offer  stability and predictive improvement for individual interpolators? }
\end{quote}

This paper addresses  the above questions in the context of bagged least square interpolators. Bagging, an abbreviation for \underline{b}ootstrap \underline{agg}regat\underline{ing} \citep{breiman1996bagging}, represents a popular  randomization-based ensemble technique that lends itself to parallel implementation.
Given a training dataset $\cD$ comprising $n$ instances, bagging generates $B$ subsamples ${\cD_1, \ldots, \cD_B}$, each of size $n$, through uniform sampling with replacement from $\cD$. This sampling procedure is called the classical bootstrap, and each of these subsamples is referred to as a bootstrap sample. Subsequently, $B$ min-norm least square estimators are individually fitted to these $B$ bootstrap samples, which are then aggregated to produce the final estimator. 
We shall call the resulting estimator  the bagged least square estimator or simply the bagged estimator\footnote{We use the bootstrap estimator interchangeably with the bagged estimator.}.

%This paper answers the questions above in the case of bagged least square interpolators. Bagging, short for \underline{b}ootstrap \underline{agg}regat\underline{ing} \citep{breiman1996bagging},  is a popular randomization-based  ensemble method  and can be implemented in parallel. Given the training dataset $\cD$ of size $n$, bagging generates $B$ subsamples $\{\cD_1, \ldots, \cD_B\}$, each of size $n$, by sampling from  $\cD$ uniformly and with replacement. Each subsample is known as a {bootstrap} sample. Then $B$ min-norm least square estimators  are fitted  on  these  $B$ bootstrap samples, and combined by simply taking the empirical average. We shall call this  averaged estimator  the bagged least square estimator, or simply the bagged estimator\footnote{We shall sometimes refer to the bagged estimator as the bootstrap estimator.}. 

%in the proportional regime where

Specifically,  we study the generalization performance, measured by the out-of-sample prediction risks, of bagged least square interpolators  when the dimensionality scales proportionally  with the sample size  $\ndim \asymp \ndata$. This allows us to understand how bagging  stabilizes and thus improves the generalization performance.  We summarize our main contributions below. 
\begin{enumerate}
\item %{\bf New formulation:} 
First, we introduce a new bagged least square estimator using a multiplier bootstrap approach, and formulate the former as an average of sketched least square estimators. This approach includes  the classical bootstrap with replacement as a special case and introduces an intriguing variant which we call the Bernoulli bootstrap. Remarkably, bagged estimators based on these two bootstrap methods yield identical limiting out-of-sample prediction risks, when the success probability in the Bernoulli bootstrap is taken as $1-1/e \approx 0.632$.

\item %{\bf Exact prediction risk formulas:} 
Second, we provide precise  formulas for the limiting out-of-sample prediction risks of both sketched and bagged least square estimators. This unveils the statistical roles of sketching and bagging in terms of the generalization performance. While sketching modifies the aspect ratio and shifts the interpolation threshold of the min-norm estimator, the prediction risk still spikes near the interpolation threshold. Remarkably, bagging consistently stabilizes and improves the risks of an individual sketched estimator, and keeps the risk of the bagged estimator bounded. This benefit arises from fact that  the bagged estimator is not an interpolator anymore. 

\item %{\bf Influence of sampling distribution:}
Third, we find that the out-of-sample prediction risk of the sketched estimator depends on the sampling distribution of the multipliers, while the risk of the bagged least square estimator is invariant to it. This demonstrates certain robustness of bagging. For sketching, we establish the optimality of the Bernoulli sketching among various sketching methods. Numerical studies prove that the Bernoulli sketching is computationally faster than other methods.

%Third, the out-of-sample prediction risk of the sketched estimator depends on the sampling distribution of the multipliers  in the underparameterized regime while being invariant to it in the overparameterized regime.  We then  establish the optimality of the Bernoulli sketching among various sketching matrices in the underparameterized regime, which extends to optimality in both regimes. These findings complement prior work by \cite{chen2023sketched}, who investigated different sketching methods. In contrast, the prediction risk of the bagged least square estimator is invariant to the sampling distribution in both regimes.

\item %{\bf Bagging as implicit regularization:}
Fourth, we prove that bagging acts as a form of implicit regularization. Specifically, we establish that the bagged estimator is equivalent to a ridge regression estimator under isotropic features. Under more general covariance matrices, we prove an  equivalence between the bagged least square estimator under model \eqref{eq:lm} and the min-norm least square estimator under a different model with a  shrunken feature covariance and a shrunken signal. 

%These findings underscore bagging's role as a form of implicit regularization.

\item %{\bf Computational efficiency gain:}
 Fifth, in comparison to the classical bootstrap, which requires multiple passes over the training data and is challenging to parallelize due to random selection of $n$ elements, the Bernoulli bootstrap stands out as a more efficient alternative. Specifically, when generating a single bootstrap sample, instead of randomly drawing from the sample data with replacement, each data point is assigned a random Bernoulli weight. %with the success probability $\sampleratio$. 
 Bernoulli bootstrap  is particularly advantageous for large values of $n$ due to its lower number of data passes.

\item  %{\bf Extensions:}
Lastly, we characterize the training error  and adversarial risk of the bagged estimator. The training error is linear in the out-of-sample prediction risk: Better the rescaled  training error, better the prediction risk. Thus, the bagged estimator with the smallest rescaled training error generalizes the best. Comparing with the full-sample min-norm estimator, bagging also helps stabilize the adversarial risk, in part by shrinking the $\ell_2$ norm of the estimator.

\end{enumerate}

\subsection*{Related work}

%\paragraph{Related work}
%\noindent{\bf Generalization properties of Interpolators}
%The phenomenon of interpolation, which refers to the ability of modern deep learning models to fit data perfectly while still achieving good generalization performance, has been extensively studied theoretically in various models. \cite{zhang2021understanding} provided significant insights into this phenomenon, and it has since garnered considerable attention. \cite{belkin2018reconciling} offered an explanation for this phenomenon by introducing the concept of a double descent curve in the risk curve, illustrating that interpolating functions have smaller norms and can be considered "simpler". Subsequently, \cite{hastie2022} further confirmed the presence of this double descent curve in the context of ridgeless regression models. \cite{liang2020} extended these findings to kernel regression models, while \cite{Ghorbani2021} explored the phenomenon in the context of random feature regression. Additionally, \cite{xiaolechao2022} investigated multiple descents in the risk curve for kernel regression, focusing on a higher-order asymptotic regime.

We briefly review work that is  closely related to ours. 

%The generalization properties of over-parametrized models have been widely studied in recently years. It starts with  the observation that 

%,including the linear model \citep{hastie2019surprises,richards2021asymptotics}, the random feature model \citep{mei2022generalization} and the partially optimized two-layer neural network \citep{ba2019generalization}, among others. 

\paragraph{Benign generalization of overparameterized models} 
Overparameterized neural networks often exhibit benign generalization performance, even when they are trained without explicit regularization \citep{he2016deep, neyshabur2014search, canziani2016analysis, novak2018sensitivity, zhang2021understanding, bartlett2020benign, liang2020just}.  
\cite{belkin2018reconciling} furthered this line of research  by positing the ``double descent" performance curve applicable to models beyond neural networks. This curve subsumes the textbook $U$-shape bias-variance-tradeoff curve  \citep{hastie2009elements}, illustrating  how increasing model capacity beyond some interpolation threshold may yield improved out-of-sample prediction risks. Subsequent studies characterized this double descent phenomenon for various simplified models \citep{hastie2019surprises, ba2019generalization, richards2021asymptotics, mei2022generalization}.

\paragraph{Sketching}
\cite{raskutti2016statistical} and \cite{edgar2019} first studied the statistical  performance of several sketching methods in the underparameterized regime. More recently,  \cite{chen2023sketched} analyzed the out-of-sample prediction risks of the sketched ridgeless least square estimators in both underparameterized and overparameterized regimes. They introduced an intriguing perspective: Sketching serves as the dual of overparameterization and may help generalize. However, the sketching matrices considered there, such as orthogonal and \iid~sketching matrices, are of full rank. In contrast, we focus on  diagonal and mostly singular sketching matrices in the form  of \eqref{eq:sketching}, which complements their results. Here, each sketch dataset is constructed by multiplying each observation in the training dataset by  a random multiplier. This approach is more computationally efficient than  multiplying the data matrix by the often dense sketching matrix as done in \cite{chen2023sketched}, which is also supported by our numerical studies. %Other works on sketching includes \cite{edgar2021} %Moreover, we provably show that bagging stabilizes and thus improves the generalization performance.  

\paragraph{Bootstrap and bagging}
Bootstrap, a resampling technique introduced by~\cite{Efron1979} and inspired by the jackknife method~\citep{Jackknife}, finds extensive applications  in estimating and inferring sampling distributions, such as the regression coefficient $\truesignal$ in model \eqref{eq:lm}. However, as the aspect ratio $\ndim/\ndata$ approaches the interpolation threshold, the regression coefficient $\truesignal$ quickly becomes unidentifiable, rendering its estimation and inference a groundless task. Indeed, \cite{el2018can} concluded that it is perhaps not possible to develop universal and robust bootstrap techniques for inferring $\truesignal$ in high dimensions. This negative observation, with hindsight, is  not surprising because   identifiability issues for parameter estimation  arise in high dimensions, while   prediction tasks are free of such issues. \cite{buhlmann2002analyzing} established that bagging acts as a smoothing operation for hard decision problems. In contrast, we analyze bagged least square interpolators, showing how bagging can stabilize the variance and benefit generalization properties.

\subsection*{Paper overview}
The rest of this paper proceeds as follows. The first introduce  notation that is used throughout the paper.  Section~\ref{sec:pre} presents the proposed bagged estimator, related definitions, and standing assumptions.    In Section~\ref{sec:isotropic}, we study the out-of-sample prediction risks in the context of isotropic features. Section~\ref{sec:correlated} delves into the analysis of correlated features. We extends the results to the deterministic signal case, characterize the training errors and adversarial risks in Section \ref{sec:extension}. Finally, Section~\ref{sec:Discus} concludes the paper with discussions. %and discusses future directions. 
All proofs are collected in the supplementary material.

\paragraph{Notation}
 We  use $c$ and $C$ to denote generic constants which may change from line to line. For a sequence $a_n$, $a_n\rightarrow a^-$ or $a_n\searrow a$ denotes that $a_n$ goes to $a$ from the left side, while $a_n\rightarrow a^+$ or $a_n\nearrow a$ denotes that $a_n$ goes to $a$ from the right side. 
For a vector $u$ and any $p\geq 1$, $\|u\|_p$ denotes its $p$-th norm.
For a matrix $A$, we use  $A^+$ to denote its Moore-Penrose pseudoinverse,  $\| A \|_{2} $ to denote its spectral norm, $\|A\|_\rF$ to denote its Frobenius norm, and $\tr(A)$ to denote its trace. 
For a sequence of random variables $\{X_n\}$, we use $X_n\overset{\as}{\rightarrow} X$ to denote that $X_n$ converges almost surely to $X$, and $X_n \rightsquigarrow  X$ to denote that $X_n$ converges in distribution to $X$.

%We shall note that the Frobenius norm coincides with the $\ell_2$ norm when $\Theta$ is a vector. 

\iffalse
For $f(x_n, y_n)$,  we shall use the following simplified notation \scomment{this is confusing? }
\$
\lim_{x_n\rightarrow\infty \atop y_n\rightarrow \infty} f(x_n, y_n ) &= \lim_{ y_n\rightarrow \infty} \lim_{x_n\rightarrow\infty} f(x_n, y_n )  , ~~~ 
\lim_{y_n\rightarrow\infty \atop x_n\rightarrow \infty} f(x_n, y_n )  = \lim_{ x_n\rightarrow \infty} \lim_{y_n\rightarrow\infty}f(x_n, y_n ) . 
\$
\fi

%%%%%%%%%%%%%%%%%%%%%%%%%%%%%%%%%%%%%%%%%%
%%%%%%%%%%%%%Preliminaries%%%%%%%%%%%%%%%%
%%%%%%%%%%%%%%%%%%%%%%%%%%%%%%%%%%%%%%%%%%

\section{Bagged linear interpolators}\label{sec:pre}

This section introduces the  multiplier bootstrap procedure and the associated bagged ridgeless least square estimator,  formally defines the out-of-sample prediction risk, and  states standing assumptions.

%%%%%%%%%%%%%%%%%%%%%%%%%%%%%%%
%%%%%Data model and risk%%%%%%%
%%%%%%%%%%%%%%%%%%%%%%%%%%%%%%

\subsection{Bagged linear interpolators}

We introduce the multiplier-bootstrap-based  bagged ridgeless least square estimator. Recall the training dataset $\cD := \{(x_i, y_i)\in \RR^\ndim \times \RR: 1\leq i \leq n  \}$. Let $\cW_k=\{\weight_{k,1}, \ldots, w_{k,\ndata}\}$ be $\ndata$ non-negative multipliers. Multiplying each summand of $L_n = \sum_{i= 1}^\ndata (y_i - x_i b)^2$ by each multiplier in $\cW_k$ and summing them up, we obtain the bootstrapped empirical loss
\$
L_n^{(k)}(b) := \sum_{i= 1}^\ndata w_{k,i}(y_i - x_i\transp b)^2. 
\$
We  then calculate the bootstrap  estimator $\hat\truesignal_k$ as the min-norm solution that minimizes the above  loss
\#\label{eq:multi_loss}
\estimatork  &= \argmin \Big\{\|b\|_2: ~b~\text{minimizes}~  L_n^{(k)}(b) = \sum_{i=1}^n \weight_{k,i} \|y_i-x_i^\T b\|_2^2\Big\}.  
\#
This procedure is repeated  $\nresample$ times to obtain a sequence of estimators $\big\{ \hat\beta_k:\, 1\le k \le B \big\}$.  The bagged least square estimator, or simply the bagged estimator, is obtained by  simply averaging these $B$ estimators as
\begin{equation}\label{eq:Bestimator} 
    \estimator = \frac{1}{\nresample}\sum_{k=1}^\nresample \estimatork.
\end{equation}

%Given the training dataset $\cD=(X, Y)$, the multiplier bootstrap generates $B$ subsamples $\cD_k$, each referred to as a bootstrap sample. 

%where  $\cW_k=\{\weight_{k,1}, \ldots, w_{k,n}\}$ are the same non-negative multipliers as above. 

The above bagged estimator $\estimator$ can be formulated as an average of sketched ridgeless least square estimators. To see this, it suffices to show  each individual  estimator $\estimatork$ is a sketched estimator.  Let $\sketchmat_k$ be the sketching matrix such that
\begin{equation}\label{eq:sketching}
\sketchmat_k= 
\begin{bmatrix}
\sqrt{\weight_{k,1}} & 0 & \ldots  & 0 \\
0 & \sqrt{\weight_{k,2}}  & \ldots  &0 \\
\vdots & \vdots & \ddots  &\vdots \\
0& 0 & \ldots    & \sqrt{\weight_{k,n}}
\end{bmatrix}.
\end{equation} 
Then the sketched dataset $\cD_k=(S_kY, S_kX)$ corresponds to the $k$-th bootstrap sample in the multiplier bootstrap, and  the $k$-th individual bootstrap estimator $\estimatork$ coincides  with sketched ridgeless least square estimator fitted on $\cD_k$:  
\begin{align}
    \estimatork
     &= \argmin \Big\{\|b\|_2: ~b~\text{minimizes}~ \sum_{i=1}^n \weight_{k,i} \|y_i-x_i^\T b\|_2^2\Big\}  \nn \\
    &= \argmin \left\{\|b\|_2: ~b~\text{minimizes}~ \left\| \sketchmat_k Y - \sketchmat_k Xb   \right\|_2^2   \right\} \nn \\
     &= \left((\sketchmat_k  X)\transp \sketchmat_k X\right)\pinv X^\top \sketchmat_k\transp \sketchmat_k Y.  \label{eq:sketched}
\end{align}

%Thus the bagged estimator is an average of $B$ sketched estimators. 

Finally, our multiplier bootstrap framework encompasses the classical bootstrap with replacement as a specific case. In the classical bootstrap, each bootstrap sample $\cD_k$ is generated by independently and uniformly sampling from $\cD$ with replacement. This sampling method permits the possibility of individual observations being repeated within $\cD_{k}$. When the sample size $n$ is sufficiently large, $\cD_{k}$ is expected to contain approximately $1 - 1/e \approx 63.2\%$ distinct examples from $\cD$. In this case, each bootstrap sample $\cD_k$ corresponds to the sketched dataset $(\sketchmat_kX, \sketchmat_k Y)$, where the sketching matrix $\sketchmat_k$ is composed of multipliers $(\weight_{k,1}, \ldots, w_{k,\ndata}) \sim {\rm Multinomial}(\ndata; p_1,\ldots, p_\ndata)$ with $p_i= 1/n$ for $1\leq i \leq n$ as the diagonal entries.

%\wcomment{$p_{i} = \frac{1}{\ndata}$?} 
%while the remaining portion consists of duplicates.

%%%%%%%%%%%%%%%%%%%%%%%%%%%%%%%%%%%%%%%%%%%%%%%%
%%%%%%%%%%%%Risk, bias, and variance%%%%%%%%%%%%
%%%%%%%%%%%%%%%%%%%%%%%%%%%%%%%%%%%%%%%%%%%%%%%%

\subsection{Risk, bias, and variance}

%In this paper, we focus on investigating the out-of-sample (OS) prediction risk, or simply the risk, with respect to the training data $X$ and the multiplier matrix $\cW$. Given a test point $x_0 \sim P_x$ that is independent of the training data and multiplier matrix, we define the OS prediction risk of an estimator $\hat{\truesignal}$ as follows:

Let us consider a test data point $x_{\new}\sim P_x$, which is independent of both the training data and the multipliers $\{\cW_k: 1\leq k \leq B\}$. To measure the generalization performance, we consider the following out-of-sample prediction risk,  also referred to as the prediction risk or simply risk: 
\begin{equation}\label{eq:riskcon}
\riskcondition(\estimator) 
= \EE\left[\left(x_\new \transp \estimator - x_\new \transp \truesignal\right)^2 \condition \variablecondition\right]
=   \EE\left[ \left\| \estimator - \truesignal\right\|_\covmat^2 \condition \variablecondition\right], 
\end{equation}
where $\|x\|_\Sigma^2 = x^\T\Sigma x$, and the conditional expectation is taken with respect to the noises $\{ \varepsilon_{i}\}_{1 \leq i \leq \ndata}$ and the test point $x_\new $. We have the following bias-variance decomposition  
\begin{align*}
\riskcondition(\estimator) &=  \big\| \EE\left(\estimator \condition \variablecondition \right) -\truesignal\big\|_\covmat^2 + \tr\left[ \cov\left(\estimator \condition \variablecondition \right)\Sigma\right]\\
&= \biascondition(\estimator) + \varcondition(\estimator),
\end{align*}
where {$\|x\|_\Sigma^2 = x^\T\Sigma x$}. 

%$\tr(A)$ represents the trace of $A$.

Our next result provides expressions for the bias and variance of the bagged least square estimator~\eqref{eq:Bestimator}. In the subsequent sections, we  will characterize the out-of-sample prediction risks by analyzing  the bias and variance terms respectively.

%\scomment{do we need assumption \ref{Assume:Covdistri}} %and Assumption~\ref{Assume:Covdistri}, 

\begin{lemma}[Bias-variance decomposition]\label{lm:biasvar}
Under model~\eqref{eq:lm} with $\cov(x)=\Sigma$, 
the bias and variance of the bagged linear regression estimator~$\estimator$ are
\begin{align}
    &\biascondition\left(\estimator\right) = \frac{1}{\nresample^{2}} \sum_{k,\ell} \truesignal\transp \projmat_{k} \covmat \projmat_{\ell} \truesignal , \label{eq:bias}\\
    &\varcondition(\estimator)=\frac{\noiselev}{\nresample^{2}} \sum_{k,\ell} \frac{1}{n^{2}} \tr\left( \hcovmat_{k}\pinv X\transp \sketchmat_{k}^2 \sketchmat_{\ell}^2 X \hcovmat_{\ell}\pinv \covmat \right), \label{eq:var}
\end{align}
where $\hcovmat_{k} = X\transp \sketchmat_{k}\transp \sketchmat_{k} X/n$ is the sketched sample covariance matrix, and $\projmat_{k}=I-\hcovmat_{k}\pinv \hcovmat_{k}$ is the projection matrix onto the null space of $S_kX$. 
\end{lemma}

%Assuming Assumption~\ref{Assume:beta4+mom}, then we have

%In the rest of the paper, we  will characterize the out-of-sample prediction risks by analyzing  the bias and variance terms respectively. % to investigate the out-of-sample prediction risks separately. 

%Specifically, we first analyze the asymptotic behaviors of the bias and variance terms by considering a single subsample, focusing on the terms where $i = j$. We then proceed to investigate cross-term ($i \neq j$) generated by Aggregating through different subsamples. These analyses are conducted under the high-dimensional regime, as defined in Assumption~\ref{Assume:highdim}.

\subsection{Standing assumptions}

This subsection collects standing assumptions. %Our first two assumptions are on the proportional asymptotic regime, and  moment and covariance conditions. 

\begin{assumption}[Proportional asymptotic regime]\label{Assume:highdim}
Assume  $\npinfty$ such that  $\ndim/\ndata \rightarrow \aspratio\in (0, + \infty)$. 
\end{assumption}

% for some $\aspratio \in (0, + \infty)$. This dimensionality-to-sample-size ratio $\aspratio$ is also referred to as the aspect ratio in the random matrix literature. %\scomment{give me a reference} 

\begin{assumption}[Moment and covariance conditions]\label{Assume:Covdistri}
Assume that the feature vector $x$ can be written as $x = \Sigma^{1/2} z$, where $z \in \mathbb{R}^{\ndim}$ has \iid~entries, each with a zero mean, unit variance, and a bounded $(8+\epsilon)$-th moment. Moreover,  the eigenvalues of $\Sigma$ are bounded away from zero and infinity, i.e., $0 < c_\lambda \leq \eigval_{\min}(\Sigma) < \eigval_{\max}(\Sigma) \leq C_\lambda < +\infty$  where $c_\lambda$ and $C_\lambda$ are constants. The empirical spectral distribution $F_{\Sigma}$ of $\Sigma$ converges weakly to a probability  measure $\limitdistr$ as $\npinfty$.
\end{assumption}

\iffalse
\begin{assumption}[Covariance matrix]\label{Assume:Covmatrix}
Let $\Sigma_d \in \mathbb{R}^{\ndim \times \ndim}$ denote the population covariance matrix of $x$. We assume the eigenvalues of $\Sigma_d$ are bounded, i.e., $0 < c < \eigval_{\min}(\Sigma_d) < \eigval_{\max}(\Sigma_d) < C < +\infty$, where $c$ and $C$ are constants. Furthermore, we assume that the empirical spectral distribution $F_{\Sigma_d}$ of $\Sigma_d$ converges weakly to a measure $\limitdistr$ as $\npinfty$.
\end{assumption}

\scolor{
We require a slightly stronger condition of a bounded $8+\epsilon$-th moment for each entry of $z$ compared to the $4+\epsilon$-th moment requirement by~\cite{hastie2019surprises}. The requirement of a bounded $8+\epsilon$-th moment in Assumption~\ref{Assume:Covdistri} is necessary to utilize the result concerning the uniform concentration of quadratic forms, as illustrated in Lemma 1 of \cite{ledoit2011eigenvectors}, to establish results on the Stieltjes transforms.  
}
\fi

Assumption \ref{Assume:highdim} specifies the proportional asymptotic regime and is frequently  adopted by recent literature on exact risk characterizations \citep{hastie2019surprises, mei2022generalization, chen2023sketched}.  Assumption~\ref{Assume:Covdistri} specifies the covariance structure and moment conditions. We require a bounded $(8+\epsilon)$-th moment for each entry of $z_i$, which is slightly stronger than the $(4+\epsilon)$-th moment condition  by \cite{hastie2019surprises}. This is because we need to derive a uniform concentration inequality on quadratic forms  to establish results on the Stieltjes transforms; see Lemma \ref{lm:concentrationontrace}. % \scomment{why do we need uniform concentration for Stieltjes transforms? Others do not need uniform concentration results?}  
Our next assumption is on the multipliers.

%Assumption~\ref{Assume:beta4+mom} specifies a random 

%eliminates the interaction between the signal and covariance structure, simplifying the analysis of the role of bagging. We note that the random signal assumption is not necessary in the isotropic case where $\covmat = I$. Furthermore, we extend our analysis to the deterministic signal $\truesignal$. This extension allows us to establish an equivalence between bagging and ridgeless regression in Corollary~\ref{cor:bootstrapequiv}.

%%%%%%%%%%%%%%%%%%%%%%%%%%%%%%%%%%%%%%%%%%%%%%%%%%%%
%%%%%%%%%%%%%%%%%%%Multipliers%%%%%%%%%%%%%%%%%%%%%%
%%%%%%%%%%%%%%%%%%%%%%%%%%%%%%%%%%%%%%%%%%%%%%%%%%%%

%lower bound for the nonzero values of $w_{k,i}$. 

% $(\weight_{1,1}, \weight_{2,1}), \dots, (\weight_{1,\ndata}, \weight_{2,\ndata})$

\begin{assumption}[Multipliers]\label{Assume:multip}
Assume the multipliers $\{\cW_k: 1\leq k \leq B\}$ are non-negative, independent of the training dataset $(X, Y)$, and the non-zero multipliers are bounded away from zero, i.e., there exists a positive constant $c_w$ such that
\begin{equation*}
\PP\left(w_{k,i}\in \{0\}\cup [c_w, +\infty)\right) = 1, ~\text{for all}~1\leq k \leq B,\,  1\leq i \leq n. 
\end{equation*} 
Moreover, the empirical measure of the multipliers $\cW_k$ %$\cW_k=\{\weight_{k,1}, \ldots, w_{k,n}\}$ 
converges weakly to some probability measure $\limitdistrweight$ as $\ndata\rightarrow \infty$ almost surely\footnote{We say $\mu_n$ converges weakly to $\mu$ almost surely if, for each continuous bounded function $g$,  $\lim \int g \mu_n(dx) = \int g \mu(dx)$ for $\mu$-almost all $x$.}. We refer to $\sampleratio := 1- \limitdistrweight(\{0\})>0$ as the {downsampling ratio}. Additionally, assume  $\cW_1, \ldots, \cW_B$ are asymptotically pairwise independent, aka the joint empirical measure of any two distinct  sets of multipliers, i.e., any $\cW_k, \cW_\ell$ with $k\ne \ell$, converges weakly to the probability measure $\limitdistrweight \otimes \limitdistrweight$ almost surely.
\end{assumption}

If $\{\cW_k, 1\leq k \leq B\}$ are independently and identically distributed, then they are also pairwise independent. In addition to requiring the multipliers to be pairwise independent, non-negative, and independent of the training data, Assumption~\ref{Assume:multip} follows a similar spirit as Assumption \ref{Assume:Covdistri} by assuming that the empirical measure of the multipliers converges weakly to a limiting probability measure $\limitdistrweight$ almost surely. The limiting measure $\limitdistrweight$ consists of point masses either at $0$ or in an interval bounded away from $0$. The downsampling ratio $\sampleratio= 1- \limitdistrweight({0})$ quantifies the long-term proportion of nonzero multipliers, reflecting the fraction of training samples picked up by each bootstrap sample. As $\aspratio/\sampleratio$ corresponds to the aspect ratio within each bootstrap sample, we shall refer to  the underparameterized and overparameterized regimes as $\aspratio/\sampleratio < 1$ and $\aspratio/\sampleratio > 1$ respectively.  Before presenting some  examples of the multiplier bootstrap procedures that satisfy  Assumption \ref{Assume:multip}, we need the following lemma. 

\begin{lemma}\label{lm:limitdistrMultinomial}
Let $\tildelimitdistrweight$ be the empirical measure of $(\weight_{k,1}, \ldots, w_{k,\ndata})\sim {\rm Multinomial}(\ndata; 1/\ndata,\ldots, 1/\ndata)$ and $\poisson(1)$ be the Poisson distribution with parameter $1$.  Then, {almost surely}, we have
\begin{equation*}
{\tildelimitdistrweight} \rightsquigarrow \poisson(1),~ \text{as}~ \npinfty. 
\end{equation*}
\end{lemma}

%%%%%%%%%%%%%%%%%%%%%%%%%%%%%%%%%%%%%
%%%%%%%%%%%%%%Examples%%%%%%%%%%%%%%%
%%%%%%%%%%%%%%%%%%%%%%%%%%%%%%%%%%%%%

\begin{example}[Classical bootstrap with replacement~\citep{Efron1979}]
The classical bootstrap with replacement is equivalent to the multiplier bootstrap with 
\iid~$\cW_1, \ldots, \cW_B$ such that $\cW_k = (\weight_{k,1}, \ldots, w_{k,\ndata}) \sim {\rm Multinomial}(\ndata; 1/\ndata,\ldots, 1/\ndata)$. %and $p_1=\ldots = p_n=1/\ndata$. 
Then Lemma \ref{lm:limitdistrMultinomial} indicates that the empirical measure of the multipliers $(\weight_{k,1}, \ldots, w_{k,\ndata})$ converges weakly to a Poisson distribution with parameter $1$ almost surely.  Consequently, the downsampling ratio is $\sampleratio = 1 - 1/e\approx 0.632>0 $.
\end{example}

\begin{example}[Bernoulli bootstrap]
The Bernoulli bootstrap samples \iid~ multipliers $w_{i,j}$ from the Bernoulli distribution with success probability $p$. In this case, the downsampling ratio is  $\sampleratio = p$. 
\end{example}

\begin{example}[Jackknife~\citep{Jackknife}]
The jackknife method samples multipliers $\{w_{k,j}\}$ such that exactly one multiplier equals zero, while all the others  equal one. In this case, the multipliers $\{\cW_k, 1\leq k \leq B\}$ are pairwise independent and  the downsampling ratio $\sampleratio$ equals $1$.
\end{example}

Lastly, we assume that the true signal vector is isotropic. 

\begin{assumption}[Random signal]\label{Assume:beta4+mom}
The true signal $\truesignal$ is a random vector with  i.i.d. entries,  $\EE[\beta] =0$, $\EE\big[d\beta_j^2\big]=r^2$, and  $\EE\big[ |\beta_j|^{4+\eta}\big]\leq C$ for some $\eta>0$ and $C<\infty$.  Moreover, the random $\beta$ is independent of the training data $(X, E)$ and the multipliers $\cW_{k}, 1 \leq k \leq \nresample$. 
\end{assumption}

%We shall write  $\signallev:= \|\truesignal\|_{2}^{2} $. %We further assume that the true signal is independent of the training data ${ X,Y }$ and the multipliers $\cW_{k}, 1 \leq k \leq \nresample$. \scomment{revise this assumption?} %where $\signallev > 0$ is a constant. 

%We first focus on the the case of random $\beta$ case in Assumption \ref{Assume:beta4+mom}, where $\beta$ follows an isotropic  distribution, allowing for clear presentation of the exact risk results. The assumption of random $\beta$ is commonly adopted in the literature \citep{dobriban2018,li2021asymptotic}. We also consider the deterministic $\beta$ in Section \ref{sec:correlated},  where the interaction between $\beta$ and $\Sigma$ needs to be taken into account.  With Assumption \ref{Assume:beta4+mom},  we first present a simplified Lemma \ref{lm:biasvar}, which will be used in the following sections. 

We first focus on the case of a random $\beta$ as specified in Assumption \ref{Assume:beta4+mom}, where $\beta$ follows an isotropic distribution. This  assumption facilitates  a clear presentation of the exact risk results. Such an assumption is commonly adopted  in the literature \citep{dobriban2018,li2021asymptotic}. We also consider the case of a deterministic $\beta$ in Section \ref{sec:extension}, where the interplay between $\beta$ and $\Sigma$ needs to be taken into account. Under Assumption \ref{Assume:beta4+mom}, we present the following simplified version of Lemma \ref{lm:biasvar}. %which will be used in forthcoming sections.

%\|\Sigma\|_2\leq C<\infty$ for some positive constant $C$, 

\begin{lemma}\label{lm:biasvar-beta}
Assume model~\eqref{eq:lm} with $\cov(x)=\Sigma$ and  Assumption \ref{Assume:beta4+mom}. If $\Sigma$ has bounded eigenvalues,  
then the bias of the bagged linear regression estimator~$\estimator$ is %satisfies
\begin{align}
    \lim_{\npinfty} \biascondition\left(\estimator\right) = \lim_{\npinfty}\frac{\signallev}{\nresample^{2}} \sum_{k,\ell} \frac{1}{\ndim} \tr\left( \projmat_{k} \covmat \projmat_{\ell} \right) \qas \label{eq:bias-beta}
\end{align}
The variance is the  same as in  Lemma \ref{lm:biasvar}. 
\end{lemma}

\section{A warm-up: Isotropic features}\label{sec:isotropic}

As a warm-up, this section studies the case of isotropic features where the covariance matrix is an identity matrix $\covmat=I$. The investigation of the correlated case will be postponed to Section \ref{sec:correlated}. We present first the limiting risk of the sketched min-norm least square estimator  $\hat\beta^1$, aka $\estimator$ with $B=1$, and then the risk of the bagged min-norm least square estimator $\estimator$. These risk characterizations  shed light on how bagging stabilizes and thus improves the generalization performance.

%%%%%%%%%%%%%%%%%%%%%%%%%%%%%%%%%%%%%%%%
%%%%%%%%%%%%%Sketching?%%%%%%%%%%%%%%%%%
%%%%%%%%%%%%%%%%%%%%%%%%%%%%%%%%%%%%%%%%

\subsection{Sketching shifts the interpolation threshold}

This subsection studies the risk of  $\hat\beta^1$ when the sketching matrix, in the form of \eqref{eq:sketching}, is diagonal and mostly singular.  While \cite{chen2023sketched} also explored the exact risks of sketched ridgeless least square estimators, they required the sketching matrices to be full rank, which distinguishes it from our work.  In order to characterize the underparametrized variance, we need the following lemma.

\begin{lemma}\label{lemma:unique_sketching_iso}
 Assume Assumptions \ref{Assume:highdim}-\ref{Assume:multip} and $\covmat = I$.  Suppose  $\aspratio > \sampleratio$.  Then the following equation  has a unique positive solution $c_0:= c_0(\aspratio, \limitdistrweight)$ with respect to $x$,
\begin{equation}\label{eq:uniqueMPlawsketching}
    \int \frac{1}{1+xt} \limitdistrweight(dt) = 1-\aspratio.
\end{equation}
\end{lemma}

The above lemma establishes the existence and uniqueness of a positive  solution to  equation \eqref{eq:uniqueMPlawsketching}. Equations of this type are known as  self-consistent equations \citep{bai2010}, and are fundamental in caldulating the exact risks. The solutions to self-consistent equations  are fundamental in calculating the exact risks. They do not generally admit  closed form solutions but can be solved numerically.   Our next result characterizes the limiting risk, as well as the limiting bias and variance, of $\hat\beta^1$. Both the variance and risk in the underparameterized regime depend on the solution to equation \eqref{eq:uniqueMPlawsketching}.

%Our next result characterizes  the limiting risk of $\hat\beta^1$ which depends on the solution to equation \eqref{eq:uniqueMPlawsketching}.  %as well as the limiting bias and variance of $\hat\beta^1$. 

\begin{theorem}[Sketching under isotropic features]\label{thm:isosketching}
 Assume  Assumptions~\ref{Assume:highdim}-\ref{Assume:beta4+mom} and $\covmat = I$. 
The out-of-sample prediction risk of $\hat\beta^{1}$ satisfies %, almost surely, 
\#\label{eq:risk_sketching_iso}
\lim_{\npinfty} \riskcondition (\hat\beta^{1})
&=
\begin{cases}
\sigma^2\left(\frac{\aspratio}{1-\aspratio - f(\aspratio)} - 1 \right),& \aspratio/\sampleratio < 1\\
\signallev \frac{\aspratio/\sampleratio - 1 }{\aspratio/\sampleratio} + \noiselev \frac{1}{\aspratio/ \sampleratio - 1}, & \aspratio/\sampleratio > 1
\end{cases} \qas
\#
where 
$
f(\aspratio) = \int \frac{1}{(1+ c_0 t )^{2}} \limitdistrweight(dt),
$ 
and the constant $c_0$ is the same as in Lemma \ref{lemma:unique_sketching_iso}. 
Specifically, the bias and variance satisfy 
\$
%\lim_{\npinfty} 
\biascondition (\hat\beta^{1})
& \overset{{\rm a.s.}}{\rightarrow}
\begin{cases}
0,& \aspratio/\sampleratio < 1  \\
\signallev \frac{\aspratio/\sampleratio - 1 }{\aspratio/\sampleratio},& \aspratio/\sampleratio > 1
\end{cases},
%\lim_{\npinfty} 
\quad
\varcondition (\hat\beta^{1})
 \overset{\as}{\rightarrow}
\begin{cases}
\sigma^2\left(\frac{\aspratio}{1-\aspratio - f(\aspratio)} - 1 \right), & \aspratio/\sampleratio < 1  \\
\noiselev \frac{1}{\aspratio/ \sampleratio - 1}, & \aspratio/\sampleratio > 1
\end{cases}. 
\$
\end{theorem}

We first compare the limiting risk of the sketched min-norm  estimator $\hat\beta^1$  with that of the min-norm estimator  $\ridgeless$ under isotropic features.  The latter's risk is given by \cite{hastie2019surprises}:
\#\label{eq:risk_ridgeless}
\lim_{\npinfty} R_{X, \, \truesignal} (\ridgeless)
\ =
\begin{cases}
\noiselev \frac{\aspratio}{1-\aspratio},& \aspratio < 1  \\
\signallev \frac{\aspratio - 1}{\aspratio} + \noiselev \frac{1}{\aspratio - 1},& \aspratio > 1
\end{cases}\qas
\#
where $R_{X, \, \truesignal}$ is the same as $\riskcondition$ but without conditioning on $\cW$. Comparing these two risks, we observes  that, in the overparameterized regime,  sketching  modifies the risk by modifying the aspect ratio from $\gamma$ to  $\gamma/\theta$, and shifts the interpolation threshold from $\aspratio = 1$ to $\gamma/\theta = 1$. However, despite these modifications, the risk still explodes  as $\gamma/\theta$ approaches the interpolation threshold from the right side, i.e., as $\aspratio/\sampleratio \searrow 1$. This observation concurs with the findings by \cite{chen2023sketched}, who focused on orthogonal and \iid~sketching. %and implies that sketching serves as a dual of increasing the model size. 

In the underparameterized regime however, the limiting risk of $\hat\beta^1$ relies on the limiting distribution of the multipliers $\limitdistrweight$ and is, therefore, different for different multipliers. Our next result demonstrates that the  risk becomes unbounded as ${\gamma/\theta\nearrow 1}$, aka $\gamma/\theta$  approaches 1 from the left side. This, in conjunction with the earlier discussion, indicates that the risk of $\hat\beta^1$ explodes when  $\aspratio/\sampleratio = 1$ from either side.

\begin{corollary}\label{cor:isointerpolating}
For any probability measure $\limitdistrweight$ satisfying Assumption~\ref{Assume:multip} and sampling ratio $\aspratio < \sampleratio \leq 1$, the function $f(\aspratio)$ satisfies
$
\lim_{\aspratio/\sampleratio  \nearrow 1} f(\aspratio) =  1 - \aspratio. 
$
Consequently, the asymptotic risk of $\hat\beta^{1}$ explodes when $\aspratio/\sampleratio  \nearrow 1$, i.e., 
$
\lim_{\aspratio/\sampleratio \nearrow 1}\lim_{\npinfty} R_{X,\cW}(\hat\beta^1) = +\infty.
$
\end{corollary}

Since the risk of $\hat\beta^1$ in the underparameterized regime depends on the limiting distribution of the multipliers, we consider three different types of multipliers and compute the risks of the associated sketched estimators.
Specifically, we consider Jackknife multipliers, Bernoulli multipliers, and multinomial multipliers, where the multinomial multipliers  correspond to the classical bootstrap with replacement.  We will refer to the corresponding sketching methods as Jackknife sketching, Bernoulli sketching, and multinomial sketching, respectively.

\begin{figure}[t!]
\centering
\subfigure{
    \includegraphics[width=.45\linewidth]{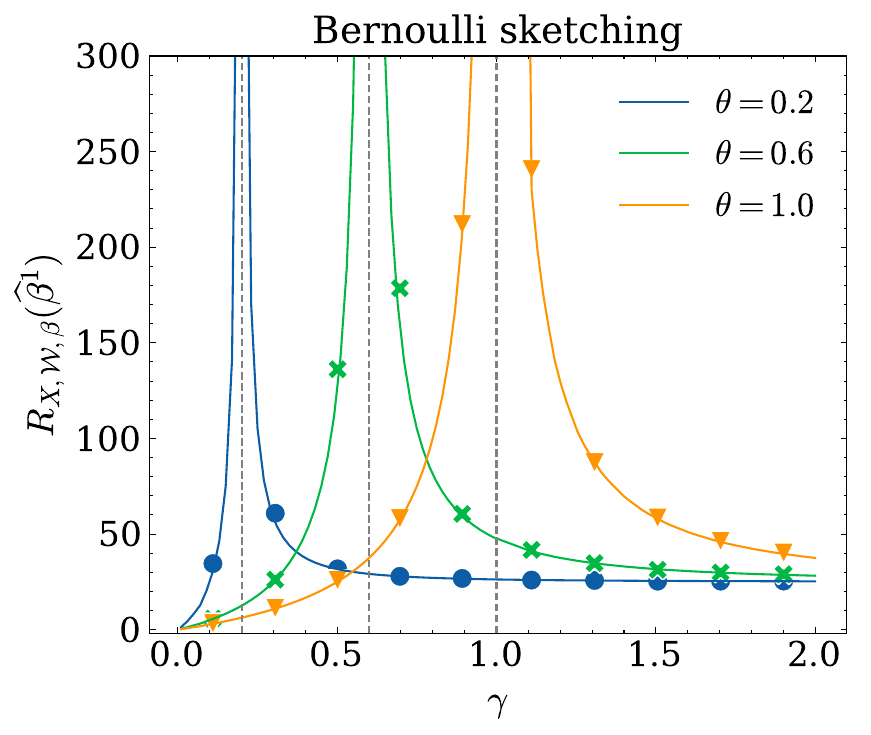}
}
% \hspace{-2em}
\subfigure{
    \includegraphics[width=.45\linewidth]{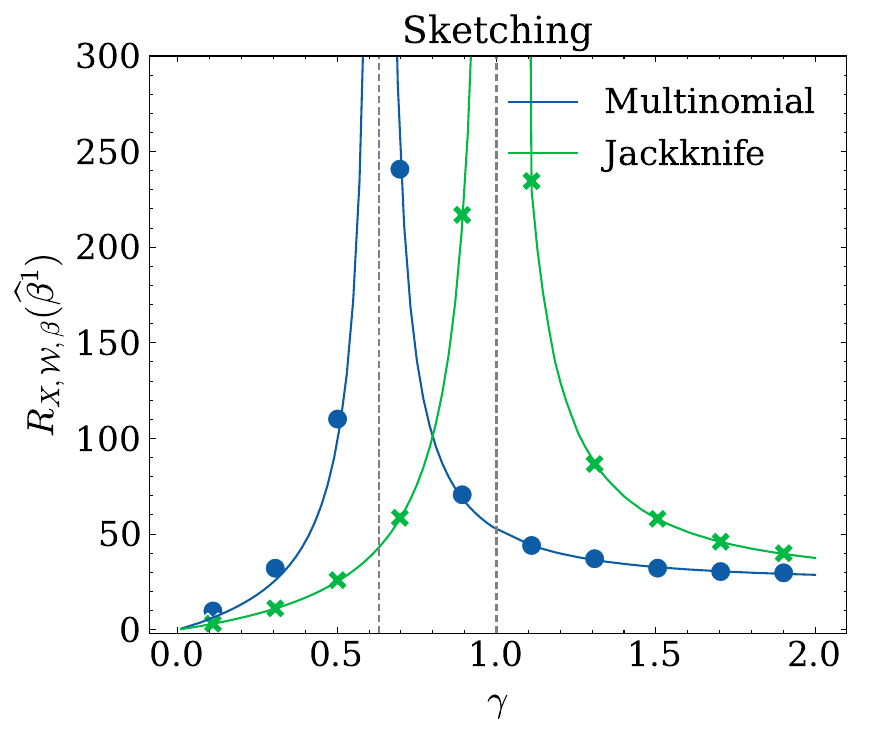}
}
    \caption{\small 
    Limiting risk curves for Bernoulli sketched estimators  (left panel) and multinomial and Jackknife sketched estimators (right panel) with $(r,\sigma)=(5,5)$. The features, errors, and $\beta$ are generated in the same way as in Figure \ref{fig:fig_1}.  {Left panel:} The blue, green, and yellow lines are theoretical risk curves for Bernoulli sketched estimators with $\theta= 0.2,\, 0.6,\, 1.0,$ respectively. 
    The blue dots, green crosses, and orange  triangles mark the corresponding finite-sample risks  with  $n=400$, $\gamma$ varying in $[0.1,10]$, and $d = [n \gamma]$.   {Right panel:} The blue and  green lines are theoretical risk curves for multinomial and Jackknife sketched estimators respectively. The blue dots and green crosses mark the corresponding finite-sample risks with the same setup as in the left panel. 
    }
     \label{fig:fig_2}
\end{figure}

%\scomment{ What is the limiting distribution $\limitdistweights$? Can we get a closed form for that? }

\begin{corollary}\label{cor:isoexample}
   Assume  Assumptions~\ref{Assume:highdim}-\ref{Assume:beta4+mom}, $\covmat = I$, and $\aspratio < \sampleratio \leq 1$.  Then the following holds.  
\begin{enumerate}[leftmargin=35pt]
    \item[(i)] The full-sample min-norm estimator $\ridgeless$ satisfies
    \begin{equation*}
    \lim_{\npinfty} \riskcondition (\ridgeless) = \noiselev \frac{\aspratio}{1-\aspratio}.
    \end{equation*}
        
    \item[(ii)] The Bernoulli sketched estimator $\hat\beta^1_{\Bern}$  with $\weight_{1,j} \overset{\iid}{\sim} \text{Bernoulli}(\sampleratio)$  satisfies
        \begin{equation*}
            \lim_{\npinfty} \riskcondition (\hat\beta^1_{\Bern}) = \noiselev \frac{\aspratio/ \sampleratio}{1-\aspratio/ \sampleratio}.
        \end{equation*}
        
    \item[(iii)] The multinomial sketched estimator   $\hat\beta^1_{\multi}$ with  
    \$
    (\weight_{1,1}, \ldots, w_{1,n}) \sim {\rm Multinomial}(n; 1/\ndata,\ldots, 1/\ndata)
    \$ 
    corresponds to the classical bootstrap with  replacement, and satisfies 
        \begin{equation*}
            \lim_{\npinfty} \riskcondition (\hat\beta^1_{\multi}) > \noiselev \frac{\aspratio/ (1 - 1/e)}{1-\aspratio/ (1 - 1/e)}.
        \end{equation*}
    \item[(iv)]    The Jackknife sketched estimator shares the same limiting risk in item (i).   
 \end{enumerate}
\end{corollary}

The corollary above confirms that taking different sketching matrices yields different limiting risks in the underparameterized regime,  while they agree in the overparameterized regime.  To provide a visual representation, Figure \ref{fig:fig_2} depicts the limiting risk curves as functions of $\gamma$, as well as the finite-sample risks for Bernoulli, multinomial, and Jackknife sketched estimators with $(r, \sigma) = (5, 5)$. The symbols (dots, crosses, triangles) indicate the finite-sample risks for $n=400$, $\gamma$ varying in $[0.1,10]$, and  $d = [n \gamma]$, whose values are averaged over 100 repetitions\footnote{In all following figures, we use the same setup for finite sample risks and omit these details.}. Notably, the downsampling ratios for multinomial and Jackknife sketching remain fixed at $1-1/e$ and $1$ respectively. However, in the case of Bernoulli sketching, the downsampling ratio $\sampleratio$, serves as a tuning parameter. This offers more flexibility. 

%Different downsampling ratios induce different effective aspect ratios and thus different interpolation thresholds. 
%On the other hand, the downsampling ratios for multinomial and Jackknife sketching remain fixed at $1-1/e$ and $1$, respectively, and can not be tuned.

\begin{figure}[t!]
\centering
\subfigure{
    \includegraphics[width=.45\linewidth]{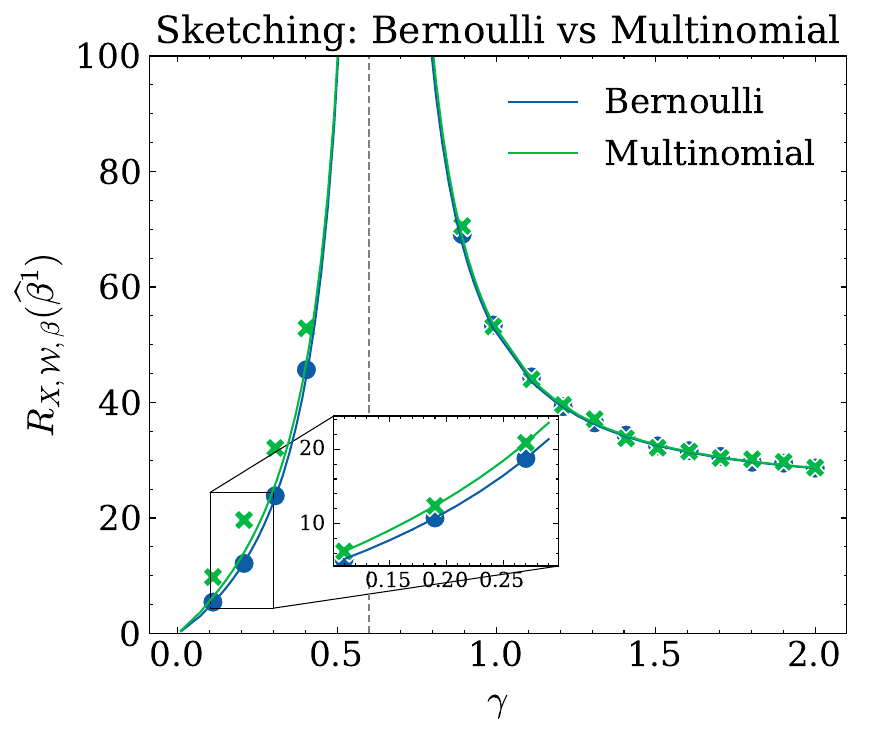}
}
% \hspace{-2em}
\subfigure{
    \includegraphics[width=.45\linewidth]{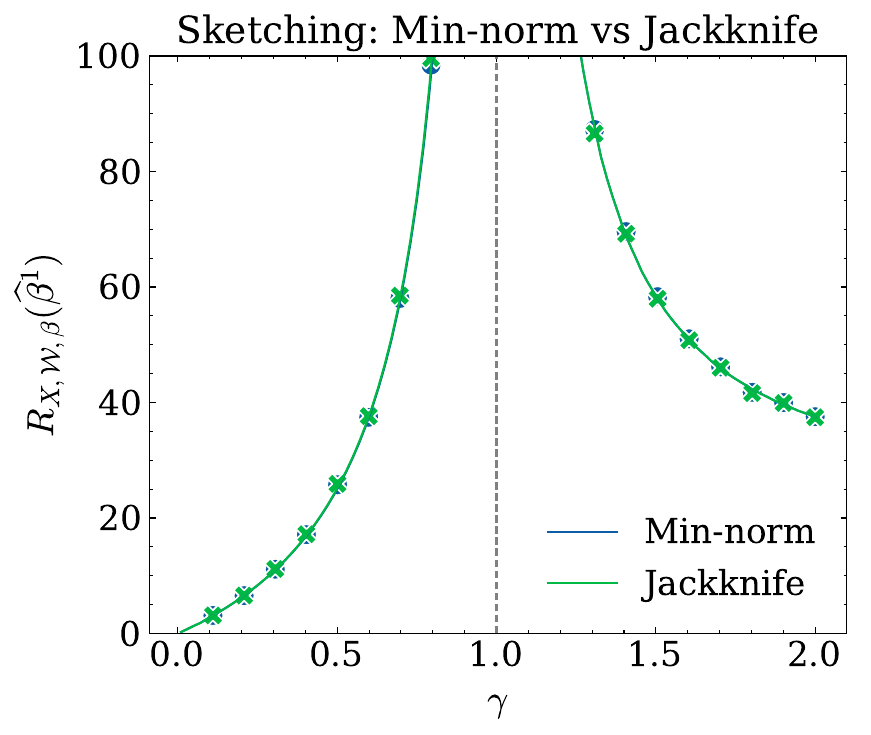}
}
\caption{\small 
Limiting risk curves for the full-sample min-norm estimator, and  Bernoulli,  multinomial, and  Jackknife sketched estimators with $(r, \sigma)=(5,5)$.    The features, errors, and $\beta$ are generated in the same way as in Figure \ref{fig:fig_1}.  {Left panel:} The blue and green lines  are theoretical risk curves for the Bernoulli sketched estimator with $\sampleratio = 1-1/e$ and the multinomial sketched estimator respectively. 
{Right panel:}  The blue and green lines  are theoretical risk curves for the full-sample min-norm and Jackknife sketched estimators respectively. 
In both panels, symbols mark the corresponding finite-sample risks in the same way as in Figure \ref{fig:fig_2}. 
}
\label{fig:fig_3}
\end{figure}

%The blue dots and green crosses mark the finite-sample risks with $n=400$, $\gamma$ varying in $[0.1,10]$, and $d = [n \gamma]$. 

%The blue dots and green crosses mark the corresponding  finite-sample risks with the same setup as in the left panel. 

Figure \ref{fig:fig_3} provides a comparison between the Bernoulli and multinomial sketched estimators as well as a comparison between the full-sample min-norm estimator and the Jackknife estimator.  These comparisons yield two observations. First, the Jackknife sketched estimator and the full-sample min-norm estimator exhibit identical limiting risks, aligning with the aforementioned corollary. Second and perhaps surprisingly, multinomial sketching results in a slightly worse limiting risk than the Bernoulli sketched estimator with $\sampleratio= 1 - 1/e$ in the underparameterized regime, while the two risks agree in the overparameterized regime. This aligns with Corollary \ref{cor:isoexample}: In the underparameterized regime, multinomial sketching leads to an increased limiting variance when compared with Bernoulli sketching. 
This naturally raises the following question:
\begin{quote}
\it What is the optimal sketching matrix among all   sketching matrices in the form of \eqref{eq:sketching}?
\end{quote}
We answer the question above by  leveraging  the variance formula in Theorem \ref{thm:isosketching}. Specifically, the following result establishes the optimality of the Bernoulli sketching, as it minimizes  the variance formula and thus the limiting risks, among all other sketching techniques considered.

\begin{corollary}\label{coro:optimal_sketching}
Taking Bernoulli multipliers with a downsampling ratio of $\sampleratio$, corresponding to the limiting probability measure $\limitdistrweight = (1 - \sampleratio)\delta_{0} + \sampleratio\delta_{1}$, minimizes the limiting risk of $\hat\beta^1$ in Theorem \ref{thm:isosketching}  among all choices of multipliers satisfying Assumption~\ref{Assume:multip} with $B=1$ and downsampling ratio $\sampleratio$. %and the fixed downsampling rate $\sampleratio$.
\end{corollary}

\begin{table}[]
\centering
\caption{Run time in seconds  for  computing the Bernoulli, orthogonal, and  multinomial sketched estimators  with $\theta=1-1/e\approx 0.632,\, \gamma=1.2$, repeated for 500 rounds. }
\label{tab:1}
\begin{tabular}{@{}cccc@{}}
\toprule
Sketching         & Bernoulli & Orthogonal & Multinomial \\ 
\midrule
$n=400$ & 7.10      & 8.92       & 7.43         \\
$n=600$ & 15.70     & 19.80      & 16.00       \\
$n=800$ & 27.50     & 37.60      & 29.70      \\
\bottomrule
\end{tabular}
\end{table}

\iffalse
The above result also holds for correlated features with a general covariance matrix \Sigma$. This is because the risk for the correlated case is the same as the risk for the isotorpic case  in the underparameterized case while the risks for both cases in the overparameterized regime are independent of the sampling distribution \limitdistrweight$; see Theorem \ref{thm:corsketching} in Section \ref{sec:correlated}. 
\cite{chen2023sketched} first showed that using the orthogonal sketching matrix is optimal among all general sketching matrices that are of full rank. 
Our result complements the results by \cite{chen2023sketched} where the sketching matrix is in the form of \eqref{eq:sketching} and is mostly  singular (as long as $\sampleratio<1$). 
Comparing with the results by \cite{chen2023sketched}, we find that Bernoulli  and orthogonal sketching methods yield exactly the same limiting risk and thus are identical to each other in terms of generalization performance. Computationally, however, Bernoulli sketching is  more  efficient than orthogonal sketching as  the former avoids matrix multiplication to obtain the sketched data. 
\fi

The above corollary holds  even for the case of correlated features. This universality stems from the fact that, in the underparameterized regime, the risks for one sketched estimator in the correlated and isotropic cases are identical, while the risks for both cases in the overparameterized regime are independent of the multiplier distribution $\limitdistrweight$; see Theorem \ref{thm:corsketching} in Section \ref{sec:correlated}. 
Recently, \cite{chen2023sketched} established the optimality of using an orthogonal sketching matrix among  sketching matrices that are of full rank. Our findings complement theirs by considering sketching matrices in the form of \eqref{eq:sketching}, which are predominantly singular.

Upon comparing our results to those by \cite{chen2023sketched}, it becomes clear that both the Bernoulli and orthogonal sketching yield identical limiting risks, thereby sharing the same generalization performance. However, computationally, the Bernoulli sketching is  more efficient than its orthogonal counterpart, primarily because the former avoids the necessity for multiplying the data matrix by a dense sketching matrix. Table \ref{tab:1} compares the run time for computing the Bernoulli,  orthogonal, and multinomial sketched estimators using  an Apple M1 CPU with data generated in the same way as in Figure \ref{fig:fig_1}, which shows that Bernoulli sketching is faster than all other sketching methods, especially when the sample size is large. When increasing the sample size from $n=400$ to $n=800$, the computational efficiency gain of the Bernoulli sketching over the othorgonal sketching improves from 26\% to 36\%.

\subsection{Bagging stabilizes the generalization performance}

This subsection studies the exact risk of the bagged min-norm least square  estimator $\estimator$ introduced in Section \ref{sec:pre}. Recall that our bagged estimator can be formulated as the average of the sketched min-norm least square estimators.

\begin{theorem}[Bagging under isotropic features]\label{thm:isoboostrap}
Assume Assumptions~\ref{Assume:highdim}-\ref{Assume:beta4+mom} and $\covmat = I$. Then the out-of-sample prediction risk of $\estimator$ satisfies 
\$
\lim_{\nresample \rightarrow + \infty} \lim_{\npinfty} \riskcondition(\estimator)
= \begin{cases}
                \noiselev \frac{\aspratio}{1-\aspratio},& \aspratio < \sampleratio  \\
                \signallev \frac{\left(\aspratio - \sampleratio \right)^2}{\aspratio \left(\aspratio - \sampleratio^{2} \right)} + \noiselev  \frac{\sampleratio^2}{\aspratio - \sampleratio^{2}},& \aspratio > \sampleratio
             \end{cases} \qas 
\$
Specifically, the bias and variance satisfy 
\$
%\lim_{\npinfty} 
\biascondition (\estimator)
& \overset{{\rm a.s.}}{\rightarrow}
\begin{cases}
0,& \aspratio/\sampleratio < 1  \\
\signallev \frac{\left(\aspratio - \sampleratio \right)^2}{\aspratio \left(\aspratio - \sampleratio^{2} \right)},& \aspratio/\sampleratio > 1
\end{cases}, 
%\lim_{\npinfty} 
\quad 
\varcondition (\estimator)
 \overset{\as}{\rightarrow} 
\begin{cases}
\sigma^2 \frac{\aspratio}{1-\aspratio}, & \aspratio/\sampleratio < 1  \\
\noiselev \frac{\sampleratio^2}{\aspratio -  \sampleratio^2}, & \aspratio/\sampleratio > 1
\end{cases}. 
\$
\end{theorem}

\iffalse
\scolor{
\begin{proof}[Proof of Theorem \ref{thm:isoboostrap}]
This result is a specific case derived from the more general Theorem~\ref{thm:corsketching}. Here, we provide a simpler proof compared to Theorem~\ref{thm:corsketching} by establishing the  limit of the crossed Stieltjes transform $\frac{1}{\ndim}\tr((X\transp \sketchmat_\ca \transp \sketchmat_\ca X/\ndata - zI)\inv(X\transp \sketchmat_\cb \transp \sketchmat_\cb X - zI)\inv)$, as demonstrated in Lemma~\ref{lm:isotrolimitresolvent}.  
Here, $\sketchmat_\ca$ and $\sketchmat_\cb$ represent two sketching matrices corresponding to different subsamples $\cW_\ca$ and $\cW_\cb$ with $\ca \neq \cb$, respectively.

\scolor{
In accordance with the bias-variance decomposition Lemma~\ref{lm:biasvar}, these limiting risks correspond to the limiting cross terms $\frac{1}{\ndim} \tr\left(\projmat_{i} \covmat \projmat_{j} \right)$ and $\frac{1}{n^{2}} \tr\left( \hcovmat_{i}\pinv X\transp \sketchmat_{i}^2 \sketchmat_{j}^2 X \hcovmat_{j}\pinv \covmat \right)$ with $i \neq j$. Since the risk of the sketched ridgeless least square estimator is bounded, \scomment{it is not bounded?} Lemma~\ref{lm:biasvar} implies that the limiting risks converge to the limiting cross terms at the order of $O(1/B)$. 
}

\end{proof}
}
\fi

%%%%%%%%%%%%%%%%%%%%%%%%%%%%%%%%%%%%%%%%
%%%%%%%%%%bagging vs sketching%%%%%%%%%%
%%%%%%%%%%%%%%%%%%%%%%%%%%%%%%%%%%%%%%%%

The above theorem characterizes the limiting risk of the bagged min-norm least square estimator. We have  two  observations. First, comparing with the sketched min-norm least square estimator, one immediate advantage is that bagging makes the risk invariant to the limiting distributions of the multipliers.  Second, in contrast to the full-sample and sketched min-norm estimator whose limiting risks diverge to infinity as $\gamma$ and $\gamma/\theta$ approaches  to $1$  and $1$ respectively, the limiting risk of the bagged estimator approaches to
$
\sigma^2 \sampleratio/(1-\sampleratio)
$
from both sides and remains bounded by $\max\{r^2, \sigma^2 \sampleratio/(1-\sampleratio) \}$  in both regimes,  improving the stability and thus the generalization performance. The following corollary proves this rigorously. %establishes an upper bound for the limiting risks of bagged estimators. 

\begin{corollary}\label{coro:risk_bound_iso}
Suppose  $\aspratio \ne \sampleratio$.  The limiting risk of the bagged estimator $\estimator$ in Theorem \ref{thm:isoboostrap} satisfies 
\$
\lim_{\nresample \rightarrow + \infty} \lim_{\npinfty} \riskcondition(\estimator)
\leq \noiselev \frac{\sampleratio}{1 - \sampleratio} \vee \signallev. 
\$
\iffalse
\scolor{
Moreover, we have
\$
\lim_{\nresample \rightarrow + \infty} \lim_{\npinfty} \riskcondition(\estimator)
\leq \min \left\{\lim_{\nresample \rightarrow + \infty} \lim_{\npinfty} \riskcondition(\ridgeless), \lim_{\nresample \rightarrow + \infty} \lim_{\npinfty} \riskcondition(\hat\beta^1) \right\}.
\$
}
\fi
\end{corollary}

To comprehend how bagging enhances the stability over an individual estimator, we begin by comparing the bagged estimator with the sketched estimator. We focus on the Bernoulli sketched estimator due to its optimality.  Bagging serves to substantially reduce the limiting variance of the Bernoulli sketched estimator across both regimes $\gamma < \theta$ and $\gamma > \theta$. Specifically, the limiting variance is consistently reduced by a factor of at least $\theta$:
\$
\frac{ \varcondition(\estimator)}{ \varcondition(\hat \beta^1)}
&\overset{\as}{\rightarrow} 
\frac{ \gamma/(1-\gamma)}{{(\gamma/\theta)}\big/(1 - \gamma/\theta)} 1(\gamma < \theta) + \frac{1/(\gamma/\theta^2 - 1)}{{1}/(\gamma/\theta  -1)} 1(\gamma > \theta)
\leq \theta. 
\$
%indicating that the variance continues to decrease as the subsampling ratio $\sampleratio$ decreases.

Furthermore, as $\aspratio \rightarrow \sampleratio >1$ from either side, the variance reduction becomes even more pronounced,  resulting in an order of difference:
\$
\lim_{\aspratio \rightarrow \theta}\lim_{\npinfty}\frac{ \varcondition(\estimator)}{ \varcondition(\hat \beta^1)} 
\overset{\as}{=} 
\lim_{\aspratio \rightarrow \theta} \frac{ \gamma/(1-\gamma)}{{(\gamma/\theta)}\big/(1 - \gamma/\theta)} 1(\gamma < \theta) + \frac{1/(\gamma/\theta^2 - 1)}{{1}/(\gamma/\theta  -1)} 1(\gamma > \theta)
= 0. 
\$
Surprisingly, at least to us,   bagging not only reduces variance but also mitigates the implicit bias in the overparameterized regime:
\$
\lim_{\nresample \rightarrow + \infty} \lim_{\npinfty} \biascondition(\estimator)
&\overset{\as}{=}  \signallev \cdot \frac{\gamma/\theta - 1}{\gamma/\theta } \cdot \frac{\gamma/\theta - 1}{\gamma/\theta - \theta}
\leq \signallev \cdot \frac{\gamma/\theta - 1 }{\gamma/\theta}
\overset{\as}{=} 
\lim_{\npinfty} \biascondition (\hat \beta^1). 
\$

%\paragraph{Comparing with full-sample least square estimator in the underparameterized regime} 

%In the overparameterized regime, the variance term  corresponds to the variance of the full-sample  ridgeless least square estimator with an aspect ratio of $\aspratio / \sampleratio^2$. Moreover, the variance is bounded by $\noiselev \aspratio/(1-\aspratio)$, except in the case when $\aspratio = \sampleratio$, effectively mitigating the peak observed in the ridgeless scenario. This indicates that bagging in ridgeless regression provides robustness. Furthermore, the over-parameterized bias is smaller compared to that of sketched ridgeless least square estimators, which will be discussed in greater detail in Section~\ref{sec:corbagging}.

\begin{figure}[t]
\centering
\subfigure{
    \includegraphics[width=.45\linewidth]{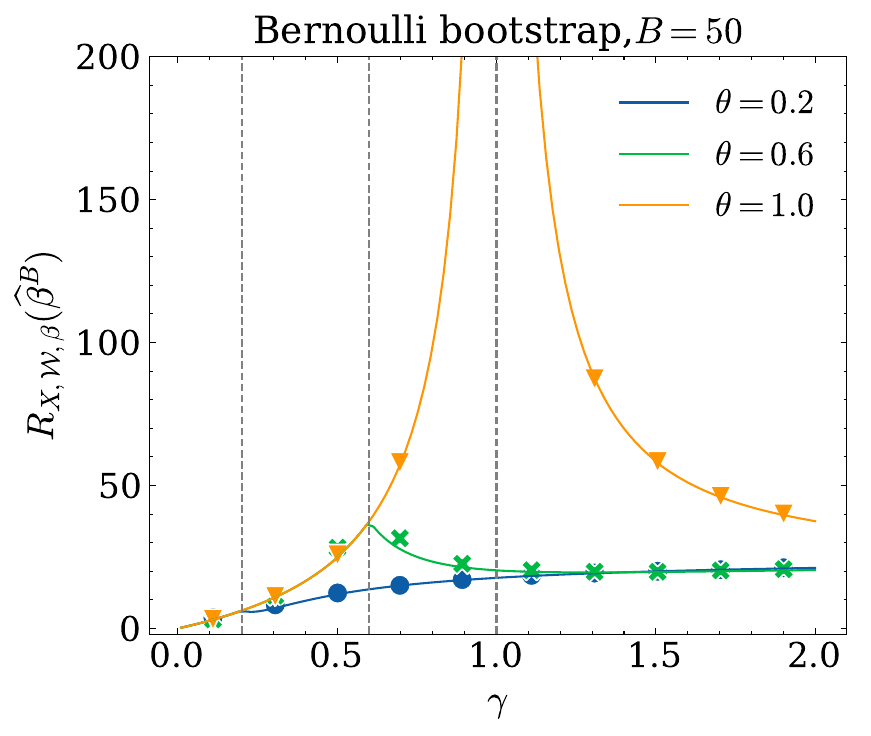}
}
% \hspace{-2em}
\subfigure{
    \includegraphics[width=.45\linewidth]{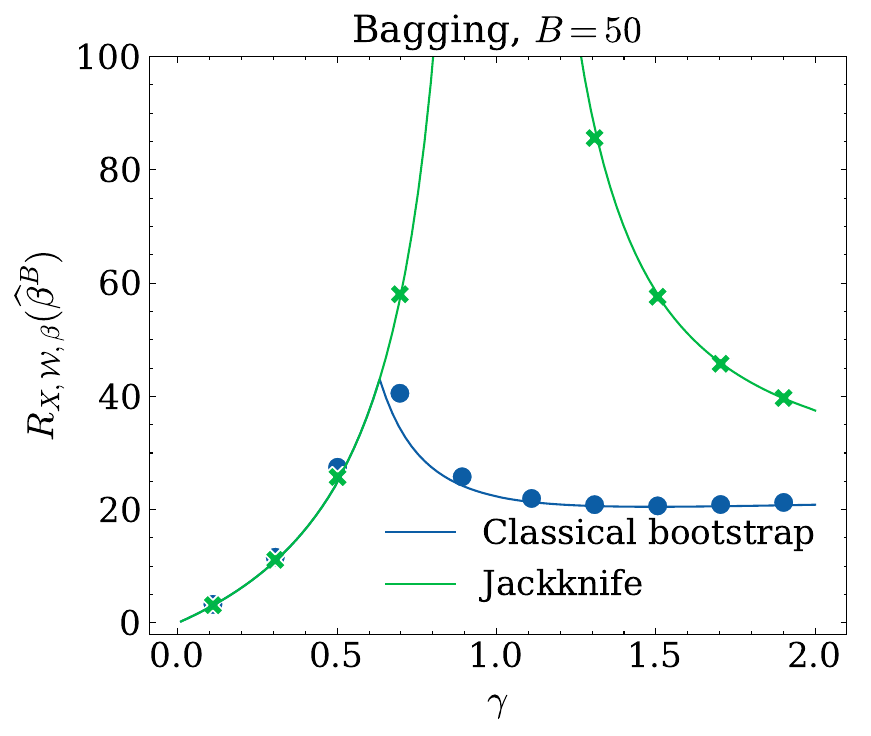}
}
    %\vspace{-1em}
    \caption{\small 
    Limiting risk curves for the Jackknife  estimator, and  Bernoulli and  classical bootstrap estimators with $(r, \sigma)=(5,5)$.    The features, errors, and $\beta$ are generated in the same way as in Figure \ref{fig:fig_1} and the number of bootstrap rounds is $B=50$.  {Left panel:} The blue, green, and orange lines  are theoretical risk curves for the Bernoulli bootstrap estimators with $\sampleratio = 0.2, 0.6, 1.0$ respectively. % where the  Bernoulli bootstrap estimator with $\sampleratio =  1.0$ is identical to the full-sample min-norm estimator. 
    {Right panel:}  The blue and green lines  are theoretical risk curves for the classical bootstrapped and Jackknife estimators respectively. 
In both panels, symbols mark the corresponding finite-sample risks in the same way as in Figure \ref{fig:fig_2}. 
}\label{fig:fig_4}
\end{figure}

%and thus the variance is reduced. %by at least $\sampleratio/\aspratio 1() +  \sampleratio^2 1()$ times. 

%\paragraph{The underparameterized regime} 
We then compare the limiting risk of the bagged  estimator with that of the full-sample min-norm estimator.  The limiting risks of these two estimators  are  identical  when $\aspratio < \sampleratio \leq 1$. %and do not depend on the sampling ratio $\sampleratio$ or the limiting distribution of the multipliers $\limitdistrweight$. 
However, when $\aspratio > \sampleratio$,  we have:
\$
\frac{\varcondition(\estimator)}{\varcondition (\ridgeless)}
\overset{\as}{\rightarrow}
\begin{cases}
    \frac{\noiselev /{(\aspratio/\sampleratio^2 - 1)}}{\noiselev \aspratio/(1-\aspratio)} \leq \frac{\sampleratio}{\aspratio} < 1 , & 1> \aspratio > \sampleratio \\
    \frac{\noiselev/{(\aspratio/\sampleratio^2 - 1)}}{\noiselev/(\aspratio - 1)} =  \frac{\aspratio - 1}{\aspratio/\sampleratio^2 - 1} \leq \sampleratio^2, & \aspratio >  1 \geq \sampleratio
\end{cases}. 
\$
This indicates that the variance is reduced  by at least a factor of $\sampleratio/\aspratio$ when $1>\aspratio>\sampleratio$, and by at least a factor of $\sampleratio^2$ when $\aspratio>1\geq \sampleratio$. 
Moreover, when $\aspratio\rightarrow 1$ from either side, we have 
\$
\lim_{\nresample \rightarrow + \infty} \lim_{\npinfty}
\frac{\varcondition(\estimator)}{\varcondition (\ridgeless)} 
\rightarrow 
\begin{cases}
    0, &  \sampleratio<\aspratio \nearrow 1 \\
    0, & \aspratio \searrow  1 > \sampleratio
\end{cases}.
\$
In other words, as $\aspratio$ approaches $1$, the variance of the bagged estimator $\estimator$ becomes of a smaller order compared with that of the full-sample min-norm estimator.

Figure \ref{fig:fig_4} plots the limiting risk curves as functions of $\gamma$ and the finite-sample risks for the Jackknife estimator, Bernoulli bootstrap estimators\footnote{With a slight abuse of notation, we shall refer to  Bernoulli (classical) bootstrap based bagged estimators as Bernoulli (classical) bootstrap estimators for simplicity.} with different downsampling ratios, and the classical bootstrap estimator, all with $(r, \sigma) = (5, 5)$. The symbols mark the finite-sample risks. It is evident that the limiting risks of Bernoulli bootstrap estimators with $\sampleratio = 0.2$ and $0.6$, as well as the classical bootstrap estimator, remain bounded. On the other hand, the Bernoulli bootstrap estimator with $\sampleratio = 1.0$ and the Jackknife estimator are identical to the full-sample min-norm estimator, whose limiting risk becomes unbounded at the interpolation threshold.

Figure \ref{fig:fig_5} provides a comparison of the limiting variances and risks between the bagged and full-sample min-norm estimators, both with $(r, \sigma)=(5,5)$. %The bagged estimator has $\sampleratio$ values of $0.2$ and $0.6$. 
When  $\sampleratio$ is taken as either $0.2$ or $0.6$, both the limiting variance and risk for the bagged estimator are of a smaller order than those for the full-sample min-norm estimator when $\aspratio \rightarrow 1$. These experimental results effectively validate our theoretical findings.

\begin{figure}[t!]
\centering

\subfigure{
\includegraphics[width=.45\textwidth]{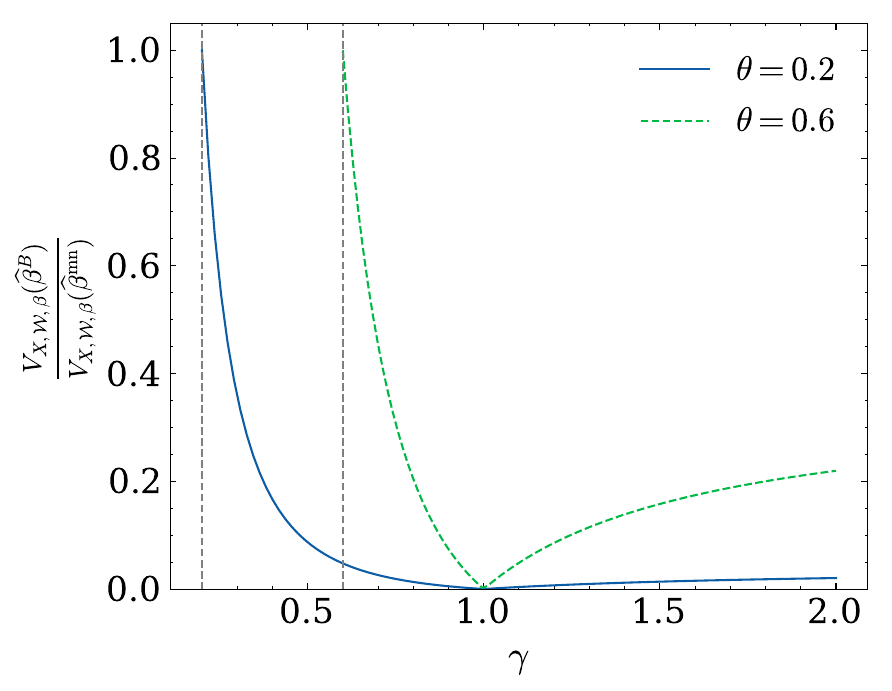}
}
\subfigure{
\includegraphics[width=.45\textwidth]{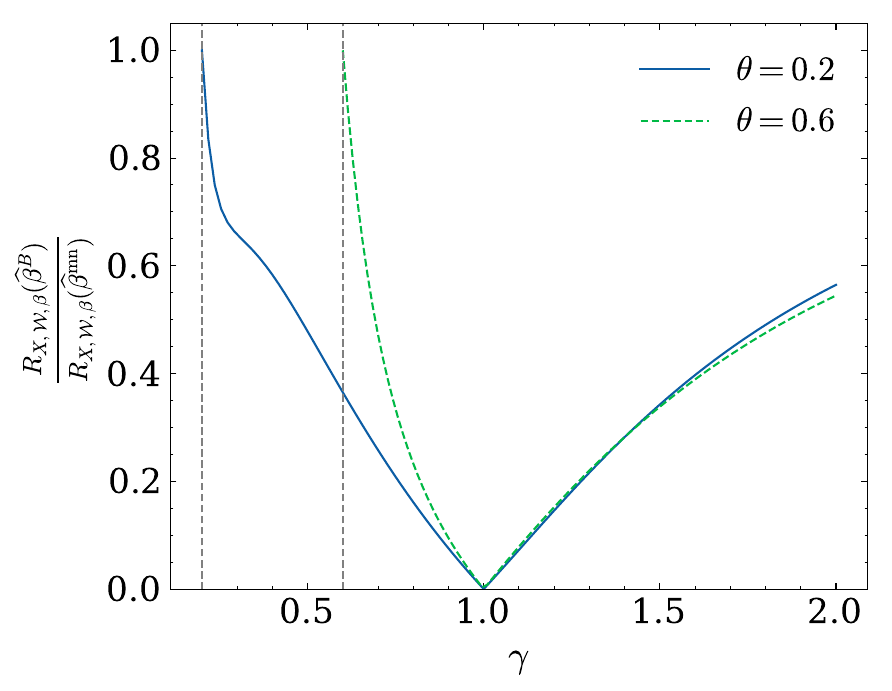}
}
\caption{The limiting  ratio curves between the bagged estimator and the full-sample min-norm estimator with $(r,\sigma)=(5,5)$. %{Left Panel:} The blue and green lines  are the  bias ratio curves for $\theta=0.2,0.6$ respectively.  
    {Left Panel:} The blue solid and green dashed lines  are the  variance ratio curves for $\theta=0.2,0.6$ respectively.  {Right Panel:} The blue solid and green dashed lines are the risk ratio curves for $\theta=0.2,0.6$ respectively.
    }
    \label{fig:fig_5}
\end{figure}

Finally,  Table \ref{tab:2} compares the run time for computing the Bernoulli, orthogonal, and classical  bootstrap\footnote{We use orthogonal bootstrap to refer to the procedure of generating subsamples in the form of $(SX, SY)$, where $S$ is an orthogonal sketching matrix.} estimators using  an Apple M1 CPU with data generated in the same way as in Figure \ref{fig:fig_1}, which shows that Bernoulli bootstrap is faster than all other bootstrap methods. 
When increasing the sample size from 400 to 800, the computational efficiency gains of the Bernoulli bootstrap over the orthogonal and classical bootstrap improves from {22\%} and 10\% to  47\% and 24\%, respectively.

% \begin{table}[]
% \centering
% \caption{Run time in seconds  for  computing the Bernoulli, orthogonal, classical, and Jackknife bootstrap estimators  with $\sampleratio = 1 - 1/e,\, \gamma=1.2, B=10$, repeated for 500 rounds. }
% \label{tab:2}
% \begin{tabular}{cccc}
% \toprule
% Boostrap &  Run time($n=400$) &Run time($n=600$)&Run time($n=800$) \\ \hline
% Bernoulli    & 51.10&95.90 &200.80                      \\
% Orthogonal                  & 59.40 &183.00& 323.00             \\
% Classical               & 60.01&150.70&371.6\\
% % Jackknife                 & 74.20\\
% \bottomrule
% \end{tabular}
% \end{table}

% \begin{table}[]
% \centering
% \caption{Run time in seconds  for  computing the Bernoulli, orthogonal, and classical bootstrap estimators  with $\sampleratio = 1 - 1/e,\, \gamma=1.2, B=10$, repeated for 500 rounds. }
% \label{tab:2}
% \begin{tabular}{@{}cccc@{}}
% \toprule
% Boostrap          & Bernoulli & Orthogonal & Classical \\ \midrule
% Run time($n=400$) & 45.90     & 63.90      & 49.80     \\
% Run time($n=600$) & 92.80     & 177.70     & 96.00    \\
% Run time($n=800$) & 200.80    & 323.40     & 214.40     \\ \bottomrule
% \end{tabular}
% \end{table}

\begin{table}[]
\centering
\caption{Run time in seconds  for  computing the Bernoulli, orthogonal, and classical bootstrap estimators  with $\sampleratio = 1 - 1/e,\, \gamma=1.2, B=10$, repeated for 500 rounds. }
\label{tab:2}
\begin{tabular}{@{}cccc@{}}
\toprule
{Bootstrap}     & Bernoulli & Orthogonal & Classical \\ \midrule
$n=400$ & 38.30     & 46.70      & 42.00     \\
$n=600$ & 93.20     & 125.10     & 103.80    \\
$n=800$ & 191.70    & 280.90     & 237.10     \\ \bottomrule
\end{tabular}
\end{table}

%%%%%%%%%%%%%%%%%%%%%%%%%%%%%%%%%%%%%%%%%%%%%%%%%
%%%%%%%%%%%%%%%%Ridge regression%%%%%%%%%%%%%%%%%
%%%%%%%%%%%%%%%%%%%%%%%%%%%%%%%%%%%%%%%%%%%%%%%%%

\subsection{Bagging as implicit regularization}

This subsection establishes an equivalence between the bagged ridgeless least square estimator and the ridge regression estimator. 

\begin{lemma}\label{lm:ridgeequivlence}
Assume Assumptions~\ref{Assume:highdim}-\ref{Assume:beta4+mom} and $\covmat = I$. Then 
\begin{equation*}
    \lim_{\nresample \rightarrow + \infty} \lim_{\npinfty} \EE \left[ \| \estimator - \hat{\beta}_{\lambda}\|_{2}^{2} \condition \variablecondition \right] = 0 \qas
\end{equation*}
where $\lambda = (1 - \sampleratio)(\aspratio/\sampleratio - 1) \vee 0.$ %\scomment{To S: check proof!}
\end{lemma}

\begin{figure}[t!]
\centering
\subfigure{
    \includegraphics[width=.45\linewidth]{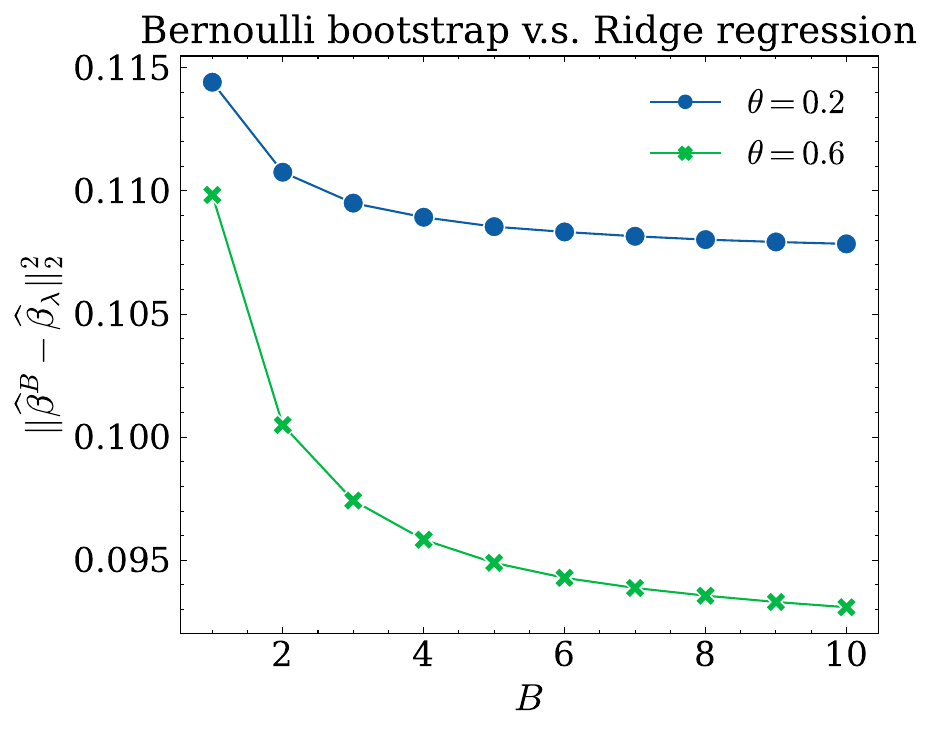}
}
% \hspace{-2em}
\subfigure{
    \includegraphics[width=.43\linewidth]{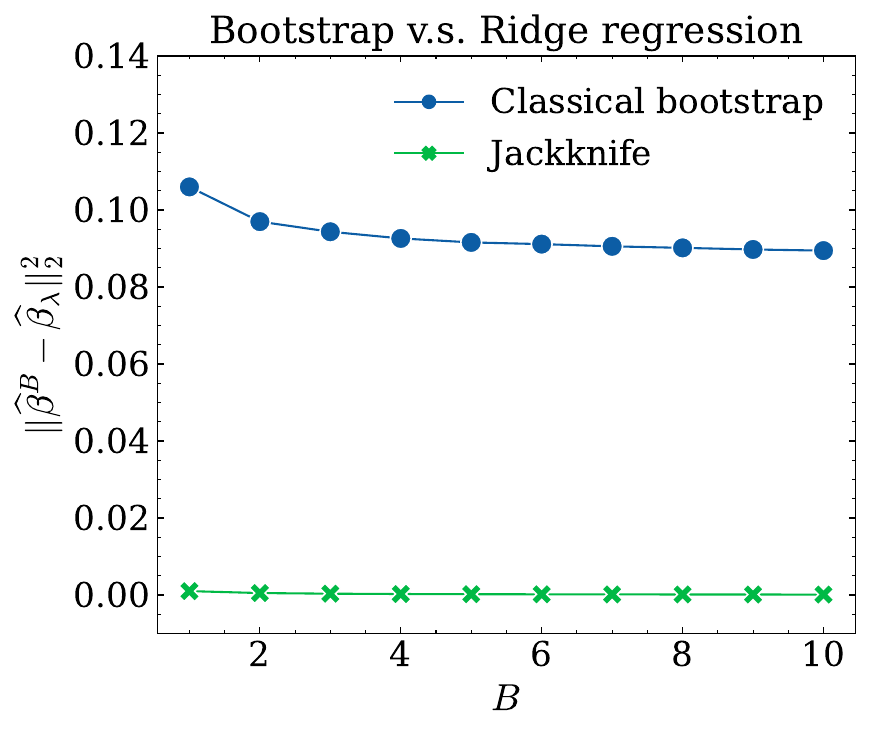}
}
\subfigure{
    \includegraphics[width=.43\linewidth]{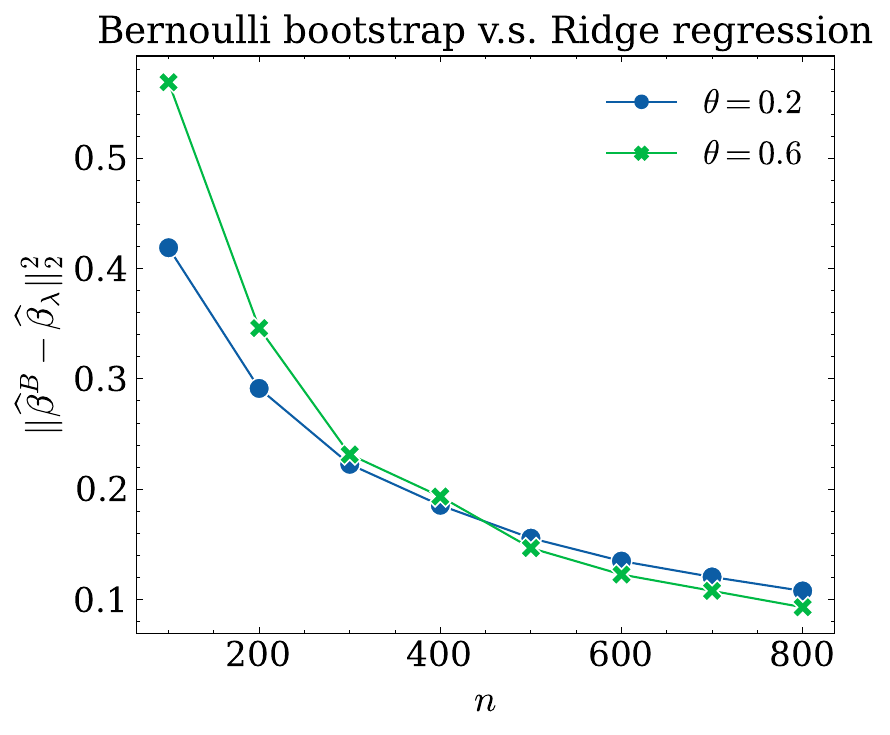}
}
\subfigure{
    \includegraphics[width=.43\linewidth]{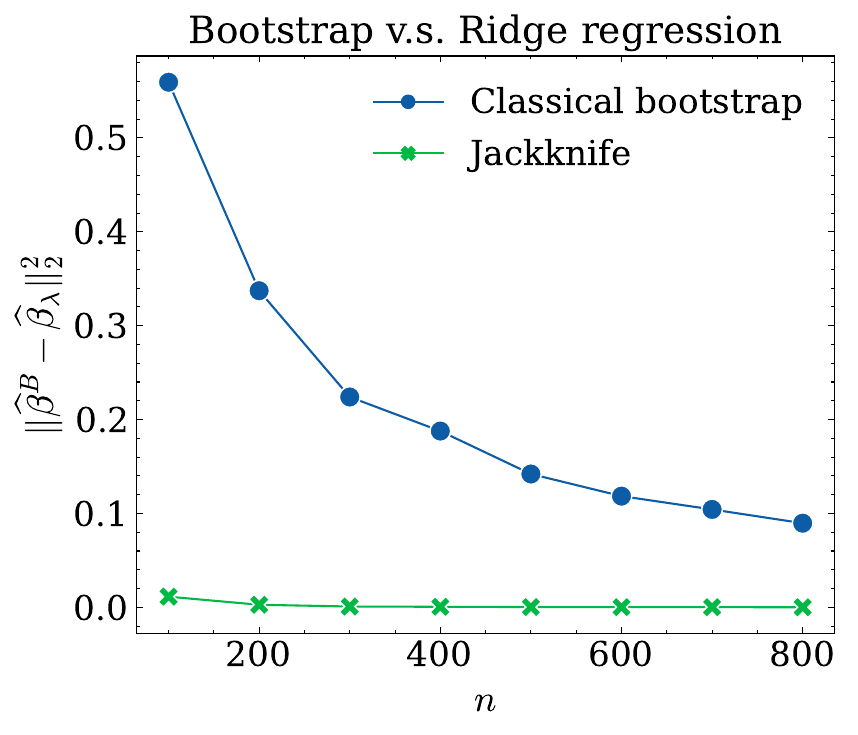}
}
    \caption{\small 
    The $\ell_2$ normed differences  between the bagged and the ridge  estimators with $(r,\sigma)=(5,5)$, $d = [n \gamma]$, and $\gamma = 1.2$ under isotropic features.  The features, errors, and $\beta$ are generated in the same way as in Figure \ref{fig:fig_1}.  {Upper left panel:} The blue and green lines are the finite-sample $\ell_2$ normed difference curves between the Bernoulli bootstrap estimator and its equivalent ridge regression estimator with $\sampleratio=0.2, 0.6$ respectively. Here $n=800$ and $B$ varies in $\{1,2,...,10\}$.   {Upper right panel:} The blue and green lines are finite-sample $\ell_2$ normed difference curves between the classical bootstrap, Jackknife estimators, and their equivalent ridge regression estimators respectively,  Here  $n=800$ and  $B$ vary in $\{1,2,...,10\}$. {Lower left panel:} The blue and green lines are the {finite-sample} $\ell_2$ normed difference curves between the Bernoulli bootstrap estimator and its equivalent ridge regression estimator with $\sampleratio=0.2, 0.6$ respectively. Here $B=10$, and $n$ varies  in $\{100,200,...,800\}$.   {Lower right panel:} The blue and green lines are the finite-sample $\ell_2$ normed difference curves between the classical bootstrap, Jackknife estimators,  and their equivalent ridge regression estimators with $ B=10$ and $n$ varying in $\{100,200,...,800\}$. 
    }
    \label{fig:fig_6}
\end{figure}

The lemma above demonstrates that the bagged estimator, when using a downsampling ratio of $\sampleratio$, can be seen as equivalent to the ridge regression estimator with a penalty parameter $\lambda = (1/\sampleratio - 1)(\aspratio - \sampleratio) \vee 0$. %as the number of bootstrap samples and the dimension of the problem tend to infinity. 
This equivalence is based on the expected $\ell_2$ normed difference between the estimators. In essence, bagging acts  as a form of implicit regularization. Figure \ref{fig:fig_6} depicts the $\ell_2$ normed differences between the bootstrap estimators and their  their corresponding ridge regression estimators. The observed differences decrease as the number of bootstrap samples $B$ or the sample size $n$ increases.

%%%%%%%%%%%%%%%%%%%%%%%%%%%%%%%%%%%%%%%%%%%%%%%%%
%%%%%%%%%%%%%%%Correlated features%%%%%%%%%%%%%%%
%%%%%%%%%%%%%%%%%%%%%%%%%%%%%%%%%%%%%%%%%%%%%%%%%

\section{Correlated features}\label{sec:correlated}

This section delves into the anaylsis of correlated features. In this case, the limiting risks are  implicitly determined by  some self-consistent equations. %which do not generally admit closed-form solutions.   %commonly referred to as self-consistent equations. Unfortunately, these equations do not generally admit closed-form solutions. 
We first introduce these equations. Given an aspect ratio $\aspratio > 0$, a downsampling ratio $0 < \sampleratio \leq 1$, a limiting spectral distribution $\limitdistr$ of $\Sigma$, and $z \leq 0$, we define $\compstieltjes(z)$ and $\tcompstieltjes(z)$ as the positive solutions to the self-consistent equations:
\#
\compstieltjes(z) &= \left( -z + \frac{\aspratio}{\sampleratio} \int \frac{t\, d\limitdistr(t)}{1 + \compstieltjes(z)t} \right)^{-1}\quad \text{and}   \label{eq:fixedpointsketching}  \\
\tcompstieltjes(z) &= \left( -z + \frac{\aspratio}{\sampleratio^{2}} \int \frac{t\, d\tlimitdistr(t)}{1 + \tcompstieltjes(z)t} \right)^{-1}, \label{eq:fixpointbootstrap}
\#
where $\tlimitdistr$ represents the limiting spectral distribution of $(I + \ensambleweight(0)\covmat)^{-1} \covmat$, and $\ensambleweight(0) = (1 - \sampleratio) \compstieltjes(0)$. Our first result provides   the existence, uniqueness, and differentiability of $\compstieltjes(z)$ and $\tcompstieltjes(z)$.

\begin{lemma}\label{lm:v0}
    For any $z \leq 0$, equations~\eqref{eq:fixedpointsketching} and \eqref{eq:fixpointbootstrap} have unique positive solutions $\compstieltjes(z)$ and $\tcompstieltjes(z)$.  Moreover, $\compstieltjes(z)$ and $\tcompstieltjes(z)$ are differentiable for any $z < 0$. When $\aspratio> \sampleratio$, $\compstieltjes(0):= \lim_{\zinfty} v(z)$, $\compstieltjes(0):= \lim_{\zinfty} \compstieltjes(z)$, $\compstieltjes'(0):= \lim_{\zinfty} \compstieltjes'(z)$, $\tcompstieltjes(0):= \lim_{\zinfty} \tcompstieltjes(z)$, and  $\tcompstieltjes'(0) :=\lim_{\zinfty} \tcompstieltjes'(z)$ exist. 
\end{lemma}

%%%%%%%%%%%%%%%%%%%%%%%%%%%%%%%%%%%%%%%%%%%%%%%%%%%%%%%%%%%%%%%%%%%%%%%%%%%%%%%%%
%%%%%%%%%%%%%%%%%%%%%Enhancing overparametrization effect%%%%%%%%%%%%%%%%%%%%%%%%
%%%%%%%%%%%%%%%%%%%%%%%%%%%%%%%%%%%%%%%%%%%%%%%%%%%%%%%%%%%%%%%%%%%%%%%%%%%%%%%%%

\subsection{Sketching under correlated features}

%%%%The following result showcases the asymptotic risk of the sketched ridgeless estimator.

We first study the exact risk of the sketched min-norm least square  estimator $\hat\beta^1$.  Recall $f(\aspratio)$ from Theorem~\ref{thm:isosketching}.

\begin{theorem}[Sketching under correlated features]\label{thm:corsketching}
     Assume Assumptions~\ref{Assume:highdim}-\ref{Assume:beta4+mom}. Then the out-of-sample prediction risk of $\hat\beta^1$ satisfies
     \begin{equation*}
    \lim_{\npinfty} \riskcondition (\hat\beta^1)
    = 
    \begin{cases}
    \sigma^2\left(\frac{\aspratio}{1-\aspratio - f(\aspratio)} - 1 \right),& \aspratio < \sampleratio  \\
    \signallev \frac{ \sampleratio}{\aspratio \compstieltjes(0)} + \noiselev \left(\frac{\compstieltjes'(0)}{\compstieltjes(0)^{2}} - 1\right),\quad& \aspratio > \sampleratio
    \end{cases} \qas
    \end{equation*}
 Specifically, the bias and variance satisfy 
\$
%\lim_{\npinfty} 
\biascondition (\estimator)
& \overset{{\rm a.s.}}{\rightarrow}
\begin{cases}
0,& \aspratio/\sampleratio < 1  \\
\signallev \frac{\sampleratio}{\aspratio \compstieltjes(0)},& \aspratio/\sampleratio > 1
\end{cases}, 
%\lim_{\npinfty} 
\quad 
\varcondition (\estimatorsketch)
 \overset{\as}{\rightarrow} 
\begin{cases}
\sigma^2 \left(\frac{\aspratio}{1-\aspratio - f(\aspratio)} -1 \right), & \aspratio/\sampleratio < 1  \\
\noiselev \left(\frac{\compstieltjes'(0)}{\compstieltjes(0)^2} - 1 \right), & \aspratio/\sampleratio > 1
\end{cases}. 
\$
\end{theorem}

%%%%%%%%%%%%%%%%%%%%%%%%%%%%%%%%%%%%%%%%%%%%%%%%%%%%%%%%%%
%%%%%%%%%%%%%%%%%%%%%%%%Features%%%%%%%%%%%%%%%%%%%%%%%%%
%%%%%%%%%%%%%%%%%%%%%%%%%%%%%%%%%%%%%%%%%%%%%%%%%%%%%%%%%%

\iffalse
Different from the isotropic feature case, the limiting  risk in the presence of correlated features does not admit closed-form solutions.% but can be computed numerically.  
In the special case of $\covmat = I$,  the limiting spectral distribution $H$ degenerates to the Dirac measure $\delta_1$. Then $v(0)$ and $v'(0)$ can be solved with closed forms as 
\$
\compstieltjes(0)= \frac{\sampleratio}{\aspratio - \sampleratio}
~~~\text{and}~~~
\compstieltjes'(0)= \frac{\sampleratio^2\aspratio}{(\aspratio - \sampleratio)^3}, 
\$
leading to 
\begin{equation*}
\lim_{\npinfty} \riskcondition (\hat\beta^1)
= 
\begin{cases}
\sigma^2\left(\frac{\aspratio}{1-\aspratio - f(\aspratio)} - 1 \right),& \aspratio < \sampleratio  \\
\signallev \frac{\aspratio/\sampleratio - 1}{\aspratio/\sampleratio } + \noiselev \frac{1}{\aspratio/\sampleratio - 1},\quad& \aspratio > \sampleratio
\end{cases} \qas
\end{equation*}
which reduces to Theorem \ref{thm:isosketching}  in Section~\ref{sec:isotropic}. In other words,  Theorem \ref{thm:corsketching} includes Theorem \ref{thm:isosketching} of Section~\ref{sec:isotropic} as a special case. Moreover, Corollary \ref{coro:optimal_sketching} continues to hold in the presence of  correlated features, that is, Bernoulli sketching optimizes  the limiting risk in Theorem \ref{thm:corsketching} among all sketching methods considered. 
\fi

The limiting risk in the presence of correlated features does not admit closed-form expressions in either regime, but it can be computed numerically. In the specific case of $\covmat = I$, the limiting spectral distribution $H$ simplifies to the Dirac measure $\delta_1$. This allows us to find the closed-form solutions for $v(0)$ and $v'(0)$, which are given by
\$
\compstieltjes(0)= \frac{\sampleratio}{\aspratio - \sampleratio}
~~~\text{and}~~~
\compstieltjes'(0)= \frac{\sampleratio^2\aspratio}{(\aspratio - \sampleratio)^3}, 
\$
resulting in the following limiting risk expressions:
\begin{equation*}
\lim_{\npinfty} \riskcondition (\hat\beta^1)
= 
\begin{cases}
\sigma^2\left(\frac{\aspratio}{1-\aspratio - f(\aspratio)} - 1 \right),& \aspratio < \sampleratio  \\
\signallev \frac{\aspratio/\sampleratio - 1}{\aspratio/\sampleratio } + \noiselev \frac{1}{\aspratio/\sampleratio - 1},\quad& \aspratio > \sampleratio
\end{cases} \qas
\end{equation*}
This result is consistent with Theorem \ref{thm:isosketching} in Section~\ref{sec:isotropic}. Essentially, Theorem \ref{thm:corsketching} encompasses Theorem \ref{thm:isosketching}  as a special case. Moreover, Corollary \ref{coro:optimal_sketching} remains valid with correlated features, implying  that Bernoulli sketching optimizes the limiting risk  in Theorem \ref{thm:corsketching}.

To comprehend the role of sketching in the correlated case, we juxtapose the Bernoulli sketched estimator, chosen due to its optimality, with the full-sample min-norm estimator. Let $\compstieltjes(z; x)$ denote the solution to the equation:
\#\label{def:vz}
\compstieltjes(z; x) &= \left( -z + x\int \frac{t\,  d\limitdistr(t)}{1 + \compstieltjes(z;x)t} \right)^{-1}, ~\text{for any}~ x>0. 
\#
With this new notation,  $\compstieltjes(z)$ defined via the self-consistent equation \eqref{eq:fixedpointsketching} can be rewritten as $v(z;\aspratio/\sampleratio)$. The limiting risk of the Bernoulli sketched estimator $\hat\beta^1_{\Bern}$ can then be expressed as:
\#\label{eq:ridgeless_corr}
\lim_{\npinfty} \riskcondition(\hat\beta^1)
&= 
\begin{cases}
\sigma^2 \frac{\aspratio/\sampleratio}{1-\aspratio/\sampleratio } ,& \aspratio/\sampleratio  < 1 \\
\signallev \frac{1}{\aspratio \compstieltjes(0; \aspratio/\sampleratio)} + \noiselev \left(\frac{\compstieltjes'(0; \aspratio/\sampleratio)}{\compstieltjes(0; \aspratio/\sampleratio)^2} - 1 \right),\quad& \aspratio/\sampleratio  > 1
\end{cases} \qas
\#
A variant of \cite[Theorem 3]{hastie2019surprises} characterizes the limiting risk of the full-sample min-norm estimator as
\#\label{eq:ridgeless_corr}
\lim_{\npinfty} \riskcondition(\ridgeless)
&= 
\begin{cases}
\sigma^2 \frac{\aspratio}{1-\aspratio } ,& \aspratio < 1 \\
\signallev \frac{1}{\aspratio \compstieltjes(0; \aspratio)} + \noiselev \left(\frac{\compstieltjes'(0; \aspratio)}{\compstieltjes(0; \aspratio)^2} - 1 \right),\quad& \aspratio > 1
\end{cases} \qas
\#

By comparing the aforementioned  limiting risks, it is evident that the limiting risk of $\hat\beta^1$ corresponds to that of the full-sample min-norm estimator, albeit with  the aspect ratio and interpolation threshold modified from $\aspratio$ and $\aspratio = 1$ to $\aspratio/\sampleratio$ and $\aspratio/\sampleratio = 1$ respectively. In other words, sketching alters the  aspect ratio and shifts the interpolation threshold, which is consistent with  findings in the isotropic case.

%%%%%%%%%%%%%%%%%%%%%%%%%%%%%%%%%%%%%%%%%%%%%%%%%%%%%%%%%%
%%%%%%%%%%%%%%%Figure 7%%%%%%%%%%%%%%%%
%%%%%%%%%%%%%%%%%%%%%%%%%%%%%%%%%%%%%%%%%%%%%%%%%%%%%%%%%%

\begin{figure}[t]
\centering
\subfigure{
    \includegraphics[width=.45\textwidth]{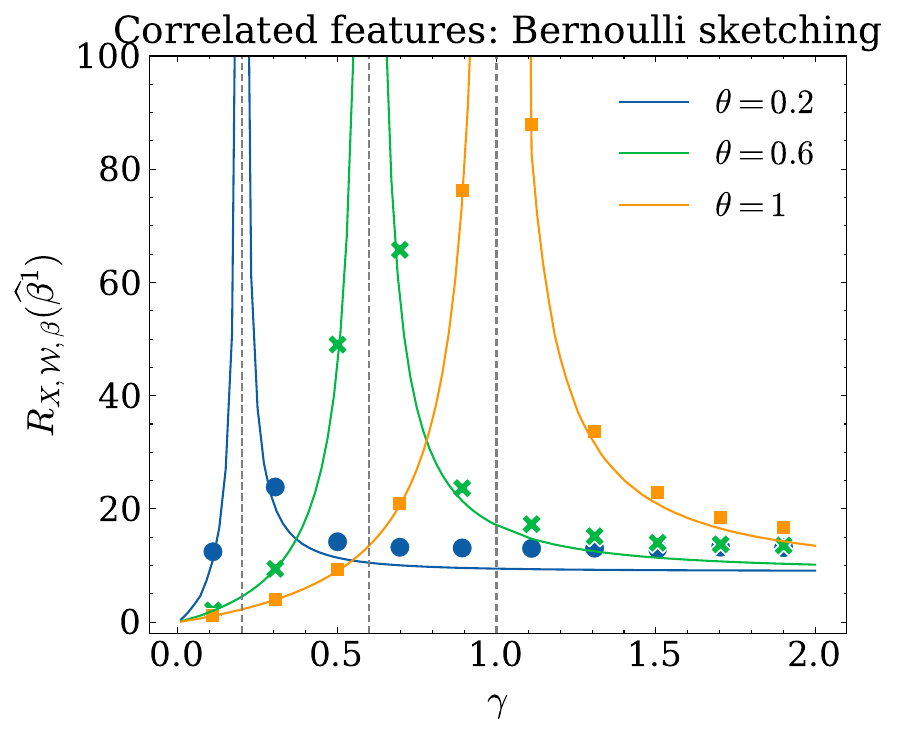}
}
% \hspace{-2em}
\subfigure{
    \includegraphics[width=.44\textwidth]{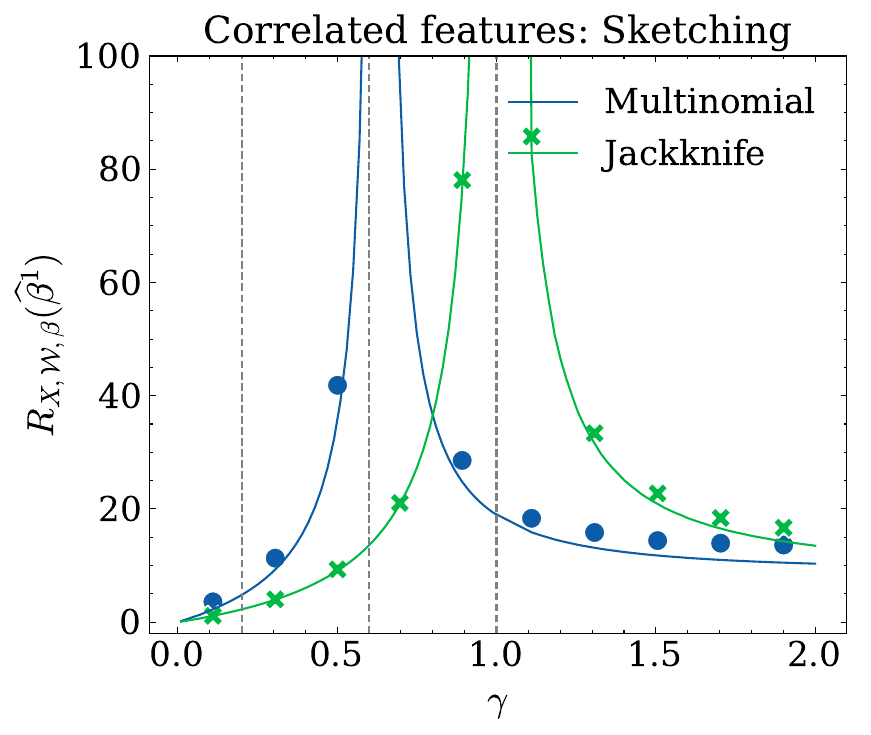}
}
    \caption{\small 
{
Limiting risk curves for Bernoulli sketched estimators  (left panel),  multinomial and Jackknife sketched estimators (right panel) with correlated features and $(r,\sigma)=(3,3)$.  Rows of $X\in \mathbb{R}^{n\times d}$ are  i.i.d drawn from $\mathcal{N}(0,\Sigma)$, $\beta\sim \cN(0, r^2I_d/d)$,  and $\Sigma$ has empirical spectral distribution $F^{\Sigma} (x)=\frac{1}{d}\sum_{i=1}^d 1(\lambda_i(\Sigma)\leq x)$ with $\lambda_i=2$ for $i = 1,...,[d/2]$, and $\lambda_i=1$ for $i = [d/2]+1,...,d$. Errors are generated in the same way as in Figure \ref{fig:fig_1}.  {Left panel:} The blue, green, and yellow lines  are theoretical risk curves for the Bernoulli sketched estimators with $\sampleratio = 0.2, 0.6, 1.0$ respectively. {Right panel:}  The blue and green lines  are theoretical risk curves for the classical sketched and Jackknife estimators respectively. 
In both panels, symbols mark the corresponding finite-sample risks in the same way as in Figure \ref{fig:fig_2}. }
}
    \label{fig:fig_7}
\end{figure}

Figure \ref{fig:fig_7}  plots the limiting risk curves for Bernoulli sketched,  multinomial,  and Jackknife sketched estimators with correlated features and $(r,\sigma)=(3,3)$. Each row of $X\in \mathbb{R}^{n\times d}$ is i.i.d drawn from $\mathcal{N}(0,\Sigma)$ and $\Sigma$ has empirical spectral distribution $F^{\Sigma} (x)=\frac{1}{d}\sum_{i=1}^d1(\lambda_i(\Sigma)\leq x)$  with $\lambda_i=2$ for $i = 1,...,[d/2]$, and $\lambda_i=1$ for $i = [d/2]+1,...,d$. It shows that the limiting risks of Bernoulli sketched estimators have the same shapes as that of the full-sample min-norm estimator but with modified  aspect ratios and interpolation thresholds.

%%%%%%%%%%%%%%%%%%%%%%%%%%%%%%%%%%%%%%%%%%%%%%%%%%%%%%%%%%%%
%%%%%%%%%%%%%%%%%%Bagging alters covariance%%%%%%%%%%%%%%%%%
%%%%%%%%%%%%%%%%%%%%%%%%%%%%%%%%%%%%%%%%%%%%%%%%%%%%%%%%%%%%

\subsection{Bagging under correlated features}\label{sec:corbagging}

This subsection first  studies  the out-of-sample prediction risk of the bagged  estimator under correlated features, and then establishes an equivalence between the bagged estimator and some full-sample min-norm  estimator. 
We begin with the characterization of the limiting risk. 
Recall $\compstieltjes(0)$,  $\tcompstieltjes(0)$, and $\tcompstieltjes'(0)$ from  Lemma \ref{lm:v0}.

%%%%%%%%%%%%%%%%%%%%%%%%%%%%%%%%%%%%%%%%%%%%%%%%%%%%%%%%%%%%%%
%%%%%%%%%%%%%Bagging under correlated features%%%%%%%%%%%%%%%%
%%%%%%%%%%%%%%%%%%%%%%%%%%%%%%%%%%%%%%%%%%%%%%%%%%%%%%%%%%%%%%

\begin{theorem}[Bagging under correlated features]\label{thm:corrbagging}
Assume Assumptions~\ref{Assume:highdim}-\ref{Assume:beta4+mom}. Then the out-of-sample prediction risk of $\estimator$ satisfies
\$
\lim_{\nresample \rightarrow + \infty} \lim_{\npinfty} \riskcondition(\estimator)
&= 
\begin{cases}
    \sigma^2\frac{\aspratio}{1-\aspratio}, & \aspratio < \sampleratio  \\
    \signallev \frac{\sampleratio}{\aspratio \compstieltjes(0)} - \signallev \frac{(1 - \sampleratio)}{\aspratio \compstieltjes(0)} \left(\frac{\tcompstieltjes'(0)}{\tcompstieltjes(0)^{2}} - 1 \right) + \noiselev \left(\frac{\tcompstieltjes'(0)}{\tcompstieltjes(0)^{2}} - 1\right), \quad& \aspratio > \sampleratio
\end{cases} \qas 
\$
Specifically, the bias and variance satisfy 
\$
%\lim_{\npinfty} 
\biascondition (\estimator)
&\overset{{\rm a.s.}}{\rightarrow}
\begin{cases}
0,& \aspratio/\sampleratio < 1  \\
\signallev \frac{\sampleratio}{\aspratio \compstieltjes(0)} - \signallev \frac{(1 - \sampleratio)}{\aspratio \compstieltjes(0)} \left(\frac{\tcompstieltjes'(0)}{\tcompstieltjes(0)^{2}} - 1 \right) ,& \aspratio/\sampleratio > 1
\end{cases}, \\
%\lim_{\npinfty} 
%~
\varcondition (\estimator)
& \overset{\as}{\rightarrow} 
\begin{cases}
\sigma^2 \frac{\aspratio}{1-\aspratio} , & \aspratio/\sampleratio < 1  \\
\noiselev \left(\frac{\tcompstieltjes'(0)}{\tcompstieltjes(0)^2} - 1 \right), & \aspratio/\sampleratio > 1
\end{cases}. 
\$
\end{theorem}

\iffalse
\scolor{
In the underparameterized regime, similar to the isotropic case, the limiting risk of our bagged least square estimator remains the same as that of the ridgeless  estimator, regardless of what multipliers are used. % bootstrap procedure. 
In contrast, in the overparameterized regime, the limiting variance is determined by the noise variance $\noiselev$, with the aspect ratio $\aspratio/\sampleratio^{2}$ and covariance matrix $(I + \ensambleweight(0)\covmat)^{-1} \covmat$. {\color{red}When compared to sketched least squares, the limiting bias is reduced by 
\$
\frac{\signallev(1 - \sampleratio)}{\aspratio \compstieltjes_{1}(0)} \left(\frac{\tcompstieltjes'(0)}{\tcompstieltjes(0)^{2}} - 1 \right),
\$ 
which combines both the bias and variance terms.}  \scomment{What do you mean by bias is reduced by XXX? which combines both bias and variance terms. I find this hard to follow.} This observation demonstrates that bagging alleviates the double descent peak and results in a bounded risk curve, {\color{red}except at the interpolating point $\aspratio = \sampleratio$.} \scomment{What happens at the interpolation threshold?}
}
\fi

%%%%%%%%%%%%%%%%%%%%%%%%%%%%%%%%%%%%%%%
%%%%%%%%After the main theorem%%%%%%%%%
%%%%%%%%%%%%%%%%%%%%%%%%%%%%%%%%%%%%%%%

%\paragraph{Overview of the result}
%%%%Overview

%The above theorem characterizes the limiting risk of the bagged ridgeless least square estimator under correlated features when $B$ goes to infinity. 
Similarly to the isotropic case, the limiting risk for the bagged estimator is  independent of the choice of multipliers. Moreover, in contrast to the full-sample and sketched min-norm estimators whose limiting risks explode at the corresponding interpolation thresholds,
the limiting risk of the bagged estimator remains bounded.  Consequently,  the bagged estimator is stabler than both estimators in terms of the generalization performance.

%%%%the sketched least square estimator
%\paragraph{Comparing with the sketched least square estimator} 
This  stability improvement  comes from the variance reduction property of bagging, especially around the interpolation threshold.
When compared with the sketched min-norm estimator, bagging helps reduce the variance by at least a factor of $\theta$ everywhere, and even more substantially around the interpolation threshold. This  is characterized by the following lemma.

%Comparing with the sketched ridgeless least square estimator, bagging helps reduce the variance at least $\theta$ times everywhere, and to a smaller order near the interpolation threshold.  This is characterized in the following lemma. 

\begin{corollary}\label{cor:overbagsketch}
Assume Assumptions \ref{Assume:highdim}-\ref{Assume:Covdistri}. Then   
\$ 
\lim_{\nresample \rightarrow + \infty}\lim_{\npinfty} \frac{ \varcondition(\estimator)  }{ \varcondition(\estimatorsketch)} 
&\leq 
\frac{\sampleratio - \aspratio}{1 - \aspratio} \cdot  1(\gamma < \theta) + \frac{  {\tcompstieltjes'(0)}/{\tcompstieltjes(0)^2} - 1 }{  {\compstieltjes'(0)}/{\compstieltjes(0)^2} - 1   } \cdot 1\left( \gamma > \theta \right) \\
&\leq \theta. 
\$
When $\aspratio/\sampleratio \rightarrow 1$ and $\sampleratio \ne 1$, we have 
\$
\lim_{\nresample \rightarrow + \infty}\lim_{\npinfty} \frac{ \varcondition(\estimator)  }{ \varcondition(\estimatorsketch)} \rightarrow 0 \qas 
\$
\end{corollary}

\begin{figure}[t]
\centering
\subfigure{
    \includegraphics[width=.45\textwidth]{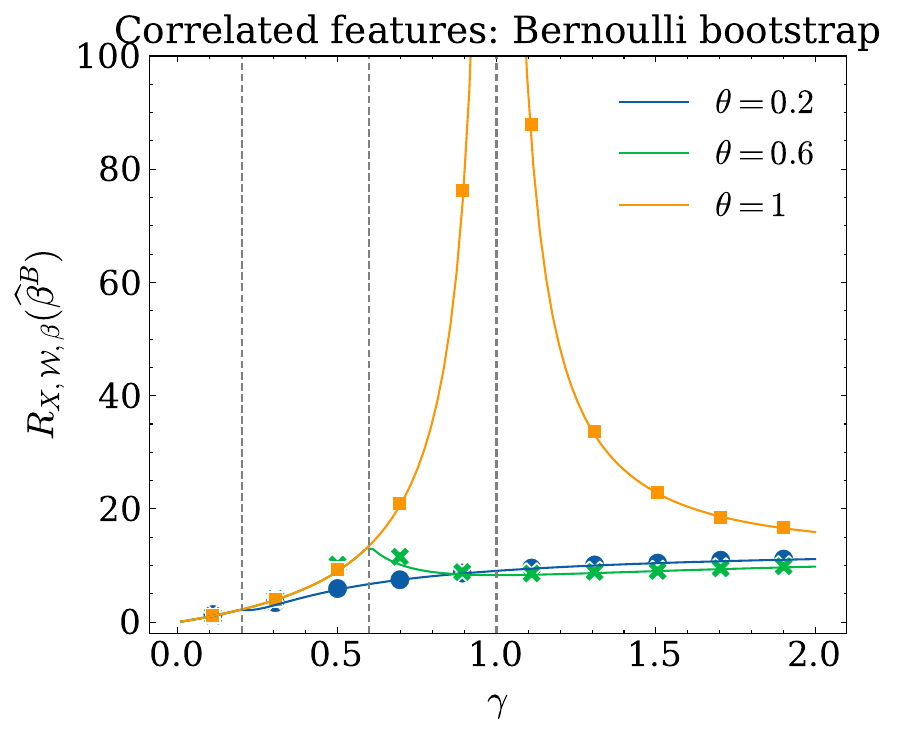}
}
% \hspace{-2em}
\subfigure{
    \includegraphics[width=.43\textwidth]{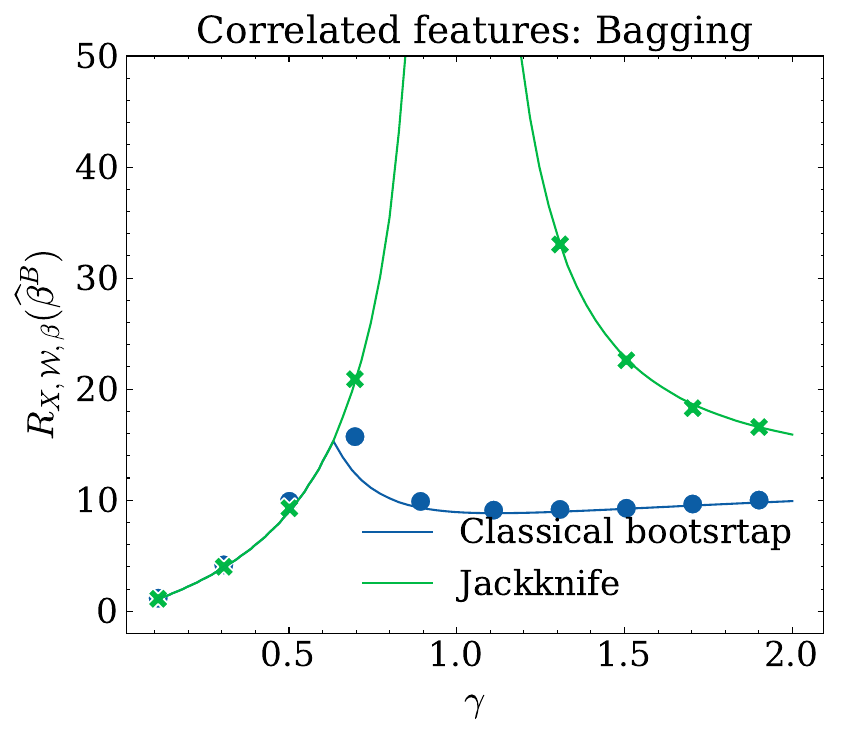}
}
    \caption{\small 
  %  Limiting risk curves for the bagged estimator with correlated features and $(r,\sigma)=(3,3), n\_train=400, n\_test=100$. The left figure depicts the risk changes with $\gamma$ under the setting of Bernoulli bootstrap, where the blue dots, green crosses, and yellow triangles individually mark $\theta=0.2,0.6,1.0$. Each row of $X\in \mathbb{R}^{n\_train\times d}$ was i.i.d drawn from $\mathcal{N}(0,\Sigma)$ and $\Sigma$ has empirical spectral distribution $F^{\Sigma} (x)=\frac{1}{d}\sum_{i=1}^d\mathbf{1}\{\lambda_i(\Sigma)\leq x\}$ with $\lambda_i=2$ for $i = 1,...,[d/2]$, and $\lambda_i=1$ for $i = [d/2]+1,...,d$. The right figure includes the risk changes in classical bootstrap (blue dots) and Jackknife bootstrap (green crosses) with $B=50$.
    {Limiting risk curves for the Jackknife estimator, Bernoulli and  classical bootstrap estimators with correlated features and $(r, \sigma)=(3,3)$. The features, errors, and $\beta$ are generated in the same way as in Figure \ref{fig:fig_7} and the number of bootstrap rounds is $B=50$.  {Left panel:} The blue,  green, and orange  lines are theoretical risk curves for the Bernoulli bootstrap estimators with $\sampleratio = 0.2, 0.6, 1.0$ respectively. 
{Right panel:}  The blue and green lines  are theoretical risk curves for the classical bootstrapped and Jackknife estimators respectively. 
In both panels, symbols mark the corresponding finite-sample risks in the same way as in Figure \ref{fig:fig_2}.}}
    \label{fig:fig_8}
\end{figure}

In addition to variance reduction, bagging also contributes to the reduction of implicit bias in the overparameterized regime:
\$
\lim_{\nresample \rightarrow + \infty} \lim_{\npinfty} \biascondition(\hat \beta^B)
&\overset{\as}{=}
\lim_{\npinfty} \biascondition (\hat \beta^1) -  \frac{\signallev(1 - \sampleratio)}{\aspratio \compstieltjes(0)} \left(\frac{\tcompstieltjes'(0)}{\tcompstieltjes(0)^{2}} - 1 \right)
\leq  \lim_{\npinfty} \biascondition (\hat \beta^1)
\$  
since ${\tcompstieltjes'(0)}/{\tcompstieltjes(0)^{2}} - 1 >0$. 
Under isotropic features, the self-consistent equations can be readily solved, leading to 
\$
\compstieltjes(0)= \frac{\sampleratio}{\aspratio - \sampleratio}, ~ \tcompstieltjes(0)= \frac{\sampleratio^{2}}{\aspratio - \sampleratio}, ~\textnormal{and}~~ \tcompstieltjes'(0)= \frac{\aspratio \sampleratio^{4}}{(\aspratio - \sampleratio)^{2} (\aspratio - \sampleratio^{2})}. 
\$ 
Consequently, the implicit bias in the overparameterized regime  is reduced by 
\$
r^2 \frac{\sampleratio(1 - \sampleratio)(\aspratio - \sampleratio)}{\aspratio (\aspratio - \sampleratio^{2})}. 
\$

Figure \ref{fig:fig_8} plots the {limiting risk curves for the Jackknife estimator, Bernoulli and  classical bootstrap estimators with correlated features and $(r, \sigma)=(3,3)$. Similar to the isotropic case, bagged estimators with $\sampleratio \ne 1$ have bounded limiting risks.

%%%%%%%%%%%%%%%%%%%%%%%%%%%%%%%%%%%%%%%%%%%%%%%%%%%%%%%%%%%%%%%
%%%%%%%%%%%%Bagging as implicit regularization%%%%%%%%%%%%%%%%%
%%%%%%%%%%%%%%%%%%%%%%%%%%%%%%%%%%%%%%%%%%%%%%%%%%%%%%%%%%%%%%%
\subsection{Bagging as implicit regularization}

In this subsection, we establish an  equivalence between the bagged estimator under model \eqref{eq:lm} and  the min-norm least square estimator under a  different model.  Specifically, let $\hat{\truesignal}_{\sampleratio}$ be the ridgeless least square estimator obtained using the following generative model %by using data generated from 
\#\label{eq:lm_modified}
\Tilde{Y} = \Tilde{X} \Tilde{\truesignal} + E \in \RR^{\lfloor\sampleratio^{2}\ndata\rfloor}, 
\#
where $\Tilde{X}= (\Tilde{x}_1, \ldots, \Tilde{x}_{[\sampleratio^2\ndata]} )\transp \in \RR^{[\sampleratio^{2}\ndata] \times \ndim}$
consists of \iid~feature vectors $\Tilde{x}_i$ with a size of $[\sampleratio^{2}\ndata]$ and a covariance matrix $\Tilde \Sigma = (I + \ensambleweight(0)\covmat)^{-1} \covmat$,  $\Tilde{\truesignal} = (I + \ensambleweight(0)\covmat)^{-1/2} \truesignal$, and $E=(\varepsilon_1, \ldots, \varepsilon_n)\transp\in \RR^n$ is the same as in model \eqref{eq:lm}.

%We assume that the signal vector $\truesignal$ is fixed rather than  random. 
Let $\covmat = \sum_{i}^{\ndim} \eigval_{i} u_{i} u_{i}\transp$ be the eigenvalue decomposition of the covariance matrix.   For a fixed deterministic signal $\truesignal$, define    the eigenvector empirical spectral distribution (VESD)  $G_{n}(s)$ as
\$
G_{n}(s) = \frac{1}{\Tilde{r}^{2}} \sum_{i=1}^{\ndim} \frac{1}{1 + \ensambleweight(0)\eigval_{i}}\langle \beta, u_{i} \rangle^{2} \, 1\left( \frac{\eigval_{i}}{1 + \ensambleweight(0)\eigval_{i}} \leq  s\right),
\$
where $\Tilde{r}^{2}$ = $\truesignal\transp (I + \ensambleweight(0)\covmat)^{-1} \truesignal$. We need the following assumption.

\begin{assumption}[Deterministic signal]\label{Assume:VESD}
 The signal $\truesignal$ is deterministic, and ${G_{n}}$ converges weakly to a probability distribution $G$.
\end{assumption}

Let $R_{\Tilde{X}}$ be defined similarly as $\riskcondition$ but conditioning only on $\Tilde{X}$. With these definitions, we are now ready to  present our main result in this subsection.

\begin{corollary}\label{cor:bootstrapequiv}
Assume Assumptions~\ref{Assume:highdim}-\ref{Assume:multip}, and~\ref{Assume:VESD}. Further assume $\Tilde{x}_i\, {\sim}\, \Tilde{x} = \Tilde{\Sigma}^{1/2}z$ %\wcomment{$\sim$?}
with $z$ satisfying Assumption \ref{Assume:Covdistri}. When $\aspratio > \sampleratio$, we have %\scomment{a.s. or almostly surely.}
    \begin{equation*}
        \lim_{\nresample \rightarrow + \infty} \lim_{\npinfty} \riskcondition(\estimator) = \lim_{\npinfty} R_{\Tilde{X}}(\hat{\truesignal}_{\sampleratio})  \qas
    \end{equation*}
\end{corollary}

 %the limiting risk of the bagged ridgeless least square estimator in the overparameterized regime under model \eqref{eq:lm} coincides with the limiting  risk of the ridgeless least square estimator under model \eqref{eq:lm_modified}. 

The above result reveals that when $\gamma > \sampleratio$, the limiting risk of the bagged estimator under the original model \eqref{eq:lm} is equivalent to that of the full-sample min-norm estimator under the new generative model \eqref{eq:lm_modified}. In this new model, the aspect ratio $\gamma$, the true signal $\truesignal$, and the covariance matrix $\Sigma$ are replaced by $\aspratio/\sampleratio^{2}$, $(I + \ensambleweight(0)\covmat)^{-1/2} \truesignal$, and $(I + \ensambleweight(0)\covmat)^{-1} \covmat$ respectively. In other words, the new model \eqref{eq:lm_modified} represents features with a shrunken covariance matrix and a shrunken signal. This suggests that bagging serves as a form of implicit regularization. When $\gamma < \sampleratio$, the limiting risk of the bagged estimator agrees with that of the full-sample min-norm estimator under the original model, regardless of the choice of multipliers.

%In other word, the features under with a shrunken covariance matrix and a shrunken signal, which suggests that  bagging serves as  a form of implicit regularization. 
%Specifically, model \eqref{eq:lm_modified} has features with a shrunken covariance matrix and a shrunken signal, which suggests that  bagging serves as  a form of implicit regularization. When $\gamma < \sampleratio$, the limiting risk of the  bagged  estimator agrees with that of the full-sample min-norm, regardless of what multipliers are used.

\iffalse  
\scolor{
 When $\gamma < \sampleratio$, the limiting risk of the  bagged  estimator agrees with that of the full-sample min-norm, regardless of what multipliers are used. When $\gamma > \sampleratio$, however,  the limiting variance  is equal to  that of the full-sample min-norm  estimator, but with the aspect ratio $\gamma$, the true signal $\truesignal$, and the covariance matrix $\Sigma$ replaced by $\aspratio/\sampleratio^{2}$, $(I + \ensambleweight(0)\covmat)^{-1/2} \truesignal$, and $(I + \ensambleweight(0)\covmat)^{-1} \covmat$ respectively; see Corollary \ref{cor:bootstrapequiv}.  
 }
 \fi

%%%%%%%%%%%%%%%%%%%%%%%%%%%%%%%%%%%%%%%%%%%%%%%%%%%%%%%%%%%%
%%%%%%%%%%%%%%%%%%%%%%%%Extensions %%%%%%%%%%%%%%%%%%%%%%%%%
%%%%%%%%%%%%%%%%%%%%%%%%%%%%%%%%%%%%%%%%%%%%%%%%%%%%%%%%%%%%
\section{Extensions}\label{sec:extension}

This section studies  the limiting risk under the deterministic signal case, characterizes the training error in terms of the out-of-sample limiting risk, and  calculates the adversarial risk, all for the  bagged estimator. 

%%%%%%%%%%%%%%%%%%%%%%%%%%%%%%%%%%%%%%%%%%%%%%%%%%%%%%%
%%%%%%%%%%%%%%%%%Deterministic signal%%%%%%%%%%%%%%%%%%
%%%%%%%%%%%%%%%%%%%%%%%%%%%%%%%%%%%%%%%%%%%%%%%%%%%%%%%

\subsection{Deterministic signal}

\iffalse
This subsection extends our analysis to the case where the signal vector $\beta$ is deterministic. Let $\covmat = \sum_{i}^{\ndim} \eigval_{i} u_{i} u_{i}\transp$ be the eigenvalue decomposition of the covariance matrix.  For a fixed deterministic signal $\truesignal$, define  the eigenvector empirical spectral distribution (VESD)  $G_{n}(s)$ as
\$
G_{n}(s) = \frac{1}{\Tilde{r}^{2}} \sum_{i=1}^{\ndim} \frac{1}{1 + \ensambleweight(0)\eigval_{i}}\langle \beta, u_{i} \rangle^{2} \, 1\left( \frac{\eigval_{i}}{1 + \ensambleweight(0)\eigval_{i}} \leq  s\right),
\$
where $\Tilde{r}^{2}$ = $\truesignal\transp (I + \ensambleweight(0)\covmat)^{-1} \truesignal$. We need  the following assumption. 

\begin{assumption}[Deterministic signal]\label{Assume:VESD}
    The signal $\truesignal$ is deterministic, and ${G_{n}}$ converges weakly to a probability distribution $G$.
\end{assumption}
\fi

We first present the limiting risk of the bagged least square estimator when the signal is deterministic as in Assumption \ref{Assume:VESD}. Recall  $\compstieltjes(0)$, $\tcompstieltjes(0)$, and  $\tcompstieltjes'(0)$ from Lemma \ref{lm:v0}.

\begin{theorem}[Deterministic signal]\label{thm:Overpara}
Assume Assumptions~\ref{Assume:highdim}-\ref{Assume:multip}, and~\ref{Assume:VESD}.  The out-of-sample prediction risk of $\estimator$ satisfies
\begin{equation*}
\lim_{\nresample \rightarrow + \infty} \lim_{\npinfty} \riskcondition(\estimator)
=
\begin{cases}
    \sigma^2\frac{\aspratio}{1-\aspratio}, & \aspratio < \sampleratio  \\
    \Tilde{r}^{2} \frac{\tcompstieltjes'(0)}{\tcompstieltjes(0)^{2}} \int \frac{s}{(1 + \tcompstieltjes(0)s)^{2}} \,d{G}(s) + \noiselev \left(\frac{\tcompstieltjes'(0)}{\tcompstieltjes(0)^{2}} - 1\right), \quad& \aspratio > \sampleratio
\end{cases} \qas 
\end{equation*}
Specifically, the bias and variance satisfy 
\$
%\lim_{\npinfty} 
\biascondition (\estimator)
& \overset{{\rm a.s.}}{\rightarrow}
\begin{cases}
0,& \aspratio/\sampleratio < 1  \\
\Tilde{r}^{2} \frac{\tcompstieltjes'(0)}{\tcompstieltjes(0)^{2}} \int \frac{s}{(1 + \tcompstieltjes(0)s)^{2}} \,d{G}(s), & \aspratio/\sampleratio > 1
\end{cases}, \\
%\lim_{\npinfty} 
\varcondition (\estimatorsketch)
 &\overset{\as}{\rightarrow} 
\begin{cases}
\sigma^2 \frac{\aspratio}{1-\aspratio}, & \aspratio/\sampleratio < 1  \\
\noiselev \left(\frac{\tcompstieltjes'(0)}{\tcompstieltjes(0)^{2}} - 1\right), & \aspratio/\sampleratio > 1
\end{cases}. 
\$
\end{theorem}

The theorem above can be viewed as a generalization of Theorem \ref{thm:corrbagging}, accounting for the interaction between $\beta$ and $\Sigma$. In comparison to Theorem \ref{thm:corrbagging}, the only differing term is the implicit bias term  in the overparameterized regime. Assuming  that $\beta$ satisfies Assumption \ref{Assume:beta4+mom}, the following corollary shows that this bias term simplifies  to the one in Theorem \ref{thm:corrbagging}. Consequently, the above theorem recovers Theorem \ref{thm:corrbagging} as a special case.

\begin{corollary}\label{cor:deterreduction}
%Assume Assumptions~\ref{Assume:highdim}, \ref{Assume:Covdistri}, \ref{Assume:beta4+mom}, and \ref{Assume:VESD}.
In addition to the assumptions in Theorem \ref{thm:Overpara}, assume Assumption \ref{Assume:beta4+mom}. Then, we have 
    \begin{equation*}
        \Tilde{r}^{2} \frac{\tcompstieltjes'(0)}{\tcompstieltjes(0)^{2}} \int \frac{s}{(1 + \tcompstieltjes(0)s)^{2}} \,d{G}(s) = \signallev \frac{\sampleratio}{\aspratio \compstieltjes(0)} - \signallev \frac{(1 - \sampleratio)}{\aspratio \compstieltjes(0)} \left(\frac{\tcompstieltjes'(0)}{\tcompstieltjes(0)^{2}} - 1 \right).
    \end{equation*}
Consequently, Theorem \ref{thm:Overpara} recovers Theorem \ref{thm:corrbagging} as a special case. 
\end{corollary}

%%%%%%%%%%%%%%%%%%%%%%%%%%%%%%%%%%%%%%%%%%%%%%%%%
%%%%%%%%%%%%%%%Training error%%%%%%%%%%%%%%%%%%%%
%%%%%%%%%%%%%%%%%%%%%%%%%%%%%%%%%%%%%%%%%%%%%%%%%

\subsection{Training error}\label{sec:Advrobust}

Let $\trainloss(\estimator; X,Y) =  \|Y - X\estimator\|^{2}_{2}$ be the training error. Then the following result characterizes the training error in terms of the out-of-sample prediction risk. %establishes a connection between  the training error and the out-of-sample prediction risk. 
\begin{theorem}[Training  error]\label{thm:trainingerr}
Assume Assumptions~\ref{Assume:highdim}-\ref{Assume:beta4+mom}.  When $\aspratio > \sampleratio$, the training error satisfies %\scomment{What about underparameterized regime?}
\begin{equation*}
    \lim_{\nresample \rightarrow + \infty} \lim_{n,p \rightarrow + \infty} \EE \left[\trainloss(\estimator; X,Y) \Big| X \right] 
    =
    (1 - \sampleratio)^{2} \lim_{\nresample \rightarrow + \infty} \lim_{n,p \rightarrow + \infty} \riskcondition(\estimator) + (1 - \sampleratio)^{2} \noiselev \qas
\end{equation*}
\end{theorem}

The above result implies  that,  in the overparameterized regime, the out-of-sample prediction risk can also be expressed  in terms of the training error:
\$
 \lim_{\nresample \rightarrow + \infty} \lim_{n,p \rightarrow + \infty} \riskcondition(\estimator) = \frac{1}{(1-\sampleratio)^2}  \lim_{\nresample \rightarrow + \infty} \lim_{n,p \rightarrow + \infty} \EE \left[\trainloss(\estimator; X,Y) \Big| X \right]  - \noiselev.
\$
In other words, the out-of-sample prediction risk of the  bagged least square interpolator is linear in the training error: Better the rescaled training error (rescaled by $1/(1-\sampleratio)^2$),  better the generalization performance. Hence, when employing bagged interpolators in practice,  we can simply choose the bagged interpolator with the smallest rescaled training error. Computationally expensive procedurse such as the  cross-validation  is not  necessary.

%We remark that Theorem \ref{thm:trainingerr} also holds for the full-sample min-norm estimator in the overparameterized regime, a linear interpolator.  Following the same argument as above, the out-of-sample prediction risk is a constant mul.  this partly explains the practice that training an overparameterized neural network to have a zero training error often leads to good generalization performance. %of swhy the training error is 

%%%%%%%%%%%%%%%%%%%%%%%%%%%%%%%%%%%%%%%%%%%%%%%%%
%%%%%%%%%%%%%Adversarial risk%%%%%%%%%%%%%%%%%%%
%%%%%%%%%%%%%%%%%%%%%%%%%%%%%%%%%%%%%%%%%%%%%%%%%

\subsection{Adversarial risk}\label{sec:Advrobust}

This subsection examines the adversarial robustness of the bagged min-norm least square estimator under adversarial attacks, focusing on the case where the features are isotropic. For any estimator $\hat\beta$, we introduce the $\ell_{2}$ adversarial risk as follows:
\begin{equation*}
    R^{\rm{adv}}(\hat{\beta}; \delta) := \EE\left[\max_{\|x\|_{2} \leq \delta} \left(x_\new\transp \truesignal  - (x_\new + x)\transp \hat{\beta}\right)^2 \right],
\end{equation*}
where the expectation is taken with respect to the test feature $x_{\new}$. 

Assuming that  $x_\new \sim \cN(0, I)$  and using \cite[Lemma 3.1]{javanmard2020precise}, the $\ell_2$ adversarial risk of any  estimator $\hat{\beta}$ can be expressed as
\#\label{eq:adv}
R^{\rm{adv}}(\hat{\beta}; \delta) = R(\hat{\beta}) + \delta^{2} \| \hat{\beta} \|_{2}^{2} + 2\sqrt{\frac{2}{\pi}}\delta \| \hat{\beta}\|_{2} \left( \noiselev +  R(\hat{\beta})   \right)^{1/2},
\#
where %the  risk $R(\hat{\beta})$ is defined as 
\$
R(\hat{\beta}):= \EE\left[\left(x_\new \transp \truesignal  - x_\new\transp \hat{\beta}\right)^2 \condition \truesignal, \hat{\beta} \right]
\$ with the conditional expectation taken  with respect to the test feature  $x_\new$. Hence, the adversarial risk of any estimator $\hat{\beta}$ depends on the risk $R(\hat\beta)$ and its norm $\|\hat\beta\|_2$.  
This newly defined  risk slightly differs from the one in equation~\eqref{eq:riskcon}. 
Our next result shows  that, with an additional assumption,  $R(\estimator)$ and $R(\ridgeless)$  are asymptotically equivalent to $\riskcondition(\estimator)$ and $\riskcondition(\ridgeless)$ respectively, which holds  in the general context of correlated features. 

\begin{assumption}\label{Assume:adv}
    Assume that the noises $\varepsilon_{i}$ have bounded $(4 + \epsilon)$-th moments for some $\epsilon>0$.
\end{assumption}

\begin{lemma}\label{lm:advriskequiv}
    Assume Assumptions~\ref{Assume:highdim}-\ref{Assume:beta4+mom}, and \ref{Assume:adv}. Then, we have almost surely %\scomment{General Sigma?}
\$
\lim_{\npinfty} R(\estimator) &= \lim_{\npinfty} \riskcondition(\estimator) ~~~\textrm{and} \\
\lim_{\npinfty} R(\ridgeless) &= \lim_{\npinfty} \riskcondition(\ridgeless).
\$
\end{lemma}

We proceed to characterize the norm of the bagged least square estimator and illustrate how bagging leads to a reduction in the norm, resulting in enhanced adversarial robustness.

\begin{lemma}\label{lemma:adversarial}
Assume Assumptions~\ref{Assume:highdim}-\ref{Assume:beta4+mom},  \ref{Assume:adv}, and $\covmat = I$. Then we have
\begin{equation*}
    \lim_{\nresample \rightarrow + \infty} \lim_{\npinfty}  \| \estimator \|_{2}^{2} 
    = 
    \begin{cases}
        \signallev + \noiselev \frac{\aspratio}{1-\aspratio}, & \aspratio < \sampleratio  \\
        \signallev \frac{\sampleratio^{2} (\aspratio + 1 - 2\sampleratio)}{\aspratio(\aspratio - \sampleratio^{2})} + \noiselev  \frac{\sampleratio^2}{\aspratio - \sampleratio^{2}}, & \aspratio > \sampleratio
    \end{cases} \qas 
\end{equation*}
\end{lemma}

We compare the norm of the bagged estimator   with that of the min-norm  estimator. Using a variant of \cite[Corollary 1]{hastie2019surprises}, we obtain that the squared  $\ell_2$-norm of the  min-norm  estimator satisfies
\begin{equation*}
    \lim_{\npinfty}  \| \ridgeless \|_{2}^{2} 
    = 
    \begin{cases}
        \signallev + \noiselev \frac{\aspratio}{1-\aspratio},  & \aspratio < 1  \\
        \signallev \frac{1}{\aspratio} + \noiselev  \frac{1}{\aspratio - 1},& \aspratio > 1
    \end{cases} \qas 
\end{equation*}
Bagging shrinks the norm of the full-sample min-norm estimator:
\$
%\lim_{\npinfty} 
\lim_{\nresample \rightarrow + \infty} \lim_{\npinfty} \frac{ \| \estimator \|_{2}^{2}}{ \| \ridgeless \|_{2}^{2} }
%\overset{\as}{\rightarrow} 
\leq 
\begin{cases}
 \frac{\sampleratio}{\aspratio} < 1 , & 1> \aspratio > \sampleratio \\
\max\left\{ \frac{\sampleratio}{\aspratio^2}, \sampleratio^2\right\}< 1, & \aspratio >  1 > \sampleratio
\end{cases}. 
\$
%Indeed, the bagged least square estimator no longer strictly interpolates the data even when $\gamma$.    \scomment{How did you see this?}

%all $ R(\estimator)$ and $\|\estimator \|_{2}^{2}$ available, 

With the limiting risks and squared norms of $\estimator$ and $\ridgeless$ available 
we can calculate the limiting adversarial risks of $\estimator$ and $\ridgeless$ according to equation \eqref{eq:adv}.  Figure \ref{fig:fig_9} compares the $\ell_2$ norms and  limiting adversarial risks of the bagged estimator and the full-sample min-norm estimator under isotropic features. In all of our cases, the bagged estimator exhibits smaller norms and smaller adversarial risks when compared with the full-sample min-norm estimator, demonstrating its superior adversarial robustness.

\iffalse
\begin{align*}
\lim_{\npinfty\atop \nresample \rightarrow + \infty} R^{\rm{adv}}(\estimator; \delta) 
&= 
\lim_{\npinfty\atop \nresample \rightarrow + \infty} R(\estimator) + \delta^{2} \lim_{\npinfty\atop \nresample \rightarrow + \infty}  \| \estimator \|_{2}^{2} \\
&\quad\quad + 2\sqrt{\frac{2}{\pi}}\delta  \lim_{\npinfty\atop \nresample \rightarrow + \infty} \| \estimator\|_{2} \left( \noiselev +  \lim_{\npinfty\atop \nresample \rightarrow + \infty} R(\estimator)\right)^{1/2}. 
\end{align*}
The limiting adversarial risk of $\ridgeless$ has a similar expression. %Since both the risk $R(\estimator)$ and norm $\|\estimator\|$ are smaller than the risk and norm of $\ridgeless$, the limiting adversarial risk of $\estimator$ is smaller than that of $\ridgeless$. 
\fi

\begin{figure}[t!]
\centering

\subfigure{
    \includegraphics[width=.45\textwidth]{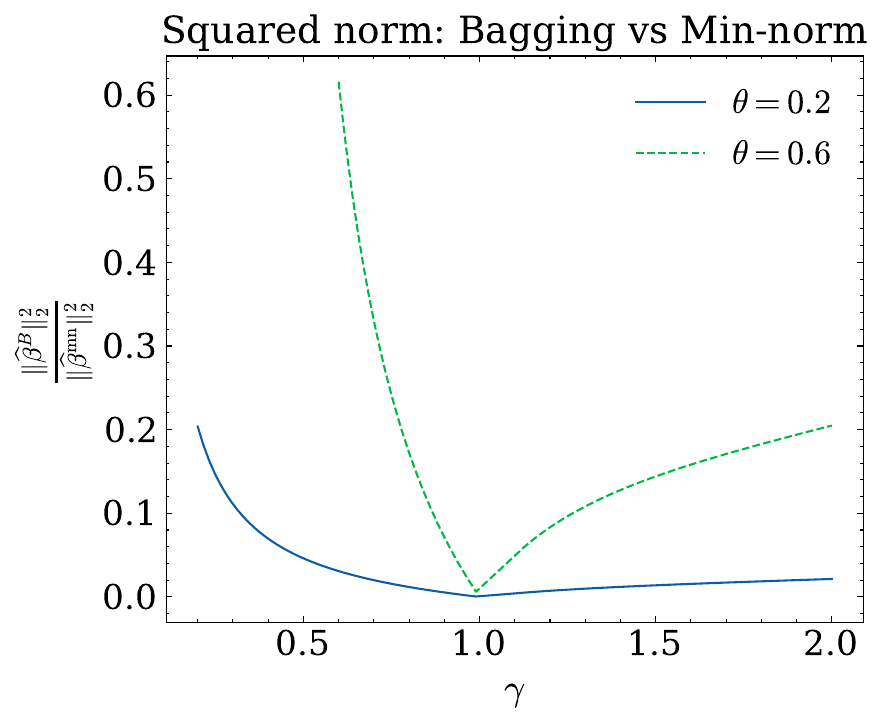}
}
% \hspace{-2em}
\subfigure{
    \includegraphics[width=.46 \textwidth]{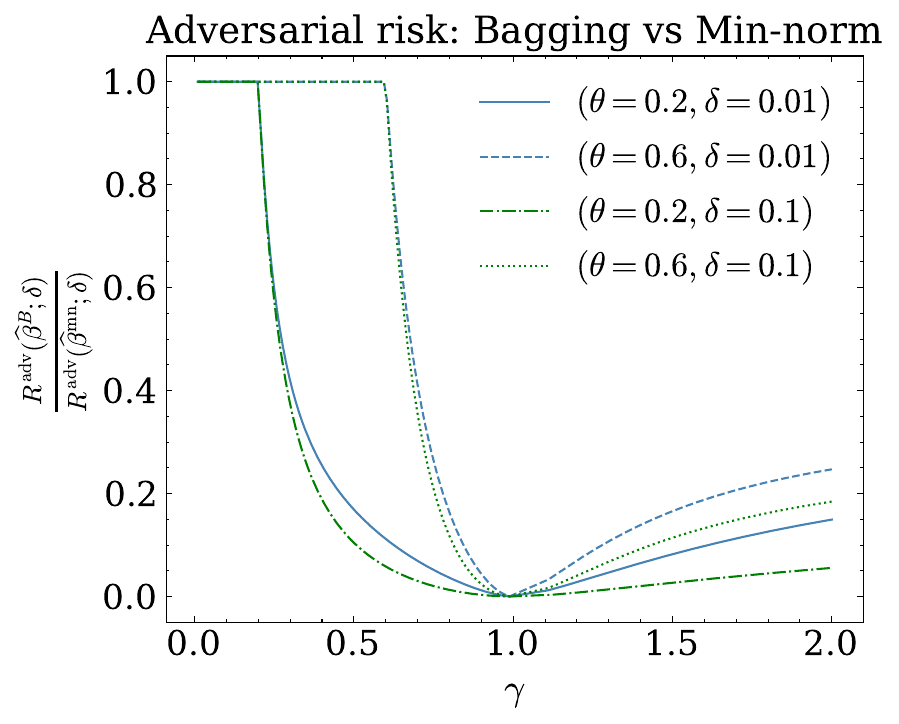}
}
%\subfigure{
%    \includegraphics[width=.45 \textwidth]{pic/lemma5-3(2)_2.pdf}
%}
% \caption{\small
% Comparing the bagged estimator with the full-sample min-norm estimator with isotropic features,  $(r,\sigma)=(3,3)$, and $\gamma$ varying $[0,2]$. {Left panel:} The blue solid  and green dashed lines are theoretical squared norm ratio curves between  the bagged estimator  with $\theta=0.2,0.6$, and the full-sample min-norm estimator, respectively. {Right panel:} 
% \protect\includegraphics[scale=.17]{pic/line1.png},
% \protect\includegraphics[scale=.2]{pic/line2.png},
% \protect\includegraphics[scale=.2]{pic/line3.png}, and 
% \protect\includegraphics[scale=.2]{pic/line4.png} lines are theoretical adversarial risk ratio curves between the bagged estimators with  $(\theta,\delta) = (0.2,0.01), (0.6,0.01),(0.6,0.1), (0.6,0.1)$,  and the full-sample min-norm estimator, respectively.
% }\label{fig:fig_9}

\caption{\small
Comparing the bagged estimator with the full-sample min-norm estimator with isotropic features,  $(r,\sigma)=(3,3)$, and $\gamma$ varying $[0,2]$. {Left panel:} The blue solid  and green dashed lines are theoretical squared norm ratio curves between  the bagged estimator  with $\theta=0.2,0.6$, and the full-sample min-norm estimator, respectively. {Right panel:} 
{\color{RoyalBlue} \full},
{\color{RoyalBlue} \dashed},
{\color{Green} \chain}, and 
{\color{Green} \dotted} lines are theoretical adversarial risk ratio curves between the bagged estimators with  $(\theta,\delta) = (0.2,0.01), (0.6,0.01),(0.6,0.1), (0.6,0.1)$,  and the full-sample min-norm estimator, respectively.
}\label{fig:fig_9}

\end{figure}

%%%%%%%%%%%%%%%%%%%%%%%%%%%%%%%%%%%%%%%%%%%%%%%%%%%%%%%%%%%%
%%%%%%%%%%%%%%%%%%%%%%%%Discussions%%%%%%%%%%%%%%%%%%%%%%%%%
%%%%%%%%%%%%%%%%%%%%%%%%%%%%%%%%%%%%%%%%%%%%%%%%%%%%%%%%%%%%

\section{Conclusions and discussions}\label{sec:Discus}

%%%%Summary
\paragraph{Summary}
Interpolators often exhibit high instability  as their test risks explode under certain model configurations. This paper delves into the mechanisms through which ensembling improves the stability and thus the generalization performance of individual interpolators. Specifically, we focus on bagging, a widely-used ensemble technique that can be implemented in parallel. Leveraging a form of multiplier bootstrap, we introduce the bagged min-norm least square estimator, which can then be formulated  as an average of sketched min-norm least square estimators. Our  multiplier bootstrap includes  the classical bootstrap with replacement as a special case, and introduces an intriguing variant which we term the Bernoulli bootstrap.

Focusing  on the proportional regime $\ndim \asymp \ndata$, where $\ndata$ denotes the sample size and $\ndim$ signifies feature dimensionality, we precisely  characterize  the out-of-sample prediction risks of both sketched and bagged  estimators in both underparameterized and overparameterized regimes. Our findings underscore the statistical roles of sketching and bagging. Specifically, sketching  modifies the aspect ratio and shifts the interpolation threshold when compared with the full-sample  min-norm estimator. Nevertheless, the risk of the sketched estimator is still unbounded around the interpolation threshold due to rapidly increasing  variance.  On the contrary, bagging effectively mitigates this variance escalation, resulting in  bounded limiting  risks.

\paragraph{General models and loss functions}
 We identify several promising avenues for  future research.
The presented multipiler-bootstrap-based bagged estimator $\estimator$ holds potential for broader applicability beyond the linear regression model and squared loss. However, in such cases,  the individual estimator \eqref{eq:multi_loss} may not possess a closed-form representation as in \eqref{eq:sketched}.  Analyzing these more general estimators poses a significant challenge.

\paragraph{Random feature subsampling}
This study primarily focuses on bagging, yet another prominent ensemble learning technique is feature subsampling, notably employed in random forests \citep{breiman2001random}. In the context of random forests, adaptive feature subsampling has demonstrated significant superiority over bagging in  low signal-to-noise ratio settings \citep{mentch2020randomization}. Consequently, it would be immensely valuable to gain a comprehensive understanding of the statistical implications of adaptive feature subsampling, both individually and in combination with bagging. Recent work by \cite{lejeune2020implicit} has delved into the exact risk analysis of underparameterized least square estimators with uniformly random feature subsampling  under isotropic features. However, it remains an open question how to investigate the effects of uniformly random feature subsampling, and even adaptive feature subsampling, under correlated features.

\paragraph{Other sketching matrices} 
 While this paper has focused on sketched estimators with diagonal and predominantly singular sketching matrices, an extension of our results to encompass other sketching matrices is worth exploring. Notably, sketching matrices, such as \iid~and orthogonal sketching matrices, as investigated by \cite{chen2023sketched}, could be considered.   It would be intriguing to investigate whether the invariance and impilcit regularization effect of bagging hold under these alternative sketching procedures.

\paragraph{Application to streaming data}
Finally, compared with the classical bootstrap,  the  proposed Bernoulli bootstrap is expected to lend itself well to  streaming data and growing data sets, since the total number of samples does not need to be known in advance of beginning to take bootstrap samples. We shall explore this  in future work.

\iffalse

\iffalse
\subsection{Subsampling features}
Given the symmetry of the sketching matrix and covariance matrix, we expect that in the over-parameterized regime, the limiting loss does not depend on the feature subsampling rate. However, when considering a general covariance matrix, the situation may be different. In this case, we can represent the data matrix as $SZ\covmat T$, where $T$ is the corresponding feature subsampling matrix. Consequently, instead of a diagonal sketching matrix $S$, we have to deal with the more general matrix $\covmat T$. This introduces new technical challenges and requires the development of additional tools and techniques.
\fi

\iffalse
\subsection{Other ideas}
\cite{dobriban2020wonder} calculate the optimal weight for the data splitting for ridge regression, which is better than simple averaging. Does optimally weighted bagging outperform ridge?
\fi

\fi

%\input{Data_splitting}

\bibliographystyle{apalike}
\bibliography{ref}
%\end{document}

\newpage
\appendix 
%Notations for Appendices

\renewcommand{\theequation}{S.\arabic{equation}}
\numberwithin{equation}{section}
\renewcommand{\thetable}{S.\arabic{table}}
\renewcommand{\thefigure}{S.\arabic{figure}}
\renewcommand{\thesection}{S.\arabic{section}}
\renewcommand{\thelemma}{S.\arabic{lemma}}
\renewcommand{\thecorollary}{S.\arabic{corollary}}
\renewcommand{\theassumption}{S.\arabic{assumption}}

\vspace{30pt}

\noindent{\bf \LARGE Appendix}
\section{Basics}\label{app:basics}

This subsection introduces necessary concepts  and results that will be used throughout the appendix. We first introduce the Stieltjes transform. 
Under certain conditions we can reconstitute the measure  $\mu$ starting from its Stieltjes transformation thanks to the inverse formula of Stieltjes-Perron.  

\begin{definition}[Stieltjes transform]
The Stieltjes transform $m_\mu(z)$ of a measure $\mu$ with support $\supp(\mu)$ is the function of the complex variable $z$ defined outside $\supp(\mu)$ by the formula
\$
m_\mu(z) =\int_{\supp(\mu)} \frac{\mu(dt)}{t-z}, ~z\in \CC\setminus\supp(\mu). 
\$
When $\mu$ is clear from the context, we shall omit the subscript $\mu$ and write $m_\mu(z)$ as $m(z)$. 
\end{definition}

\iffalse
For example, if the density $\rho$ of the measure $\mu$ is continuous throughout $\supp(\mu)$, one will have inside this interval that 
\$
\rho (x)=\lim _{\varepsilon \to 0^{+}}{\frac {m_{\mu }(x-i\varepsilon )- m_{\mu }(x+i\varepsilon )}{2\pi i }}. 
\$
\fi

We discuss three Stieljes transforms that will be used in the proofs for sketching, baggging, and ridge equivalence.  
Let $S_k$ be a  possibly singular diagonal sketching matrix in the form of \eqref{eq:sketching} whose diagonal entries consist of multipliers $\cW$ satisfying Assumption \ref{Assume:multip}. Recall that $X$ is  the data matrix. Let $\hcovmat_k = X\transp S_k\transp S_k X/n$ be the sketched covariance matrix, and $\mu_{\hcovmat_k}= \sum_{i=1}^d \delta_{\lambda_i(\hcovmat_k)}$ be the empirical spectral measure of $\hcovmat_k$. Then the associate Stieltjes transform is 
\$
\stieltjes_{1,n}(z) 
&= \int_{\supp(\mu_{\hcovmat_k})} \frac{1}{t-z} \mu_{\hcovmat_k}(d t) 
 = \frac{1}{d} \sum_{i=1}^d \frac{1}{\lambda_i(\hcovmat_k) - z}= \frac{1}{d}\tr\left((\hcovmat_k - zI)\inv\right). 
\$
Define $m_1(z)$ as a solution to the following self-consistent equation:
\begin{equation}\label{eq:MPlawsketching}
\stieltjes_{1}(z) \EE_{\limitdistrweight}\left[ \frac{\weight}{1 + \aspratio \weight \stieltjes_{1}(z)} \right] - z\stieltjes_{1}(z) = 1, ~~~z < 0. 
\end{equation}
Our first lemma concerns the almost sure convergence of $\mu_{\hcovmat}$ to some probability measure $\mu$ characterized by $\stieltjes_{1}(z)$ and the existence of $\stieltjes_{1}(0)$, which will be used for proving results on sketched estimators.  

%Let $\mu_{n}$ be the empirical spectral distribution of  $X\transp \sketchmat\transp \sketchmat X/ \ndata $. 

\begin{lemma}\label{lm:MPlawsketching}
Assume Assumptions \ref{Assume:highdim}-\ref{Assume:multip}, and $\covmat = I$. As $\npinfty$, we have
\$
 \mu_{\hcovmat_k} \rightarrow \mu \qas
\$
 where $\mu$ is the  probability distribution defined by the Stieltjes transform $\stieltjes_{1}(z)$ and  $\stieltjes_{1}(z)$ is  the unique positive solution to \eqref{eq:MPlawsketching}.  Consequently, $\stieltjes_{1,n}(z) \rightarrow \stieltjes_{1}(z)$ almost surely. %\scomment{Why does $m_1(z)$ have unique positive solution?}
When $\aspratio < \sampleratio$,  $m_1(0):= \lim_{z\rightarrow 0-} m_1(z)$ exist and satisfies  equation \eqref{eq:MPlawsketching} at $z=0$. %equation~\eqref{eq:MPlawsketching}.
\end{lemma}

\iffalse
\scolor{
Let  $\stieltjes_{1,n}'(z)$ be the first-order derivative of $m_{1,n}(z)$, that is, \scomment{result on this?} %the Stieltjes transforms of $\hcovmat$ as follows: 
\$
\stieltjes_{1,n}'(z) = \int \frac{1}{(t-z)^2} \mu_{\hcovmat}(dt)= \frac{1}{d}\tr\left((\hcovmat - zI)^{-2}\right). 
\$
}
\fi

For two sketching matrices $S_k$ and $S_\ell$, let  $\hcovmat_k = X\transp S_k\transp S_k X/n$ and $\hcovmat_\ell  = X\transp S_\ell \transp S_\ell X/n$  be the corresponding sketched covariance matrices. Let  $\mu_{\hcovmat_k}= \sum_{i=1}^d \delta_{\lambda_i(\hcovmat_k)}/d $  and $\mu_{\hcovmat_\ell}= \sum_{i=1}^d \delta_{\lambda_i(\hcovmat)}/d$  be the empirical spectral measures of $\hcovmat_k$ and $\hcovmat_\ell$ respectively.  
Define $\stieltjes_{2,\ndata} (z)$ as 
\$
\stieltjes_{2,\ndata} (z):= {\ndim}^{-1} \tr((\hcovmat_{k} - zI)^{-1}(\hcovmat_{\ell} - zI)^{-1}). 
\$
Then  the following result establishes the almost sure convergence of $\stieltjes_{2,\ndata} (z)$, which will be used  for proving results on bagged estimators. 

%Recall that $\stieltjes_{1}(z)$ is the unique positive solution to equation~\eqref{eq:MPlawsketching}. \scolor{ Recall $\stieltjes_{2,\ndata} (z)= {\ndim}^{-1} \tr((\hcovmat_{k} - zI)^{-1}(\hcovmat_{\ell} - zI)^{-1})$ as defined in Lemmma \ref{lm:convergprodresolvent}. } 

\begin{lemma}\label{lm:isotrolimitresolvent}
    Assume Assumptions~\ref{Assume:highdim}-\ref{Assume:multip} and $\covmat = I$. For any $z < 0$, as $\npinfty$,
    \begin{equation*}
        \lim_{\npinfty} \stieltjes_{2,\ndata} (z) = \stieltjes_{2} (z) \qas
    \end{equation*}
    where $\stieltjes_{2} (z)$ satisfies %\scomment{unique positiveness of $m_2(z)$?}
    \begin{equation}\label{eq:fixpoint_iso_Bootstrap}
    \stieltjes_{2}(z) \EE_{\limitdistrweight} \left[ \frac{1}{1 + \aspratio \weight \stieltjes_{1}(z)}\right] \EE_{\limitdistrweight} \left[ \frac{\weight}{1 + \aspratio \weight \stieltjes_{1}(z)}\right] - z \stieltjes_{2}(z) = \stieltjes_{1}(z). 
    \end{equation}
   The above result  also holds for $z = 0$ when $\aspratio < \sampleratio$.
\end{lemma}

Let 
$
 \stieltjes_{2, \lambda, n}(z):= \tr( (X\transp X/n + \lambda I)\inv (X\transp \sketchmat_{k}\transp \sketchmat_{k}X/n - z)\inv)/\ndim. 
$
Then the following result establishes the almost sure convergence of $ \stieltjes_{2, \lambda, n}(z)$ and  will be used for proving  the ridge equivalence result. %scomment{moved to later?}

\begin{lemma}\label{lm:ridgeresolvent}
    Assume Assumptions~\ref{Assume:highdim}-\ref{Assume:multip} and $\covmat = I$. %Let $X_k := \sketchmat_k X$ be a subsample. %and $\stieltjes(\lambda)$ be the Stieltjes transform of the limiting spectral distribution of full-sample matrix $X$. Then 
    Then, for any $z < 0$, we have, as $\npinfty$,
    \begin{equation*}
        \lim_{\npinfty} \stieltjes_{2, \lambda, n}(z) = \stieltjes_{2, \lambda}(z) \qas
    \end{equation*}
    where $\stieltjes_{2, \lambda}(z)$ satisfies the equation
    \begin{equation}\label{eq:fixpoint_ridgeequiv}
        \stieltjes(-\lambda) = \stieltjes_{2, \lambda}(z) \frac{1}{1 + \aspratio \stieltjes(-\lambda)} \EE_{\limitdistrweight} \left[ \frac{\weight}{1 + \aspratio \weight \stieltjes_{1}(z)}\right] - z\stieltjes_{2,\lambda}(z),
    \end{equation}
    where $\stieltjes(-\lambda)$ is the Stieltjes transform of the limiting spectral distribution of $X\transp X/n$ evaluated at $-\lambda$. When  $\aspratio < \sampleratio$, the above result also holds for $z = 0$. 
\end{lemma}

%We then prove Lemma \ref{lm:v0} on  $\compstieltjes(z)$ and $\tcompstieltjes(z)$. We first restate the lemma. 
\iffalse
\begin{lemma}
    For $z \leq 0$, equations~\eqref{eq:fixedpointsketching} and \eqref{eq:fixpointbootstrap} have unique positive solutions $\compstieltjes(z)$ and $\tcompstieltjes(z)$.  Moreover, $\compstieltjes(z)$ and $\tcompstieltjes(z)$ are differentiable for any $z < 0$. When $\aspratio> \sampleratio$, $\compstieltjes(0):= \lim_{\zinfty} v(z)$, $\compstieltjes(0):= \lim_{\zinfty} \compstieltjes(z)$, $\compstieltjes'(0):= \lim_{\zinfty} \compstieltjes'(z)$, $\tcompstieltjes(0):= \lim_{\zinfty} \tcompstieltjes(z)$, and  $\tcompstieltjes'(0)=\lim_{\zinfty} \tcompstieltjes'(z)$ exist. 
\end{lemma}
\fi

%%%%%%%%%%%%%%%%%%%%%%%%%%%%%%%%%%%%%%%%%%%%%%%%%%
%%%%%%%%%%%%%%%%%%%%%Proofs%%%%%%%%%%%%%%%%%%%%%%%
%%%%%%%%%%%%%%%%%%%%%%%%%%%%%%%%%%%%%%%%%%%%%%%%%%
\subsection{Proofs for Section \ref{app:basics}}

This subsection proves Lemmas \ref{lm:MPlawsketching}-\ref{lm:ridgeresolvent} in order.

%%%%%%%%%%%%%%%%%%%%%%%%%%%%%%%%%%%%%%%%%%%%%%%%%%
%%%%%%%%%%%%%%%%%%%%Lemma S.1.3%%%%%%%%%%%%%%%%%%%
%%%%%%%%%%%%%%%%%%%%%%%%%%%%%%%%%%%%%%%%%%%%%%%%%%

\subsubsection{Proof of Lemma \ref{lm:MPlawsketching}}

\begin{proof}[Proof of Lemma \ref{lm:MPlawsketching}]
We prove this lemma in two steps. 

\paragraph{Almost sure convergence of $\mu_n$.} %and characterization of $m_1(z)$}

If the multipliers are deterministic and bounded, Theorem 2.7 by~\cite{couillet2022} proves that  equation \eqref{eq:MPlawsketching} has the unique positive solution  $\stieltjes_{1}(z)$ for $z < 0$ and  $\mu_{\hcovmat_k}$ converges almost surely  to $\mu$ which is defined by the Stieltjes transform $\stieltjes(z)$.  The deterministic and bounded assumption on the multipliers can be relaxed to Assumption \ref{Assume:relaxed_assumption}  by using the arguments presented in Section 4.4 of~\citep{zhang2007spectral}. 
%\scomment{collect these reults and assumptions in the preliminary section. }

\paragraph{Proving that  $\stieltjes_{1}(0) := \lim_{\zinfty} \stieltjes_{1}(z)$ exits. }

When $\aspratio < \sampleratio$, Lemma~\ref{lm:singularmat} and Lemma~\ref{lm:lowerboundeigval} imply  that the sketched sample covariance matrix 
$X\transp \sketchmat_k\transp \sketchmat_k X/ \ndata$ 
is almost surely invertible as $\npinfty$, and its smallest eigenvalue is strictly greater than zero in the limit. Consequently, $\stieltjes_{1}(0): = \lim_{\zinfty} \stieltjes_{1}(z)$ exists and satisfies
    \begin{align}\label{eq:m0limit}
        \stieltjes_{1}(0) &= \lim_{\zinfty} \stieltjes_{1}(z) =  \lim_{\zinfty}\int \frac{1}{t-z} d\mu(t)= \int \frac{1}{t} d\mu(t),
    \end{align}
    where the last  line follows from the dominated convergence theorem. Furthermore, for $z \leq 0$, when $\aspratio < \sampleratio$, $\stieltjes_{1}(z)$ is bounded. Applying the dominated convergence theorem, we conclude that $\stieltjes(0)=\lim_{\zinfty} \stieltjes_{1}(z)$ satisfies equation~\eqref{eq:MPlawsketching}. Hence, the desired result follows. %\wcomment{However, the fixed-point equation in Theorem 1.2.1 in~\citep{zhang2007spectral} does not have such form and requires $z \in \mathbb{C^{+}}$. (need to prove the uniqueness of the solution?). Another approach, through Theorem 2.7 in~\citep{couillet2022}, directly yields the above fixed-point equation, but it requires the multipliers to be bounded and $z \in \mathbb{C}/\RR^{+}$.} 
\end{proof}

%%%%%%%%%%%%%%%%%%%%%%%%%%%%%%%%%%%%%%%%%%%%%%%%%%
%%%%%%%%%%%%%%%%%%%%Lemma S.1.3%%%%%%%%%%%%%%%%%%%
%%%%%%%%%%%%%%%%%%%%%%%%%%%%%%%%%%%%%%%%%%%%%%%%%%

\subsubsection{Proof of Lemma \ref{lm:isotrolimitresolvent}}
\begin{proof}[Proof of Lemma \ref{lm:isotrolimitresolvent}]
    Let $z < 0$. We start by expressing $(\hcovmat_{k} - zI)^{-1}$ as
    \begin{align*}
    (\hcovmat_{k} - zI)^{-1} &= (\hcovmat_{k} - zI)^{-1} (\hcovmat_{\ell} - zI)^{-1} (\hcovmat_{\ell} - zI)\\
    &= \frac{1}{n} \sum_{i=1}^{\ndata} \weight_{\ell,i} (\hcovmat_{k} - zI)^{-1} (\hcovmat_{\ell} - zI)^{-1} x_{i}x_{i}\transp - z (\hcovmat_{k} - zI)^{-1}(\hcovmat_{\ell} - zI)^{-1}.
    \end{align*}
    Taking the trace and applying the Sherman–Morrison formula, we obtain
    \begin{align*}
    \tr (\hcovmat_{k} - zI)^{-1} &= \frac{1}{\ndata} \sum_{i=1}^{\ndata} \weight_{\ell,i} \frac{\tr( x_{i}\transp (\hcovmat_{k,-i} - zI)^{-1} (\hcovmat_{\ell,-i} - zI)^{-1} x_{i})}{(1 + \frac{1}{\ndata} \weight_{k,i} x_{i}\transp (\hcovmat_{k,-i} - zI)^{-1} x_{i}) (1 + \frac{1}{\ndata} \weight_{\ell,i} x_{i}\transp (\hcovmat_{\ell,-i} - zI)^{-1} x_{i})}\\
    &\quad\quad\quad  - z\tr (\hcovmat_{k} - zI)^{-1} (\hcovmat_{\ell} - zI)^{-1}.
    \end{align*}
Applying Lemma~\ref{lm:boundrankoneperturb} and Lemma~\ref{lm:concentrationontrace}, we obtain 
    \begin{align*}
        & \frac{1}{\ndim} \tr (\hcovmat_{k} - zI)^{-1} - \frac{1}{\ndata} \sum_{i=1}^{\ndata} \weight_{\ell,i} \frac{ \frac{1}{\ndim} \tr( (\hcovmat_{k} - zI)^{-1} (\hcovmat_{\ell} - zI)^{-1})}{(1 + \frac{1}{\ndata} \weight_{k,i}  \tr(\hcovmat_{k} - zI)^{-1} ) (1 + \frac{1}{\ndata} \weight_{\ell,i}  \tr(\hcovmat_{\ell} - zI)^{-1})}\\
        &\quad\quad\quad - z \,\frac{1}{\ndim} \tr \left((\hcovmat_{k} - zI)^{-1} (\hcovmat_{\ell} - zI)^{-1}\right) \rightarrow 0 \qas 
    \end{align*}
 as $\npinfty$. Applying Lemma~\ref{lm:MPlawsketching} then acquires 
   \begin{equation*}
        \stieltjes_{1}(z) - \frac{1}{\ndata} \sum_{i=1}^{\ndata} \weight_{\ell,i} \frac{  \stieltjes_{2,n}(z)}{ (1 + \aspratio \weight_{k,i} \stieltjes_{1} (z)) (1 + \aspratio \weight_{\ell,i} \stieltjes_{1} (z))}  + z \stieltjes_{2,n}(z)\rightarrow 0 \qas
   \end{equation*}
% as $\npinfty$.
Using Lemma \ref{lm:MPlawsketching}, we obtain 
   \begin{equation*}
        \frac{\weight_{\ell,i}}{ (1 + \aspratio \weight_{k,i} \stieltjes_{1} (z)) (1 + \aspratio \weight_{\ell,i} \stieltjes_{1} (z))} \leq \frac{1}{\aspratio \stieltjes_{1}(z)}<\infty. 
   \end{equation*}
 Hence, using  Assumption~\ref{Assume:multip} and the dominated convergence theorem, we obtain
   \begin{equation*}
        \stieltjes_{1}(z) -  \stieltjes_{2,n}(z) \EE_{\limitdistrweight} \left[ \frac{1}{1 + \aspratio \weight \stieltjes_{1}(z)}\right] \EE_{\limitdistrweight} \left[ \frac{\weight}{1 + \aspratio \weight \stieltjes_{1}(z)}\right] + z \stieltjes_{2,n}(z) \rightarrow 0 \qas
   \end{equation*}
 Consequently, the limit of $\stieltjes_{2,n}(z)$ exists and satisfies equation~\eqref{eq:fixpoint_iso_Bootstrap}. This completes the proof for the case of  $z < 0$. 
   
   In the case where $\aspratio < \sampleratio$ and $z = 0$, we can assume both sample covariance matrices are invertible almost surely as $\npinfty$, as indicated by Lemma~\ref{lm:singularmat}. Then applying  the same argument as in the case of $z<0$ finishes the proof. 
\end{proof}

%%%%%%%%%%%%%%%%%%%%%%%%%%%%%%%%%%%%%%%%%%%%%%%%%%
%%%%%%%%%%%%%%%%%%%%Lemma S.1.4%%%%%%%%%%%%%%%%%%%
%%%%%%%%%%%%%%%%%%%%%%%%%%%%%%%%%%%%%%%%%%%%%%%%%%
\subsubsection{Proof of Lemma \ref{lm:ridgeresolvent}}
\begin{proof}[Proof of Lemma~\ref{lm:ridgeresolvent}]
    The proof of Lemma~\ref{lm:ridgeresolvent} is similar to that of Lemma~\ref{lm:isotrolimitresolvent}. The only difference is that the ridge regression has term $(X\transp X + \lambda I)\inv$, which can be written as $(X\transp S SX + \lambda I)\inv$ with sketching matrix $S = I$. One can check that all arguments in Lemma~\ref{lm:isotrolimitresolvent} carry through.
\end{proof}

%%%%%%%Technical lemmas
\subsection{Technical lemmas}
This subsection proves technical lemmas that are used in the proofs of the supporting lemmas in the previous subsection.

%%%%%Tech Lemma 1

\begin{lemma}\label{lm:singularmat}
    Let $\Omega_{\ndata} := \{ \omega \in \Omega :  (X\transp \sketchmat\transp \sketchmat X/\ndata)(\omega)~\text{is invertible} \}$. Assume Assumptions~\ref{Assume:highdim}-\ref{Assume:multip}. Then
    \begin{equation*} 
         \PP \left( \lim_{\npinfty} \Omega_{\ndata} \right)  = \begin{cases}
            1, \quad&~ \aspratio < \sampleratio,\\
            0, \quad&~ \aspratio > \sampleratio.
        \end{cases}
    \end{equation*}
\end{lemma}
\begin{proof}[Proof of Lemma \ref{lm:singularmat}]
    Let $A_{0} := \{i:~\sketchmat_{i,i} \neq 0\}$. The smallest eigenvalue can be lower bounded as %\scomment{do we need to put the assumption on the $w_i$ or the limiting law?} \scomment{should only depend on the limiting law...}
\begin{align*}
        \eigval_{\min}(X\transp \sketchmat\transp \sketchmat X/\ndata) &= \eigval_{\min}\left( \sum_{i=1}^{n} \weight_{i} x_{i}x_{i}\transp/\ndata \right)\\
        &= \eigval_{\min}\left( \sum_{i \in A_{0}} \weight_{i} x_{i}x_{i}\transp/\ndata \right)\\
        & \geq c_{w} \eigval_{\min}\left( \sum_{i \in A_{0}} x_{i}x_{i}\transp/\ndata \right)\\
        & \geq c_{w} c_{\lambda} \eigval_{\min}\left( \sum_{i \in A_{0}} z_{i}z_{i}\transp/\ndata \right). 
\end{align*}
Applying Lemma \ref{lm:Bai-yin}, when $\lim |A_{0}|/ \ndim> 1$, we obtain %\scomment{revise}
    \begin{equation*}
        \eigval_{\min}(X\transp \sketchmat\transp \sketchmat X/\ndata) \geq c_{w} c_{\lambda} \left( 1 - \sqrt{ \lim \ndim/ |A_{0}|} \right) \lim |A_{0}|/\ndata> 0,
    \end{equation*}
    which implies that $X\transp \sketchmat\transp \sketchmat X/\ndata$ is invertible in the limit. 
    
     On the other hand, for any vector $v \in \RR^{\ndim}$, we have
    \begin{align*}
        v\transp X\transp \sketchmat\transp \sketchmat Xv &= \sum_{i \in A_{0}} \weight_{i}v\transp x_{i}x_{i}\transp v\\
        &= \sum_{i \in A_{0}} \weight_{i} (x_{i}\transp v)^2.
    \end{align*}  
Therefore, when $|A_{0}| < \ndim$, there exists a non-zero vector $v$ such that $v\transp X \sketchmat\transp \sketchmat\transp Xv = 0$.

The cardinality of $A_{0}$ can be expressed as:
\begin{equation*}
    |A_{0}| = \sum_{i = 1}^{\ndata} 1\left({\weight_{i} \neq 0}\right).
\end{equation*}

Using Assumption~\ref{Assume:multip}, we have
\begin{equation}\label{eq:SSLNA0}
    |A_{0}|/\ndata = \frac{1}{\ndata} \sum_{i = 1}^{\ndata} 1\left({\weight_{i} \neq 0}\right) \rightarrow \EE_{\limitdistrweight} 1\left({\weight \neq 0}\right) = \sampleratio~ \qas
\end{equation} 
as $\npinfty$. This finishes the proof.
\end{proof}

%%%%%Tech Lemma 2

\begin{lemma}\label{lm:lowerboundeigval}
Assume Assumptions \ref{Assume:highdim}-\ref{Assume:multip}.  There exists a  positive constant  $c$ such that, as $\npinfty$, %he following inequality holds almost surely
\begin{equation*}
\lim_{\npinfty}\eigval^{+}_{\min}(X\transp \sketchmat\transp \sketchmat X/\ndata) \geq  c \left(1 - \sqrt{\aspratio /\sampleratio} \right)^2 \qas
\end{equation*}
where $\eigval^{+}_{\min}(A)$ is the smallest non-zero eigenvalue of $A$. %the sample covariance matrix.
\end{lemma}

\begin{proof}[Proof of Lemma \ref{lm:lowerboundeigval}]
    Given Assumption~\ref{Assume:Covdistri}, we have $x_i = \covmat^{1/2}_{\ndim} z_i$. Let $A_{0} := \{i~|~\sketchmat_{i,i} \neq 0\}$, and $X_{A_0}\transp X_{A_0} = \sum_{i \in A_{0}} x_{i} x_{i}\transp$ and $Z_{A_0}\transp Z_{A_0} = \sum_{i \in A_{0}} z_{i} z_{i}\transp$. %Consider a unit vector 
    Let $v$ be the unit vector corresponding to the smallest non-zero eigenvalue of $X\transp \sketchmat\transp \sketchmat X/\ndata$. Then we have  %\scomment{???}%we can write
\begin{align*}
\lim_{\npinfty} \eigval^{+}_{\min}(X\transp \sketchmat\transp \sketchmat X/\ndata) 
    &= \lim_{\npinfty} {v\transp X\transp \sketchmat\transp \sketchmat X v/\ndata}\\
    &= \lim_{\npinfty} {\sum_{i \in A_{0}} \weight_{i} (x_{i}\transp v)^2/\ndata}\\
    &\geq \lim_{\npinfty} {c_w} {\sum_{i \in A_{0}} (x_{i}\transp v)^2/\ndata}\\
    &=  \lim_{\npinfty} {c_w} {v\transp X_{A_0}\transp X_{A_0} v/\ndata}\\
    &\geq \lim_{\npinfty} {c_w} \eigval_{\min}(\covmat) \eigval^{+}_{\min}(Z_{A_0}\transp Z_{A_0}/\ndata)\\
    &\geq \lim_{\npinfty} c_\lambda {c_w} \left(1 - \sqrt{\aspratio / \sampleratio} \right)^2,
\end{align*}
where the first inequality follows from Assumption~\ref{Assume:multip}, and the second inequality uses  the fact that $v\transp X\transp S\transp S X v \neq 0$, which implies $v\transp X_{A_0}\transp X_{A_0} v/n \geq \eigval^{+}_{\min}( X_{A_0}\transp X_{A_0}/n )$. Finally, the last line uses Lemma \ref{lm:Bai-yin}, along with Assumption~\ref{Assume:highdim} and equation~\eqref{eq:SSLNA0}. 
%\wcomment{$\lim \eigval$?}
\end{proof}

%%%%%Tech Lemma 3
\begin{lemma}\label{lm:boundrankoneperturb} Consider two subsamples $X_{k} := \sketchmat_{k} X $ and $X_{\ell} := \sketchmat_{\ell}X $. Let $z < 0$. Then the following holds
\begin{gather*}
\left| \tr\left((\hcovmat_{h, -i} - zI)^{-1}  - (\hcovmat_{h} - zI)^{-1}\right) \right| \leq -\frac{1}{z}~~~\textrm{and}\\
    \left| \tr((\hcovmat_{k} - zI)\inv(\hcovmat_{\ell} - zI)\inv) - \tr((\hcovmat_{k, -i} - zI)\inv(\hcovmat_{\ell, -i} - zI)\inv)\right| \leq \frac{2}{z^2}, 
\end{gather*}
where $\hcovmat_{h, -i} := \hcovmat_{h} - \frac{1}{\ndata}\weight_{h,i} x_{i} x_{i}\transp$ for $h = k, \ell$. Furthermore, when $\aspratio < \sampleratio$, there exists a positive constant $c$ such that, under Assumptions \ref{Assume:highdim}-\ref{Assume:multip}, as $\npinfty$, we have
\begin{equation*}
    \left| \tr((\hcovmat_{k})^{-1}(\hcovmat_{\ell})^{-1}) - \tr((\hcovmat_{k, -i})^{-1}(\hcovmat_{\ell, -i} )^{-1})\right| \leq c \left(1 - \sqrt{\aspratio /\sampleratio} \right)^{-4} ~ \qas 
\end{equation*}
\end{lemma}
\begin{proof}[Proof of Lemma \ref{lm:boundrankoneperturb}]
We first prove the first two  results. For any real positive semi-definite matrix $A\in \RR^{\ndim \times \ndim}$ and any $z < 0$, we have $\big\|(A - zI)^{-1}\big\|_{2} \leq  - 1/z$. 
Using the Sherman-Morrison formula, we obtain that
\$
(\hcovmat_{h, -i} - zI)^{-1}  - (\hcovmat_{h} - zI)^{-1} %- (\hcovmat_{\ell, -i} - zI)^{-1}
&=   \frac{(\hcovmat_{h, -i}- z I)^{-1}\weight_{h, i}x_i x_i^\T (\hcovmat_{h, -i}- z I)^{-1}/n  }{1 + \weight_{h, i} x_{i}\transp (\hcovmat_{h, -i }- z I)^{-1} x_{i} / \ndata}
%\geq  0, 
\$
%where $A \geq 0$ indicates that $A$ 
is positive semi-definite. Thus applying Lemma \ref{lm:petubation} acquires 
\$
0\leq \tr\left((\hcovmat_{h, -i} - zI)^{-1}  - (\hcovmat_{h} - zI)^{-1}\right) \leq -\frac{1}{z}. 
\$
This finishes the proof for the first result. 
Let $\lambda_1(A), \ldots, \lambda_\ndim(A)$ be the eigenvalues of $A$ in decreasing order. Then %\scomment{more details. }
\begin{align*}
&\left| \tr((\hcovmat_{k} - zI)^{-1}(\hcovmat_{\ell} - zI)^{-1}) - \tr((\hcovmat_{k, -i} - zI)^{-1}(\hcovmat_{\ell, -i} - zI)^{-1}) \right|\\
&=\left| \tr( (\hcovmat_{k} - zI)^{-1} ((\hcovmat_{\ell} - zI)^{-1} - (\hcovmat_{\ell, -i} - zI)^{-1} ) ) + \tr( (\hcovmat_{\ell, -i} - zI)^{-1} ((\hcovmat_{k} - zI)^{-1} - (\hcovmat_{k, -i} - zI)^{-1} ) ) \right|\\
&\leq  \sum_{j=1}^\ndim \lambda_j\left( (\hcovmat_{k} -zI)^{-1} \right) \lambda_j \left( (\hcovmat_{\ell, -i} -zI)^{-1} - (\hcovmat_{\ell} -zI)^{-1}  \right) \\
&\quad\quad\quad + \sum_{j=1}^\ndim \lambda_j\left( (\hcovmat_{\ell, -i } -zI)^{-1} \right) \lambda_j \left( (\hcovmat_{k, -i} -zI)^{-1} - (\hcovmat_{k} -zI)^{-1}  \right)   \\
&\leq  - \frac{1}{z}
        \left(
            \tr( (\hcovmat_{\ell, -i} - zI)^{-1} - (\hcovmat_{\ell} - zI)^{-1} ) + \tr( (\hcovmat_{k, -i} - zI)^{-1} - (\hcovmat_{k} - zI)^{-1} )     
        \right) \\
&\leq \frac{2}{z^2},
\end{align*}
where the last second line uses the Von Neumann's trace inequality, aka Lemma \ref{lm:Von_neumann}. 

We then prove the last result. Because $\aspratio < \sampleratio$ and by Lemma \ref{lm:singularmat}, $\hcovmat_h$ and $\hcovmat_{h, -i}$ are invertible almost surely as $\npinfty$.  Using the Sherman-Morrison formula, aka Lemma \ref{lm:Sherman–Morrison}, we have for either $h = k, \ell$,  
\begin{align*}
\left| \tr((\hcovmat_{h})^{-1} - (\hcovmat_{h, -i})^{-1}) \right|
    &= \left| 
            \tr\left(
                \left( (\hcovmat_{h, - i})^{-1}  - \frac{(\hcovmat_{h, -i})^{-1}\weight_{h, i}x_i x_i^\T (\hcovmat_{h, -i})^{-1}/n  }{1 + \weight_{h, i} x_{i}\transp (\hcovmat_{h, -i })^{-1} x_{i} / \ndata}  \right) 
                - (\hcovmat_{h, -i})^{-1}  
                \right) \right| \\
    &= \tr\left(\frac{(\hcovmat_{h, -i})^{-1}\weight_{h,i} x_{i} x_{i}\transp  (\hcovmat_{h, -i})^{-1} / \ndata}{1 + \weight_{h,i} x_{i}\transp (\hcovmat_{h, -i})^{-1} x_{i} / \ndata} \right)\\
    &= \frac{ \weight_{h,i} \|(\hcovmat_{h, -i})^{-1} x_{i} \|_{2}^{2}/ \ndata}{1 + \weight_{h,i} x_{i}\transp (\hcovmat_{h, -i})^{-1} x_{i}/ \ndata}\\
    &\leq \frac{ \|(\hcovmat_{h, -i})^{-1} x_{i} \|_{2}^{2}}{ x_{i}\transp (\hcovmat_{h, -i})^{-1} x_{i}}\\
    &\leq \|(\hcovmat_{h, -i})^{-1}\|_{2}\\
    &= \frac{1}{\lambda_{\min}\left(\hcovmat_{h, -i} \right)} \\
    &\leq c_0 \left(1 - \sqrt{\aspratio/ \sampleratio}\right)^{-2} \qas 
\end{align*}
where the last line follows from Lemma~\ref{lm:lowerboundeigval}. 
Then following the same argument to the proof of the first result, we obtain for some constant $c$
\$
& \left| \tr((\hcovmat_{k})^{-1}(\hcovmat_{\ell})^{-1}) - \tr((\hcovmat_{k, -i})^{-1}(\hcovmat_{\ell, -i} )^{-1})\right|
\leq c\left(   1 - \sqrt{\gamma/\theta}   \right)^{-4}. 
\$
This finishes the proof.
\end{proof}

\section{Proofs for Section \ref{sec:pre}}

%%%%%%%%%%%%%%%%%%%%%%%%%%%%%%%%%%%%%%%%%%%%%%%%%%%%%%%%
%%%%%%%%%%%%%%%%%%%%%%Lemma 2.1%%%%%%%%%%%%%%%%%%%%%%%%%
%%%%%%%%%%%%%%%%%%%%%%%%%%%%%%%%%%%%%%%%%%%%%%%%%%%%%%%%

\subsection{Proof of Lemma~\ref{lm:biasvar}}

\begin{proof}[Proof of Lemma \ref{lm:biasvar}]
Using Definition~\eqref{eq:Bestimator}, we obtain 
\begin{align*}
   \estimator &= \frac{1}{\nresample} \sum_{k=1}^{\nresample} \estimatork\\
    &= \frac{1}{\nresample} \sum_{k=1}^{\nresample} (X\transp \sketchmat_{k}\transp \sketchmat_{k} X)\pinv X\transp \sketchmat_{k}\transp \sketchmat_{k} X \truesignal + (X\transp \sketchmat_{k}\transp \sketchmat_{k} X)\pinv X\transp \sketchmat_{k}\transp \sketchmat_{k} E.
\end{align*}

For the bias part,
\begin{align*}
    \biascondition\left(\estimator \right) &= \left\| \EE(\estimator \condition \variablecondition) -\truesignal\right\|_\Sigma^2\\
    &= \left\| \frac{1}{\nresample} \sum_{k=1}^{\nresample} (\hcovmat_{k}\pinv \hcovmat_{k} - I) \truesignal\right\|_\Sigma^2\\
    &= \frac{1}{\nresample^{2}} \sum_{k,\ell} \truesignal^{\top} \projmat_{k} \covmat \projmat_{\ell} \truesignal.
\end{align*}

For the variance part, since $\varepsilon_{i}$ are \iid, we have %\scomment{Revise?}
\begin{align*}
    \varcondition(\estimator)&= \tr\left[ \cov\left(\estimator \condition \variablecondition \right)\Sigma \right]\\
    &= \tr\left[ \cov\left(\frac{1}{\nresample} \sum_{k=1} (X\transp \sketchmat_{k}\transp \sketchmat_{k} X)\pinv X\transp \sketchmat_{k}\transp \sketchmat_{k} E \condition \variablecondition \right)\Sigma \right]\\
    &= \frac{\noiselev}{\nresample^{2}} \sum_{k,\ell}\tr\left(\frac{1}{n^{2}} \hcovmat_{k}\pinv X\transp \sketchmat_{k}^2 \sketchmat_{\ell}^2 X \hcovmat_{\ell}\pinv \covmat \right).
\end{align*}
\end{proof}

%%%%%%%%%%%%%%%%%%%%%%%%%%%%%%%%%%%%%%%%%%%%%%%%%%%%%%%%
%%%%%%%%%%%%%%%%%%%%%%Lemma 2.2%%%%%%%%%%%%%%%%%%%%%%%%%
%%%%%%%%%%%%%%%%%%%%%%%%%%%%%%%%%%%%%%%%%%%%%%%%%%%%%%%%

\subsection{Proof for Lemma \ref{lm:limitdistrMultinomial}}

\begin{proof}[Proof of Lemma \ref{lm:limitdistrMultinomial}]
%\scolor{This proof extends the proof  of Proposition 4.10 in~\citep{el2010} to show the almost sure convergence.}
Let $\pi_{1}, \dots, \pi_{\ndata}$ be independent and identically distributed random variables following a Poisson distribution with parameter 1. Define $\Pi_{\ndata} = \sum_{i=1}^{\ndata} \pi_{i}$. Using the same arguments as in the proof of Proposition 4.10 in~\citep{el2010}, we can establish that
    \begin{equation*}
        (\pi_{1}, \dots, \pi_{\ndata}) \,|\,   \Pi_{\ndata} = n \sim  {\rm Multinormial}(\ndata; 1/\ndata,\ldots, 1/\ndata).
    \end{equation*}
    Now, consider a bounded function $f$, and let $W \sim \text{Poisson}(1)$ be a random variable. We have
    \begin{align*}
        \Pr\left( \left| \frac{1}{\ndata} \sum_{i=1}^{\ndata} f(\weight_{i}) - \EE f(W) \right| \geq \epsilon \right) &= \Pr\left( \left| \frac{1}{\ndata} \sum_{i=1}^{\ndata} f(\pi_{i}) - \EE f(W) \right| \geq \epsilon \condition \Pi_{\ndata} = n \right)\\
        &\leq \Pr\left( \left| \frac{1}{\ndata} \sum_{i=1}^{\ndata} f(\pi_{i}) - \EE f(W) \right| \geq \epsilon \right)/P(\Pi_{\ndata} = n).
    \end{align*}
    Since $f$ is bounded and $\EE f(\pi_{i}) = \EE f(W)$, it can be shown that
    \begin{align*}
        \EE \left( \frac{1}{\ndata} \sum_{i=1}^{\ndata} f(\pi_{i}) - \EE f(W) \right)^4 &= \sum_{i=1}^{\ndata} \frac{1}{n^4} \EE \left( f(\pi_{i}) - \EE f(W) \right)^4 + \sum_{i \neq j} \frac{1}{n^4} \EE \left( f(\pi_{i}) - \EE f(W) \right)^{2} \EE \left( f(\pi_{j}) - \EE f(W) \right)^{2} \\&
        \leq O(n^{-2}).
    \end{align*}
    Considering that $\Pi_{\ndata}$ follows a Poisson distribution with parameter $n$, and $\Pr(\Pi_{\ndata} = \ndata) \sim 1/\sqrt{2 \pi \ndata}$, we can deduce that $\Pr\left( \left| \frac{1}{\ndata} \sum_{i=1}^{\ndata} f(\weight_{i}) - \EE f(W) \right| \geq \epsilon \right) \leq O(n^{-3/2})$. By applying the Borel–Cantelli lemma, we have    
    \begin{equation*}
        \int f(w) d{\tildelimitdistrweight}(w) \rightarrow \EE f(W) \qas
    \end{equation*}
    as $\npinfty$.
\end{proof}

%%%%%%%%%%%%%%%%%%%%%%%%%%%%%%%%%%%%%%%%%%%%%%%%%%%%%%%%
%%%%%%%%%%%%%%%%%%%%%%Lemma 2.3%%%%%%%%%%%%%%%%%%%%%%%%%
%%%%%%%%%%%%%%%%%%%%%%%%%%%%%%%%%%%%%%%%%%%%%%%%%%%%%%%%

\subsection{Proof of Lemma~\ref{lm:biasvar-beta}}
\begin{proof}[Proof of Lemma \ref{lm:biasvar-beta}]
By  Lemma~\ref{lm:biasvar}, we have 
\begin{equation*}
    \biascondition\left(\estimator \right) = \frac{1}{\nresample^{2}} \sum_{k,\ell} \truesignal^{\top} \projmat_{k} \covmat \projmat_{\ell} \truesignal.
\end{equation*}
Since the eigenvalues of the projection matrices $\projmat_{k} = I-\hcovmat_{k}\pinv \hcovmat_{k}$ are either zero or one, and $\|\Sigma\|_2\leq C$ for all $\ndim$, the norms of $\projmat_{k} \covmat \projmat_{\ell}$  are uniformly bounded. The desired result then follows from Lemma~\ref{lm:quadconcentration}.
\end{proof}
%{\color{red}
%\$
%?=? \frac{\sigma^2}{\nresample^{2}} \sum_{i,j}\tr\left( n^{2} \hcovmat_{i}\pinv X\transp \weightmat_{i} \weightmat_{j} X \hcovmat_{j}\pinv \covmat \right)
%\$
%}

\section{Proofs for Section \ref{sec:isotropic}}
%%%%%%%%%%%%%%%%%%%%%%%%%%%%%%%%%%%%%%%%%%%%%%%%%
%%%%%%%%%%%%%%%%Isotropic%%%%%%%%%%%%%%%%%%%%%%%%
%%%%%%%%%%%%%%%%%%%%%%%%%%%%%%%%%%%%%%%%%%%%%%%%%

This section proves the results in Section \ref{sec:isotropic}.

\subsection{Proof of Lemma \ref{lemma:unique_sketching_iso}}

Lemma \ref{lemma:unique_sketching_iso}  follows from the following stronger result. % proving a stronger result as follows.  
\begin{lemma}\label{lm:uniqueMPlawsketching}
    When $\aspratio < \sampleratio$, $\aspratio \stieltjes_{1}(0)$ is the unique positive solution to  equation~\eqref{eq:uniqueMPlawsketching}. 
\end{lemma}

\begin{proof}[Proof of Lemma \ref{lm:uniqueMPlawsketching}]
We first prove that $\gamma m_1(0)$ is a solution to \eqref{eq:uniqueMPlawsketching}.  Because $\aspratio<\sampleratio$, we apply Lemma \ref{lm:MPlawsketching} and obtain
\$
1 &= \stieltjes_{1}(0) \cdot \EE_{\limitdistrweight}\left[ \frac{\weight}{1 + \aspratio \weight \stieltjes_{1}(0)} \right], 
\$
which is equivalent to 
\$
1 &= \frac{1}{\aspratio} \left( 1 - \EE_{\limitdistrweight}\left[ \frac{1}{1 + \aspratio \weight \stieltjes_{1}(0)} \right]\right).
\$
Comparing the equality above with equation \eqref{eq:uniqueMPlawsketching},  we conclude that $\aspratio \stieltjes_{1}(0)$ is a solution to equation~\eqref{eq:uniqueMPlawsketching}. 
    
 We then prove the uniqueness of the solution in the positive half line.   Let 
 \$
 f(x) = \EE_{\limitdistrweight}\left[ \frac{1}{1 + \weight x} \right].
 \$
Then $f(x)$ is a continuous and decreasing function on the interval $[0, +\infty)$. Additionally, we have $f(0) = 1 > 1 - \aspratio$, $\lim_{x \rightarrow +\infty} f(x) = 0 < 1 - \aspratio$. Therefore, $\aspratio \stieltjes_{1}(0)$ is the unique positive solution to equation~\eqref{eq:uniqueMPlawsketching}.

\end{proof}

\subsection{Proof of Theorem~\ref{thm:isosketching}}

\begin{proof}[Proof of Theorem \ref{thm:isosketching}]

Theorem \ref{thm:isosketching} is a special case of Theorem~\ref{thm:corsketching}. In the underparameterized regime,  the limiting  risk in Theorem~\ref{thm:corsketching} is independent of the covariance matrix and thus is the same in the isotropic case.  It suffices to derive the limiting risk of the sketched estimator in the overparameterized regime. 
%Thus  thereby requiring our focus solely on the over-parameterized regime. 
Recall that $\delta_x$ is the Dirac measure at $x$. By setting $\covmat = I$ and $\limitdistr = \delta_{1}$, the self-consistent equation~\eqref{eq:fixedpointsketching} reduces to 
\begin{align*}
    \compstieltjes(z) = \left(-z + \frac{\aspratio}{\sampleratio}  \frac{1}{1 + \compstieltjes(z)} \right)^{-1},
\end{align*}
which further gives  that $\compstieltjes(0) = \sampleratio/(\aspratio - \sampleratio)$ and
\begin{align*}
    \compstieltjes(z) = \frac{\aspratio/\sampleratio - 1 - z - \sqrt{(z - \aspratio/\sampleratio + 1)^{2} - 4z}}{2z}, \quad z < 0.
\end{align*}
Furthermore, by taking the derivative of $\compstieltjes(z)$, we have
\begin{align*}
    \compstieltjes'(z) = \frac{ \left(-1 - \frac{z - \aspratio/\sampleratio - 1}{\sqrt{(z - \aspratio/\sampleratio + 1)^{2} - 4z}} \right)2z - 2\left(\aspratio/\sampleratio - 1 - z - \sqrt{(z - \aspratio/\sampleratio + 1)^{2} - 4z} \right)}{4z^{2} }.
\end{align*}
By applying L'H\^{o}pital's rule, we get
\begin{align*}
    \compstieltjes'(0) &= \lim_{\zinfty} \frac{ 2\left(-1 - \frac{z - \aspratio/\sampleratio - 1}{\sqrt{(z - \aspratio/\sampleratio + 1)^{2} - 4z}} \right) + 2z\left( - \frac{1}{\sqrt{(z - \aspratio/\sampleratio + 1)^{2} - 4z}} + \frac{\left(z - \aspratio/\sampleratio - 1\right)^{2}}{\left((z - \aspratio/\sampleratio + 1)^{2} - 4z\right)^{3/2}}\right) +2 +2\frac{z - \aspratio/\sampleratio - 1}{\sqrt{(z - \aspratio/\sampleratio + 1)^{2} - 4z}}}{8z}\\
    &= \lim_{\zinfty} \frac{1}{4} \left( - \frac{1}{\sqrt{(z - \aspratio/\sampleratio + 1)^{2} - 4z}} + \frac{\left(z - \aspratio/\sampleratio - 1\right)^{2}}{\left((z - \aspratio/\sampleratio + 1)^{2} - 4z\right)^{3/2}}\right)\\
    & = \frac{\sampleratio}{4(\aspratio - \sampleratio)}\left( -1 + \frac{\sampleratio^{2}}{(\aspratio - \sampleratio)^{2}}\frac{(\aspratio + \sampleratio)^{2}}{\sampleratio^{2}}\right)\\
    &= \frac{\sampleratio^{2} \aspratio}{(\aspratio - \sampleratio)^{3}}.
\end{align*}
Therefore, the limiting  risk satisfies almost surely that
\begin{align*}
\lim_{\nresample \rightarrow + \infty} \lim_{\npinfty} \riskcondition(\hat{\beta}^{1})
&\rightarrow \frac{\signallev \sampleratio}{\aspratio \compstieltjes(0)} + \noiselev \left(\frac{\compstieltjes'(0)}{\compstieltjes(0)^{2}} - 1\right)\\
&= \frac{\signallev \sampleratio}{\aspratio} \frac{\aspratio - \sampleratio}{\sampleratio} + \noiselev \left( \frac{\sampleratio^{2} \aspratio}{(\aspratio - \sampleratio)^{3}} \frac{(\aspratio - \sampleratio)^{2}}{\sampleratio^{2}} - 1\right)\\
&= \signallev \frac{\aspratio - \sampleratio}{\aspratio} + \noiselev  \frac{\sampleratio}{\aspratio - \sampleratio}. 
\end{align*}

\end{proof}

\subsection{Proof of Corollary~\ref{cor:isointerpolating}}

\begin{proof}[Proof of Corollary~\ref{cor:isointerpolating}]
    Recall that $\sampleratio = 1 - \limitdistrweight(\{0 \})$. We rewrite equation~\eqref{eq:uniqueMPlawsketching} as:
    \begin{align*}
        1-\aspratio &= \int \frac{1}{1+c t}  \limitdistrweight(dt)\\
        &=  1 - \sampleratio + \int \frac{1}{1+c t} 1\left({t \neq 0}\right)  \limitdistrweight(dt).
    \end{align*}
    Then, we obtain
    \begin{equation*}
        \lim_{\aspratio/\sampleratio  \nearrow 1} \int \frac{1}{1+c t} 1\left({t \neq 0}\right)  \limitdistrweight(dt) = 0.
    \end{equation*}
    Since $\limitdistrweight$ satisfies Assumption~\ref{Assume:multip} and $c$ is positive, we have
    \begin{equation*}
        0 \leq \lim_{\aspratio/\sampleratio  \nearrow 1} \int \frac{1}{(1+ c t )^{2}} 1\left({t \neq 0}\right)  \limitdistrweight(dt) \leq \lim_{\aspratio/\sampleratio  \nearrow 1} \int \frac{1}{1+c t} 1\left({t \neq 0}\right)  \limitdistrweight(dt) = 0.
    \end{equation*}
    Therefore,
    \begin{align*}
        \lim_{\aspratio/\sampleratio  \nearrow 1} f(\aspratio) &= \lim_{\aspratio/\sampleratio  \nearrow 1} \int \frac{1}{(1+ c t )^{2}}  \limitdistrweight(dt)\\
        &= \lim_{\aspratio/\sampleratio  \nearrow 1} \left(1 - \sampleratio + \int \frac{1}{(1+ c t )^{2}} 1\left({t \neq 0}\right)  \limitdistrweight(dt) \right)\\
        &= 1 - \aspratio.
    \end{align*}
\end{proof}

\subsection{Proof of Corollary~\ref{cor:isoexample}}

\begin{proof}[Proof of Corollary~\ref{cor:isoexample}]
Following Theorem~\ref{thm:isosketching}, it suffices to calculate $f(\aspratio) = \int \frac{1}{(1+ c t )^{2}} d \limitdistrweight(t)$ for each case. 
    \begin{enumerate}
        \item[(i)] \textbf{Full-sample.} It is obvious that $\limitdistrweight(\{ 1\} ) = 1$. From equation~\eqref{eq:uniqueMPlawsketching}, we have 
        \begin{align*}
            1 - \aspratio = \int \frac{1}{1+c t}   \limitdistrweight(dt) = \frac{1}{1 + c}.
        \end{align*}
        Therefore,
            \begin{align*}
                f(\aspratio) &= \int \frac{1}{(1+ c t )^{2}}  \limitdistrweight(dt)\\
                &= \frac{1}{(1 + c)^{2}}\\
                &= (1 - \aspratio)^{2}.
            \end{align*}
            
        \item[(ii)] \textbf{Bernoulli multipliers.} Since the multipliers $\weight_{i,j}$ are i.i.d, by the Glivenko–Cantelli theorem, $\limitdistrweight \sim \text{Bernoulli}(\sampleratio)$. Then, from equation~\eqref{eq:MPlawsketching},
        \begin{equation*}
            1 - \aspratio = 1 - \sampleratio + \frac{\sampleratio}{1 + c}.
        \end{equation*}
        Therefore,
            \begin{align*}
                f(\aspratio) &= \int \frac{1}{(1+ c t )^{2}}  \limitdistrweight(dt)\\
                &= 1 - \sampleratio + \sampleratio\frac{1}{(1 + c)^{2}}\\
                &= 1 - \sampleratio + \sampleratio (1 - \aspratio/\sampleratio)^{2}.
            \end{align*}
        
        \item[(iii)] \textbf{Multinomial multipliers.} By Lemma~\ref{lm:limitdistrMultinomial}, we have $\limitdistrweight = \text{Poisson}(1)$. Then, from equation~\eqref{eq:uniqueMPlawsketching},
        \begin{equation*}
            1 - \aspratio = 1 - (1 - 1/e) + \int \frac{1}{1+c t} 1\left({t \neq 0}\right)  \limitdistrweight(dt).
        \end{equation*}
        Therefore,
        \begin{align*}
             f(\aspratio) &= \int \frac{1}{(1+ c t )^{2}}  \limitdistrweight(dt)\\
             &= 1 - (1 - 1/e) + \int \frac{1}{(1+c t)^{2}} 1\left({t \neq 0}\right)  \limitdistrweight(dt)\\
             &\geq 1 - (1 - 1/e) + \left(\int \frac{1}{1+c t} 1\left({t \neq 0}\right)  \limitdistrweight(dt)\right)^{2}/ \int  1\left({t \neq 0}\right)  \limitdistrweight(dt)\\
             &= 1 - (1 - 1/e) + (1 - 1/e)(1 - \aspratio/(1 - 1/e)),
        \end{align*}
        where the third line follows from Holder's inequality.
    \item[(iv)] \textbf{Jackknife sketching.} This is the same as item (i). 
    \end{enumerate}
\end{proof}

\subsection{Proof of Corollary~\ref{coro:optimal_sketching}}
\begin{proof}[Proof of Corollary~\ref{coro:optimal_sketching}]
    For any probability measure $\limitdistrweight$ that satisfies  Assumption~\ref{Assume:multip} with a sampling rate of $\sampleratio$, from equation~\eqref{eq:MPlawsketching}, 
    \begin{equation*}
            1 - \aspratio = 1 - \sampleratio + \int \frac{1}{1+c t} 1\left({t \neq 0}\right)  \limitdistrweight(dt).
        \end{equation*}
    Then, 
    \begin{align*}
         f(\aspratio) &= \int \frac{1}{(1+ c t )^{2}}  \limitdistrweight(dt)\\
         &= 1 - \sampleratio + \int \frac{1}{(1+c t)^{2}} 1\left({t \neq 0}\right)  \limitdistrweight(dt)\\
         &\geq 1 - \sampleratio + \left(\int \frac{1}{1+c t} 1\left({t \neq 0}\right)  \limitdistrweight(dt)\right)^{2}/ \int  1\left({t \neq 0}\right)  \limitdistrweight(dt)\\
         &= 1 - \sampleratio + \sampleratio(1 - \aspratio/\sampleratio),
    \end{align*}
    where the third line follows from Holder's inequality. The inequality becomes equality if and only if there exist real numbers $\alpha, \beta \geq 0$, not both of them zero, such that
    \begin{equation*}
       \alpha \frac{1}{(1+c t)^{2}} 1\left({t \neq 0}\right)  =  \beta 1\left({t \neq 0}\right) \quad \limitdistrweight-\rm{almost~everywhere.}
    \end{equation*}
    Therefore $\limitdistrweight = (1 - \sampleratio)\delta_{0} + \sampleratio\delta_{1}$ 
    satisfies the above equation and minimizes the limiting risk of $\hat\beta^1$ in Theorem \ref{thm:isosketching}  among all choices of multipliers satisfying Assumption~\ref{Assume:multip} with $B=1$ and a sampling rate of $\sampleratio$. %\scomment{can take $\limitdistrweight = (1 - \sampleratio)\delta_{0} + \sampleratio\delta_{a}$ for any $a \geq 0$.}
\end{proof}

%%%%%%%%%%%%%%%%%%%%%%%%%%%%%%%%%%
%%%%%%%%%%%%Theorem 3.6%%%%%%%%%%%
%%%%%%%%%%%%%%%%%%%%%%%%%%%%%%%%%%

\subsection{Proof of Theorem~\ref{thm:isoboostrap}}\label{sec:Proofisobagging}

We first prove  Theorem \ref{thm:isoboostrap} directly and then prove it as a special case of Theorem \ref{thm:corrbagging}. 

\subsubsection{A direct proof}
\begin{proof}[Proof of Theorem \ref{thm:isoboostrap}]

For notational simplicity, we  %introduce notations to simplify the expressions and 
drop the explicit dependence on $\{X, \cW, \estimator\}$, and define %the following notations:
\begin{align*}
    B_{k,\ell} &= \frac{\signallev}{\ndim} \tr \left(\left(I - \hcovmat_{k}\pinv \hcovmat_{k} \right) \left(I - \hcovmat_{\ell}\pinv \hcovmat_{\ell} \right) \right),\\
    V_{k,\ell} &= \frac{\noiselev}{\ndata^{2}} \tr\left(\hcovmat_{k}\pinv X\transp \sketchmat_{k}^2 \sketchmat_{\ell}^2 X \hcovmat_{\ell}\pinv \right).
\end{align*} 
Applying Lemma~\ref{lm:biasvar} and Lemma~\ref{lm:biasvar-beta}, we can rewrite the out-of-sample prediction risk as
\begin{equation*}
    \lim_{\nresample \rightarrow + \infty} \lim_{\npinfty} \riskcondition(\estimator) = \lim_{\nresample \rightarrow + \infty} \lim_{\npinfty} \frac{1}{\nresample^{2}} \sum_{k,\ell}^{\nresample} \left( B_{k,\ell} + V_{k,\ell} \right).
\end{equation*}
The terms with $k = \ell$ correspond to the bias and variance terms for the sketched least square estimators. According to Theorem~\ref{thm:isosketching}, these terms converge almost surely to constants for any $\aspratio/ \sampleratio \neq 1$ as $\npinfty$. Therefore, under Assumption~\ref{Assume:multip}, taking the limit $\nresample \rightarrow \infty$, we obtain  %,that 
\begin{align}\label{eq:baggedbiasvariance}
    \lim_{\nresample \rightarrow + \infty} \lim_{\npinfty} \riskcondition(\estimator) = \lim_{\npinfty} \left(B_{k,\ell} + V_{k,\ell}\right) {\quad \as}
\end{align}
for $\aspratio/ \sampleratio \neq 1$ and $k \neq \ell$. 

In what follows, we will compute the limits $\lim_{n\to\infty} V_{k,\ell}$ using Lemma~\ref{lm:Isobootvar_under} for the underparameterized regime and Lemma~\ref{lm:Isobootvar_over} for the overparameterized regime for $k\ne \ell$.
Additionally, we will determine the limit $\lim_{n\to\infty} B_{k,\ell}$ in Lemma~\ref{lm:isobootbias} for $k \neq \ell$.

%and $\lim_{\npinfty} B_{k,\ell}$ in Lemma~\ref{lm:isobootbias} for $k \neq \ell$, respectively. 

%\subsubsection{Variance under isotropic features}
%\paragraph{Variance under isotropic features}

\begin{lemma}[Underparameterized variance under isotropic features]\label{lm:Isobootvar_under}
Assume Assumptions~\ref{Assume:highdim}-\ref{Assume:multip}, $\covmat = I$, and $\aspratio < \sampleratio$. For $k \neq \ell$, the variance term $V_{k,\ell}$ of the estimator $\estimator$ satisfies %the following convergence almost surely:
    \begin{equation*}
        \lim_{\npinfty} V_{k,\ell} = \noiselev \frac{\aspratio}{1-\aspratio} {\quad \as}
    \end{equation*}
\end{lemma}

\begin{lemma}[Overparameterized variance  under isotropic features]\label{lm:Isobootvar_over}
Assume Assumptions~\ref{Assume:highdim}-\ref{Assume:multip}, $\covmat = I$, and $\aspratio > \sampleratio$. For $k \neq \ell$, the variance term $V_{k,\ell}$ of the estimator $\estimator$ satisfies %the following convergence almost surely:
    \begin{equation*}
        \lim_{\npinfty} V_{k,\ell} = \noiselev  \frac{\sampleratio^2}{\aspratio - \sampleratio^{2}} {\quad \as} 
    \end{equation*}
\end{lemma}

%%%%%%%%%%%%%%%%%%%%%%%%%%%%%%%%%%%%%%%%%%%%%%%%%%%%%%
%%%%%%%%%%%%Bias under isotropic features%%%%%%%%%%%%%
%%%%%%%%%%%%%%%%%%%%%%%%%%%%%%%%%%%%%%%%%%%%%%%%%%%%%%

%\paragraph{Bias under isotropic features}

\begin{lemma}[Bias under isotropic features]\label{lm:isobootbias}
Assume Assumptions~\ref{Assume:highdim}-\ref{Assume:beta4+mom} and $\covmat = I$. For $k \neq \ell$, the bias term $B_{k,\ell}$ of the estimator $\estimator$ satisfies %the following convergence almost surely:
    \begin{equation*}
        \lim_{\nresample \rightarrow + \infty} \lim_{\npinfty} B_{k,\ell}(\estimator)
        =
        \begin{cases}
        0,& \aspratio < \sampleratio \\
        \signallev \frac{\left(\aspratio - \sampleratio \right)^2}{\aspratio \left(\aspratio - \sampleratio^{2} \right)},& \aspratio > \sampleratio 
        \end{cases}{\quad \as}
    \end{equation*}
\end{lemma}

Putting the above lemmas together finishes the proof. 

\end{proof}

%%%%%%%%%%%%%%%%%%%%%%%%%%%%%%%%%%%%%%%%%%%%%%%%
%%%%%%%%%%%%%%%%Isotropic case%%%%%%%%%%%%%%%%%%
%%%%%%%%%%%%%%%%%%%%%%%%%%%%%%%%%%%%%%%%%%%%%%%%

\subsubsection{As the isotropic case of Theorem \ref{thm:corrbagging}}

\begin{proof}[Proof of Theorem \ref{thm:isoboostrap}]
In this subsection, we derive the results for isotropic features as a special case of 
 Theorem~\ref{thm:corrbagging}. In the underparameterized regime $\aspratio < \sampleratio$, the risk is independent of the covariance matrix and thus is the same as that for the correlated case. Therefore, we focus on the overparameterized case, where $\aspratio > \sampleratio$. Since $\covmat = I$, we obtain $\limitdistr = \delta_{1}$. Equation~\eqref{eq:fixedpointsketching} at $z = 0$ simplifies to
\begin{equation*}
    \compstieltjes(0) = \left(\frac{\aspratio}{\sampleratio}  \frac{1}{1 + \compstieltjes(0)} \right)^{-1}.
\end{equation*}
This equation has a unique positive solution given by $\compstieltjes(0) = \sampleratio/(\aspratio - \sampleratio) > 0$. Consequently, we obtain  $\ensambleweight(0) = (1 - \sampleratio)\compstieltjes(0) = \sampleratio(1 - \sampleratio)/(\aspratio - \sampleratio)$. 

 Similarly, since $\covmat = I$, equation~\eqref{eq:fixpointbootstrap} simplifies to 
\begin{equation*}
    \tcompstieltjes(z) = \left( -z + \frac{\aspratio}{\sampleratio^{2}} \frac{t}{1 + \tcompstieltjes(z)t} \right)^{-1},
\end{equation*}
where $t = (1 + \ensambleweight(0))\inv = (\aspratio - \sampleratio)/(\aspratio - \sampleratio^{2})$. After some calculations, we obtain that the unique positive solution to equation~\eqref{eq:fixpointbootstrap} is  $\tcompstieltjes(0) = \sampleratio^{2}/(\aspratio - \sampleratio)$ when $z=0$, and
\begin{equation*}
    \tcompstieltjes(z) = \frac{(\aspratio - \sampleratio^{2})t/\sampleratio - z - \sqrt{\left(z - {(\aspratio - \sampleratio^{2})}t/\sampleratio^{2}\right)^{2} - 4zt}}{2zt}, ~\text{when}~ z < 0.
\end{equation*}
Then, by taking derivative of $\tcompstieltjes(z)$, we have
\begin{align*}
    \tcompstieltjes'(z) = \left(\left(-1 - \frac{z - \frac{\aspratio - \sampleratio^{2}}{\sampleratio^{2}}t - 2t}{\sqrt{(z - \frac{\aspratio - \sampleratio^{2}}{\sampleratio^{2}}t)^{2} - 4zt}} \right)2z - 2\left(\frac{\aspratio - \sampleratio^{2}}{\sampleratio^{2}}t - z - \sqrt{(z - \frac{\aspratio - \sampleratio^{2}}{\sampleratio^{2}}t)^{2} - 4zt} \right)\right)/ (4z^{2} t).
\end{align*}
By applying L'H\^opital's rule, we obtain
\begin{align*}
    \tcompstieltjes'(0) &= \lim_{\zinfty} \frac{ 2\left(-1 - \frac{z - \frac{\aspratio - \sampleratio^{2}}{\sampleratio^{2}}t - 2t}{\sqrt{(z - \frac{\aspratio - \sampleratio^{2}}{\sampleratio^{2}}t)^{2} - 4zt}} \right) + 2z\left( - \frac{1}{\sqrt{(z - \frac{\aspratio - \sampleratio^{2}}{\sampleratio^{2}}t)^{2} - 4zt}} + \frac{\left(z - \frac{\aspratio - \sampleratio^{2}}{\sampleratio^{2}}t - 2t\right)^{2}}{\left((z - \frac{\aspratio - \sampleratio^{2}}{\sampleratio^{2}}t)^{2} - 4zt \right)^{3/2}}\right) +2 +2\frac{z - \frac{\aspratio - \sampleratio^{2}}{\sampleratio^{2}}t - 2t}{\sqrt{(z - \frac{\aspratio - \sampleratio^{2}}{\sampleratio^{2}}t)^{2} - 4zt}}}{8zt}\\
    &= \lim_{\zinfty} \frac{1}{4t} \left( - \frac{1}{\sqrt{(z - \frac{\aspratio - \sampleratio^{2}}{\sampleratio^{2}}t)^{2} - 4zt}} + \frac{\left(z - \frac{\aspratio - \sampleratio^{2}}{\sampleratio^{2}}t - 2t\right)^{2}}{\left((z - \frac{\aspratio - \sampleratio^{2}}{\sampleratio^{2}}t)^{2} - 4zt \right)^{3/2}}\right)\\
    & = \frac{1}{4t} \left(- \frac{\sampleratio^{2}}{(\aspratio - \sampleratio^{2})t} + \frac{\sampleratio^{6}}{(\aspratio - \sampleratio^{2})^{3}t^{3}} \frac{(\aspratio + \sampleratio^{2})^{2}}{\sampleratio^{4}}t^{2} \right)\\
    &= \frac{(\aspratio - \sampleratio^{2}) \sampleratio^{2}}{4(\aspratio - \sampleratio)^{2}} \left( -1 + \frac{(\aspratio + \sampleratio^{2})^{2}}{(\aspratio - \sampleratio^{2})^{2}} \right)\\
    &= \frac{\aspratio \sampleratio^{4}}{(\aspratio - \sampleratio)^{2} (\aspratio - \sampleratio^{2})}.
\end{align*}
Therefore, the prediction risk satisfies, almost surely,
\begin{align*}
    \lim_{\nresample \rightarrow + \infty} \lim_{\npinfty} \riskcondition(\estimator)
&= \frac{\signallev \sampleratio}{\aspratio \compstieltjes(0)} - \frac{\signallev(1 - \sampleratio)}{\aspratio \compstieltjes(0)} \left(\frac{\tcompstieltjes'(0)}{\tcompstieltjes(0)^{2}} - 1 \right) + \noiselev \left(\frac{\tcompstieltjes'(0)}{\tcompstieltjes(0)^{2}} - 1\right)\\
&= \frac{\signallev \sampleratio}{\aspratio} \frac{\aspratio - \sampleratio}{\sampleratio} - \frac{\signallev(1 - \sampleratio)}{\aspratio} \frac{\aspratio - \sampleratio}{\sampleratio} \left( \frac{\aspratio \sampleratio^{4}}{(\aspratio - \sampleratio)^{2} (\aspratio - \sampleratio^{2})} \frac{(\aspratio - \sampleratio)^{2}}{\sampleratio^{4}} - 1\right)\\
& \qquad + \noiselev \left( \frac{\aspratio \sampleratio^{4}}{(\aspratio - \sampleratio)^{2} (\aspratio - \sampleratio^{2})} \frac{(\aspratio - \sampleratio)^{2}}{\sampleratio^{4}} - 1\right)\\
& = \signallev \frac{\left(\aspratio - \sampleratio \right)^2}{\aspratio \left(\aspratio - \sampleratio^{2} \right)} + \noiselev  \frac{\sampleratio^2}{\left(\aspratio - \sampleratio^{2} \right)}.
\end{align*}
Thus Theorem~\ref{thm:corrbagging}  reduces to  Theorem~\ref{thm:isoboostrap} when $\covmat= I$. 
\end{proof}

%%%%%%%%%%%%%%%%%%%%%%%%%%%%%%%%%%%%
%%%%%%%%%%%Lemma 3.7%%%%%%%%%%%%%%%
%%%%%%%%%%%%%%%%%%%%%%%%%%%%%%%%%%%%

\subsection{Proof of Lemma~\ref{coro:risk_bound_iso}}

\begin{proof}[Proof of Lemma~\ref{coro:risk_bound_iso}]
%Following from Theorem~\ref{thm:isoboostrap}, we can observe that 
For $\aspratio < \sampleratio$, we have obviously that
\begin{equation*}
\lim_{\nresample \rightarrow + \infty} \lim_{\npinfty} \riskcondition(\estimator) = \noiselev \frac{\aspratio}{1-\aspratio} 
\leq  \noiselev \frac{\sampleratio}{1 - \sampleratio}.     
\end{equation*}

In the case where $\aspratio > \sampleratio$, we first calculate  the derivatives of the limiting  risk respect to $\aspratio$:
\begin{align*}
    &\frac{d}{d\aspratio}\left(\signallev \frac{\left(\aspratio - \sampleratio \right)^2}{\aspratio \left(\aspratio - \sampleratio^{2} \right)} + \noiselev  \frac{\sampleratio^2}{\aspratio - \sampleratio^{2}} \right)
    = \signallev \frac{(\aspratio - \sampleratio)((2\sampleratio - \sampleratio^{2})\aspratio - \sampleratio^{3})}{\aspratio^{2}(\aspratio - \sampleratio^{2})^{2}} - \noiselev \frac{\aspratio^{2}\sampleratio^{2}}{\aspratio^{2}(\aspratio - \sampleratio^{2})^{2}}
    =: \frac{f(\aspratio)}{\aspratio^2(\aspratio - \sampleratio^2)^2}. 
\end{align*}
It can be shown that the function 
\$
 f(\aspratio) = \signallev (\aspratio - \sampleratio)((2\sampleratio - \sampleratio^{2})\aspratio - \sampleratio^{3}) - \noiselev \aspratio^{2}\sampleratio^{2} 
\$ satisfies $f(\sampleratio) = - \noiselev \sampleratio^{4} < 0$ and has at most one root on $[\sampleratio, + \infty]$. Thus, $\lim_{\nresample \rightarrow + \infty} \lim_{\npinfty} \riskcondition(\estimator)$ is bounded from above by the larger value between its values at $\aspratio = \sampleratio$ and $\aspratio = +\infty$. The result follows from that $\lim_{\aspratio \rightarrow +\infty} \lim_{\nresample \rightarrow + \infty} \lim_{\npinfty} \riskcondition(\estimator) = \signallev$.
\end{proof}

\subsection{Proof of Lemma~\ref{lm:ridgeequivlence}}

\begin{proof}[Proof of Lemma~\ref{lm:ridgeequivlence}]
First note that  $\covmat = I$. We begin by writing
\begin{align*}
    \EE \left[\| \estimator - \hat{\beta}_{\lambda}\|_{2}^{2} \condition \variablecondition \right] &= \EE \left[\| \estimator - \truesignal + \truesignal - \hat{\beta}_{\lambda}\|_{2}^{2} \condition \variablecondition \right]\\
     &=\EE \left[\| \estimator - \truesignal\|_{2}^{2} \condition \variablecondition \right] - 2\EE \left[ (\estimator - \truesignal)\transp (\hat{\beta}_{\lambda} - \truesignal) \condition \variablecondition \right]\\
     &\qquad +  \EE \left[\| \hat{\beta}_{\lambda} - \truesignal\|_{2}^{2} \condition X, \truesignal \right], 
\end{align*}
where the first and the third term represent the risk of the bagged least square estimator and the ridge regression estimator, respectively.

We decompose the second term of the above equation as 
\begin{align}
    &\EE \left[ (\estimator - \truesignal)\transp (\hat{\beta}_{\lambda} - \truesignal) \condition \variablecondition \right] \nn \\
    &= \frac{\signallev}{\nresample} \sum_{k=1}^{\nresample} \truesignal \transp  \left( X\transp \sketchmat_{k}\transp \sketchmat_{k} X (X\transp \sketchmat_{k}\transp \sketchmat_{k} X)\pinv - I \right)\left(  (X\transp  X + n \lambda I)\inv X\transp X - I\right) \truesignal \notag\\ 
    &\qquad + \frac{\noiselev}{\nresample} \sum_{k=1}^{\nresample} \tr \left( (X\transp  X + n \lambda I)\inv X\transp \sketchmat_{k}\transp \sketchmat_{k} X (X\transp \sketchmat_{k}\transp \sketchmat_{k} X)\pinv \right) \notag \\ 
    &=: \frac{1}{\nresample} \sum_{k=1}^{\nresample} \left( B_{k, \lambda} + V_{k ,\lambda}\right),\label{eq:ridgeequivdecomp}
\end{align}
which follows from the same argument in the proof of Lemma~\ref{lm:biasvar} and thus is omitted. Under Assumption~\ref{Assume:beta4+mom}, the bias term $B_{k, \lambda}$ satisfies
\begin{equation*}
    \lim_{\npinfty} B_{k, \lambda} = \lim_{\npinfty} \tr \left( \left( X\transp \sketchmat_{k}\transp \sketchmat_{k} X (X\transp \sketchmat_{k}\transp \sketchmat_{k} X)\pinv - I \right)\left(  (X\transp  X + n \lambda I)\inv X\transp X - I\right) \right)/\ndim {\quad \as}
\end{equation*}
which follows from the same argument in the proof of Lemma~\ref{lm:biasvar-beta}.%We omit the proof since it mirrors the previous derivation.

Let $\lambda = (1 - \sampleratio)(\aspratio/\sampleratio - 1) \vee 0$. Then we can prove this lemma by showing that the ridge regression estimator and the bagged  estimator share the same limiting risk as in Lemma~\ref{lm:ridgerisk}, 
\$
\lim_{\npinfty} V_{k,\lambda} = \lim_{\nresample \rightarrow \infty} \lim_{\npinfty} \varcondition {\quad \as}
\$ 
as in Lemma~\ref{lm:ridgeequivar}, and 
\$
\lim_{\npinfty} B_{k, \lambda} = \lim_{\nresample \rightarrow \infty} \lim_{\npinfty} \biascondition {\quad \as}
\$ 
as in Lemma~\ref{lm:ridgeequibias}. We collect these lemmas below, with their proofs deferred later. 

\begin{lemma}\label{lm:ridgerisk}
    Let $\lambda = (1 - \sampleratio)(\aspratio/\sampleratio - 1) \vee 0$. Assume Assumptions~\ref{Assume:highdim}-\ref{Assume:Covdistri}, Assumption~\ref{Assume:beta4+mom}, and $\covmat = I$. Then, almost surely,
    \begin{equation*}
        \lim_{\npinfty} \EE \left[\| \hat{\beta}_{\lambda} - \truesignal\|_{2}^{2} \condition X, \truesignal \right] = \begin{cases}
                \noiselev \frac{\aspratio}{1-\aspratio},& \aspratio < \sampleratio , \\
                \signallev \frac{\left(\aspratio - \sampleratio \right)^2}{\aspratio \left(\aspratio - \sampleratio^{2} \right)} + \noiselev  \frac{\sampleratio^2}{\aspratio - \sampleratio^{2}},& \aspratio > \sampleratio.
             \end{cases}
    \end{equation*}
\end{lemma}

\begin{lemma}\label{lm:ridgeequivar}
    Let $\lambda = (1 - \sampleratio)(\aspratio/\sampleratio - 1) \vee 0$. Assume Assumption~\ref{Assume:highdim}-\ref{Assume:multip}, and $\covmat = I$. Then, almost surely
    \begin{equation*}
        \lim_{\npinfty} V_{k ,\lambda} = \begin{cases}
                \noiselev \frac{\aspratio}{1-\aspratio},& \aspratio < \sampleratio , \\
                 \noiselev  \frac{\sampleratio^2}{\aspratio - \sampleratio^{2}},& \aspratio > \sampleratio.
             \end{cases}
    \end{equation*}
\end{lemma}

%%%%%%%%%%%%%%%%%%%%%%%%%%%%%%%%%%%%%%%%
%%%%%%%%%%%%%%ridgeequibias%%%%%%%%%%%%%
%%%%%%%%%%%%%%%%%%%%%%%%%%%%%%%%%%%%%%%%

\begin{lemma}\label{lm:ridgeequibias}
    Let $\lambda = (1 - \sampleratio)(\aspratio/\sampleratio - 1) \vee 0$. Assume Assumption~\ref{Assume:highdim}-\ref{Assume:beta4+mom}, and $\covmat = I$. Then, almost surely
    \begin{equation*}
        \lim_{\npinfty} B_{k ,\lambda} = \begin{cases}
                0,& \aspratio < \sampleratio , \\
                 \signallev  \frac{(\aspratio - \sampleratio)^{2}}{\aspratio \left( \aspratio - \sampleratio^{2} \right)},& \aspratio > \sampleratio.
             \end{cases}
    \end{equation*}
\end{lemma}

\end{proof}

\subsection{Supporting lemmas}
This subsection proves the supporting lemmas used in previous subsections, aka Lemmas \ref{lm:Isobootvar_under}-\ref{lm:ridgeequibias}.

%\ref{lm:Isobootvar_over}, \ref{lm:isobootbias}, \ref{lm:ridgerisk}, \ref{lm:ridgeequivar}, 

%%%%%%%%%%%%%Proof of Lemma S.4.2%%%%%%%%%%%%
\subsubsection{Proof of Lemma \ref{lm:Isobootvar_under}}
\begin{proof}[Proof of Lemma \ref{lm:Isobootvar_under}]
We assume $\noiselev = 1$ without loss of generality. 
{Let $\Omega_{\ndata} := \{ \omega \in \Omega : \eigval(\hcovmat)(\omega) > 0 \}$. Using Lemma~\ref{lm:singularmat} and Lemma~\ref{lm:lowerboundeigval}, we obtain
\begin{equation*}
    \PP \left( \lim_{\npinfty} \Omega_{\ndata} \right) = 1,
\end{equation*}
when $\aspratio < \sampleratio$. Thus, without loss of generality}, we shall assume $\eigval_{\min} (\hcovmat_{k})>0$ and $\eigval_{\min} (\hcovmat_{\ell})>0$. %\scomment{should not refer back to Lemma \ref{lm:sketchundervar}. rewrite a little bit. } 
We begin by rewriting the  variance term as
\begin{align*}
    V_{k,\ell} \left(\estimator\right) =&  \tr\left( (X\transp \sketchmat_{k} \sketchmat_{k} X)\pinv X\transp \sketchmat_{k}^2 \sketchmat_{\ell}^2 X (X\transp \sketchmat_{\ell} \sketchmat_{\ell} X)\pinv \right)\\
    =& \frac{1}{\ndata^2}\sum_{i=1}^{\ndata} \weight_{k,i} \weight_{\ell,i} x_{i}\transp ( X\transp \sketchmat_{k} \sketchmat_{k} X / \ndata)\inv ( X\transp \sketchmat_{\ell} \sketchmat_{\ell} X / \ndata)\inv x_{i}\\
    =& \frac{1}{\ndata^2} \sum_{i=1}^{\ndata}  \weight_{k,i} \weight_{\ell,i} \frac{x_{i}\transp \hcovmat_{k, -i}\inv \hcovmat_{\ell, -i}\inv x_{i}}{(1 + \weight_{k,i} x_{i}\transp \hcovmat_{k, -i}\inv x_{i}/\ndata)(1 + \weight_{\ell,i} x_{i}\transp \hcovmat_{\ell, -i}\inv x_{i}/ \ndata)},
\end{align*}
where the last line uses  the Sherman–Morrison formula twice.  Recall that $\stieltjes_{1}(0)$ and $\stieltjes_{2}(0)$, from Lemmas \ref{lm:MPlawsketching} and \ref{lm:isotrolimitresolvent}, are solutions to the equations 
\begin{gather*}
    \stieltjes_{1}(0) \EE_{\limitdistrweight} \left[ \frac{\weight}{1 + \aspratio \weight \stieltjes_{1}(0)} \right] = 1,\\
    \stieltjes_{2}(0) \EE_{\limitdistrweight} \left[ \frac{1}{1 + \aspratio \weight \stieltjes_{1}(0)}\right] \EE_{\limitdistrweight} \left[ \frac{\weight}{1 + \aspratio \weight \stieltjes_{1}(0)}\right] = \stieltjes_{1}(0).
\end{gather*} 
Using a similar argument as in the proof of Lemma~\ref{lm:sketchundervar}, and applying Lemma~\ref{lm:MPlawsketching}, Lemma~\ref{lm:boundrankoneperturb}, Lemma~\ref{lm:concentrationontrace}, and Lemma~\ref{lm:isotrolimitresolvent}, we obtain
\begin{equation*}
    \lim_{\npinfty} V_{k,\ell} \left(\estimator\right) 
    =  \lim_{\npinfty} \frac{\aspratio}{\ndata} \sum_{i = 1}^{\ndata}  \weight_{k,i} \weight_{\ell,i} \frac{\stieltjes_{2} (0)}{(1 + \aspratio \weight_{k,i} \stieltjes_{1}(0) )(1 + \aspratio \weight_{\ell,i} \stieltjes_{1}(0))} \qas
\end{equation*}

Because $\stieltjes_{1}(0)$ is positive by Lemma \ref{lm:uniqueMPlawsketching},  we have 
\begin{equation*}
    \frac{\weight_{k,i} \weight_{\ell,i}}{(1 + \aspratio \weight_{k,i} \stieltjes_{1}(0) )(1 + \aspratio \weight_{\ell,i} \stieltjes_{1}(0))} \leq \frac{1}{(\aspratio \stieltjes_{1}(0))^{2}}<\infty. 
\end{equation*}
Then, using  Assumption~\ref{Assume:multip} and the strong law of large numbers, we have 
\begin{align*}
    \lim_{\npinfty} V_{k,\ell} \left(\estimator\right) 
    &\overset{\as}{=} \aspratio \stieltjes_{2} (0) \left(\EE_{\limitdistrweight} \left[ \frac{\weight}{1 + \aspratio \weight \stieltjes_{1}(0)}\right] \right)^{2} \\
    &= \aspratio \stieltjes_{1}(0) \EE_{\limitdistrweight} \left[ \frac{\weight}{1 + \aspratio \weight \stieltjes_{1}(0)} \right] / \EE_{\limitdistrweight} \left[ \frac{1}{1 + \aspratio \weight \stieltjes_{1}(0)}\right]\\
    &= \aspratio/\left( 1 - \aspratio \stieltjes_{1}(0) \EE_{\limitdistrweight} \left[ \frac{\weight}{1 + \aspratio \weight \stieltjes_{1}(0)} \right] \right)\\
    &= \frac{\aspratio}{1 - \aspratio}.    
\end{align*}
\end{proof}

%%%%%%%%%%%%%Proof of Lemma S.4.3%%%%%%%%%%%%
\subsubsection{Proof of Lemma \ref{lm:Isobootvar_over}}
\begin{proof}[Proof of Lemma \ref{lm:Isobootvar_over}]
We begin by expressing $V_{k,\ell}(\estimator)$ in terms of sketching matrices $\Bernoulliweight_{k}$ and $\Bernoulliweight_{\ell}$ as defined in equation~\eqref{eq:defBernoulliweight} for the corresponding sketch matrices $\sketchmat_{k}$ and $\sketchmat_{\ell}$.  Applying Lemma~\ref{lm:Pseudo_products}, we assume without generality that
\begin{align*}
    V_{k,\ell} \left(\estimator\right) =& \noiselev \tr\left( (X\transp \sketchmat_{k} \sketchmat_{k} X)\pinv X\transp \sketchmat_{k}^2 \sketchmat_{\ell}^2 X (X\transp \sketchmat_{\ell} \sketchmat_{\ell} X)\pinv \right)\\
    =& \noiselev \tr\left( (X\transp \Bernoulliweight_{k}\transp \Bernoulliweight_{k} X)\pinv X\transp (\Bernoulliweight_{k}\transp)^{2} (\Bernoulliweight_{\ell})^{2} (X\transp \Bernoulliweight_{\ell}\transp \Bernoulliweight_{\ell} X)\pinv X\transp \right). 
\end{align*}
%which holds almost surely as $\npinfty$.  
Using the above equality, we shall further assume that the multipliers $\weight_{i,j}, 1\leq i \leq j \leq \ndata$, are either one or zero and $\noiselev = 1$. Applying the identity~\eqref{eq:pesudoidentity}, we can rewrite $V_{k,\ell}(\estimator)$ as
\begin{align*}
    V_{k,\ell} \left(\estimator\right) 
    =&\lim_{\zinfty} \frac{1}{\ndata^2}\tr\left( (\hcovmat_{k} - zI)\inv X\transp \sketchmat_{k}^{2} \sketchmat_{\ell}^{2} X (\hcovmat_{\ell} - zI) \inv \right) \\
    =&\lim_{\zinfty} \frac{1}{\ndata^2}\sum_{i=1}^{\ndata} \weight_{k,i} \weight_{\ell,i} x_{i}\transp ( X\transp \sketchmat_{k} \sketchmat_{k} X / \ndata - zI)\inv ( X\transp \sketchmat_{\ell} \sketchmat_{\ell} X / \ndata -zI)\inv x_{i}\\
    =&\lim_{\zinfty} \frac{1}{\ndata^2} \sum_{i=1}^{\ndata}  \weight_{k,i} \weight_{\ell,i} \frac{x_{i}\transp ( \hcovmat_{k, -i} - zI)\inv ( \hcovmat_{\ell, -i} -zI)\inv x_{i}}{(1 + \weight_{k,i} x_{i}\transp ( \hcovmat_{k, -i} -zI)\inv x_{i}/\ndata)(1 + \weight_{\ell,i} x_{i}\transp ( \hcovmat_{\ell, -i} -zI)\inv x_{i}/ \ndata)},
\end{align*}
where, in the last line, we applied the Sherman–Morrison formula twice. By using similar arguments as in the proof of Lemma~\ref{lm:Isobootvar_under}, and applying Lemma~\ref{lm:MPlawsketching}, Lemma~\ref{lm:boundrankoneperturb}, Lemma~\ref{lm:concentrationontrace}, and Lemma~\ref{lm:isotrolimitresolvent}, we obtain 
\begin{equation*}
    \lim_{\npinfty} \frac{1}{\ndata^2}\tr\left( (\hcovmat_{k} - zI)\inv X\transp \sketchmat_{k}^{2} \sketchmat_{\ell}^{2} X (\hcovmat_{\ell} - zI) \inv \right) = \aspratio \stieltjes_{2} (z) \left(\EE_{\limitdistrweight} \left[ \frac{\weight}{1 + \aspratio \weight \stieltjes_{1}(z)}\right] \right)^{2}
\end{equation*}
almost surely. According to Lemma~\ref{lm:isotrolimitresolvent}, for any $z <0$, $\stieltjes_{2} (z)$ satisfies
\begin{equation}\label{eq:isotropicbootvar2}
    \stieltjes_{2}(z) \EE_{\limitdistrweight} \left[ \frac{1}{1 + \aspratio \weight \stieltjes_{1}(z)}\right] \EE_{\limitdistrweight} \left[ \frac{\weight}{1 + \aspratio \weight \stieltjes_{1}(z)}\right] - z \stieltjes_{2}(z) = \stieltjes_{1}(z),
\end{equation}
where, according to Lemma~\ref{lm:MPlawsketching},  $\stieltjes_{1}(z)$ is the unique positive solution to
\begin{equation}\label{eq:fixpoint_resample}
    \stieltjes_{1}(z) \EE_{\limitdistrweight} \left[ \frac{\weight}{1 + \aspratio \stieltjes_{1}(z)\weight} \right] - z \stieltjes_{1}(z) = 1
\end{equation}
on $z < 0$.  When $\aspratio > \sampleratio$, Lemma~\ref{lm:singularmat} implies that the sketched covariance matrix $X\transp \sketchmat_{k} \sketchmat_{k} X$ is almost surely singular, leading to $\lim_{\zinfty} \stieltjes_{1}(z) = + \infty$.   Taking the limit as $\zinfty$ in equation~\eqref{eq:fixpoint_resample}, we obtain %\scomment{How does the last inequality follow?}
\$
    \lim_{\zinfty} -z \stieltjes_{1}(z) 
    &= 1 - \lim_{\zinfty} \stieltjes_{1}(z) \EE_{\limitdistrweight} \left[ \frac{\weight}{1 + \aspratio \weight \stieltjes_{1}(z)}\right] \nn \\
    &= 1 - \lim_{\zinfty} \frac{1}{\aspratio} \left(1 -  \EE_{\limitdistrweight} \left[ \frac{1}{1 + \aspratio \weight \stieltjes_{1}(z)}\right]\right) \nn \\
    &= \frac{\aspratio - \sampleratio}{\aspratio}, %~\text{and}\nn %\\%\nn \\
\$
and
\#
\lim_{\zinfty} \frac{1}{z} \EE_{\limitdistrweight} \left[ \frac{\weight}{1 + \aspratio \weight \stieltjes_{1}(z)}\right] 
= \lim_{\zinfty} \frac{1}{z \stieltjes_{1}(z)} + 1 
= - \frac{\sampleratio}{\aspratio - \sampleratio}.\label{eq:Isobootstrap_over1}
\#
Combining the above calculations with equation~\eqref{eq:isotropicbootvar2} and multiplying both sides by $z$, we obtain 
\begin{align}
    \lim_{\zinfty} z^2 \stieltjes_{2}(z) &=  \lim_{\zinfty} z \stieltjes_{2}(z) \EE_{\limitdistrweight} \left[ \frac{1}{1 + \aspratio \weight \stieltjes_{1}(z)}\right] \EE_{\limitdistrweight} \left[ \frac{\weight}{1 + \aspratio \weight \stieltjes_{1}(z)}\right] - z \stieltjes_{1}(z) \notag \\
    &= \lim_{\zinfty} z^2 \stieltjes_{2}(z) \left(1 - \sampleratio + o(z)\right) \frac{1}{z} \EE_{\limitdistrweight} \left[ \frac{\weight}{1 + \aspratio \weight \stieltjes_{1}(z)}\right] + \frac{\aspratio - \sampleratio}{\aspratio} \notag\\ %\label{eq:isobootvar3}
    &= -\frac{\sampleratio(1-\sampleratio)}{\aspratio - \sampleratio} \, \lim_{\zinfty} z^2 \stieltjes_{2}(z)  + \frac{\gamma - \theta}{\gamma}, 
\end{align}
which then gives
\#\label{eq:z2m2}
\lim_{\zinfty}  z^2 \stieltjes_{2}(z) 
    = \frac{(\aspratio - \sampleratio)^{2}}{\aspratio(\aspratio - \sampleratio^{2})}.
\#
Therefore, we obtain 
\begin{align*}
    &\lim_{\zinfty} \aspratio \stieltjes_{2} (z) \left(\EE_{\limitdistrweight} \left[ \frac{\weight}{1 + \aspratio \weight \stieltjes_{1}(z)}\right] \right)^{2} \\
    =& \aspratio \frac{(\aspratio - \sampleratio)^{2}}{\aspratio(\aspratio - \sampleratio^{2})} \left(- \frac{\sampleratio}{\aspratio - \sampleratio} \right)^{2}\\
    =& \frac{\sampleratio^{2}}{\aspratio - \sampleratio^{2}}.    
\end{align*}

Finally, we prove that we can exchange the limits between $\npinfty$ and $\zinfty$ by applying the Arzela-Ascoli theorem and the Moore-Osgood theorem. To accomplish this, we establish that $\tr\left( (\hcovmat_{k} - zI)\inv X\transp \sketchmat_{k}^{2} \sketchmat_{\ell}^{2} X (\hcovmat_{\ell} - zI) \inv \right)/\ndata^{2}$ and its derivative are uniformly bounded for $z < 0$. We assume that the diagonal elements of $\sketchmat_{k}$ and $\sketchmat_{\ell}$ are either zero or one. For $z < 0$, we have the following inequality %\scomment{The right hand side is exploding? because we have $\aspratio>\sampleratio$. }
\begin{equation*}
    \frac{1}{\ndata^{2}}\tr\left( (\hcovmat_{k} - zI)\inv X\transp \sketchmat_{k}^{2} \sketchmat_{\ell}^{2} X (\hcovmat_{\ell} - zI) \inv \right) 
    \leq \frac{d}{n} \cdot \frac{\eigval_{\max}\big(X\transp \sketchmat_{k}^{2} \sketchmat_{\ell}^{2} X/n \big)}{\eigval^{+}_{\min}\big(\hcovmat_{k}\big) \eigval^{+}_{\min}\big(\hcovmat_{\ell}\big)},
\end{equation*}
which is almost surely uniformly bounded by applying Lemma~\ref{lm:lowerboundeigval} and Theorem 2 in~\citep{Bai1993}. Similarly, we can show that the derivative 
\begin{align*}
    &\frac{1}{\ndata^2}\frac{\partial \, \tr \left( (\hcovmat_{k} - zI)\inv X\transp \sketchmat_{k}^{2} \sketchmat_{\ell}^{2} X (\hcovmat_{\ell} - zI) \inv \right)}{\partial z}\\
    =& \frac{1}{\ndata^2} \tr\left( (\hcovmat_{k} - zI)^{-2} X\transp \sketchmat_{k}^{2} \sketchmat_{\ell}^{2} X (\hcovmat_{\ell} - zI) \inv \right) + \tr\left( (\hcovmat_{i} - zI)\inv X\transp \sketchmat_{k}^{2} \sketchmat_{\ell}^{2} X (\hcovmat_{\ell} - zI)^{-2} \right)
\end{align*}
is also almost surely uniformly bounded. Then,  using the Arzela-Ascoli theorem and the Moore-Osgood theorem,  we obtain almost surely that  %Therefore, we obtain,
\begin{align*}
    \lim_{\npinfty} V_{k,\ell} \left(\estimator \right) &= \lim_{\npinfty} \lim_{\zinfty} \frac{\noiselev}{\ndata^{2}} \tr\left( (\hcovmat_{k} - zI)\inv X\transp \sketchmat_{k}^{2} \sketchmat_{\ell}^{2} X (\hcovmat_{\ell} - zI) \inv \right)\\
    &= \lim_{\zinfty} \lim_{\npinfty} \frac{\noiselev}{\ndata^{2}} \tr\left( (\hcovmat_{k} - zI)\inv X\transp \sketchmat_{k}^{2} \sketchmat_{\ell}^{2} X (\hcovmat_{\ell} - zI) \inv \right)\\
    &= \lim_{\zinfty} \aspratio \stieltjes_{2} (z) \left(\EE_{\limitdistrweight} \left[ \frac{\weight}{1 + \aspratio \weight \stieltjes_{1}(z)}\right] \right)^{2}\\
    &=\noiselev  \frac{\sampleratio^2}{\aspratio - \sampleratio^{2}}.
\end{align*}
\end{proof}

%%%%%%%%%%%%%Proof of Lemma S.4.4%%%%%%%%%%%%%%
\subsubsection{Proof of Lemma \ref{lm:isobootbias}}

\begin{proof}[Proof of Lemma \ref{lm:isobootbias}]
According to Lemma~\ref{lm:singularmat}, the sketched sample covariance matrices are almost surely invertible when $\aspratio < \sampleratio$. Thus, in the underparameterized regime, by Lemma~\ref{lm:biasvar}, the bias term converges almost surely to zero.

When $\aspratio > \sampleratio$, we apply Lemma~\ref{lm:biasvar-beta} to rewrite the limit of $B_{k,\ell}(\estimator)$ as 
    \begin{align*}
        \lim_{\npinfty} B_{k,\ell}(\estimator) &= \lim_{\npinfty} \frac{\signallev}{\ndim} \tr\left(\left(I - \hcovmat_{k}\pinv \hcovmat_{k} \right) \left(I - \hcovmat_{\ell}\pinv \hcovmat_{\ell} \right) \right) \\
        &=\lim_{\npinfty} \lim_{\zinfty} \frac{\signallev}{\ndim} \tr \left[ \left(I - (\hcovmat_{k} - zI)\inv \hcovmat_{k}\right) \left(I - (\hcovmat_{\ell} - zI)\inv \hcovmat_{\ell}\right) \right]\\
        &= \lim_{\npinfty} \lim_{\zinfty} \frac{\signallev z^2}{\ndim} \, \tr\left((\hcovmat_{k} - zI)\inv (\hcovmat_{\ell} - zI)\inv \right),
    \end{align*}
    where we used the identity~\eqref{eq:pesudoidentity}. 
    As $\npinfty$, ${\ndim}^{-1} z^2 \tr\left((\hcovmat_{k} - zI)\inv (\hcovmat_{\ell} - zI)\inv \right)$ and its derivative
    \$
        &\left| \frac{1}{\ndim}\frac{d\, z^2\tr\left((\hcovmat_{k} - zI)\inv (\hcovmat_{\ell} - zI)\inv \right)}{d\, z} \right| \\
        &= \left| \frac{1}{\ndim} \left(z\,\tr\left((\hcovmat_{k} - zI)\inv\hcovmat_{k} (\hcovmat_{\ell} - zI)\inv \right) + z\,\tr\left((\hcovmat_{k} - zI)\inv (\hcovmat_{\ell} - zI)\inv\hcovmat_{\ell} \right)\right) \right|\\
       &= \left| \frac{1}{\ndim} \left(z\, \tr(\hcovmat_{k} - zI)\inv + z^{2}\,\tr\left((\hcovmat_{k} - zI)\inv(\hcovmat_{\ell} - zI)\inv\right) + z\, \tr(\hcovmat_{\ell} - zI)\inv\right) \right|\\
       &\leq 4
    \$
    is almost surely bounded by Lemma~\ref{lm:lowerboundeigval}. Therefore, by the Arzela-Ascoli theorem and the Moore-Osgood theorem, we can exchange the limits between $\npinfty$ and $\zinfty$ to obtain 
    \begin{align*}
         \lim_{\npinfty} B_{k,\ell}(\estimator) &= \lim_{\npinfty} \lim_{\zinfty} \frac{\noiselev}{\ndim} z^2 \, \tr\left((\hcovmat_{k} - zI)\inv (\hcovmat_{\ell} - zI)\inv \right)\\
         &=\lim_{\zinfty} \lim_{\npinfty} \frac{\noiselev}{\ndim} z^2\, \tr\left((\hcovmat_{k} - zI)\inv (\hcovmat_{\ell} - zI)\inv \right)\\
         &= \lim_{\zinfty} \signallev z^2\, \stieltjes_{2}(z)\\
         &= \signallev \frac{\left(\aspratio - \sampleratio \right)^2}{\aspratio \left(\aspratio - \sampleratio^{2} \right)},
    \end{align*}
    where the last equality follows from equation~\eqref{eq:z2m2}.
    
\end{proof}

\subsubsection{Proof of Lemma \ref{lm:ridgerisk}}
\begin{proof}[Proof of Lemma~\ref{lm:ridgerisk}]
    In the underparameterized regime, where $\aspratio < \sampleratio$, we have $\lambda = 0$, which corresponds to the ridgeless regression. The result follows directly from equation~\eqref{eq:risk_ridgeless}.

    When $\aspratio > \sampleratio$, \cite{edgar2019} showed that
    \begin{equation*}
        \lim_{\npinfty} \EE \left[\| \hat{\beta}_{\lambda} - \truesignal\|_{2}^{2} \condition X,\truesignal \right] = \frac{\signallev}{\aspratio \compstieltjes(-\lambda)} \left( 1 - \frac{\lambda \compstieltjes{'}(-\lambda)}{\compstieltjes(-\lambda)}\right) + \noiselev \left( \frac{\compstieltjes{'}(-\lambda)}{\compstieltjes^{2}(-\lambda)} - 1 \right),
    \end{equation*}
    where $\compstieltjes(-\lambda)$ is the unique solution of equation~\eqref{eq:fixedpointsketching} with $\aspratio/\sampleratio$ replaced by $\aspratio$. We have the derivative
    \begin{align*}
        \compstieltjes{'}(-\lambda) = \frac{\compstieltjes^{2}(-\lambda)(1 + \compstieltjes(-\lambda))^{2}}{(1 + \compstieltjes(-\lambda)) - \aspratio \compstieltjes^{2}(-\lambda)}.
    \end{align*}
    Therefore, the result holds if and only if $\compstieltjes(-\lambda) = {\sampleratio}/{(\aspratio - \sampleratio)}$, which is the unique positive solution of equation~\eqref{eq:fixedpointsketching} ($\aspratio/\sampleratio$ replaced by $\aspratio$) with $\lambda = (1 - \sampleratio)(\aspratio/\sampleratio - 1)$.
\end{proof}

\subsubsection{Proof of Lemma \ref{lm:ridgeequivar}}
\begin{proof}[Proof of Lemma~\ref{lm:ridgeequivar}]

We begin with the underparameterized case, that is $\aspratio < \sampleratio$. Consider the subsample $\sketchmat_{k} X$, that is $\aspratio< \sampleratio$. By Lemmas~\ref{lm:singularmat}--\ref{lm:lowerboundeigval}, we can assume $\eigval_{\min}(\hcovmat_{k}) > 0$ without loss of generality. We shall also assume  $\noiselev = 1$. Following Lemma \ref{lm:ridgeresolvent}, let $\stieltjes_{2, \lambda}(z):= \lim_{\npinfty} \ndim^{-1}\tr( (X\transp X/n + \lambda I)\inv (X\transp \sketchmat_{k}\transp \sketchmat_{k}X/n - z I)\inv)$, and $\stieltjes(-\lambda)$ be the Stieltjes transform of the limiting spectral distribution of the full-sample matrix $X$. %Lemma~\ref{lm:ridgeresolvent} shows that 
Then, by Lemma \ref{lm:ridgeresolvent}, $\stieltjes_{2, \lambda}(0)$ exists and satisfies
    \begin{equation}\label{eq:ridgeequivar1}
        \stieltjes(-\lambda) = \stieltjes_{2, \lambda}(0) \frac{1}{1 + \aspratio \stieltjes(-\lambda)} \EE_{\limitdistrweight} \left[ \frac{\weight}{1 + \aspratio \weight \stieltjes_{1}(0)}\right].
    \end{equation}
    Using equation~\eqref{eq:ridgeequivdecomp} and a similar argument as in the proof of Lemma~\ref{lm:Isobootvar_under}, we obtain almost surely that
    \begin{align*}
        \lim_{\npinfty}V_{k ,\lambda} &= \tr \left( (X\transp  X + n \lambda I)\inv X\transp \sketchmat_{k}\transp \sketchmat_{k} X (X\transp \sketchmat_{k}\transp \sketchmat_{k} X)\inv \right)\\
        &= \aspratio \stieltjes_{2, \lambda}(0) \frac{1}{1 + \aspratio \stieltjes(-\lambda)} \EE_{\limitdistrweight} \left[ \frac{\weight}{1 + \aspratio \weight \stieltjes_{1}(0)}\right] \\
        &= \aspratio \stieltjes(-\lambda),
    \end{align*}
    where the last line follows from equation~\eqref{eq:ridgeequivar1}. In the underparameterized case,  $\lambda = (1 - \sampleratio)(\aspratio/\sampleratio - 1) \vee 0=0$.  Thus it suffices to show $\stieltjes(0)= 1/(1-\aspratio)$, whose proof can be found in Chapter 6 of  \cite{serdobolskii2007multiparametric}. % \scolor{which we refer to \cite[Proposition 2]{hastie2019surprises}.} \scomment{they did not calculate $m(0)$ in Proposotion 2?}

    In the overparameterized regime, using a similar argument in the proof of Lemma~\ref{lm:Isobootvar_over},  we shall assume that the multipliers $\weight_{i,j}$ are either zero or one and $\noiselev = 1$ without loss of generality. Using the definition of $V_{k, \lambda}$ in equation~\eqref{eq:ridgeequivdecomp} and identity~\eqref{eq:pesudoidentity}, we acquire
    \begin{equation*}
        V_{k ,\lambda} = \lim_{\zinfty} \tr \left( (X\transp  X + n \lambda I)\inv X\transp \sketchmat_{k}\transp \sketchmat_{k} X (X\transp \sketchmat_{k}\transp \sketchmat_{k} X - zI)\inv \right)
    \end{equation*}
Using a similar argument as in the proof of Lemma~\ref{lm:Isobootvar_over}, %and applying Lemma~\ref{lm:ridgeresolvent}, 
we obtain almost surely that %\scomment{we have only defined $m_{2,\lambda}(0)$ but not $m_{2,\lambda}(z)$.  }
\$
    &\lim_{\npinfty} \tr \left( (X\transp  X + n \lambda I)\inv X\transp \sketchmat_{k}\transp \sketchmat_{k} X (X\transp \sketchmat_{k}\transp \sketchmat_{k} X - zI)\inv \right) \\
    &= \aspratio \stieltjes_{2, \lambda}(z) \frac{1}{1 + \aspratio \stieltjes(-\lambda)} \EE_{\limitdistrweight} \left[ \frac{\weight}{1 + \aspratio \weight \stieltjes_{1}(z)}\right]  %{\quad \qas}
\$
    where $\stieltjes_{2, \lambda}(z)$ satisfies the equation
    \#\label{eq:m2lambda}
        \stieltjes(-\lambda) = \stieltjes_{2, \lambda}(z) \frac{1}{1 + \aspratio \stieltjes(-\lambda)} \EE_{\limitdistrweight} \left[ \frac{\weight}{1 + \aspratio \weight \stieltjes_{1}(z)}\right] - z\stieltjes_{2,\lambda}(z).
    \#
    From the proof of Lemma~\ref{lm:ridgerisk}, we know that, for $\lambda = (1 - \sampleratio)(\aspratio/\sampleratio -1)$, the companion Stieltjes transform is  $\compstieltjes(-\lambda) = {\sampleratio}/({\aspratio - \sampleratio})$. Thus we have
    \begin{equation*}
        \aspratio \stieltjes(-\lambda) = \compstieltjes(- \lambda) - 1/ \lambda + \aspratio/\lambda = \frac{\sampleratio}{1 - \sampleratio}.
    \end{equation*}
    Furthermore, equation~\eqref{eq:Isobootstrap_over1} implies
    \begin{equation*}
        \lim_{\zinfty} \frac{1}{z} \EE_{\limitdistrweight} \left[ \frac{\weight}{1 + \aspratio \weight \stieltjes_{1}(z)}\right] = - \frac{\sampleratio}{\aspratio - \sampleratio}.
    \end{equation*}
Plugging the above two equalities into equation \eqref{eq:m2lambda}, we obtain
\$
\frac{\sampleratio}{\aspratio(1-\sampleratio)} 
&= \frac{1}{1 + \sampleratio/(1-\sampleratio)}\cdot \left(- \frac{\sampleratio}{\aspratio - \sampleratio}\right)\lim_{\zinfty}zm_{2,\lambda}(z)  - \lim_{\zinfty} z m_{2,\lambda}(z)   
\$
    which yields 
    \begin{equation}\label{eq:ridgez0}
        \lim_{\zinfty} -z \stieltjes_{2, \lambda}(z) = \frac{\sampleratio}{\aspratio (1 - \sampleratio)}\frac{\aspratio - \sampleratio}{\aspratio - \sampleratio^{2}}.
    \end{equation}

Using a similar argument as in the proof of  Lemma~\ref{lm:Isobootvar_over}, we can verify that 
\$
\tr \left( (X\transp  X + n \lambda I)\inv X\transp \sketchmat_{k}\transp \sketchmat_{k} X (X\transp \sketchmat_{k}\transp \sketchmat_{k} X - zI)\inv \right)
\$
and its derivative are uniform bounded on $R_{\leq 0}$.  %_{-}\cup \{ 0\}$.  
%Therefore, using , we can exchange the limits (as $\npinfty$ and $\zinfty$) and obtain:
Therefore, by the Arzela-Ascoli theorem and the Moore-Osgood theorem, we can exchange the limits between $\npinfty$ and $\zinfty$ to obtain 
    \begin{align*}
        \lim_{\npinfty}  V_{k ,\lambda}  &= \lim_{\zinfty} \aspratio \stieltjes_{2, \lambda}(z) \frac{1}{1 + \aspratio \stieltjes(-\lambda)} \EE_{\limitdistrweight} \left[ \frac{\weight}{1 + \aspratio \weight \stieltjes_{1}(z)}\right]\\
        &= \frac{\sampleratio^{2}}{\aspratio - \sampleratio^{2}}.
    \end{align*}
\end{proof}

\subsubsection{Proof of Lemma \ref{lm:ridgeequibias}}

\begin{proof}[Proof of Lemma~\ref{lm:ridgeequibias}]
In the underparameterized regime, according to Lemma~\ref{lm:singularmat}, the sketched sample covariance matrix $X\transp \sketchmat_{k} \sketchmat_{k} X$ is almost surely invertible. Therefore, the result follows directly from equation~\eqref{eq:ridgeequivdecomp}.

In the overparameterized regime, by the definition of $B_{k,\ell}$ in  equation~\eqref{eq:ridgeequivdecomp}, we have %\scomment{$r^2$ is not in the second line????}
\begin{align*}
    \lim_{\npinfty} B_{k ,\lambda} &= \lim_{\npinfty} \signallev \tr \left( \left( X\transp \sketchmat_{k}\transp \sketchmat_{k} X (X\transp \sketchmat_{k}\transp \sketchmat_{k} X)\pinv - I \right)\left(  (X\transp  X + n \lambda I)\inv X\transp X - I\right) \right)/\ndim \\
    &=  \lim_{\npinfty} \lim_{\zinfty} r^2  -\frac{z \lambda }{\ndim} \tr \left( ( X\transp \sketchmat_{k}\transp \sketchmat_{k} X/ \ndata - zI)\inv (X\transp  X/\ndata + \lambda I)\inv  \right).
\end{align*}
Using Lemma~\ref{lm:ridgeresolvent} and equation~\eqref{eq:ridgez0}, we obtain almost surely that
\begin{align*}
    &\lim_{\zinfty} \lim_{\npinfty} -\frac{z \lambda}{\ndim} \tr \left( ( X\transp \sketchmat_{k}\transp \sketchmat_{k} X/ \ndata - zI)\inv (X\transp  X/\ndata + \lambda I)\inv  \right) \\
    =& \lambda \lim_{\zinfty} -z \stieltjes_{2, \lambda}(z)\\
    =& (1 - \sampleratio)\left(\frac{\aspratio}{\sampleratio} -1\right)\cdot \frac{\sampleratio}{\aspratio (1 - \sampleratio) }\frac{\aspratio - \sampleratio}{\aspratio - \sampleratio^{2}}\\
    =& \frac{(\aspratio - \sampleratio)^{2}}{\aspratio \left( \aspratio - \sampleratio^{2} \right)}.
\end{align*}

Finally, as in the proof of Lemma~\ref{lm:isobootbias}, one can easily check that 
\$
\frac{1}{d} \, \tr \left( ( X\transp \sketchmat_{k}\transp \sketchmat_{k} X/ \ndata - zI)\inv (X\transp  X/\ndata + \lambda I)\inv  \right)
\$ 
and its derivative with respect to $z$ are uniformly bounded for $z \leq 0$. By the Arzela-Ascoli theorem and the Moore-Osgood theorem, we can exchange the limits between $\npinfty$ and $\zinfty$,  and the desired result follows. 
\end{proof}

\subsection{Technical lemmas}
This subsection proves a technical lemma that is used in the proofs of the supporting lemmas in the previous subsection. 
%%%%%%Tech Lemma 5
Let $\sketchmat \in \mathbb{R}^{\ndata \times \ndata}$ be a sketching matrix, and $S_{ij}$ be its $(i,j)$-th element. Define $\Bernoulliweight$ as
\begin{align}\label{eq:defBernoulliweight}
    \Bernoulliweight_{ij} = \begin{cases}
        1(\sketchmat_{ij} \neq 0), &\quad i=j\\
        0, &\quad i\neq j
    \end{cases}.
\end{align}

\begin{lemma}\label{lm:Pseudo_products}
%Let $X \in \mathbb{R}^{\ndata \times \ndim}$ be a data matrix, and 
Assume   Assumptions~\ref{Assume:highdim},~\ref{Assume:Covdistri}, and $\aspratio > \sampleratio$. Then, as $\npinfty$, 
\begin{equation*}
    (X\transp \sketchmat\transp \sketchmat X)\pinv X\transp \sketchmat\transp 
       \sketchmat
    = (X\transp \Bernoulliweight \transp \Bernoulliweight X)\pinv X\transp      
       \Bernoulliweight\transp \Bernoulliweight~\qas 
\end{equation*}
\iffalse
where $\Bernoulliweight$ is defined as
\begin{align}\label{eq:defBernoulliweight}
    \Bernoulliweight_{ij} 
    = \begin{cases}
        1(\sketchmat_{ij} \neq 0), &\quad i=j\\
        0, &\quad i\neq j
    \end{cases}.
\end{align}
\fi
\end{lemma}

\begin{proof}[Proof of Lemma \ref{lm:Pseudo_products}]
We begin by reordering the rows of $\sketchmat$ and $\sketchmat X$ as % \scomment{dimensions are not compatible.}
\begin{equation*}
    U \sketchmat 
    =  \begin{bmatrix}
            \sketchmat_{1}  & 0\\
            0               & 0
    \end{bmatrix}  \quad 
\text{and} 
\quad U 
    \sketchmat X 
    =  \begin{bmatrix}
            \sketchmat_{1} & 0  \\
            0              & 0 
        \end{bmatrix}
        \begin{bmatrix}
            X_{1} \\
            X_{2}
        \end{bmatrix}
    =   \begin{bmatrix}
            \sketchmat_{1} X_{1} \\
            0
        \end{bmatrix}, 
\end{equation*}
where $\sketchmat_{1} \in \mathbb{R}^{|A_{0}| \times |A_{0}|}$ is a diagonal matrix with non-zero elements on its diagonal, corresponding to the non-zero diagonal elements of $\sketchmat$,  $A_{0}:= \{i:~\sketchmat_{ii} \neq 0\}$, $\sketchmat_{1}X_{1} \in \mathbb{R}^{|A_{0}| \times \ndim}$ is the subsampled data matrix, and $U \in \mathbb{R}^{\ndata \times \ndata}$ is an orthogonal matrix. 
Then %\scomment{invertibility of $\sketchmat_{1} X_{1} X_{1}\transp \sketchmat_{1}\transp$? }%using these notations, we have:
    \begin{align*}
        (X\transp \sketchmat\transp \sketchmat X)\pinv X\transp \sketchmat\transp \sketchmat
        &= (X_{1}\transp \sketchmat_{1}\transp \sketchmat_{1} X_{1})\pinv \begin{bmatrix}
            X_{1}\transp \sketchmat_{1}\transp \sketchmat_{1} & 0
        \end{bmatrix}  \\
        &= (X_{1}\transp \sketchmat_{1}\transp \sketchmat_{1} X_{1})\pinv X_{1}\transp \sketchmat_{1}\transp \sketchmat_{1}
        \begin{bmatrix}
            I & 0
        \end{bmatrix} \\
        &= X_{1}\transp \sketchmat_{1}\transp (\sketchmat_{1} X_{1} X_{1}\transp \sketchmat_{1}\transp )\pinv \sketchmat_{1}
        \begin{bmatrix}
            I & 0
        \end{bmatrix},
    \end{align*}
where the last line uses the following property of the pseudoinverse of a matrix $A$: %is derived from the following two equivalences involving the pseudoinverse of a matrix A:
    \$
        A\pinv &= (A\transp A)\pinv A\transp = A\transp (A A\transp)\pinv.
    \$
Using equation~\eqref{eq:SSLNA0}, we can show that when $\aspratio >\sampleratio$, as $\npinfty$, it is almost surely that $| A_{0}|/d \leq \sampleratio/\aspratio < 1$. Applying  \cite[Theorem 2]{Bai1993}, we obtain almost surely that 
\$
    \eigval_{\min}(\sketchmat_{1} X_{1} X_{1}\transp\sketchmat_{1}\transp / \ndim ) 
    &\geq c_w^{2}c_\lambda \left( 1 - \sqrt{\aspratio/\sampleratio}\right) > 0 \qas \\
    \eigval_{\min}(X_{1} X_{1}\transp/d) 
    &\geq c_\lambda  \left( 1 - \sqrt{\aspratio/\sampleratio}\right) \qas
\$
as $\npinfty$.  Therefore, we can assume that both $\sketchmat_{1} X_{1} X_{1}\transp \sketchmat_{1}\transp$ and $X_1X_1\transp$ are invertible. Then %we have
    \begin{align*}
        (X\transp \sketchmat\transp \sketchmat X)\pinv X\transp \sketchmat\transp \sketchmat
        &= X_{1}\transp \sketchmat_{1}\transp (\sketchmat_{1} X_{1} X_{1}\transp \sketchmat_{1}\transp )\inv \sketchmat_{1}
        \begin{bmatrix}
            I & 0
        \end{bmatrix} \\
        &= X_{1}\transp  ( X_{1} X_{1}\transp  )\inv  \begin{bmatrix}
            I & 0
        \end{bmatrix} \\
        &=  ( X_{1} X_{1}\transp  )\inv X_{1}\transp  \begin{bmatrix}
            I & 0
        \end{bmatrix} \\
        &= (X\transp \Bernoulliweight\transp \Bernoulliweight X)\pinv %(DX)\transp 
            X\transp  \Bernoulliweight\transp \Bernoulliweight,
    \end{align*}
which holds almost surely as $\npinfty$. This finishes the proof. 
\end{proof}

%%%%%%%%%%%%%%%%%%%%%%%%%%%%%%%%%%%%%%%%%%%%%%%%%
%%%%%%%%%%%%%%%%Correlated%%%%%%%%%%%%%%%%%%%%%%%
%%%%%%%%%%%%%%%%%%%%%%%%%%%%%%%%%%%%%%%%%%%%%%%%%

\section{Proofs for Section \ref{sec:correlated}}

This section proves the results in Section \ref{sec:correlated}. %\cp

%%%%%%%%%%%%%%%%%%%%%%%%%%%%%%%%%%%%%%%%%%%%%%%%%
%%%%%%%%%%%%%%%%Lemma 4.1$%%%%%%%%%%%%%%%%%%%%%%%
%%%%%%%%%%%%%%%%%%%%%%%%%%%%%%%%%%%%%%%%%%%%%%%%%

\subsection{Proof of Lemma~\ref{lm:v0}}

\begin{proof}[Proof of Lemma~\ref{lm:v0}]
We only prove the results for $\compstieltjes(z)$. The results for  $\tcompstieltjes(z)$ follow similarly. 

\paragraph{Proving equation \eqref{eq:fixedpointsketching} has a unique positive solution.}
Let us define the function 
\$
f(x, z) := \frac{\sampleratio}{\aspratio} zx + \int \frac{1}{1 + xt} \limitdistr(dt). 
\$ Using equation~\eqref{eq:fixedpointsketching}, we have $f(\compstieltjes(z), z) = 1 - {\sampleratio}/{\aspratio} > 0$. Moreover, for $z \leq 0$, $f(x,z)$ is a continuous decreasing function with respect to $x$ on $(0, +\infty)$, where $f(0,z) = 1$ and $f(+\infty, z) = -\infty$. Thus, we conclude that equation~\eqref{eq:fixedpointsketching} has a unique positive solution.

\paragraph{Differentiability of $\compstieltjes(z)$.}
According to \cite[Theorem 2.7]{couillet2022}, $\compstieltjes(z)$ is the companion Stieltjes transform of $\stieltjes_1(z)$, and is also the  Stieltjes transform of the limiting  empirical spectral distribution $\mu$ of $X X\transp /n$ with $X\in \RR^{\ndata \times \ndim}$ satisfying Assumption~\ref{Assume:Covdistri} and $\ndim/\ndata \rightarrow \aspratio/\sampleratio$. Because the Stieltjes transform is analytic outside the support of $\mu$, we apply Lemma~\ref{lm:singularmat} and Lemma~\ref{lm:lowerboundeigval} to conclude that $\lambda_{\min}(X X\transp/n ) > 0$ almost surely as $\npinfty$ when $\aspratio > \sampleratio$. Therefore, when $\aspratio > \sampleratio$, the function $\compstieltjes(z)$ is differentiable for $z < 0$.

\paragraph{Proving that $\compstieltjes(0):=\lim_{\zinfty} \compstieltjes(z)$ and $\tcompstieltjes(0):= \lim_{\zinfty}\tcompstieltjes(z)$ exist.} 
The proof follows from a similar argument used in the proof of Lemma \ref{lm:MPlawsketching} and thus is omitted.

\end{proof}

%%%%%%%%%%%%%%%%%%%%%%%%%%%%%%%%%%%%%%%%%%%%%%%%%
%%%%%%%%%%%%%%%%Theorem 4.2%%%%%%%%%%%%%%%%%%%%%%
%%%%%%%%%%%%%%%%%%%%%%%%%%%%%%%%%%%%%%%%%%%%%%%%%

\subsection{Proof of Theorem~\ref{thm:corsketching}}

\begin{proof}[Proof of Theorem \ref{thm:corsketching}]
We begin by simplifying the expressions for the bias and variance of $\hat\beta^{1}$ in Lemma~\ref{lm:biasvar} and Lemma~\ref{Assume:beta4+mom}
\begin{align*}
    \varcondition(\hat\beta^{1}) &= \frac{\noiselev}{n^{2}} \tr\left( \hcovmat_{1}\pinv X\transp \sketchmat_{1}^2 \sketchmat_{1}^2 X \hcovmat_{1}\pinv \covmat \right),\\
    \lim_{\npinfty} \biascondition\left(\hat\beta^{1}\right) &= \lim_{\npinfty} \frac{\signallev}{\ndim}  \tr\left( \projmat_{1} \covmat \projmat_{1} \right),  
\end{align*}
where $\projmat_{1}=I-\hcovmat_{1}\pinv \hcovmat_{1}$. 
We  discuss the underparameterized and  overparameterized regimes separately.

\paragraph{The underparameterized regime}
We first derive the underparameterized bias and then underparameterized variance. 
When $\aspratio < \sampleratio$, by Lemma~\ref{lm:singularmat}, $\hcovmat_{1}$ is almost surely invertible as $\npinfty$. This implies $\projmat_{1} = 0$. Therefore, when $\aspratio< \sampleratio$, the bias convergences to zero almost surely as $\npinfty$. The following lemma derives the underparameterized variance. %We then derive the underparameterized variance, which is collected in the following lemma. 

%In what follows, we proceed to calculate the underparameterized variance (Lemma~\ref{lm:sketchundervar}), overparameterized variance (Lemma \ref{lm:sketchovervar}), and the overparameterized bias (Lemma \ref{lm:sketchoverbias}) respectively.

%%%%%%%%%%%%%%%%%%%%%%%%%%%%%%%%%%%%%%%%%%%%%%%%%
%%%%%%%%%%Under-parameterized Variance%%%%%%%%%%%
%%%%%%%%%%%%%%%%%%%%%%%%%%%%%%%%%%%%%%%%%%%%%%%%%

%We then derive the underparameterized variance, which is collected in the following lemma. 
%\subsubsection{Under-parameterized variance}
%\paragraph{Underparameterzied variance}

\begin{lemma}[Underparameterized variance under correlated features]\label{lm:sketchundervar}
Assume Assumptions~\ref{Assume:highdim}-\ref{Assume:multip}, and $\aspratio < \sampleratio$. The variance satisfies
     \begin{equation*}
    \lim_{\npinfty} \varcondition(\hat{\beta}^{1})
    =
    \sigma^2\left(\frac{\aspratio}{1-\aspratio - f(\aspratio)} - 1 \right) {\qas}
    \end{equation*}
    where
\$
 f(\aspratio) = \int \frac{1}{(1+ c t )^{2}} d \limitdistrweight(t), 
\$
and the constant $c$ is the unique positive solution to equation~\eqref{eq:uniqueMPlawsketching}.
\end{lemma}

\paragraph{The overparameterized regime}

We first derive the overparameterized variance and then overparameterized bias. %\scomment{why not first bias and then the variance? Because for the underparameterized regime we talk about bias and then variance. } \wcomment{personal preference}
Recall that, by {Lemma~\eqref{lm:v0}, there is an unique positive solution to equation~\ref{eq:fixedpointsketching} for any $z\leq 0$.}

%%%%%%%%%%%%%%%%%%%%%%%%%%%%%%%%%%%%%%%%
%%%%%%%%%%%%%Sketch over var%%%%%%%%%%%%
%%%%%%%%%%%%%%%%%%%%%%%%%%%%%%%%%%%%%%%%

\begin{lemma}\label{lm:sketchovervar}
Assume Assumptions~\ref{Assume:highdim}-\ref{Assume:multip} and $\aspratio > \sampleratio$.  The variance satisfies 
\begin{equation*}
\lim_{\npinfty} \varcondition(\hat{\beta}^{1})
=  \noiselev \left(\frac{\compstieltjes'(0)}{\compstieltjes(0)^{2}} - 1\right)~{\qas}
\end{equation*}
where $\compstieltjes_{1}(0)$ is the unique positive solution to equation~\eqref{eq:fixedpointsketching} for any  $z = 0$. %\scomment{why does \eqref{eq:fixedpointsketching} have a unique positive solution? Can you first prove it or cite elsewhere?}
\end{lemma}

%\subsubsection{Over-parameterized Bias}
Our next lemma collects the limiting bias of $\hat\beta^1$ in the overparameterized regime. 
\begin{lemma}\label{lm:sketchoverbias}
Assume Assumptions~\ref{Assume:highdim}-\ref{Assume:beta4+mom}, and $\aspratio > \sampleratio$. Then the bias satisfies
     \begin{equation*}
    \lim_{\npinfty} \biascondition(\hat{\beta}^{1})
    =
    \frac{\signallev \sampleratio}{\aspratio \compstieltjes(0)} {\qas}
    \end{equation*}
    where $\compstieltjes(z)$ is the unique positive solution to  equation~\eqref{eq:fixedpointsketching} at $z = 0$. 
\end{lemma}

Putting above results together finishes the proof. 
\end{proof}

%%%%%%%%%%%%%%%%%%%%%%%%%%%%%%%%%%%%%%%%%%%%%%%%%%%%%%%%%%%%%
%%%%%%%%%%%%%%%%%%%Proof of Theorem 4.6%%%%%%%%%%%%%%%%%%%%%%
%%%%%%%%%%%%%%%%%%%%%%%%%%%%%%%%%%%%%%%%%%%%%%%%%%%%%%%%%%%%%

\subsection{Proof of Theorem~\ref{thm:corrbagging}}

\begin{proof}[Proof of Theorem \ref{thm:corrbagging}]  %\label{Proof of Theorem~\ref{thm:Overpara}}

%We start with defining the following notation: 
For notational simplicity, let
\begin{align*}
    B_{k,\ell} &= \frac{\signallev}{\ndim} \tr \left( \left(I - \hcovmat_{k}\pinv \hcovmat_{k} \right) \covmat \left(I - \hcovmat_{\ell}\pinv \hcovmat_{\ell} \right) \right),\\
    V_{k,\ell} &= \frac{\noiselev}{\ndata^{2}} \tr\left(\hcovmat_{k}\pinv X\transp \sketchmat_{k}^2 \sketchmat_{\ell}^2 X \hcovmat_{\ell}\pinv \covmat \right).
\end{align*} 
Following the same argument as in the proof of Theorem~\ref{thm:isoboostrap}, we have  %\scomment{need independence assumption.}
\begin{align}\label{eq:corbagging1}
    \lim_{\nresample \rightarrow + \infty} \lim_{\npinfty} \riskcondition(\estimator) = \lim_{\npinfty} B_{k,\ell} + V_{k,\ell} \qas
\end{align}
In the following two lemmas, we characterize the variance and bias respectively. 
Recall the definitions of $\compstieltjes(0)$, $\tcompstieltjes(0)$, and $\tcompstieltjes'(0)$ from Lemma \ref{lm:v0}.

\begin{lemma}[Variance under correlated features]\label{lm:corBaggingvar}
Assume Assumptions~\ref{Assume:highdim}-\ref{Assume:multip}, and $\aspratio < \sampleratio$. For $k \neq \ell$, the variance term $V_{k,\ell}$ of the estimator $\estimator$ satisfies %almost surely that %the following convergence almost surely:
    \begin{equation*}
        \lim_{\npinfty} V_{k,\ell} = \begin{cases}
    \sigma^2\frac{\aspratio}{1-\aspratio}, & \aspratio < \sampleratio  \\
    \noiselev \left(\frac{\tcompstieltjes'(0)}{\tcompstieltjes(0)^{2}} - 1\right), \quad& \aspratio > \sampleratio \qas
\end{cases}
    \end{equation*}
   % where $\tcompstieltjes(0)$ is the unique {positive} solution of equation~\eqref{eq:fixpointbootstrap}. %\scomment{why does 4.2 have unique positive solution?}
\end{lemma}

\begin{lemma}[Bias under correlated features]\label{lm:corBaggingbias}
Assume Assumptions~\ref{Assume:highdim}-\ref{Assume:beta4+mom}. For $k \neq \ell$, the bias term $B_{k,\ell}$ of the estimator $\estimator$ satisfies %almost surely that %the following convergence almost surely:
    \begin{equation*}
        \lim_{\npinfty} B_{k,\ell} = \begin{cases}
    0, & \aspratio < \sampleratio \\
    \signallev \frac{\sampleratio}{\aspratio \compstieltjes(0)} - \signallev \frac{(1 - \sampleratio)}{\aspratio \compstieltjes(0)} \left(\frac{\tcompstieltjes'(0)}{\tcompstieltjes(0)^{2}} - 1 \right), \quad& \aspratio > \sampleratio
\end{cases} {\quad \as} 
    \end{equation*}
\end{lemma}

Putting the above results together finishes the proof. 

\end{proof}

%%%%%%%%Corollary 4.4
\subsection{Proof of Corollary~\ref{cor:overbagsketch}}
\begin{proof}[Proof of Corollary~\ref{cor:overbagsketch}]
    It suffices to show
    \begin{equation*}
        \frac{( {\tcompstieltjes'(0)}/{\tcompstieltjes(0)^2} - 1 )}{\left( {\compstieltjes'(0)}/{\compstieltjes(0)^2} - 1 \right)   } \leq \sampleratio
    \end{equation*}
    when $\aspratio > \sampleratio$, and 
    \begin{equation*}
        \lim_{\aspratio/\sampleratio \rightarrow 1}\frac{( {\tcompstieltjes'(0)}/{\tcompstieltjes(0)^2} - 1 )}{\left( {\compstieltjes'(0)}/{\compstieltjes(0)^2} - 1 \right)   } = 0
    \end{equation*}
    when $\sampleratio \neq 1$. 
    Using equation~\eqref{eq:vderivform}, we obtain
    \begin{equation*}
        \frac{\tcompstieltjes'(0)}{\tcompstieltjes(0)^2} - 1 = \frac{\frac{\aspratio}{\sampleratio^{2}} \int \frac{\tcompstieltjes(0)^{2}t^{2}}{(1 + \tcompstieltjes(0)t )^{2}} \tlimitdistr(dt)}{1 - \frac{\aspratio}{\sampleratio^{2}} \int \frac{\tcompstieltjes(0)^{2}t^{2}}{(1 + \tcompstieltjes(0)t )^{2}}\tlimitdistr(dt)}.
    \end{equation*}
To proceed, we need the following two lemmas.

\begin{lemma}\label{lm:v0tv}
    Assume Assumption~\ref{Assume:Covdistri}. Then $\tcompstieltjes(0) = \sampleratio \compstieltjes(0)$.
\end{lemma}

\begin{lemma}\label{lm:limitdistrequiv}
Assume Assumption~\ref{Assume:Covdistri} and $x \leq 0$. Let $\limitdistr$ and $\tlimitdistr_{x}$ be the limiting empirical spectral distribution of $\covmat$ and $(I + \ensambleweight(x)\covmat)\inv \covmat$, where $\ensambleweight(z):= (1 - \sampleratio)v(z/\sampleratio)$. Then, for any continuous function $f \in \mathcal{C}\left(\text{supp}(\limitdistr) \cup \text{supp}(\tlimitdistr_{x}) \right)$, it holds that
\begin{equation*}
    \int f(t) d\tlimitdistr_{x}(t) = \int f\left( \frac{t}{1 + \ensambleweight(x)t} \right) d\limitdistr(t).
\end{equation*}
\end{lemma}
    
Then,  using  Lemma~\ref{lm:v0tv} and Lemma~\ref{lm:limitdistrequiv}, we have
    \begin{align*}
        \frac{\aspratio}{\sampleratio^{2}} \int \frac{\tcompstieltjes(0)^{2}t^{2}}{(1 + \tcompstieltjes(0)t )^{2}} \tlimitdistr(dt) &= \aspratio \int \frac{\compstieltjes(0)^{2}t^{2}}{(1 + \sampleratio \compstieltjes(0)t )^{2}} \tlimitdistr(dt)\\
        &= \aspratio \int \frac{\compstieltjes(0)^{2}t^{2}}{(1 + \sampleratio \compstieltjes(0)t (1 - \sampleratio) \compstieltjes(0)t)^{2}} \limitdistr(dt)\\
        &= \aspratio \int \frac{\compstieltjes(0)^{2}t^{2}}{(1 + \compstieltjes(0)t } \limitdistr(dt).
    \end{align*}
    Similarly, using equation~\eqref{eq:vderivform}, we have
    \begin{equation*}
        \frac{\compstieltjes'(0)}{\compstieltjes(0)^2} - 1 = \frac{\frac{\aspratio}{\sampleratio} \int \frac{\compstieltjes(0)^{2}t^{2}}{(1 + \compstieltjes(0)t )^{2}} \limitdistr(dt)}{1 - \frac{\aspratio}{\sampleratio} \int \frac{\compstieltjes(0)^{2}t^{2}}{(1 + \compstieltjes(0)t )^{2}}\limitdistr(dt)}.
    \end{equation*}
    Therefore, 
    \begin{align*}
        \frac{\tcompstieltjes'(0)}{\tcompstieltjes(0)^2} - 1 &= \frac{\aspratio \int \frac{\compstieltjes(0)^{2}t^{2}}{(1 + \compstieltjes(0)t)^{2} } \limitdistr(dt)}{1 - \aspratio \int \frac{\compstieltjes(0)^{2}t^{2}}{(1 + \compstieltjes(0)t)^{2} } \limitdistr(dt)}\\
        &\leq \sampleratio \frac{\frac{\aspratio}{\sampleratio} \int \frac{\compstieltjes(0)^{2}t^{2}}{(1 + \compstieltjes(0)t)^{2} } \limitdistr(dt)}{1 - \frac{\aspratio}{\sampleratio} \int \frac{\compstieltjes(0)^{2}t^{2}}{(1 + \compstieltjes(0)t)^{2} } \limitdistr(dt)}\\
        &= \sampleratio \left( \frac{\compstieltjes'(0)}{\compstieltjes(0)^2} - 1 \right). 
    \end{align*}
    This finishes the proof for the first result.

    Taking $\aspratio/\sampleratio \rightarrow 1$, we have 
    \#\label{eq:var_small_order}     
        \lim_{\aspratio/\sampleratio \rightarrow 1}\frac{( {\tcompstieltjes'(0)}/{\tcompstieltjes(0)^2} - 1 )}{\left( {\compstieltjes'(0)}/{\compstieltjes(0)^2} - 1 \right)   } = \lim_{\aspratio/\sampleratio \rightarrow 1} \sampleratio \frac{1 - \aspratio/\sampleratio \int \frac{\compstieltjes(0)^{2}t^{2}}{(1 + \compstieltjes(0)t)^{2} } \limitdistr(dt)}{1 - \aspratio \int \frac{\compstieltjes(0)^{2}t^{2}}{(1 + \compstieltjes(0)t)^{2} } \limitdistr(dt)}.
    \#
   To proceed, we first calculate 
   \$
    \lim_{\aspratio/\sampleratio \rightarrow 1}  \int \frac{\compstieltjes(0)^{2}t^{2}}{(1 + \compstieltjes(0)t)^{2} } \limitdistr(dt). 
   \$
   Using equation~\eqref{eq:fixedpointsketching}, we have 
    \begin{equation*}
        \int \frac{1}{1 + \compstieltjes(0)t} d\limitdistr(t) = 1 - \aspratio/\sampleratio,
    \end{equation*}
    which implies $\lim_{\aspratio/\sampleratio \rightarrow 1} \compstieltjes(0) = +\infty$.
    Therefore, by the dominated convergence theorem, we have 
    \begin{equation*}
       \lim_{\aspratio/\sampleratio \rightarrow 1}  \int \frac{\compstieltjes(0)^{2}t^{2}}{(1 + \compstieltjes(0)t)^{2} } \limitdistr(dt) = \int \lim_{\aspratio/\sampleratio \rightarrow 1} \frac{\compstieltjes(0)^{2}t^{2}}{(1 + \compstieltjes(0)t)^{2} } \limitdistr(dt) = 1.
    \end{equation*}
 Plugging the above equality into \eqref{eq:var_small_order}, we obtain
\$
\lim_{\aspratio/\sampleratio \rightarrow 1}\frac{( {\tcompstieltjes'(0)}/{\tcompstieltjes(0)^2} - 1 )}{\left( {\compstieltjes'(0)}/{\compstieltjes(0)^2} - 1 \right)   } 
= \lim_{\aspratio/\sampleratio \rightarrow 1} \sampleratio \frac{1 - \aspratio/\sampleratio \int \frac{\compstieltjes(0)^{2}t^{2}}{(1 + \compstieltjes(0)t)^{2} } \limitdistr(dt)}{1 - \aspratio \int \frac{\compstieltjes(0)^{2}t^{2}}{(1 + \compstieltjes(0)t)^{2} } \limitdistr(dt)} 
= \theta \cdot \frac{0}{1- \gamma } 
= 0 . 
\$
    This finishes the proof. 
\end{proof}

%%%%%%%%Corollary 4.5
\subsection{Proof of Corollary \ref{cor:bootstrapequiv}}
\begin{proof}[Proof of Corollary \ref{cor:bootstrapequiv}]
The result can be observed from equations~\eqref{eq:fixpointcrossterm} and \eqref{eq:biasequiv}. As $\npinfty$, $B_{k,\ell}$ and $V_{k,\ell}$ converge to the limiting bias and variance terms of ridge regression with the covariance matrix of $(I + \ensambleweight(z)\covmat)^{-1} \covmat$ and the coefficient vector $(I + \ensambleweight(z)\covmat)^{-1/2} \truesignal$, respectively. Therefore, the proof is omitted.
\end{proof}

\subsection{Supporting lemmas}
This subsection proves the supporting lemmas in previous subsections, aka Lemmas \ref{lm:sketchundervar}-\ref{lm:sketchoverbias}. 

\subsubsection{Proof of Lemma \ref{lm:sketchundervar}}
\begin{proof}[Proof of Lemma \ref{lm:sketchundervar}]
For simplicity, we assume $\sigma^2 = 1$, and omit the subscript $1$ in $\hcovmat_{1}$ and $\sketchmat_{1}$, which become $\hcovmat$ and $\sketchmat$ respectively. Let $\Omega_{\ndata} := \{ \omega \in \Omega : \eigval(\hcovmat)(\omega) > 0 \}$. Using Lemma~\ref{lm:singularmat} and Lemma~\ref{lm:lowerboundeigval}, we obtain
\begin{equation*}
    \PP \left( \lim_{\npinfty} \Omega_{\ndata} \right) = 1,
\end{equation*}
when $\aspratio < \sampleratio$. Thus, without loss of generality, we assume $\eigval_{\min}(\hcovmat) > 0$. Then we can write $\varcondition(\hat{\beta}^{1})$ as 
\begin{align*}
\varcondition(\hat{\beta}^{1}) &=  \frac{1}{n^{2}} \tr\left( \hcovmat\pinv X\transp \sketchmat^2 \sketchmat^2 X \hcovmat\pinv \covmat \right),\\
    &= \frac{1}{n^{2}} \tr\left( (\covmat^{-1/2} \hcovmat \covmat^{-1/2} )\inv \covmat^{-1/2} X\transp \sketchmat^2 \sketchmat^2 X \covmat^{-1/2} (\covmat^{-1/2} \hcovmat \covmat^{-1/2} )\inv \right).
\end{align*}
Hence, given Assumption~\ref{Assume:Covdistri}, we can assume $\covmat =I$ without loss of generality. 

Assuming $\covmat =I$, we can simplify the variance term as %proceed by expressing the variance term as follows:
\begin{align*}
    \varcondition(\hat{\beta}^{1}) &=  \frac{1}{n^{2}} \tr\left( \hcovmat\inv X\transp \sketchmat^2 \sketchmat^2 X \hcovmat\inv \right),\\
    &= \frac{1}{n^{2}} \tr\left( \hcovmat\inv \sum_{i=1}^{n} \weight_{i}^2 x_{i}\transp x_{i} \hcovmat\inv \right)\\
    &= \frac{1}{n^{2}} \sum_{i=1}^{n} \weight_{i}^2 x_{i}\transp \hcovmat^{-2} x_{i}\\
    &= \frac{1}{n^{2}} \sum_{i=1}^{n} \weight_{i}^2 \frac{x_{i}\transp \hcovmat_{-i}^{-2} x_{i}}{\left(1 + \weight_{i} x_{i}\transp \hcovmat_{-i}^{-1} x_{i} \right)^2},
\end{align*}
where $\hcovmat_{-i}:= \hcovmat - \frac{1}{\ndata}\weight_{i} x_{i} x_{i}\transp$, and the last line uses the Sherman–Morrison formula twice.

Define $\stieltjes_{1,n}(z)$, the Stieltjes transform of $\hcovmat$,  and  $\stieltjes_{1,n}'(z)$ as %the Stieltjes transforms of $\hcovmat$ as follows: 
\begin{align*}
    \stieltjes_{1,n}(z) &= \int \frac{1}{t-z} dF_{\hcovmat}(t), ~~~
    \stieltjes_{1,n}'(z) = \int \frac{1}{(t-z)^2} dF_{\hcovmat}(t),
\end{align*}
which are well-defined for any $z \leq 0$. Applying Lemma~\ref{lm:boundrankoneperturb} and Lemma~\ref{lm:concentrationontrace}, we obtain almostly surely that
\begin{equation*}
    \lim_{\npinfty} \varcondition(\hat{\beta}^{1}) = \lim_{\npinfty} \frac{1}{\ndata} \sum_{i=1}^{n} \weight_{i}^2 \frac{\aspratio \stieltjes_{1,n}'(0)}{\left(1 + \aspratio \weight_{i} \stieltjes_{1,n}(0) \right)^2}. 
\end{equation*}  
Lemma~\ref{lm:MPlawsketching} shows that the empirical spectral distribution of $\hcovmat$ converges weakly to a deterministic distribution $\mu$ almost surely as $\npinfty$, characterized by its Stieltjes transform $\stieltjes_{1}(z)$ that satisfies the following equation:
\begin{equation}\label{eq:varsketching1}
    \stieltjes_{1}(z) \EE_{\limitdistrweight}\left[ \frac{\weight}{1 + \aspratio \weight \stieltjes_{1}(z)} \right] - z\stieltjes_{1}(z) = 1,
\end{equation}
for any $z \leq 0$. Now since $\eigval_{\min}(\hcovmat_{k}) > 0$, we have almost surely
\begin{align*}
    \lim_{\npinfty} \stieltjes_{1,n}(0) &= \lim_{\npinfty} \int \frac{1}{t} dF_{\hcovmat}(t) = \stieltjes_{1}(0), \\
    \lim_{\npinfty} \stieltjes_{1,n}'(0) &= \lim_{\npinfty} \int \frac{1}{t^{2}} dF_{\hcovmat}(t) = \stieltjes_{1}'(0),
\end{align*}
where the second line follows from the fact that $\stieltjes_{1,n}(z)$ is analytic and bounded on $\RR^{-} \cup \{0\}$, and we apply the Vitali's convergence theorem. 
 Since $\stieltjes_{1}(0) > 0$ according to Lemma~\ref{lm:MPlawsketching}, we have
\begin{equation*}
     \frac{\weight_{i}^2}{\left(1 + \aspratio \weight_{i} \stieltjes_{1}(0) \right)^2} \leq \frac{1}{\gamma^2 (\stieltjes_{1}(0))^{2}} <\infty. 
\end{equation*}
Thus, under Assumption~\ref{Assume:multip}, we use the dominated convergence theorem to obtain
\begin{equation*}
    \lim_{\npinfty} \varcondition(\hat{\beta}^{1}) 
    =   \aspratio \stieltjes_{1}'(0) \, \EE_{\limitdistrweight}\left[ \frac{\weight^2}{\left(1 + \aspratio \weight \stieltjes_{1}(0) \right)^2} \right],
\end{equation*}
almost surely.

We can simplify this result by using equation~\eqref{eq:varsketching1}. Since the Stieltjes transform $\stieltjes_{1}(z) = \int \frac{1}{t - z} d \mu(t)$ is strictly increasing and positive for any $z \leq 0$, we obtain for any $z \leq 0$
\begin{equation*}
      \frac{\weight^2 \stieltjes_{1}'(z)}{(1 +  \weight \stieltjes_{1}(z))^2} \leq \frac{\stieltjes_{1}'(z)}{(\aspratio\stieltjes_{1}(z))^{2}} \leq \frac{1}{(\aspratio \eigval_{\min}(\hcovmat_{k}) \stieltjes_{1}(z))^{2}},
\end{equation*}
which is uniformly bounded over any compact interval $I \subset (-\infty, 0]$. Applying the dominated convergence theorem, we take derivatives on both sides of equation~\eqref{eq:varsketching1} to obtain
\begin{equation*}
    \stieltjes_{1}'(z) \EE_{\limitdistrweight} \left[ \frac{\weight}{1 + \aspratio \weight \stieltjes_{1}(z)} \right] + \aspratio \stieltjes_{1}'(z) \stieltjes_{1}(z) \EE_{\limitdistrweight}\left[ \frac{\weight^2}{(1 + \aspratio \weight \stieltjes_{1}(z))^2} \right] - \stieltjes_{1}(z) - z \stieltjes_{1}'(z) = 0.
\end{equation*}
Taking the limit as $z \rightarrow 0^-$, we obtain
\begin{align*}
    \stieltjes_{1}'(0) \EE_{\limitdistrweight} \left[ \frac{\weight}{(1 + \aspratio \weight \stieltjes_{1}(0))^2} \right] = \stieltjes_{1}(0), 
\end{align*}
which leads to
\begin{align*}
    \lim_{\npinfty} \varcondition(\hat{\beta}^{1}) &= \aspratio \stieltjes_{1}'(0) \EE_{\limitdistrweight}\left[ \frac{\weight^2}{\left(1 + \aspratio \weight \stieltjes_{1}(0) \right)^2} \right]\\
    &= \EE_{\limitdistrweight} \left[ \frac{\aspratio \weight^2 \stieltjes_{1}(0)}{(1 + \aspratio \weight \stieltjes_{1}(0))^2} \right] / \EE_{\limitdistrweight} \left[ \frac{\weight}{(1 + \aspratio \weight \stieltjes_{1}(0))^2} \right]\\
    &= \left( \EE_{\limitdistrweight} \left[ \frac{\weight}{1 + \aspratio \weight \stieltjes_{1}(0)} \right] - \EE_{\limitdistrweight} \left[ \frac{\weight}{(1 + \aspratio \weight \stieltjes_{1}(0))^2} \right]\right) / \EE_{\limitdistrweight} \left[ \frac{\weight}{(1 + \aspratio \weight \stieltjes_{1}(0))^2} \right]\\
    &= \frac{1}{\EE_{\limitdistrweight} \left[ \frac{\weight \stieltjes_{1}(0)}{(1 + \aspratio \weight \stieltjes_{1}(0))^2} \right]} - 1\\
    &= \frac{\aspratio}{\EE_{\limitdistrweight} \left[ \frac{1}{1 + \aspratio \weight \stieltjes_{1}(0)} \right] - \EE_{\limitdistrweight} \left[ \frac{1}{(1 + \aspratio \weight \stieltjes_{1}(0))^2} \right]} - 1\\
    &= \frac{\aspratio}{1 - \aspratio - \EE_{\limitdistrweight} \left[ \frac{1}{(1 + \aspratio \weight \stieltjes_{1}(0))^2} \right]} - 1~~~\text{almost surely.}
\end{align*} 
The fourth and the last lines use equation~\eqref{eq:varsketching1}. Finally, by  Lemma~\ref{lm:MPlawsketching}, $\stieltjes_{1}(0)$ is the unique positive solution of equation~\eqref{eq:uniqueMPlawsketching}. This finishes the proof. %Thus, the proof is complete.
\end{proof}

%%%%%%%%%%%Proof of Lemma sketchovervar
\subsubsection{Proof of Lemma \ref{lm:sketchovervar}}
\begin{proof}[Proof of Lemma \ref{lm:sketchovervar}]

Without loss of generality, we assume $\sigma^2 = 1$,  and omit the subscript $1$ in 
$\hcovmat_{1}$ and $\sketchmat_{1}$. {Let $S_{ij}$ be the $(i,j)$-th element of $S$}, and
\$ 
\Bernoulliweight_{ij} 
= 
\begin{cases}
    1(\sketchmat_{ij} \neq 0), &\quad i=j \\
    0, &\quad i\neq j
\end{cases}.
\$ 
Applying Lemma \ref{lm:Pseudo_products}, we can rewrite $\varcondition(\hat{\beta}^{1})$ as
%\$ 
%\Bernoulliweight_{i,j} 
%= \begin{cases}
%    1_{\{\sketchmat_{i,j} \neq 0\}}, &\quad i=j \\
%    0, &\quad i\neq j
%  \end{cases}.
%\$ 
%\end{align}
%$\Bernoulliweight$   for the corresponding sketch matrices $\sketchmat$. 
%Applying Lemma~\ref{lm:Pseudo_products}, we have: 
\begin{align*}
    \varcondition(\hat{\beta}^{1}) 
    &=  \tr\left( (X\transp \sketchmat \sketchmat X)\pinv X\transp \sketchmat^2 \sketchmat^2 X (X\transp \sketchmat \sketchmat X)\pinv \covmat \right)\\
    &=  \tr\left( (X\transp \Bernoulliweight\transp \Bernoulliweight X)\pinv X\transp (\Bernoulliweight\transp)^{2} \Bernoulliweight^{2} X (X\transp \Bernoulliweight\transp \Bernoulliweight X)\pinv  \covmat \right).
\end{align*}
Without loss of generality, we assume that each $\weight_{i,j},\, 1\leq i \leq j \leq \ndata$,  is either one or zero. Applying the following identity of the pseudoinverse of a matrix $A$
\begin{equation}\label{eq:pesudoidentity}
    (A\transp A)\pinv A\transp = \lim_{\zinfty} (A\transp A - zI)\inv A\transp,
\end{equation}
we obtain
\begin{align}
    \varcondition(\hat{\beta}^{1}) &= \lim_{\zinfty} \frac{1}{\ndata} \tr\left( (\hcovmat - zI)\inv \hcovmat (\hcovmat - zI)\inv \covmat \right) \notag\\ \label{eq:sketchovervar1}
    &=\frac{1}{\ndata} \lim_{\zinfty} \left\{ \tr \left((\hcovmat - zI)\inv \covmat \right) + z \, \tr\left((\hcovmat - zI)^{-2} \covmat \right) \right\}.
\end{align}
For any $z < 0$, Lemma~\ref{lm:stieltjescov} shows 
\begin{equation*}
   \lim_{\npinfty} \frac{1}{\ndim}\tr\left((\hcovmat - zI)\inv \covmat \right) = \stieltjescov_{1}(z)~{\qas}
\end{equation*}
%\scomment{current point.}
The second term in equation~\eqref{eq:sketchovervar1} is the derivative of $\tr\left((\hcovmat - zI)\inv \covmat \right)$. For any small constant $\epsilon >0$, $\tr\left((\hcovmat - zI)\inv \covmat \right)$ is almost surely uniformly bounded on all $z < -\epsilon$.  Moreover, it is analytic with respect to $z$. Thus we can apply Vitali convergence theorem, aka Lemma \ref{lm:Vitali}, to obtain %for any $z<0$ that 
\begin{align*}
    \lim_{\npinfty} \frac{1}{\ndim} \tr\left((\hcovmat - zI)\inv \covmat \right) + z\,  \tr\left((\hcovmat - zI)^{-2} \covmat \right) = \stieltjescov_{1}(z) + z\stieltjescov'_{1}(z) {\qas} 
\end{align*}
for any $z < 0$.

Finally, we show that we can exchange the limits between $\npinfty$ and $\zinfty$ by applying the Arzela-Ascoli theorem and the Moore-Osgood theorem (Lemma \ref{lm:Moore-Osgood}). To achieve  this, we first establish that $\tr\left( (\hcovmat - zI)\inv \hcovmat (\hcovmat - zI)\inv \covmat \right)/\ndim$ and its derivative are uniformly bounded for all $z < 0$. By taking the derivative, we obtain 
\begin{align*}
  \frac{1}{\ndim}\frac{{\rm d}\, \tr\left( (\hcovmat - zI)\inv \hcovmat (\hcovmat - zI)\inv \covmat \right)}{{\rm d} z} 
  &= \frac{2}{\ndim}\tr\left((\hcovmat - zI)^{-2} \hcovmat (\hcovmat - zI)\inv \covmat \right) \\
  &\leq \frac{2 \eigval_{\max}(\covmat)}{\eigval^{+}_{\min}(\hcovmat)^{2}},
\end{align*}
which is almost surely bounded by Lemma~\ref{lm:lowerboundeigval}. 
For $\tr\left( (\hcovmat - zI)\inv \hcovmat (\hcovmat - zI)\inv \covmat \right)/\ndim$, a similar calculation leads to
\$
\frac{\tr\left( (\hcovmat - zI)\inv \hcovmat (\hcovmat - zI)\inv \covmat \right)}{\ndim} 
\leq \frac{\eigval_{\max}(\covmat)}{\eigval^{+}_{\min}(\hcovmat)}. 
\$
Then,  using the Arzela-Ascoli theorem, we obtain the uniform convergence of $ \varcondition(\hat{\beta}^{1})$. Applying the Moore-Osgood theorem,  we obtain almost surely that %\scomment{what theorem guarantees this?}
\begin{align*}
    \lim_{\npinfty}  \varcondition(\hat{\beta}^{1}) &= \lim_{\npinfty} \lim_{\zinfty} \frac{1}{\ndata} \tr\left( (\hcovmat - zI)\inv \hcovmat (\hcovmat - zI)\inv \covmat \right) \\
    &= \lim_{\zinfty} \lim_{\npinfty}  \frac{1}{\ndata} \tr\left( (\hcovmat - zI)\inv \hcovmat (\hcovmat - zI)\inv \covmat \right) \\
    &=  \lim_{\zinfty} \aspratio (\stieltjescov_{1}(z) + z\stieltjescov'(z))\\
    &= \lim_{\zinfty} \frac{\compstieltjes'(z)}{\compstieltjes(z)^{2}} - 1\\
    &= \frac{\compstieltjes'(0)}{\compstieltjes(0)^{2}} - 1,
\end{align*}
where the existence of $\compstieltjes'(z)$  and the last line follow from Lemma~\ref{lm:v0}. The fourth line follows from Lemma~\ref{lm:stieltjescov}.
\end{proof}

\subsubsection{Proof of Lemma \ref{lm:sketchoverbias}}
\begin{proof}[Proof of Lemma \ref{lm:sketchoverbias}]

For notational simplicity, we assume $\signallev = 1$, omit the subscript $i$, and write $\hcovmat_{i}$ and $\sketchmat_{i}$ as $\hcovmat$ and $\sketchmat$ respectively. Applying Lemma~\ref{lm:biasvar}, we can rewrite the bias term as %\scomment{missing negative and inverse?}
    \begin{align*}
        \lim_{\npinfty} \biascondition(\hat{\beta}^{1}) &= \lim_{\npinfty} \frac{1}{\ndim} \tr \left( \projmat \covmat \projmat \right)\\
        &= \lim_{\npinfty} \frac{1}{\ndim} \tr \left( (I - \hcovmat\pinv \hcovmat) \covmat \right)\\
        &= \lim_{\npinfty} \lim_{\zinfty} \frac{1}{\ndim} \, \tr \left( (I - (\hcovmat - zI)\inv \hcovmat) \covmat \right)\\
        &= \lim_{\npinfty} \lim_{\zinfty} -\frac{z}{\ndim} \, \tr \left( (\hcovmat - zI)\inv \covmat \right),
    \end{align*}
    where we used equation~\eqref{eq:pesudoidentity}. According to Lemma~\ref{lm:stieltjescov}, it holds almost surely that
    \begin{equation*}
        \lim_{\npinfty} \frac{1}{\ndim}  \, \tr \left(  (\hcovmat - zI)\inv \covmat \right) 
        =  \, \stieltjescov_{1}(z). 
    \end{equation*}
    By using similar arguments as in the proof of Lemma~\ref{lm:sketchovervar}, we can exchange the limits between $\npinfty$ and $\zinfty$. By applying Lemma~\ref{lm:stieltjescov}, the following holds almost surely
    \begin{align*}
        \lim_{\npinfty} \biascondition(\hat{\beta}^{1}) &= \lim_{\npinfty} \lim_{\zinfty} -\frac{z}{\ndim}  \tr \left( (\hcovmat - zI)\inv \covmat \right)\\
        &= \lim_{\zinfty} \lim_{\npinfty} -\frac{z}{\ndim}  \tr \left( (\hcovmat - zI)\inv \covmat \right)\\
        &= \lim_{\zinfty} -z \stieltjescov(z)\\
        &= \frac{ \sampleratio}{\aspratio \compstieltjes_{1}(0)},
    \end{align*}
    where the last line follows from Lemma~\ref{lm:v0}.
\end{proof}

%%%%%%%%%Bagging variance%%%%%%%%%%
\subsubsection{Proof of Lemma \ref{lm:corBaggingvar}}
\begin{proof}[Proof of Lemma~\ref{lm:corBaggingvar}]

We start with the underparameterized regime. Using a similar argument as in the proof of Lemma~\eqref{lm:sketchundervar}, we assume $\eigval_{\min} (\hcovmat_{k})>0$ and $\eigval_{\min} (\hcovmat_{\ell})>0$. We first rewrite the variance term as
\begin{align*}
    V_{k,\ell} &= \frac{\noiselev}{\ndata^{2}} \tr\left(\hcovmat_{k}\inv X\transp \sketchmat_{k}^2 \sketchmat_{\ell}^2 X \hcovmat_{\ell}\inv \covmat \right)\\
    &= \frac{\noiselev}{\ndata^{2}} \tr\left((\covmat^{-1/2}X\transp \sketchmat_{k}\sketchmat_{k} X\covmat^{-1/2})\inv \covmat^{-1/2} X\transp \sketchmat_{k}^2 \sketchmat_{\ell}^2 X \covmat^{-1/2} (\covmat^{-1/2}X\transp \sketchmat_{\ell}\sketchmat_{\ell} X\covmat^{-1/2})\inv \right).
\end{align*}
Under Assumption~\ref{Assume:Covdistri},  $X\covmat^{-1/2}$ has isotropic features. Therefore, the limiting variance $V_{k,\ell}$ is the same as in the case of the isotropic features in Lemma~\ref{lm:Isobootvar_under}.

In the overparameterized regime, following the same argument as in Lemma~\ref{lm:Isobootvar_over}, we can  assume that the multipliers $\weight_{i,j},~1\leq i \leq j \leq \ndata$ are either one or zero, without loss of generality. We further assume $\noiselev = 1$. Using identity~\eqref{eq:pesudoidentity}, we obtain
\begin{align*}
    V_{k,\ell} = \lim_{\zinfty} \frac{1}{\ndata^{2}} \tr\left((\hcovmat_{k}-zI)\inv X\transp \sketchmat_{k}^2 \sketchmat_{\ell}^2 X (\hcovmat_{\ell}-zI)\inv \covmat \right).
\end{align*}
Let
\begin{equation*}
    \C: = \{ i~|~ \weight_{k,i} \neq 0,\weight_{\ell,i} \neq 0 \},  \quad\hcovmat_{\C} =: \sum_{i \in \C} \frac{1}{\ndata} x_{i}x_{i}\transp, \quad \ensambleweight(z): = (1 - \sampleratio)v(z/\sampleratio).
\end{equation*}
Using  Lemma~\ref{lm:Crossterm} acquires
\begin{align}
        &\lim_{\npinfty} \frac{1}{\ndata^{2}}\tr\left( (\hcovmat_{k} - zI)\inv X\transp \sketchmat_{k}^{2} \sketchmat_{\ell}^{2} X (\hcovmat_{\ell} - zI) \inv \covmat \right) \notag\\ 
        &\quad\quad=\lim_{\npinfty} \frac{\aspratio}{\ndim}\tr\left( (-z\ensambleweight(z)\covmat + \hcovmat_{\C} - zI)^{-1} \hcovmat_{\C} (-z\ensambleweight(z)\covmat + \hcovmat_{\C} - zI)^{-1} \covmat \right)  \notag\\ \label{eq:fixpointcrossterm}
        &\quad\quad=\lim_{\npinfty} \frac{\aspratio}{\ndim}\tr\left( (\tcovmat_{\C} - zI)^{-1} \tcovmat_{\C} (\tcovmat_{\C} - zI)^{-1} \covmat (I + \ensambleweight(z)\covmat)\inv \right) {\quad \as} 
    \end{align}
where $\tcovmat_{\C}(z):= (I + \ensambleweight(z)\covmat)^{-1/2} \hcovmat_{\C}(I + \ensambleweight(z)\covmat)^{-1/2}$. Furthermore, under Assumption~\ref{Assume:multip}, the cardinality of $\C$ satisfies
     \begin{equation} \label{eq:crossterm5}
         |\C|/\ndata = \frac{1}{\ndata} \sum_{i=1}^{\ndata} 1\left({\weight_{k,i} \neq 0, \weight_{\ell,i} \neq 0}\right) \rightarrow \sampleratio^{2} {\quad \as}
     \end{equation} 
Thus, we can see $\tcovmat_{\C}(z)$ as a sample covariance matrix with sample size $\sampleratio^{2}\ndata$ and a population covariance matrix $(I + \ensambleweight(z)\covmat)^{-1} \covmat$. Define $\tlimitdistr_{x}$ as the limiting empirical spectral distribution of $(I + \ensambleweight(x)\covmat)^{-1} \covmat$, which exists under Assumption~\ref{Assume:Covdistri}. Let  $\tcompstieltjes(z,x)$ be the unique positive solution of the following equation
\begin{equation}\label{eq:fixpointkx}
    \tcompstieltjes(z,x) = \left( -z + \frac{\aspratio}{\sampleratio^{2}} \int \frac{t\tlimitdistr_{x}(dt)}{1 + \tcompstieltjes(z,x)t} \right)^{-1}.
\end{equation}
The existence and uniqueness of the positive solution to  equation~\eqref{eq:fixpointkx} follows from the same argument as in the proof of Lemma~\ref{lm:v0}. %\scomment{existence and uniqueness? check one more time.. } 
Note that the term in equation~\eqref{eq:fixpointcrossterm} can be seen as the variance of the sketched estimator with aspect ratio $\aspratio/\sampleratio^{2}$ and covariance matrix $(I + \ensambleweight(x)\covmat)^{-1} \covmat$. 
Then, it has been proved in Lemma~\ref{lm:sketchovervar} that 
\begin{equation*}
    \lim_{\npinfty} \frac{\aspratio}{\ndim}\tr\left( (\tcovmat_{\C} - zI)^{-1} \tcovmat_{\C} (\tcovmat_{\C} - zI)^{-1} \covmat (I + \ensambleweight(z)\covmat)\inv \right) = \frac{\tcompstieltjes'(z,z)}{\tcompstieltjes(z,z)^{2}} - 1 {\quad \as}
\end{equation*}
Following the same argument as in the proof of Lemma~\ref{lm:Isobootvar_over}, we can exchange the limits between $\npinfty$ and $\zinfty$, and obtain 
\begin{align*}
    \lim_{\npinfty} V_{k,\ell} &= \lim_{\npinfty} \lim_{\zinfty} \frac{\aspratio}{\ndim}\tr\left( (\tcovmat_{\C} - zI)^{-1} \tcovmat_{\C} (\tcovmat_{\C} - zI)^{-1} \covmat (I + \ensambleweight(z)\covmat)\inv \right)\\
    &= \lim_{\zinfty} \lim_{\npinfty} \frac{\aspratio}{\ndim}\tr\left( (\tcovmat_{\C} - zI)^{-1} \tcovmat_{\C} (\tcovmat_{\C} - zI)^{-1} \covmat (I + \ensambleweight(z)\covmat)\inv \right)\\
    &= \lim_{\zinfty} \frac{\tcompstieltjes'(z,z)}{\tcompstieltjes(z,z)^{2}} - 1 {\quad \as} \\
    &= \frac{\tcompstieltjes'(0)}{\tcompstieltjes(0)^{2}} - 1,
\end{align*}
where the last line uses Lemma~\ref{lm:fixpointequiv}.
\end{proof}

\subsubsection{Proof of Lemma~\ref{lm:corBaggingbias}}
\begin{proof}[Proof of Lemma~\ref{lm:corBaggingbias}]
Using a similar argument as in the proof of Lemma~\ref{lm:isobootbias}, in the underparameterized regime, the bias term converges almost surely to zero. 

When $\aspratio > \sampleratio$, using the argument as in the proof of Lemma~\ref{lm:Isobootvar_over}, we can assume that the multipliers $\weight_{i,j}$ are either zero or one and $\noiselev = 1$ without loss of generality. We rewrite the bias term as 
\begin{align*}
    B_{k,\ell} &= \lim_{\zinfty} \frac{\signallev}{\ndim} \tr\left((I - (\hcovmat_{k} - zI)\inv \hcovmat_{k}) \covmat (I - (\hcovmat_{\ell} - zI)\inv \hcovmat_{\ell}) \right)\\
    &= \lim_{\zinfty} \left(-\frac{\signallev z}{\ndim} \tr\left((\hcovmat_{k} - zI)\inv \covmat \right) + \frac{\signallev z}{\ndim} \tr\left((\hcovmat_{k} - zI)\inv \hcovmat_{\ell} (\hcovmat_{\ell} - zI)\inv \covmat \right)\right).
\end{align*}
Assuming we can exchange the limits between $\npinfty$ and $\zinfty$, we obtain
\$
\lim_{\npinfty} B_{k,\ell} 
&= \lim_{\npinfty} \lim_{\zinfty} \left(-\frac{\signallev z}{\ndim} \tr\left((\hcovmat_{k} - zI)\inv \covmat \right) + \frac{\signallev z}{\ndim} \tr\left((\hcovmat_{k} - zI)\inv \hcovmat_{\ell} (\hcovmat_{\ell} - zI)\inv \covmat \right)\right) \\
&= \lim_{\zinfty}  \lim_{\npinfty} -\frac{\signallev z}{\ndim} \tr\left((\hcovmat_{k} - zI)\inv \covmat \right) \\
&\quad\quad\quad +   \lim_{\zinfty}  \lim_{\npinfty}   \frac{\signallev z}{\ndim} \tr\left((\hcovmat_{k} - zI)\inv \hcovmat_{\ell} (\hcovmat_{\ell} - zI)\inv \covmat \right)\\
&= \lim_{\zinfty}  \lim_{\npinfty}  \Rom{1} + \lim_{\zinfty}  \lim_{\npinfty}  \Rom{2}. 
\$
The first term \Rom{1} in the above equation has already appeared in Lemma~\ref{lm:sketchoverbias} and satisfies
\begin{equation*}
    \lim_{\zinfty} \lim_{\npinfty}  \Rom{1} =  \lim_{\zinfty} \lim_{\npinfty} -\frac{\signallev z}{\ndim} \tr\left((\hcovmat_{k} - zI)\inv \covmat \right) = \signallev \frac{\sampleratio}{\aspratio \compstieltjes(0)} {\quad \as}
\end{equation*}
Recall the definitions of $\hcovmat_{\A}$, $\hcovmat_{\B}$, and $\hcovmat_{\C}$ defined in Lemma~\ref{lm:Crossterm} and its proof.  We rewrite the second term as
\begin{align*}
    \Rom{2} &= \frac{r^2 z}{\ndim} \tr\left((\hcovmat_{k} - zI)\inv \hcovmat_{\ell} (\hcovmat_{\ell} - zI)\inv \covmat \right) \\
    &= \frac{r^2 z}{\ndim} \tr\left((\hcovmat_{\A} + \hcovmat_{\C} - zI)\inv (\hcovmat_{\B}  + \hcovmat_{\C}) (\hcovmat_{\B} + \hcovmat_{\C} - zI)\inv \covmat \right)\\
    &= \frac{r^2 z}{\ndim} \tr\left((\hcovmat_{\A} + \hcovmat_{\C} - zI)\inv \hcovmat_{\B}   (\hcovmat_{\B} + \hcovmat_{\C} - zI)\inv \covmat \right) \\
    &\qquad + \frac{r^2 z}{\ndim} \tr\left((\hcovmat_{\A} + \hcovmat_{\C} - zI)\inv \hcovmat_{\C}   (\hcovmat_{\B} + \hcovmat_{\C} - zI)\inv \covmat \right) \\
    &=: \Rom{2}_1 + \Rom{2}_2. 
\end{align*}
We derive the limits of $\Rom{2}_1$ and $\Rom{2}_2$ respectively. We start with $\Rom{2}_2$. Following the proof of Lemma~\ref{lm:corBaggingvar}, we obtain 
\$
  \lim_{\zinfty} \lim_{\npinfty} \Rom{2}_2 
  &= \lim_{\zinfty} \lim_{\npinfty} \frac{r^2 z}{\ndim} \tr\left((\hcovmat_{\A} + \hcovmat_{\C} - zI)\inv \hcovmat_{\C} (\hcovmat_{\B} + \hcovmat_{\C} - zI)\inv \covmat \right) \\
  &= \lim_{\zinfty} \lim_{\npinfty} \frac{r^2 z}{\aspratio} V_{k,\ell} = 0 {\quad \as}
\$
For term $\Rom{2}_1$, using Lemma~\ref{lm:biassecond}, we obtain %s\scomment{$v(z/\theta)$ instead of $v(z)$?}
\begin{align*}
   \lim_{\zinfty} \lim_{\npinfty}\Rom{2}_1
    &= \lim_{\zinfty} \lim_{\npinfty} \frac{r^2 z}{\ndim} \tr\left((\hcovmat_{\A} + \hcovmat_{\C} - zI)\inv \hcovmat_{\B} (\hcovmat_{\B} + \hcovmat_{\C} - zI)\inv \covmat \right)\\
    &=  \lim_{\zinfty} \lim_{\npinfty} -\frac{r^2 (1 - \sampleratio)}{\compstieltjes(z/\sampleratio)} \frac{1}{\ndim}\tr\left((\hcovmat_{\A} + \hcovmat_{\C} - zI)\inv \hcovmat_{\C} (\hcovmat_{\B} + \hcovmat_{\C} - zI)\inv \covmat \right)\\
    &= \lim_{\zinfty}  \lim_{\npinfty}  -\frac{r^2(1 - \sampleratio)}{\aspratio \compstieltjes(z/\sampleratio)}     V_{k,\ell}\\
    &= -\frac{r^2(1 - \sampleratio)}{\aspratio \compstieltjes(0)} \left(\frac{\tcompstieltjes '(0)}{\tcompstieltjes(0)^{2}} - 1 \right) {\quad \as}
\end{align*}

Lastly, the validity of exchanging the limits between $\npinfty$ and $\zinfty$ follows from the  argument as in the proof of  Lemma~\ref{lm:isobootbias}. This finishes the proof. 

\end{proof}

\subsubsection{Proof of Lemma~\ref{lm:v0tv}}

\begin{proof}[Proof of Lemma~\ref{lm:v0tv}]
    By Lemma~\ref{lm:limitdistrequiv}, we have 
    \begin{align*}
        \frac{\aspratio}{\sampleratio^{2}} \int \frac{\sampleratio \compstieltjes(0) t  \tlimitdistr(dt)}{1 + \sampleratio \compstieltjes(0)t} &= \frac{\aspratio}{\sampleratio^{2}} \int \frac{\sampleratio \compstieltjes(0) t  \limitdistr(dt)}{1 + (\sampleratio \compstieltjes(0) + (1 - \sampleratio)\compstieltjes(0))t}\\
        &= \frac{\aspratio}{\sampleratio} \int \frac{\compstieltjes(0) t  \limitdistr(dt)}{1 + \compstieltjes(0)t}\\
        &= 1,
    \end{align*}
   % where the first line follows from Lemma~\ref{lm:limitdistrequiv}, and 
   where the last line follows from the fact that $\compstieltjes(0)$ is a solution to equation~\eqref{eq:fixedpointsketching} at $z=0$. The desired result follows from that $\tcompstieltjes(0)$ is the unique positive solution of equation~\eqref{eq:fixpointbootstrap}.
\end{proof}

\subsubsection{Proof of Lemma~\ref{lm:limitdistrequiv}}
\begin{proof}[Proof of Lemma~\ref{lm:limitdistrequiv}]
Using  Assumption~\ref{Assume:Covdistri}, we have %\scomment{upper bound is not enough?}
    \begin{equation*}
        c_{\eigval}\leq \eigval_{\min}(\covmat) \leq \eigval_{\max}(\covmat) \leq C_{\eigval}, \quad 0\leq \eigval_{\min}((I + \ensambleweight(x)\covmat)\inv \covmat) \leq \eigval_{\max}((I + \ensambleweight(x)\covmat)\inv \covmat) \leq 1/\ensambleweight(x).
    \end{equation*}
    Therefore, for any $x \leq 0$ and continuous function $f$, $f$ is bounded on $\text{supp}(\limitdistr) \cup \text{supp}(\tlimitdistr_{x}) $. By the definitions of $\limitdistr$ and $\tlimitdistr_{x}$, we have 
    \begin{align*}
        \int f(t) d\tlimitdistr_{x}(t) &= \lim_{\npinfty} \frac{1}{\ndim}\sum_{i=1}^{\ndim} f \left( \frac{\eigval_{i}(\covmat)}{1 + \ensambleweight(x)\eigval_{i}(\covmat)} \right) \\
        &= \int f\left( \frac{t}{(1 + \ensambleweight(x)t)} \right) d\limitdistr(t).
    \end{align*}
    This completes the proof.
\end{proof}

\subsection{Technical lemmas}
This subsection proves technical lemmas that are used in the proofs of the supporting lemmas in the previous subsection. Our first lemma provides an extension of \cite[Lemma 2.1]{ledoit2011} to the sketched covariance matrix $X\transp \sketchmat\transp \sketchmat X/n $.

\begin{lemma}\label{lm:stieltjescov} 
    Assume Assumptions~\ref{Assume:highdim}-\ref{Assume:multip},  and $\aspratio > \sampleratio$. Let $X_{k}:= \sketchmat_{k}X $ be a subsample with multipliers of either zero or one. For any $z < 0$, it holds 
    \begin{equation*}
        \lim_{\npinfty} \frac{1}{\ndim} \tr((\hcovmat_{k} - zI)\inv \covmat) = \stieltjescov_{1}(z) \qas 
    \end{equation*}
    where %$\stieltjescov_{1}(z)$ is defined as
    \begin{align*}
    \stieltjescov_{1}(z) &= \frac{\sampleratio/\aspratio^{2}}{\sampleratio/\aspratio - 1 - z\stieltjes_{1}(z)} - 1/\aspratio 
    = \frac{1}{\aspratio} \left( \frac{\sampleratio}{-z \, \compstieltjes(z/\sampleratio)} - 1\right). 
    \end{align*}
\end{lemma}

%and $\compstieltjes(z)$ represents the unique positive solution to equation~\eqref{eq:fixedpointsketching} for $z < 0$.
% $\stieltjes_{1}(z)$ denotes  the limiting empirical spectral distribution of the matrix $X\transp \sketchmat_{k} \sketchmat_{k} X/ \ndata$, 

\begin{proof}[Proof of Lemma~\ref{lm:stieltjescov}]
    Let $z < 0$. We begin with the following identity
    \begin{equation*}
        (\hcovmat_{k} - zI)\inv (\hcovmat_{k} - zI) = I.
    \end{equation*}
    Taking the trace and then multiplying both sides by $1/\ndim$, we obtain for any $z<0$ that %\scomment{How did you get the last line?}
    \$
    1  &= \frac{1}{\ndim}\, \tr((\hcovmat_{k} - zI)\inv \hcovmat_{k}) -  \frac{z}{\ndim}\, \tr(\hcovmat_{k} - zI)\inv\\
       &= \frac{1}{\ndim}\, \tr\left((\hcovmat_{k} - zI)\inv \frac{1}{\ndata} \sum_{i = 1}^{\ndata} \weight_{k,i} x_{i}x_{i}\transp \right) - \frac{z}{\ndim}\,\tr(\hcovmat_{k} - zI)\inv\\
       &= \frac{1}{\ndim}\, \tr\left( \sum_{i = 1}^{\ndata} \frac{(\hcovmat_{k, -i} - zI)\inv \weight_{k,i} x_{i}x_{i}\transp /\ndata }{1 +  \weight_{k,i} x_{i}\transp (\hcovmat_{k, -i} - zI)\inv x_{i}/ {\ndata} } \right) -  \frac{z}{\ndim}\,\tr(\hcovmat_{k} - zI)\inv, 
    \$
 where the last line uses  the Sherman–Morrison formula. {Using a similar argument as in the proof of Lemma~\ref{lm:isotrolimitresolvent}, we obtain} %\scomment{check?}
    \begin{align*}
        1 &= \frac{1}{\ndim}\, \tr\left( \sum_{i = 1}^{\ndata} \frac{(\hcovmat_{k, -i} - zI)\inv \weight_{k,i} x_{i}x_{i}\transp /\ndata }{1 +  \weight_{k,i} x_{i}\transp (\hcovmat_{k, -i} - zI)\inv x_{i}/ {\ndata} } \right) -  \frac{z}{\ndim}\,\tr(\hcovmat_{k} - zI)\inv \\
        &\overset{\as}{=} \frac{\sampleratio \stieltjescov_{1}(z)}{1 + \aspratio \stieltjescov_{1}(z)} - z\stieltjes_{1}(z)\\
        &= \frac{-\sampleratio/ \aspratio}{1 + \aspratio \stieltjescov_{1}(z)} + \sampleratio/ \aspratio - z\stieltjes_{1}(z).
    \end{align*}
    This leads to 
    \begin{equation}\label{eq:stieltjescov1}
        \stieltjescov_{1}(z) = \frac{\sampleratio/\aspratio^{2}}{\sampleratio/\aspratio - 1 - z\stieltjes_{1}(z)} - 1/\aspratio.
    \end{equation}
    
   Let $\compstieltjes_{1}(z)$ be the Stieltjes transform of the limiting empirical spectral distribution of the matrix $\sketchmat_{k} X   X \sketchmat_{k} \transp / \ndata$. Since the matrices $X\transp \sketchmat_{k} \sketchmat_{k} X/ \ndata$ and $\sketchmat_{k} X   X\transp \sketchmat_{k}  / \ndata$ share the same non-zero eigenvalues, we can establish %\scomment{Why?? It is hard to follow. }
    \begin{equation}\label{eq:stieltjescov2}
        \compstieltjes_{1}(z) + \frac{1 - \sampleratio}{z}= \aspratio \left( \stieltjes_{1}(z) + \frac{\aspratio -\sampleratio}{\aspratio} \frac{1}{z}\right).
    \end{equation}
    
     According to \citep[Theorem 2.7]{couillet2022} and since the multipliers are either  zero or one, we obtain  for $z < 0$ that
    \begin{equation*}
        \compstieltjes_{1}(z) = -\frac{1}{z} \EE_{\limitdistrweight}\left[ \frac{1}{1 +\delta(z) \weight}\right] = -\frac{1 - \sampleratio}{z} + \frac{\sampleratio}{-z - z\delta(z)},
    \end{equation*}
    where $\delta(z)$ and $\Tilde{\delta}(z)$ are the unique positive solution of the following equations for any $z<0$ %\scomment{Why do the following two equations have unique positive solutions?}
    \begin{align*}
        \delta(z) &= - \frac{\aspratio}{z} \int \frac{t \limitdistr(dt)}{1 + \Tilde{\delta}(z) t},\\
        \Tilde{\delta}(z) &= -\frac{1}{z} \EE_{\limitdistrweight}\left[ \frac{t}{1 +\delta(z) \weight}\right] = \frac{\sampleratio}{-z - z\delta(z)}.
    \end{align*}
Using these facts, we obtain
    \begin{align*}
        \frac{1}{\compstieltjes_{1}(z) + \frac{1 - \sampleratio}{z}} &= - \frac{z}{\sampleratio} (1 + \delta(z))\\
        &= -\frac{z}{\sampleratio} + \frac{\aspratio}{\sampleratio} \int \frac{t \limitdistr(dt)}{1 + \Tilde{\delta}(z) t}\\
        &= -\frac{z}{\sampleratio} + \frac{\aspratio}{\sampleratio} \int \frac{t \limitdistr(dt)}{1 + \left(\compstieltjes_{1}(z) + \frac{1 - \sampleratio}{z} \right) t}.
    \end{align*}
Comparing the above results with equation~\eqref{eq:fixedpointsketching}, we conclude 
    \begin{equation}\label{eq:stieltjescov3}
        \compstieltjes\left(\frac{z}{\sampleratio}\right) = \compstieltjes_{1}(z) + \frac{1 - \sampleratio}{z}.
    \end{equation}
The result follows from equations \eqref{eq:stieltjescov1}, \eqref{eq:stieltjescov2}, and \eqref{eq:stieltjescov3}.
\end{proof}

\begin{lemma}\label{lm:Crossterm}
    Let $\C: = \{ i~|~ \weight_{k,i} \neq 0, \weight_{\ell,i} \neq 0 \}$  and $\hcovmat_{\C} := \sum_{i \in \C} x_{i}x_{i}\transp/\ndata.$
    %\begin{equation*}
        %\A: = \{ i~|~ \weight_{k,i} \neq 0, \weight_{\ell,i} = 0 \}, \quad \B: = \{ i~|~ \weight_{k,i} = 0, \weight_{\ell,i} \neq 0 \}, \quad \C: = \{ i~|~ \weight_{k,i} \neq 0, \weight_{\ell,i} \neq 0 \},
    %\end{equation*}
    %\begin{equation*}
        %\hcovmat_{\A} =: \sum_{i \in \A} \frac{1}{\ndata} x_{i}x_{i}\transp, \quad %\hcovmat_{\B} =: \sum_{i \in \B} \frac{1}{\ndata} x_{i}x_{i}\transp, \quad %\hcovmat_{\C} =: \sum_{i \in \C} \frac{1}{\ndata} x_{i}x_{i}\transp.
    %\end{equation*}
    Assume Assumption~\ref{Assume:highdim}-\ref{Assume:multip}, and the multipliers are either zero or one. Then, for any $z < 0$, we have %\scomment{Bernoulli weights only?}
    \begin{align}
        &\lim_{\npinfty} \frac{1}{\ndata \ndim}\tr\left( (\hcovmat_{k} - zI)\inv X\transp \sketchmat_{k}^{2} \sketchmat_{\ell}^{2} X (\hcovmat_{\ell} - zI) \inv \covmat \right) \notag\\ \label{eq:crossterm2}
        &\quad\quad=\lim_{\npinfty} \frac{1}{\ndim}\tr\left( (-z\ensambleweight(z)\covmat + \hcovmat_{\C} - zI)^{-1} \hcovmat_{\C} (-z\ensambleweight(z)\covmat + \hcovmat_{\C} - zI)^{-1} \covmat \right) \qas
    \end{align}
    where $\ensambleweight(z): = (1 - \sampleratio)v(z/\sampleratio).$
\end{lemma}
\begin{proof}[Proof of Lemma~\ref{lm:Crossterm}]
In addition to $C$ and $\hcovmat_{\C}$, we define 
\$
\A: &= \{ i~|~ \weight_{k,i} \neq 0, \weight_{\ell,i} = 0 \}, \quad \B: = \{ i~|~ \weight_{k,i} = 0, \weight_{\ell,i} \neq 0 \},\\
\hcovmat_{\A} &:= \sum_{i \in \A} \frac{1}{\ndata} x_{i}x_{i}\transp, \quad \hcovmat_{\B} := \sum_{i \in \B} \frac{1}{\ndata} x_{i}x_{i}\transp. 
\$
It is straightforward to see that  $A, B$, and $C$ are disjoint sets. Since the multipliers are either zero or one, we have
    \begin{align*}
        \hcovmat_{k} &= \sum_{i : \weight_{k,i} \neq 0} \frac{1}{\ndata} x_{i}x_{i}\transp = \hcovmat_{\A} + \hcovmat_{\C},\\
        \hcovmat_{\ell} &= \sum_{i : \weight_{\ell,i} \neq 0} \frac{1}{\ndata} x_{i}x_{i}\transp = \hcovmat_{\B} + \hcovmat_{\C},\\
        \frac{1}{n} X\transp \sketchmat_{k}^{2} \sketchmat_{\ell}^{2} X &= \sum_{i : \weight_{k,i} \neq 0, \weight_{\ell,i} \neq 0} \frac{1}{\ndata} x_{i}x_{i}\transp = \hcovmat_{\C}.
    \end{align*}
    Therefore, we have 
    \begin{equation*}
        \frac{1}{\ndata \ndim}\tr\left( (\hcovmat_{k} - zI)\inv X\transp \sketchmat_{k}^{2} \sketchmat_{\ell}^{2} X (\hcovmat_{\ell} - zI) \inv \covmat \right) = \frac{1}{\ndim}\tr\left( (\hcovmat_{\A} + \hcovmat_{\C} - zI)\inv \hcovmat_{\C} (\hcovmat_{\B} + \hcovmat_{\C} - zI)\inv \covmat \right).
    \end{equation*}

Let  
\begin{align*}
        T_{1}: &= \frac{1}{\ndim} \tr\left( (\hcovmat_{\A} + \hcovmat_{\C} - zI)\inv \hcovmat_{\C} (\hcovmat_{\B} + \hcovmat_{\C} - zI)\inv \covmat \right) \\
        &\quad \quad -\frac{1}{\ndim} \tr\left( (-z\ensambleweight(z)\covmat + \hcovmat_{\C} - zI)\inv \hcovmat_{\C} (\hcovmat_{\B} + \hcovmat_{\C} - zI)\inv \covmat \right),\\
        T_{2}: &= \frac{1}{\ndim}\tr\left( (-z\ensambleweight(z)\covmat + \hcovmat_{\C} - zI)\inv \hcovmat_{\C} (\hcovmat_{\B} + \hcovmat_{\C} - zI)\inv \covmat \right) \\
        &\quad \quad -\frac{1}{\ndim}\tr\left( (-z\ensambleweight(z)\covmat + \hcovmat_{\C} - zI)^{-1} \hcovmat_{\C} (-z\ensambleweight(z)\covmat + \hcovmat_{\C} - zI)^{-1} \covmat \right).
    \end{align*} 
To prove \eqref{eq:crossterm2}, we compare both sides by writing their difference as
\begin{align*}
&\frac{1}{d}\tr\left( (\hcovmat_{\A} + \hcovmat_{\C} - zI)\inv \hcovmat_{\C} (\hcovmat_{\B} + \hcovmat_{\C} - zI)\inv \covmat \right) \\
&\quad - \frac{1}{d}\tr\left( (-z\ensambleweight(z)\covmat + \hcovmat_{\C} - zI)\inv \hcovmat_{\C} (-z\ensambleweight(z)\covmat + \hcovmat_{\C} - zI)\inv \covmat \right) \notag \\
%&=\tr\left( (\hcovmat_{\A} + \hcovmat_{\C} - zI)\inv \hcovmat_{\C} (\hcovmat_{\B} + \hcovmat_{\C} - zI)\inv \covmat \right) - \tr\left( (-z\ensambleweight(z)\covmat + \hcovmat_{\C} - zI)\inv \hcovmat_{\C} (\hcovmat_{\B} + \hcovmat_{\C} - zI)\inv \covmat \right)\\
%&+ \tr\left( (-z\ensambleweight(z)\covmat + \hcovmat_{\C} - zI)\inv \hcovmat_{\C} (\hcovmat_{\B} + \hcovmat_{\C} - zI)\inv \covmat \right) - \tr\left( (-z\ensambleweight(z)\covmat + \hcovmat_{\C} - zI)^{-1} \hcovmat_{\C} (-z\ensambleweight(z)\covmat + \hcovmat_{\C} - zI)^{-1} \covmat \right)\\
&= T_1 + T_2. 
\end{align*}
Therefore, it suffices to show   %we aim to show that:
    \begin{align*}
        \lim_{\npinfty} T_{1} = 0, \qquad \lim_{\npinfty} T_{2} = 0 \qas
    \end{align*}

For $T_{1}$, we have 
\$
    T_1 %&\tr\left( (\hcovmat_{\A} + \hcovmat_{\C} - zI)\inv \hcovmat_{\C} (\hcovmat_{\B} + \hcovmat_{\C} - zI)\inv \covmat \right) - \tr\left( (-z\ensambleweight(z)\covmat + \hcovmat_{\C} - zI)\inv \hcovmat_{\C} (\hcovmat_{\B} + \hcovmat_{\C} - zI)\inv \covmat \right)\notag \\ \label{eq:crossterm2}
    &= \frac{1}{d} \tr\left( (\hcovmat_{\A} + \hcovmat_{\C} - zI)\inv (-z\ensambleweight(z)\covmat - \hcovmat_{\A}) (-z\ensambleweight(z)\covmat + \hcovmat_{\C} - zI)\inv \hcovmat_{\C} (\hcovmat_{\B} + \hcovmat_{\C} - zI)\inv \covmat \right)\\
    &= - \frac{1}{d} \tr\left( (\hcovmat_{\A} + \hcovmat_{\C} - zI)\inv  \hcovmat_{\A} (-z\ensambleweight(z)\covmat + \hcovmat_{\C} - zI)\inv \hcovmat_{\C} (\hcovmat_{\B} + \hcovmat_{\C} - zI)\inv \covmat \right) \\
    &\qquad - \frac{1}{d} \tr\left( (\hcovmat_{\A} + \hcovmat_{\C} - zI)\inv  z\ensambleweight(z)\covmat  (-z\ensambleweight(z)\covmat + \hcovmat_{\C} - zI)\inv \hcovmat_{\C} (\hcovmat_{\B} + \hcovmat_{\C} - zI)\inv \covmat \right) \\
    &= - T_{11} - T_{12},
\$
    where the first equality uses the identity $A\inv - B\inv = A\inv(B - A)B\inv$ for any invertible matrices $A$ and $B$. For  $T_{11}$, we have %\scomment{$-zk(z)$ or $zk(z)$????}
    \begin{align*}
        T_{11}&=\frac{1}{d} \tr\left( (\hcovmat_{\A} + \hcovmat_{\C} - zI)\inv \hcovmat_{\A} (-z\ensambleweight(z)\covmat + \hcovmat_{\C} - zI)\inv \hcovmat_{\C} (\hcovmat_{\B} + \hcovmat_{\C} - zI)\inv \covmat \right)\\
        =&\frac{1}{d} \tr\left( (\hcovmat_{\A} + \hcovmat_{\C} - zI)\inv \sum_{i \in \A} \frac{1}{\ndata} x_{i}x_{i}\transp (-z\ensambleweight(z)\covmat + \hcovmat_{\C} - zI)\inv \hcovmat_{\C} (\hcovmat_{\B} + \hcovmat_{\C} - zI)\inv \covmat \right)\\
        =&\frac{1}{d} \tr\Bigg( \sum_{i \in \A} \left( (\hcovmat_{\A, -i} + \hcovmat_{\C} - zI)\inv  \frac{1}{\ndata} \frac{x_{i}x_{i}\transp}{1 + \frac{1}{\ndata}x_{i}\transp (\hcovmat_{\A, -i} + \hcovmat_{\C} - zI)\inv x_{i}}\right) \\
        &\qquad\qquad\qquad \cdot (-z\ensambleweight(z)\covmat + \hcovmat_{\C} - zI)\inv \hcovmat_{\C} (\hcovmat_{\B} + \hcovmat_{\C} - zI)\inv \covmat \Bigg),
    \end{align*}
    where $\hcovmat_{\A, -i} := \hcovmat_{\A} - \frac{1}{\ndata}x_{i}x_{i}\transp$, and the last line uses the Sherman–Morrison formula. Applying Lemma~\ref{lm:concentrationontrace} and Lemma \ref{lm:petubation}, we obtain
    \begin{align}\label{eq:crossterm3}
         &\lim_{\npinfty} \frac{1}{\ndim}\tr\left( (\hcovmat_{\A} + \hcovmat_{\C} - zI)\inv \hcovmat_{\A} (-z\ensambleweight(z)\covmat + \hcovmat_{\C} - zI)\inv \hcovmat_{\C} (\hcovmat_{\B} + \hcovmat_{\C} - zI)\inv \covmat \right) \notag \\ 
        &\quad\quad\quad\quad   -\frac{1}{\ndim} \tr\Bigg( (\hcovmat_{\A} + \hcovmat_{\C} - zI)\inv \sum_{i \in \A} \frac{1}{\ndata} \frac{\covmat}{1 + \aspratio \frac{1}{\ndim} \tr\big((\hcovmat_{\A} + \hcovmat_{\C} - zI)\inv\covmat\big)} \nn\\
        & \qquad \qquad \qquad \qquad \cdot (-z\ensambleweight(z)\covmat + \hcovmat_{\C} - zI)\inv \hcovmat_{\C} (\hcovmat_{\B} + \hcovmat_{\C} - zI)\inv \covmat \Bigg) \nn \\
        &= 0 \qas
    \end{align}
     Furthermore, by Lemma~\ref{lm:stieltjescov}, we obtain
     \begin{align}
        \lim_{\npinfty} \frac{1}{1 + \aspratio \frac{1}{\ndim} \tr(\hcovmat_{\A} + \hcovmat_{\C} - zI)\inv\covmat} &= \frac{1}{1 + \aspratio \stieltjescov_{1}(z)} \qas \notag \\
        &= \frac{1}{1 + \left(-\frac{\sampleratio}{z \compstieltjes(z/\sampleratio) - 1}\right)} \notag \\
        &= -\frac{z\compstieltjes(z/\sampleratio)}{\sampleratio} \notag \\ \label{eq:crossterm4}
        &= -\frac{z \ensambleweight(z)}{\sampleratio(1-\sampleratio)} 
     \end{align}
     Moreover, under Assumption~\ref{Assume:multip}, the cardinality of $\A$ satisfies
     \begin{equation} \label{eq:crossterm5}
         |\A|/\ndata = \frac{1}{\ndata} \sum_{i=1}^{\ndata} 1\left({\weight_{k,i} \neq 0, \weight_{\ell,i} = 0}\right) \rightarrow \sampleratio(1 - \sampleratio) \qas
     \end{equation}
     Therefore, by combining~\eqref{eq:crossterm3}-\eqref{eq:crossterm5}, we obtain
     \begin{equation*}
         \lim_{\npinfty} T_{11} + T_{12} = 0 \qas,
     \end{equation*}
     and thus $ \lim_{\npinfty} T_{1} =0$ almost surely. 

 The above argument can be applied to $T_2$ to obtain
     \begin{equation*}
         \lim_{\npinfty} T_{2} = 0 \qas
     \end{equation*}
     This completes the proof.
\end{proof}

\begin{lemma}\label{lm:fixpointequiv}
    Assume Assumption~\ref{Assume:Covdistri}, $\aspratio/\sampleratio > 1$ and $x \leq 0$. Let $\tcompstieltjes(z,x)$ be the unique positive solution of equation~\eqref{eq:fixpointkx} for $z \leq 0$. Then,
    \begin{equation*}
        \lim_{\zinfty} \tcompstieltjes(z,z) = \tcompstieltjes(0), \quad \lim_{\zinfty} \tcompstieltjes '(z,z) = \tcompstieltjes '(0),
    \end{equation*}
    where the derivative is taken with respect to the first variable, $\tcompstieltjes(0)$ is the unique positive solution of equation~\eqref{eq:fixpointbootstrap}, and $\tcompstieltjes'(0) = \lim_{\zinfty} \tcompstieltjes'(z)$.
\end{lemma}
\begin{proof}[Proof of Lemma~\ref{lm:fixpointequiv}]
Recall from the proof of Lemma \ref{lm:corBaggingvar} that $k(x)= (1-\sampleratio)v(x/\sampleratio)$ for $x\leq 0$,  $\tlimitdistr_{x}$ is  the limiting empirical spectral distribution of $(I + \ensambleweight(x)\covmat)^{-1} \covmat$, and recall  equation \eqref{eq:fixpointkx}
\#\label{eq:fixpointkx_2}
\tcompstieltjes(z,x) = \left( -z + \frac{\aspratio}{\sampleratio^{2}} \int \frac{t\tlimitdistr_{x}(dt)}{1 + \tcompstieltjes(z,x)t} \right)^{-1}.
\#
We start by rewriting the above  equation as
    \begin{align*}
        1 - \frac{\sampleratio^{2}}{\aspratio} &= \frac{\sampleratio^{2}}{\aspratio} z \tcompstieltjes(z,x) + \int \frac{1}{1 + \tcompstieltjes(z,x) t} d\tlimitdistr_{x}(t) \\
        &= \frac{\sampleratio^{2}}{\aspratio} z \tcompstieltjes(z,x) + \int \frac{1 + \ensambleweight(x)t}{1 + \ensambleweight(x)t + \tcompstieltjes(z,x) t} d\limitdistr(t),
    \end{align*}
    where the last line follows from Lemma~\ref{lm:limitdistrequiv}. Let 
    \begin{equation*}
        f(c,x) = \frac{\sampleratio^{2}}{\aspratio} z c + \int \frac{1 + \ensambleweight(x)t}{1 + \ensambleweight(x)t + c t} d\limitdistr(t),
    \end{equation*}
    and thus $f(\tcompstieltjes(z,x), x) =  1 - \sampleratio^{2}/ \aspratio$. 
    In what follows, we upper and lower bound $\tcompstieltjes(z,x)$ in terms of $\tcompstieltjes(z,0)$. %\scomment{it's really hard to follow what you want to express.} 

   \paragraph{Upper bound}  Recall that $\compstieltjes(x/\sampleratio)$ is the limiting Stieltjes transform of $\sketchmat X X\transp \sketchmat/\ndata$, which is an increasing function of $x$ on $(-\infty, 0 ]$. Then $\ensambleweight(x) = (1 - \sampleratio) \compstieltjes(x/ \sampleratio)$  is also an increasing function of $x$ on $(-\infty, 0 ]$. 
    Therefore, for any fixed $c$, $f(c,x)$ is an increasing function of $x$ on $(-\infty, 0]$. Furthermore, for $z \leq 0$ and any fixed $x$, $f(c,x)$ is a decreasing function of $c$ on $[0, +\infty)$. Then, fixing some $z \leq 0$ and for any $x \leq 0$, we have %\scomment{Why? We should use the monotonicity of $\tcompstieltjes(z,x)$ of $x$?}
    \begin{align*}
        f(\tcompstieltjes(z,0), x) = \frac{\sampleratio^{2}}{\aspratio} z \tcompstieltjes(z,0)  + \int \frac{1 + \ensambleweight(x)t}{1 + \ensambleweight(x)t + \tcompstieltjes(z,0) t} d\limitdistr(t) \leq 1 - \frac{\sampleratio^{2}}{\aspratio}. 
    \end{align*}
    Thus we obtain
    \#\label{eq:vzx_lower}
     \tcompstieltjes(z,x)\leq \tcompstieltjes(z,0). 
    \#

   \paragraph{Lower bound} For $z \leq 0$ and $x \leq 0$, we have 
    \begin{align*}
        f\left(\frac{\ensambleweight(x)}{\ensambleweight(0)}\tcompstieltjes(z,0), x \right) &= \frac{\sampleratio^{2}}{\aspratio} z \frac{\ensambleweight(x)}{\ensambleweight(0)}\tcompstieltjes(z,0) + \int \frac{1 + \ensambleweight(x)t}{1 + \ensambleweight(x)t + \frac{\ensambleweight(x)}{\ensambleweight(0)}\tcompstieltjes(z,0) t} d\limitdistr(t)\\
        &\geq \frac{\sampleratio^{2}}{\aspratio} z \tcompstieltjes(z,0) +  \int \frac{\frac{\ensambleweight(0)}{\ensambleweight(x)} + \ensambleweight(0)t}{\frac{\ensambleweight(0)}{\ensambleweight(x)} + \ensambleweight(0)t + \tcompstieltjes(z,0) t} d\limitdistr(t)\\
        &\geq 1 - \frac{\sampleratio^{2}}{\aspratio} \\
        &= f(\tcompstieltjes(z,x), x), 
    \end{align*}
    which implies
    \#\label{eq:vzx_upper}
    \tcompstieltjes(z,0) \frac{\ensambleweight(x)}{\ensambleweight(0)}  \leq \tcompstieltjes(z,x). 
    \#
%    Therefore
%    \begin{align*}
%        \tcompstieltjes(z,0) - \tcompstieltjes(z,x) \leq \frac{\ensambleweight(0) - \ensambleweight(x)}{\ensambleweight(x)}\tcompstieltjes(z,0).
%    \end{align*}
   % This upper bound does not depend on $x$ and implies 
    %\begin{equation}\label{eq:fixpointequiv1}
      %  \lim_{\zinfty} \tcompstieltjes(z,z) = \lim_{\zinfty} \tcompstieltjes(z,0) = \tcompstieltjes(0).
   % \end{equation}

Combining the lower and upper bounds for  $\tcompstieltjes(z,x)$ and plugging $x=z$ into $\tcompstieltjes(z,x)$, we obtain
\$
\tcompstieltjes(z,0) \frac{\ensambleweight(z)}{\ensambleweight(0)}  \leq \tcompstieltjes(z,z) \leq  \tcompstieltjes(z,0). 
\$
Taking $\zinfty$, we obtain
\begin{equation}\label{eq:fixpointequiv1}
 \lim_{\zinfty} \tcompstieltjes(z,z) = \lim_{\zinfty} \tcompstieltjes(z,0) = \tcompstieltjes(0).
\end{equation}
For the derivative of $\tcompstieltjes(z,x)$, Using equation~\eqref{eq:fixpointkx_2} and Lemma~\ref{lm:limitdistrequiv}, we have
    \begin{align}
        \tcompstieltjes'(z,x) &= \frac{\tcompstieltjes(z,x)^{2}}{1 - \frac{\aspratio}{\sampleratio^{2}} \int \frac{\tcompstieltjes(z,x)^{2} t^{2}}{(1 + \tcompstieltjes(z,x)t)^{2}} d\tlimitdistr_{x}(t)}\notag\\ \label{eq:vderivform}
        &= \frac{\tcompstieltjes(z,x)^{2}}{1 - \frac{\aspratio}{\sampleratio^{2}} \int \frac{\tcompstieltjes(z,x)^{2} t^{2}}{(1 + \ensambleweight(x)t + \tcompstieltjes(z,x)t)^{2}} d\limitdistr(t)}.
    \end{align}
    The result follows from equation~\eqref{eq:fixpointequiv1}, the fact that $\lim_{\zinfty}\tcompstieltjes'(0)=\lim_{\zinfty} \tcompstieltjes'(z)$ ,   and the dominated convergence theorem.
\end{proof}

Let $\hcovmat_{\A}, \hcovmat_{\B},$ and $\hcovmat_{\C}$ be the same as in Lemma~\ref{lm:Crossterm} and its proof.

\begin{lemma}\label{lm:biassecond}
    Assume Assumptions~\ref{Assume:highdim}-~\ref{Assume:multip}. For any $z < 0$, we have 
    \begin{align*}
        &\lim_{\npinfty}\Bigg( \frac{z}{\ndim} \tr\left((\hcovmat_{\A} + \hcovmat_{\C} - zI)\inv \hcovmat_{\B} (\hcovmat_{\B} + \hcovmat_{\C} - zI)\inv \covmat \right) \\
        &\quad \quad \qquad +\frac{(1 - \sampleratio)}{\compstieltjes(z/\sampleratio)} \frac{1}{\ndim}\tr\left((\hcovmat_{\A} + \hcovmat_{\C} - zI)\inv \hcovmat_{\C} (\hcovmat_{\B} + \hcovmat_{\C} - zI)\inv \covmat \right)\Bigg) = 0 \qas
    \end{align*}
    where $\compstieltjes(0)$ is the unique positive solution to equation~\eqref{eq:fixedpointsketching}.
\end{lemma}
\begin{proof}[Proof of Lemma \ref{lm:biassecond}]
    We start by writing %rewriting ${\ndim}^{-1}z \tr\big((\hcovmat_{\A} + \hcovmat_{\C} - zI)\inv \hcovmat_{\B} (\hcovmat_{\B} + \hcovmat_{\C} - zI)\inv \covmat \big)$ as
    \begin{align*}
        &\frac{z}{\ndim} \tr\left((\hcovmat_{\A} + \hcovmat_{\C} - zI)\inv \hcovmat_{\B} (\hcovmat_{\B} + \hcovmat_{\C} - zI)\inv \covmat \right) \\
        &=\frac{z}{\ndim} \tr\left((\hcovmat_{\A} + \hcovmat_{\C} - zI)\inv \frac{1}{\ndata} \sum_{i \in \B}x_{i}x_{i}\transp (\hcovmat_{\B} + \hcovmat_{\C} - zI)\inv \covmat \right)\\
        &=\frac{z}{\ndim} \tr\left(\frac{1}{\ndata} \sum_{i \in \B} (\hcovmat_{\A} + \hcovmat_{\C} - zI)\inv \frac{x_{i}x_{i}\transp}{1 + \frac{1}{n}x_{i}\transp(\hcovmat_{\B, -i} + \hcovmat_{\C} - zI)\inv x_{i} } (\hcovmat_{\B, -i} + \hcovmat_{\C} - zI)\inv \covmat \right),
    \end{align*}
    where the last equality uses the Sherman–Morrison formula. Following the same argument as in the proof of Lemma~\ref{lm:Crossterm}, we obtain %\scomment{should be $z/d$ instead of $z^2/d$? }
    \#\label{eq:biassecond_1}
        &\lim_{\npinfty} \Bigg( \frac{z}{\ndim} \tr\left((\hcovmat_{\A} + \hcovmat_{\C} - zI)\inv \hcovmat_{\B} (\hcovmat_{\B} + \hcovmat_{\C} - zI)\inv \covmat \right) \nn\\
        &\quad \quad\qquad  - \frac{z}{\ndim} \frac{\sampleratio(1 - \sampleratio)}{1 + \aspratio \stieltjescov_{1}(z)} \tr\left((\hcovmat_{\A} + \hcovmat_{\C} - zI)\inv \covmat (\hcovmat_{\B} + \hcovmat_{\C} - zI)\inv \covmat \right) \Bigg ) = 0 \qas
    \#
    Similarly, we have
    \begin{align*}
        &\frac{1}{\ndim}\tr\left((\hcovmat_{\A} + \hcovmat_{\C} - zI)\inv \hcovmat_{\C} (\hcovmat_{\B} + \hcovmat_{\C} - zI)\inv \covmat \right)\\
        =&\frac{1}{\ndim} \tr\left((\hcovmat_{\A} + \hcovmat_{\C} - zI)\inv \frac{1}{\ndata} \sum_{i \in \C}x_{i}x_{i}\transp (\hcovmat_{\B} + \hcovmat_{\C} - zI)\inv \covmat \right)\\
        =&\frac{1}{\ndim} \tr\left(\frac{1}{\ndata} \sum_{i \in \C}  \frac{(\hcovmat_{\A} + \hcovmat_{\C, -i} - zI)\inv x_{i}x_{i}\transp (\hcovmat_{\B} + \hcovmat_{\C, -i} - zI)\inv \covmat}{(1 + \frac{1}{n}x_{i}\transp(\hcovmat_{\A} + \hcovmat_{\C, -i} - zI)\inv x_{i} )(1 + \frac{1}{n}x_{i}\transp(\hcovmat_{\B} + \hcovmat_{\C, -i} - zI)\inv x_{i} )}  \right),
    \end{align*}
    and thus 
    \#\label{eq:biassecond_2}
        &\lim_{\npinfty}\frac{1}{\ndim}\tr\left((\hcovmat_{\A} + \hcovmat_{\C} - zI)\inv \hcovmat_{\C} (\hcovmat_{\B} + \hcovmat_{\C} - zI)\inv \covmat \right) \nn\\
        &  \quad \quad-\frac{1}{\ndim} \frac{\sampleratio^{2}}{(1 + \aspratio \stieltjescov_{1}(z))^{2}} \tr\left((\hcovmat_{\A} + \hcovmat_{\C} - zI)\inv \covmat (\hcovmat_{\B} + \hcovmat_{\C} - zI)\inv \covmat \right) = 0 \qas
    \#
    Combining \eqref{eq:biassecond_1} and \eqref{eq:biassecond_2} and using Lemma~\ref{lm:stieltjescov}, that is 
    \begin{equation*}
        \stieltjescov_{1}(z)
    = \frac{1}{\aspratio} \left( \frac{\sampleratio}{-z \, \compstieltjes(z/\sampleratio)} - 1\right),
    \end{equation*}
    we complete the proof. 
\end{proof}

%%%%%%%%%%%%%%%%%%%%%%%%%%%%%%%%%%%%%%%%%%%%%%%%%
%%%%%%%%%%%%%%Proofs for Section 5%%%%%%%%%%%%%%%
%%%%%%%%%%%%%%%%%%%%%%%%%%%%%%%%%%%%%%%%%%%%%%%%%

\section{Proofs for Section \ref{sec:extension}}

%%%%%%%%%%%%%%%%%%%%%%%%%%%%%%%%%%%%%%%%%%%%%%%%%%%%%%%%%%%%
%%%%%%%%%%%%%%%%%Deterministic signal case%%%%%%%%%%%%%%%%%%
%%%%%%%%%%%%%%%%%%%%%%%%%%%%%%%%%%%%%%%%%%%%%%%%%%%%%%%%%%%%

\subsection{Proof of Theorem~\ref{thm:Overpara}}

\begin{proof}[Proof of Theorem~\ref{thm:Overpara}]
    Let %\scomment{already defined somewhere?}
    \begin{align}\label{eq:deterministic0}
        B_{k,\ell} &= \frac{1}{\ndim} \truesignal\transp \left(I - \hcovmat_{k}\pinv \hcovmat_{k} \right) \covmat \left(I - \hcovmat_{\ell}\pinv \hcovmat_{\ell} \right) \truesignal,\\ \label{eq:deterministic1}
        V_{k,\ell} &= \frac{\noiselev}{\ndata^{2}} \tr\left(\hcovmat_{k}\pinv X\transp \sketchmat_{k}^2 \sketchmat_{\ell}^2 X \hcovmat_{\ell}\pinv \covmat \right).
    \end{align} 
    Applying Lemma~\ref{lm:biasvar}, we can rewrite the out-of-sample prediction risk as
    \begin{equation*}
         \riskcondition(\estimator) =  \frac{1}{\nresample^{2}} \sum_{k,\ell}^{\nresample} \left( B_{k,\ell} + V_{k,\ell} \right).
    \end{equation*}
    Note that the eigenvalues of $I - \hcovmat_{k}\pinv \hcovmat_{k}$ are either zero or one. For $k \neq \ell$, under Assumption~\ref{Assume:Covdistri}, we have %\scomment{do we always write $r^2= \|\truesignal\|_2^2$. }
    \begin{align*}
         B_{k,\ell} &= \frac{1}{\ndim} \truesignal\transp \left(I - \hcovmat_{k}\pinv \hcovmat_{k} \right) \covmat \left(I - \hcovmat_{\ell}\pinv \hcovmat_{\ell} \right) \truesignal\transp \leq C_{\lambda}\|\beta\|_2^2/d. 
    \end{align*}
    Furthermore, $V_{k,k}$ corresponds to the variance of the sketched least square estimator. Therefore, %\scomment{need independence assumption.}
    \begin{align*}
        \lim_{\nresample \rightarrow + \infty} \lim_{\npinfty} \riskcondition(\estimator) = \lim_{\npinfty} \left( B_{k,\ell} + V_{k,\ell}\right) {\quad \as}
    \end{align*}
    for  $k \neq \ell$.
    
    By equation~\eqref{eq:deterministic1}, the variance term does not depend on $\truesignal$, and thus is the same as  in the random signal case in Theorem~\ref{thm:corrbagging}. Moreover, using a similar argument as in the proof of Lemma \ref{lm:isobootbias},  the bias term in the underparameterized regime  converges almost surely to zero. In what follows, we will prove the almost sure convergence of  the bias term $B_{k,\ell}$ with $k\neq \ell$ under  the overparameterized regime, aka %\scomment{How did you get this??} %\scomment{current point. }
    \begin{equation*}
        \lim_{\npinfty} B_{k,\ell} = \Tilde{r}^{2} \frac{\tcompstieltjes'(0)}{\tcompstieltjes(0)^{2}} \int \frac{s}{(1 + \tcompstieltjes(0)s)^{2}} d\Tilde{G}(s) {\quad \as}
    \end{equation*}

    When  $\aspratio/\sampleratio >1$ and $k \neq \ell$, we rewrite $B_{k,\ell}$ in~\eqref{eq:deterministic0} as
\begin{align*}
    B_{k,\ell} &= \frac{1}{\ndim} \truesignal\transp \left(I - \hcovmat_{k}\pinv \hcovmat_{k} \right) \covmat \left(I - \hcovmat_{\ell}\pinv \hcovmat_{\ell} \right) \truesignal\transp\\
    &= \lim_{\zinfty} \frac{1}{\ndim} \truesignal\transp(I - (\hcovmat_{k} - zI)\inv \hcovmat_{k}) \covmat (I - (\hcovmat_{\ell} - zI)\inv \hcovmat_{\ell}) \truesignal\\
    &= \lim_{\zinfty} \frac{z^2}{\ndim} \truesignal\transp (\hcovmat_{k} - zI)\inv \covmat (\hcovmat_{\ell} - zI)\inv  \truesignal
\end{align*}
where the second line uses the identity~\eqref{eq:pesudoidentity}. Furthermore, using   Lemma~\ref{lm:Pseudo_products}, we can assume that the multipliers $\weight_{i,j}$ are either zero or one and $\noiselev = 1$ without loss of generality. Let
\begin{gather*}
    \ensambleweight(z) := (1 - \sampleratio) \compstieltjes(z/\sampleratio), \quad 
    \hcovmat_{\C} :=\sum_{i : \weight_{k,i} \neq 0, \weight_{\ell,i} \neq 0} \frac{1}{\ndata} x_{i}x_{i}\transp, 
    \tcovmat_{\C}(z) : = (I + \ensambleweight(z)\covmat)^{-1/2} \hcovmat_{\C}(I + \ensambleweight(z)\covmat)^{-1/2}, \\
    \Tilde{\covmat} : = (I + \ensambleweight(z)\covmat)^{-1/2} \covmat (I + \ensambleweight(z)\covmat)^{-1/2}, \quad
    \Tilde{\truesignal} : = (I + \ensambleweight(z)\covmat)^{-1/2} \truesignal,
\end{gather*}
where $\compstieltjes(z)$ is the unique positive solution of equation~\eqref{eq:fixedpointsketching}. Using a similar argument as in the proof of  Lemma~\ref{lm:Crossterm}, we have 
\begin{align}
    & \lim_{\npinfty} \frac{z^2}{\ndim} \truesignal\transp (\hcovmat_{k} - zI)\inv \covmat (\hcovmat_{\ell} - zI)\inv  \truesignal \notag\\
    =& \lim_{\npinfty} \frac{z^2}{\ndim} \truesignal\transp z^2 (-z\ensambleweight(z)\covmat + \hcovmat_{\C} - zI)^{-1} \covmat (-z\ensambleweight(z)\covmat + \hcovmat_{\C} - zI)^{-1} \truesignal \notag\\ \label{eq:biasequiv}
    =& \lim_{\npinfty} \frac{z^2}{\ndim} \Tilde{\truesignal}\transp (\tcovmat_{\C}(z) - zI)^{-1} \Tilde{\covmat} (\tcovmat_{\C}(z) - zI)^{-1} \Tilde{\truesignal}. {\quad \as}
\end{align}
The above is equivalent to the bias term of ridge regression with the covariance matrix of $(I + \ensambleweight(z)\covmat)^{-1} \covmat$, the  coefficient vector $(I + \ensambleweight(z)\covmat)^{-1/2} \truesignal$, and the ridge regularization parameter $-z$. Using a similar argument as in~\cite[Theorem 5]{hastie2019surprises}, we obtain
\begin{equation*}
    \lim_{\ndim \rightarrow \infty} \frac{z^2}{\ndim} \truesignal\transp  (\hcovmat_{k} - zI)\inv \covmat (\hcovmat_{\ell} - zI)\inv  \truesignal = \Tilde{r_{z}}^{2} \frac{\tcompstieltjes'(z, z)}{\tcompstieltjes(z, z)^{2}} \int \frac{s}{(1 + \tcompstieltjes(z,z)s)^{2}} d\Tilde{G}_{z}(s) {\quad \as}
\end{equation*}
where $\tcompstieltjes(z, x)$ is the unique positive solution of equation~\eqref{eq:fixpointkx}, $\Tilde{r_{x}}^{2} = \truesignal\transp (I + \ensambleweight(x)\covmat)^{-1} \truesignal$, and $\Tilde{G}_{x}$ is the weak convergence limit of
\begin{equation}\label{eq:VESDkx}
    \Tilde{G}_{x, \ndim}(s) = \frac{1}{\Tilde{r_{x}}^{2}} \sum_{i=1}^{\ndim} \frac{1}{1 + \ensambleweight(x)\eigval_{i}}\langle \beta, u_{i} \rangle^{2} \, 1 \left\{ s \geq \frac{\eigval_{i}}{1 + \ensambleweight(x)\eigval_{i}}\right\},
\end{equation}
which exists under Assumption~\ref{Assume:VESD}.
By Lemma~\ref{lm:limitVESD}, we obtain 
\begin{align*}
&\lim_{\zinfty} \Tilde{r_{z}}^{2} \frac{\tcompstieltjes'(z, z)}{\tcompstieltjes(z, z)^{2}} \int \frac{s}{(1 + \tcompstieltjes(z,z)s)^{2}} d\Tilde{G}_{z}(s) \\
&=  \Tilde{r}^{2} \frac{\tcompstieltjes'(0)}{\tcompstieltjes(0)^{2}} \lim_{\zinfty} \int \frac{s}{(1 + \tcompstieltjes(0)s)^{2}} d\Tilde{G}_{z}(s) \\ %+ \wcolor{\lim_{\zinfty} \Tilde{r}^{2} \frac{\tcompstieltjes'(0)}{\tcompstieltjes(0)^{2}} \sup_{s \in [0, C_{\lambda}]} \left| \frac{s}{(1 + \tcompstieltjes(z)s)^{2}} - \frac{s}{(1 + \tcompstieltjes(0)s)^{2}}\right|}\\ 
&= \Tilde{r}^{2} \frac{\tcompstieltjes'(0)}{\tcompstieltjes(0)^{2}} \int \frac{s}{(1 + \tcompstieltjes(0)s)^{2}} d\Tilde{G}(s),
\end{align*}
where in the second line we used the following inequality:
\begin{align*}
    \lim_{\zinfty} \left|\int \frac{s}{(1 + \tcompstieltjes(z,z)s)^{2}} d\Tilde{G}_{z}(s) - \int \frac{s}{(1 + \tcompstieltjes(0)s)^{2}} d\Tilde{G}_{z}(s) \right| \leq \lim_{\zinfty} \sup_{s \in [0, C_{\lambda}]} \left| \frac{s}{(1 + \tcompstieltjes(z)s)^{2}} - \frac{s}{(1 + \tcompstieltjes(0)s)^{2}}\right| = 0.
\end{align*}
Using a similar argument as in the proof of Lemma \ref{lm:isobootbias}, we can exchange the limits between $\npinfty$ and $\zinfty$. This completes the proof.
\end{proof}

\subsection{Proof of Corollary \ref{cor:deterreduction}}
\begin{proof}[Proof of Corollary \ref{cor:deterreduction}]
We have
\begin{align*}
    \Tilde{r}^{2} \frac{\tcompstieltjes'(0)}{\tcompstieltjes(0)^{2}} \int \frac{s}{(1 + \tcompstieltjes(0)s)^{2}} \,d{G}(s) &= \lim_{\npinfty} \Tilde{r}^{2} \frac{\tcompstieltjes'(0)}{\tcompstieltjes(0)^{2}} \int \frac{s}{(1 + \tcompstieltjes(0)s)^{2}} \,d{G}_n(s)\tag{Assumption~\ref{Assume:VESD}}\\
    &= \lim_{\npinfty}  \frac{\tcompstieltjes'(0)}{\tcompstieltjes(0)^{2}} \sum_{i=1}^\ndim  \frac{\lambda_i}{(1 + k(0)\lambda_i + \tcompstieltjes(0)\lambda_i)^{2}} \truesignal\transp u_iu_i\transp \truesignal\\
    &= \lim_{\npinfty}  \frac{\tcompstieltjes'(0)}{\tcompstieltjes(0)^{2}}  \truesignal\transp \sum_{i=1}^\ndim  \frac{\lambda_i}{(1 + \compstieltjes(0)\lambda_i)^{2}}  u_iu_i\transp \truesignal \tag{Lemma~\ref{lm:v0tv}}\\
    &= \lim_{\npinfty}  \frac{\signallev}{\ndim}\frac{\tcompstieltjes'(0)}{\tcompstieltjes(0)^{2}}  \tr \left(\sum_{i=1}^\ndim  \frac{\lambda_i}{(1 + \compstieltjes(0)\lambda_i)^{2}}  u_iu_i\transp \right) \qas  \tag{Lemma \ref{lm:concentrationontrace}}\\
    &= {\signallev}\frac{\tcompstieltjes'(0)}{\tcompstieltjes(0)^{2}} \int \frac{t}{(1+\compstieltjes(0)t )^2} \, dH(t) \tag{Assumption~\ref{Assume:Covdistri}}\\
    &= {\signallev}\frac{\tcompstieltjes'(0)}{\tcompstieltjes(0)^{2}} \int \frac{t(1 + \compstieltjes(0)t)}{(1+\compstieltjes(0)t )^2} \, dH(t) -  {\signallev}\frac{\tcompstieltjes'(0)}{\tcompstieltjes(0)^{2}} \int \frac{\compstieltjes(0)t^{2}}{(1+\compstieltjes(0)t )^2} \, dH(t)\\
    &= {\signallev}\frac{\tcompstieltjes'(0)}{\tcompstieltjes(0)^{2}} \frac{\sampleratio}{\aspratio} \frac{1}{\compstieltjes(0)} -  {\signallev}\frac{\tcompstieltjes'(0)}{\tcompstieltjes(0)^{2}} \int \frac{\compstieltjes(0)t^{2}}{(1+\compstieltjes(0)t )^2} \, dH(t) \tag{equation~\eqref{eq:fixedpointsketching}}\\
    &= {\signallev}\frac{\tcompstieltjes'(0)}{\tcompstieltjes(0)^{2}} \frac{1}{\aspratio \compstieltjes(0)}\left(1 - \aspratio \int \frac{\compstieltjes(0)^{2}t^{2}}{(1+\compstieltjes(0)t )^2} \, dH(t) \right) - {\signallev}\frac{\tcompstieltjes'(0)}{\tcompstieltjes(0)^{2}} \frac{1 - \sampleratio}{\aspratio \compstieltjes(0)}\\
    &={\signallev}\frac{1}{\aspratio \compstieltjes(0)} - {\signallev}\frac{\tcompstieltjes'(0)}{\tcompstieltjes(0)^{2}} \frac{1 - \sampleratio}{\aspratio \compstieltjes(0)} \tag{equation~\eqref{eq:vderivform}}\\
    &=\signallev \frac{\sampleratio}{\aspratio \compstieltjes(0)} - \signallev \frac{(1 - \sampleratio)}{\aspratio \compstieltjes(0)} \left(\frac{\tcompstieltjes'(0)}{\tcompstieltjes(0)^{2}} - 1 \right).
\end{align*}
This finishes the proof. 
\end{proof}

\subsection{Proof of Theorem~\ref{thm:trainingerr}}
\begin{proof}[Proof of Theorem~\ref{thm:trainingerr}]
    By Lemma~\ref{lm:singularmat}, we can assume, without loss of generality, that $X\transp \sketchmat_{k} \sketchmat_{k} X$ are singular. Consequently, the sketched ridgeless least square estimators $\estimatork$ interpolate the sketched data $(\sketchmat_{k} X, \sketchmat_{k} Y)$. This enables us to express the training error $\trainloss(\estimator; X,Y)$ as follows: 
    \begin{align*}
        \trainloss(\estimator; X,Y) &= \frac{1}{\ndata} \|X(\truesignal - \estimator) + E\|^{2}_{2}\\
        &= \frac{1}{\ndata} \sum_{i=1}^{\ndata} (x_{i}\transp(\truesignal - \estimator) + \varepsilon_{i})^{2}\\
        &= \frac{1}{\ndata} \sum_{i=1}^{\ndata} \frac{1}{\nresample^{2}}\sum_{k,\ell}(x_{i}\transp(\truesignal - \estimatork) + \varepsilon_{i})(x_{i}\transp(\truesignal - \hat{\beta_{\ell}}) + \varepsilon_{i})\\
        &= \frac{1}{\nresample^{2}} \frac{1}{\ndata} \sum_{k,\ell} \sum_{i: \weight_{k,i}= \weight_{\ell,i} = 0} (x_{i}\transp(\truesignal - \estimatork) + \varepsilon_{i})(x_{i}\transp(\truesignal - \hat{\beta_{\ell}}) + \varepsilon_{i}),
    \end{align*}
    where the last equality uses the fact that the sketched  estimators are interpolators. 
    Applying similar arguments as in Lemma~\ref{lm:biasvar}, we obtain the following decomposition:
    \begin{align*}
       E \left[ (x_{i}\transp(\truesignal - \estimatork) + \varepsilon_{i})(x_{i}\transp(\truesignal - \hat{\beta_{\ell}}) + \varepsilon_{i}) \condition \variablecondition \right] &=  \truesignal\transp \projmat_{k} x_{i}\transp x_{i} \projmat_{\ell} \truesignal + \noiselev (1 + x_{i}\transp \hcovmat_{k}\pinv X\transp \sketchmat_{k}^2 \sketchmat_{\ell}^2 X \hcovmat_{\ell}\pinv x_{i}).
    \end{align*} 
    Furthermore, by Lemma~\ref{lm:concentrationontrace}, we have
    \begin{align*}
        &\lim_{\npinfty} E \left[ \trainloss(\estimator; X,Y) \condition \variablecondition \right] \\
        &= \lim_{\npinfty} \frac{1}{\nresample^{2}} \frac{1}{\ndata} \sum_{k,\ell} \sum_{i: \weight_{k,i}= \weight_{\ell,i} = 0} \truesignal\transp \projmat_{k} \covmat \projmat_{\ell} \truesignal + \noiselev (1 + \tr( \hcovmat_{k}\pinv X\transp \sketchmat_{k}^2 \sketchmat_{\ell}^2 X \hcovmat_{\ell}\pinv \covmat)) \qas \\
        &= \lim_{\npinfty} \frac{1}{\nresample^{2}} \sum_{k,\ell} (1 - \sampleratio)^{2} [ \truesignal\transp \projmat_{k} \covmat \projmat_{\ell} \truesignal + \noiselev (1 + \tr( \hcovmat_{k}\pinv X\transp \sketchmat_{k}^2 \sketchmat_{\ell}^2 X \hcovmat_{\ell}\pinv \covmat))].
    \end{align*}
    The last line follows from the same argument as in equation~\eqref{eq:crossterm5}. The result then follows from equation~\eqref{eq:corbagging1}.
\end{proof}
%%%%%%%%%%%%%%%%%%%%%%%%%%%%%%%%%%%%%%%%%%%%%%%%%
%%%%%%%%%%%%%%%%%Proof of Lemma 5.2%%%%%%%%%%%%%%
%%%%%%%%%%%%%%%%%%%%%%%%%%%%%%%%%%%%%%%%%%%%%%%%%

\subsection{Proof of Lemma~\ref{lm:advriskequiv}}
\begin{proof}[Proof of Lemma~\ref{lm:advriskequiv}]
Following the proof of Lemma~\ref{lm:biasvar}, we can decompose the risk $R_{X}(\estimator)$ as:
\begin{equation*}
    R_{X}(\estimator) = \frac{1}{\nresample^{2}} \sum_{k, \ell} \truesignal \transp (I - \hcovmat_{k}\pinv \hcovmat_{k})\covmat(I - \hcovmat_{\ell}\pinv \hcovmat_{\ell}) \truesignal + \truesignal \transp (I - \hcovmat_{k}\pinv \hcovmat_{k}) \covmat \hcovmat_{\ell}\pinv X\transp \sketchmat_{\ell} \sketchmat_{\ell} E + E  \transp \sketchmat_{k}^2 X\hcovmat_{k}\pinv  \covmat \hcovmat_{\ell}\pinv X\transp \sketchmat_{\ell}^2E.
\end{equation*}
For any $k$ and $\ell$, we define
\begin{align*}
    T_{1, k, \ell} &:= \truesignal \transp (I - \hcovmat_{k}\pinv \hcovmat_{k})\covmat(I - \hcovmat_{\ell}\pinv \hcovmat_{\ell}) \truesignal,\\
    T_{2, k, \ell} &:= \truesignal \transp (I - \hcovmat_{k}\pinv \hcovmat_{k}) \covmat \hcovmat_{\ell}\pinv X\transp \sketchmat_{\ell} \sketchmat_{\ell} E\\
    T_{3, k, \ell} &:= E  \transp \sketchmat_{k}^2 X\hcovmat_{k}\pinv  \covmat \hcovmat_{\ell}\pinv X\transp \sketchmat_{\ell}^2 E.
\end{align*}
Suppose  the following equations hold:
\begin{align}\label{eq:adveq1}
    \lim_{\npinfty} T_{1, k, \ell} &\overset{\as}{=} \lim_{\npinfty} \frac{\signallev}{\ndim} \tr( (I - \hcovmat_{k}\pinv \hcovmat_{k})\covmat(I - \hcovmat_{\ell}\pinv \hcovmat_{\ell}) ),\\ \label{eq:adveq2}
    \lim_{\npinfty} T_{2, k, \ell} &\overset{\as}{=} 0,\\ \label{eq:adveq3}
    \lim_{\npinfty} T_{3, k, \ell} &\overset{\as}{=} \lim_{\npinfty} \frac{\noiselev}{\ndata^{2}} \tr(\hcovmat_{k}\pinv X\transp \sketchmat_{k}^2 \sketchmat_{\ell}^2 X \hcovmat_{\ell}\pinv \covmat).
\end{align}
From equation~\eqref{eq:baggedbiasvariance}, we have
\begin{equation*}
    \lim_{\npinfty} \riskcondition(\estimator) = \lim_{\npinfty} T_{1, k, \ell} +  \lim_{\npinfty} T_{3, k, \ell} \qas
\end{equation*}
Therefore, the desired result follows.

It suffices to show equations~\eqref{eq:adveq1}-\eqref{eq:adveq3}. We begin with $T_{3, k, \ell}$. By Lemma~\ref{lm:quadconcentration}, it suffices to show that $\hcovmat_{k}\pinv X\transp \sketchmat_{k}^2 \sketchmat_{\ell}^2 X \hcovmat_{\ell}\pinv \covmat/\ndata$ has bounded spectral norm. By Lemma~\ref{lm:singularmat}, when $\aspratio < 1$, we can, without loss of generality, assume that $(X\transp X)\inv$ exists. Then, we have
\begin{align*}
    \| \hcovmat_{k}\pinv X\transp \sketchmat_{k}^2 \sketchmat_{\ell}^2 X \hcovmat_{\ell}\pinv /\ndata \|_2 &= \| \hcovmat_{k}\pinv X\transp \sketchmat_{k}^2 X (X\transp X)\inv X \transp \sketchmat_{\ell}^2 X \hcovmat_{\ell}\pinv /\ndata \|_2\\
    &\leq \| (X\transp X/\ndata)\inv  \|_2\\
    &\leq \left(1 - \sqrt{\aspratio} \right)\inv,
\end{align*}
where the last line follows from Lemma~\ref{lm:Bai-yin}. When $\aspratio \geq 1$, by Lemma~\ref{lm:Pseudo_products}, we can, without loss of generality, assume that the multipliers$\weight_{i,j}$ are either zero or one. Then,
\begin{align*}
    \| \hcovmat_{k}\pinv X\transp \sketchmat_{k}^2 \sketchmat_{\ell}^2 X \hcovmat_{\ell}\pinv /\ndata \|_2 &= \|(X\transp \sketchmat_{k} \sketchmat_{k} X)\pinv X\transp \sketchmat_{k} \sketchmat_{k} \sketchmat_{\ell} \sketchmat_{\ell} X (X\transp \sketchmat_{\ell} \sketchmat_{\ell} X)\pinv \covmat /n\|_2 \\
    &\leq \frac{\| \sketchmat_{k} \sketchmat_{\ell} \|_2 \| \covmat \|_2}{\sqrt{\eigval^{+}_{\min}(X\transp \sketchmat_{k}\transp \sketchmat_{k} X/\ndata)} \sqrt{\eigval^{+}_{\min}(X\transp \sketchmat_{\ell}\transp \sketchmat_{\ell} X/\ndata)}} \\
    &\leq C \left(1 - \sqrt{\aspratio /\sampleratio} \right)^{-2} \qas
\end{align*}
where the last line follows from Lemma~\ref{lm:lowerboundeigval} and Assumption~\ref{Assume:Covdistri}. This argument can be directly applied to $T_{1, k, \ell}$, and equation~\eqref{eq:adveq1} holds. Finally, $T_{2, k, \ell}$ follows from Lemma~\ref{lm:quadcrossconcentration}.

\end{proof}

\subsection{Proof of Lemma~\ref{lemma:adversarial}}

\begin{proof}[Proof of Lemma~\ref{lemma:adversarial}]
    
Following the proof of Lemma~\ref{lm:biasvar}, we can decompose the norm of the bagged least square estimator $\| \estimator \|_{2}^{2}$ as:
\begin{align*}
    \| \estimator \|_{2}^{2} &= \frac{1}{\nresample^{2}} \sum_{k, \ell} \truesignal \transp \hcovmat_{k}\pinv \hcovmat_{k} \hcovmat_{\ell}\pinv \hcovmat_{\ell} \truesignal + \truesignal \transp \hcovmat_{k}\pinv \hcovmat_{k} \hcovmat_{\ell}\pinv X\transp \sketchmat_{\ell} \sketchmat_{\ell} E + E  \transp \sketchmat_{k}^2 X\hcovmat_{k}\pinv  \hcovmat_{\ell}\pinv X\transp \sketchmat_{\ell}^2 E.
\end{align*}
Using a similar argument
as in the proof of Lemma~\ref{lm:advriskequiv}, we obtain that $\lim_{\npinfty} \truesignal \transp \hcovmat_{k}\pinv \hcovmat_{k} \hcovmat_{\ell}\pinv X\transp \sketchmat_{\ell} \sketchmat_{\ell} E \overset{\as}{=} 0$, $\lim_{\npinfty} E  \transp \sketchmat_{k}^2 X\hcovmat_{k}\pinv  \hcovmat_{\ell}\pinv X\transp \sketchmat_{\ell}^2 E \overset{\as}{=} \lim_{\npinfty} \frac{\noiselev}{\ndata^{2}} \tr(\hcovmat_{k}\pinv X\transp \sketchmat_{k}^2 \sketchmat_{\ell}^2 X \hcovmat_{\ell}\pinv)$ equals the limiting variance of the bagged least square estimator, and 
\begin{align*}
    \lim_{\npinfty} \truesignal \transp \hcovmat_{k}\pinv \hcovmat_{k} \hcovmat_{\ell}\pinv \hcovmat_{\ell} \truesignal &= \lim_{\npinfty} \frac{\signallev}{\ndim} \tr( \hcovmat_{k}\pinv \hcovmat_{k} \hcovmat_{\ell}\pinv \hcovmat_{\ell} )\\
    &= \lim_{\npinfty} \frac{\signallev}{\ndim} \tr( (I - \hcovmat_{k}\pinv \hcovmat_{k}) (I - \hcovmat_{\ell}\pinv \hcovmat_{\ell}) ) - \frac{\signallev}{\ndim} \tr (I - \hcovmat_{k}\pinv \hcovmat_{k}) - \frac{\signallev}{\ndim}\tr (I - \hcovmat_{\ell}\pinv \hcovmat_{\ell}) + \signallev\\
    &= \begin{cases}
\signallev, & \aspratio/\sampleratio < 1\\
\signallev \frac{\left(\aspratio - \sampleratio \right)^2}{\aspratio \left(\aspratio - \sampleratio^{2} \right)} -2\signallev \frac{\aspratio/\sampleratio - 1 }{\aspratio/\sampleratio} + \signallev, & \aspratio/\sampleratio > 1
\end{cases}\\
&= \begin{cases}
\signallev, & \aspratio/\sampleratio < 1\\
\signallev \frac{\sampleratio^{2} (\aspratio + 1 - 2\sampleratio)}{\aspratio(\aspratio - \sampleratio^{2})} , & \aspratio/\sampleratio > 1
\end{cases} \qas
\end{align*}
where the third line follows from the proof of Lemma~\ref{lm:isobootbias} and Lemma~\ref{lm:sketchoverbias}. Therefore, the desired result follows. 
\end{proof}

\subsection{Technical lemmas}

%Let $\rightsquigarrow$ denote weak convergence. 
\begin{lemma}\label{lm:limitVESD}
    Assume Assumption~\ref{Assume:Covdistri} and Assumption~\ref{Assume:VESD}. Let $\Tilde{G}_{x}$ be the weak convergence limit of $\Tilde{G}_{x, \ndim }$  defined in  equation~\eqref{eq:VESDkx}. Then, as $x \rightarrow 0^{-}$,
    \begin{equation*}
        \Tilde{G}_{x} \rightsquigarrow \Tilde{G},
    \end{equation*}
where $\rightsquigarrow$ denotes weak convergence. 
\end{lemma}
\begin{proof}[Proof of Lemma~\ref{lm:limitVESD}]
  Recall $\Tilde{r_{x}}^{2} = \truesignal\transp (I + \ensambleweight(x)\covmat)^{-1} \truesignal$.   To prove the weak convergence, it suffices to show, for any bounded continuous function $f \in \mathcal{C}_{b}(\RR)$, that %\scomment{why can we exchange limits?}
    \begin{align*}
        \lim_{x \rightarrow 0^{-}} \int f(t) d\Tilde{G}_{x}(t) &= \lim_{x \rightarrow 0^{-}} \lim_{\npinfty} \frac{1}{\Tilde{r_{x}}^{2}} \sum_{i=1}^{\ndim} \frac{1}{1 + \ensambleweight(x)\eigval_{i}}\langle \beta, u_{i} \rangle^{2} \, f \left( \frac{\eigval_{i}}{1 + \ensambleweight(x)\eigval_{i}} \right)\\
        &= \lim_{\npinfty} \lim_{x \rightarrow 0^{-}} \frac{1}{\Tilde{r_{x}}^{2}} \sum_{i=1}^{\ndim} \frac{1}{1 + \ensambleweight(x)\eigval_{i}}\langle \beta, u_{i} \rangle^{2} \, f \left( \frac{\eigval_{i}}{1 + \ensambleweight(x)\eigval_{i}} \right)\\
        &= \int f(t) d\Tilde{G}(t),
    \end{align*}
    where the first and third equality follow from the definition of $\Tilde{G}_{x}$ and continuity of $\ensambleweight(x)= (1 - \sampleratio) v(x/\sampleratio)$, and we exchange the limits between $x \rightarrow 0^{-}$ and $\npinfty$ in the second equality. This obtains the desired result. 

    It remains to prove that the limits in the above displayed equality are exchangeable. By the Moore-Osgood theorem, %to exchange the limit between $x \rightarrow 0^{-}$ and $\npinfty$ in the above equation, 
    it suffices to show, as $\npinfty$, that
    \begin{equation*}
        T_{\ndim}(x):= \frac{1}{\Tilde{r_{x}}^{2}} \sum_{i=1}^{\ndim} \frac{1}{1 + \ensambleweight(x)\eigval_{i}}\langle \beta, u_{i} \rangle^{2} \, f \left( \frac{\eigval_{i}}{1 + \ensambleweight(x)\eigval_{i}} \right) \rightarrow \int f(t) d\Tilde{G}_{x}(t)
    \end{equation*}
    uniformly for all $x$ on $[0, c]$, for some $c$ such that $0 < c < +\infty$. By the Arzela–Ascoli theorem, we only need to prove that $T_{\ndim}(x)$ is uniformly bounded and equicontinuous on $[0, c]$.
    \paragraph{Uniform boundedness} Since we assume that $f$ is a bounded function such that  $f \leq M$ for some constant $M$. Then,
\$
T_{\ndim}(x)&= \frac{1}{\Tilde{r_{x}}^{2}} \sum_{i=1}^{\ndim} \frac{1}{1 + \ensambleweight(x)\eigval_{i}}\langle \beta, u_{i} \rangle^{2} \, f \left( \frac{\eigval_{i}}{1 + \ensambleweight(x)\eigval_{i}} \right) \\
&\leq \frac{1}{\Tilde{r_{x}}^{2}} \sum_{i=1}^{\ndim} \frac{1}{1 + \ensambleweight(x)\eigval_{i}}\langle \beta, u_{i} \rangle^{2} \, M\\
&\leq M
\$
where the last line follows from the definition of $\Tilde{r_{x}}^{2}$. %that $\Tilde{G}_{x}(t)$ is  a probability distribution.

\paragraph{Uniform equicontinuity}
    First, since $\ensambleweight(x)= (1 - \sampleratio) v(x/\sampleratio)$ is a continuous function on the compact interval $[0, c]$, then $\ensambleweight(x)$ is equicontinuous. Furthermore, since $T_{\ndim}(x)$ can be written as a function of  $\ensambleweight(x)$. It suffices  to prove that $T_{\ndim}(x)$ is equicontinuous with respect to $\ensambleweight(x)$ on $\{x: x\in [0, c]\}$. For any $0 \leq x \leq c$, we have 
    \begin{align*}
        \frac{\eigval_{i}}{1 + \ensambleweight(x)\eigval_{i}} - \frac{\eigval_{i}}{1 + \ensambleweight(0)\eigval_{i}} \leq \frac{\eigval_{i}^{2}(\ensambleweight(0) - \ensambleweight(x))}{(1 + \ensambleweight(x)\eigval_{i}) (1 + \ensambleweight(0)\eigval_{i})} \leq \frac{\ensambleweight(0) - \ensambleweight(x)}{\ensambleweight(0) \ensambleweight(x)} \leq C (\ensambleweight(0) - \ensambleweight(x)),
    \end{align*}
    for some constant $C >0$, where we used the fact that the continuous function $\ensambleweight(x)$ is bounded on the compact set $[0, c]$. Thus, $\eigval_{i}/(1 + \ensambleweight(x)\eigval_{i})$ is equicontinuous with respect to $\ensambleweight(x)$. Similarly, we can see $1/(1 + \ensambleweight(x)\eigval_{i})$ is equicontinuous with respect to $\ensambleweight(x)$. Since we assumed $f$ is continuous, we obtain $f(\eigval_{i}/(1 + \ensambleweight(x)\eigval_{i}))$ is equicontinuous with respect to $\ensambleweight(x)$. Recall $r^2 = \|\beta\|_2^2$. Then
    \begin{align*}
        \frac{1}{\Tilde{r_{0}}^{2}} - \frac{1}{\Tilde{r_{x}}^{2}} &= \frac{\Tilde{r_{x}}^{2}-\Tilde{r_{0}}^{2}}{\Tilde{r_{x}}^{2}\Tilde{r_{0}}^{2}}\\
        &= \frac{\truesignal \transp (I + \ensambleweight(x)\covmat)^{-1} (\ensambleweight(0) - \ensambleweight(x)) \covmat (I + \ensambleweight(0)\covmat)^{-1} \truesignal}{\truesignal \transp (I + \ensambleweight(x)\covmat)^{-1} \truesignal \truesignal \transp (I + \ensambleweight(0)\covmat)^{-1} \truesignal}\\
        &\leq \frac{(1 + \ensambleweight(x) c_{\lambda})\inv C_{\lambda} (1 + \ensambleweight(0) c_{\lambda})\inv}{\signallev } (\ensambleweight(0) - \ensambleweight(x))\\
        &\leq C(\ensambleweight(0) - \ensambleweight(x))
    \end{align*}
    for some constant $C > 0$, where the second line  uses  the identity $A\inv - B\inv = A\inv(B - A)B\inv$, the first inequality uses Assumption~\ref{Assume:Covdistri}, and boundedness of $\ensambleweight(x)$ on a compact interval.
    This completes the proof.
\end{proof}

%%%%%%%%%%%%%%%%%%%%%%%%%%%%%%%%%%%%%%%%%%%%%%%%%%%%%%%%
%%%%%%%%%%%%%%%%%Prelminary results%%%%%%%%%%%%%%%%%%%
%%%%%%%%%%%%%%%%%%%%%%%%%%%%%%%%%%%%%%%%%%%%%%%%%%%%%%%%
\newpage

\section{Preliminary lemmas}

This section collects preliminary results. 

\begin{assumption}[Assumption 4.4.1 in~\cite{zhang2007spectral}]\label{Assume:relaxed_assumption}
 We assume the followings. %\scomment{You did not cite this assumpiton: where is it needed?}
    \begin{enumerate}[itemindent=30pt]
        \item[(i)] $\left\|T_{1 n}\right\|$ and $\left\|T_{2 n}\right\|$ are uniformly bounded for $n$, where $\|\cdot\|$ denotes the spectral norm of a matrix.
        \item[(ii)] $E x_{i j}=0, E\left|x_{i j}\right|^2 \leq 1,\left|x_{i j}\right| \leq \delta_n \sqrt{n}$, with $\delta_n \rightarrow 0$,
        $$
        \frac{1}{\delta_n^2 n N} \sum_{i j}\left(1-E\left|x_{i j}\right|^2\right) \rightarrow 0,
        $$
        as $n \rightarrow \infty$.
        \item[(iii)] $T_{1 n}$ and $T_{2 n}$ are non-random.
    \end{enumerate}   
\end{assumption}

\begin{lemma}[Vitali convergence theorem]\label{lm:Vitali}
    Let $f_1, f_2, \cdots$ be analytic on the domain $D$, satisfying $\left|f_n(z)\right| \leq M$ for every $n$ and $z \in D$. Suppose that there is an analytic function $f$ on $D$ such that $f_n(z) \rightarrow f(z)$ for all $z \in D$. Then it also holds that $f_n^{\prime}(z) \rightarrow f^{\prime}(z)$ for all $z \in D$.
\end{lemma}

\begin{lemma}[Moore-Osgood theorem]\label{lm:Moore-Osgood}
If $\lim _{x \rightarrow a} f(x, y)=g(y)$ uniformly (in $y$) on $Y \backslash\{b\}$, and $\lim _{y \rightarrow b} f(x, y)=h(x)$ for each $x$ near $a$, then both $\lim _{y \rightarrow b} g(y)$ and $\lim _{x \rightarrow a} h(x)$ exists and
\begin{equation*}
    \lim _{y \rightarrow b} \lim _{x \rightarrow a} f(x, y)=\lim _{x \rightarrow a} \lim _{y \rightarrow b} f(x, y)=\lim _{\substack{x \rightarrow a \\ y \rightarrow b}} f(x, y) .
\end{equation*}
The $a$ and $b$ here can possibly be infinity.
\end{lemma}

By combining Theorem 2 and Remark 1 of \cite{Bai1993}, we obtain the following lemma. 
\begin{lemma}\label{lm:Bai-yin}
    Assume that the feature vector $x$ has \iid~entries with zero mean, unit variance, and bounded $4$-th moment. As $\npinfty$, $0<\aspratio<\infty$,
    \begin{align*}
        \lim_{\npinfty} \eigval_{\min}^{+} (X\transp X/\ndata) &= (1 - \sqrt{\aspratio})^{2} \qas \\
        \lim_{\npinfty} \eigval_{\max} (X\transp X/\ndata) &= (1 + \sqrt{\aspratio})^{2} \qas
    \end{align*}
    where $\eigval_{\min}^{+}$ denotes the smallest positive eigenvalue.
\end{lemma}

\begin{lemma}[Theorem A.43 in~\cite{bai2010}] \label{lm:petubation}
Let $A$ and $B$ be two $n \times n$ Hermitian matrices with their empirical spectral distributions $F^{A}$ and $F^{B}$. Then 
\begin{equation*}
    \left\|F^{A}-F^{B}\right\|_\infty \leq \frac{1}{n} \operatorname{rank}(A-B),
\end{equation*}
where $\|F\|_\infty=\sup _x|F(x)|$.
\end{lemma}

\begin{lemma}[Von Neumann's trace inequality] \label{lm:Von_neumann}
If $A, B$ are complex $n \times n$ matrices with singular values,
\begin{gather*}
    \alpha_1 \geq \cdots \geq \alpha_n, \quad \beta_1 \geq \cdots \geq \beta_n,
\end{gather*}
respectively, then
\begin{equation*}
    |\operatorname{tr}(A B)| \leq \sum_{i=1}^n \alpha_i \beta_i
\end{equation*}
with equality if and only if $A$ and $B$ share singular vectors.
\end{lemma}

\begin{lemma}[Sherman–Morrison formula] \label{lm:Sherman–Morrison}
Suppose $A \in \mathbb{R}^{n \times n}$ is an invertible square matrix and $u, v \in \mathbb{R}^n$ are column vectors. Then $A+u v^{\top}$ is invertible iff $1+v^{\top} A^{-1} u \neq 0$. In this case,
\begin{equation*}
    \left(A+u v^{\top}\right)^{-1}=A^{-1}-\frac{A^{-1} u v^{\top} A^{-1}}{1+v^{\top} A^{-1} u}.
\end{equation*}
Additionally, we will frequently use the following form:
\begin{equation*}
    \left(A+u v^{\top}\right)^{-1}u=\frac{A^{-1}u}{1+v^{\top} A^{-1} u}.
\end{equation*}
\end{lemma}

\begin{lemma}[Burkholder inequality, Lemma
B.26 in~\cite{bai2010}]\label{lm:quadraticbound}
Let $A=\left(a_{i j}\right)$ be an $n \times n$ nonrandom matrix and $X=$ $\left(x_1, \cdots, x_n\right)^{\prime}$ be a random vector of independent entries. Assume that $\mathrm{E} x_i=0$, $\mathrm{E}\left|x_i\right|^2=1$, and $\mathrm{E}\left|x_j\right|^{\ell} \leq \nu_{\ell}$. Then, for any $p \geq 1$,
\begin{equation*}
\mathbf{E}\left|X^* A X-\operatorname{tr} A\right|^p \leq C_p\left(\left(\nu_4 \operatorname{tr}\left(AA^*\right)\right)^{p / 2}+\nu_{2 p} \operatorname{tr}\left(AA^*\right)^{p / 2}\right), 
\end{equation*}
where $C_p$ is a constant depending on $p$ only.
\end{lemma}

The following two lemmas are direct consequences of Lemma~\ref{lm:quadraticbound} and the Borel-Cantelli Lemma.

\begin{lemma}[Lemma C.3 in~\cite{dobriban2018}]\label{lm:quadconcentration}
    Let $x \in \mathbb{R}^\ndim$ be a random vector with i.i.d. entries and $\mathbb{E}[x]=0$, for which $\mathbb{E}\left[\left(\sqrt{\ndim} x_i\right)^2\right]=\sigma^2$ and $\sup _i \mathbb{E}\left[\left|\sqrt{\ndim} x_i\right|^{4+\eta}\right]$ $<C$ for some $\eta>0$ and $C<\infty$. Moreover, let $A_\ndim$ be a sequence of random $\ndim \times \ndim$ symmetric matrices independent of $x$, with uniformly bounded eigenvalues. Then the quadratic forms $x^{\top} A_\ndim x$ concentrate around their means: $x^{\top} A_\ndim x-\ndim^{-1} \sigma^2 \operatorname{tr} A_\ndim \overset{\as}{\rightarrow} 0$.
\end{lemma}

\begin{lemma}\label{lm:concentrationontrace}
    Assume Assumptions~\ref{Assume:highdim}-\ref{Assume:Covdistri}. Then, for any triangular array of matrices $M_{\ndim, i}$ with bounded {spectral} norm and independent with $x_{i}$, $1 \leq i \leq \ndata$, it holds that %\scomment{need a result for $z=0$?}
    \begin{equation*}
       \lim_{\npinfty} \max_{i \in \{1, \dots, \ndata\}} \left| \frac{1}{\ndata} x_{i}\transp M_{\ndim, i} x_{i} - \frac{1}{\ndata} \tr(M_{\ndim, i} \covmat) \right| = 0 ~\qas 
    \end{equation*}
\end{lemma}

\begin{lemma}[Lemma C.1 in~\cite{dobriban2018}]\label{lm:quadcrossconcentration}
    Let $x_n \in$ $\mathbb{R}^n$ and $y_n \in \mathbb{R}^n$ be independent sequences of random vectors, such that for each $n$ the coordinates of $x_n$ and $y_n$ are independent random variables. Moreover, suppose that the coordinates of $x_n$ are identically distributed with mean 0 , variance $C / n$ for some $C>0$ and fourth moment of order $1 / n^2$. Suppose the same conditions hold for $y_n$, where the distribution of the coordinates of $y_n$ can be different from those of $x_n$. Let $A_n$ be a sequence of $n \times n$ random matrices such that $\left\|A_n\right\|$ is uniformly bounded. Then $x_n^{\top} A_n y_n \overset{\as}{\rightarrow} 0$.
\end{lemma}

%\input{Isotropicresampling}
%\input{IsotropicBootstrap}
%\input{IsotropicMultiplier_Bootstrap}
%\input{ResidualWild}
%\input{Correlatedresampling}
%\input{CorrelatedBootstrap}
%\input{Conditional_Method}
%\input{eigenvectors}

%\vspace{-10pt}

\end{document}